\documentclass[accepted]{uai2026} 
                        
\usepackage{amsmath}
\usepackage{amssymb}
\usepackage{mathtools}
\usepackage{tcolorbox}
\usepackage{amsthm}
\newtheorem{theorem}{Theorem}[section]
\theoremstyle{plain}
\newtheorem{proposition}[theorem]{Proposition}
\newtheorem{lemma}[theorem]{Lemma}
\newtheorem{corollary}[theorem]{Corollary}
\theoremstyle{definition}
\newtheorem{definition}[theorem]{Definition}
\newtheorem{assumption}[theorem]{Assumption}
\theoremstyle{remark}
\newtheorem{remark}[theorem]{Remark}

\usepackage{booktabs}
\usepackage{tabularx}
\usepackage{adjustbox}
\newcommand{\mstd}[2]{\ensuremath{#1 \pm #2}} 
\usepackage{graphicx}

\usepackage[american]{babel}

\usepackage{natbib} 
\usepackage{mathtools} 
\usepackage{booktabs} 
\usepackage{tikz} 

\usepackage{graphicx}
\usepackage{float} 
\usepackage{booktabs} 
\usepackage{wrapfig}
\usepackage{lipsum}
\usepackage{subfigure}
\usepackage{subcaption}
\usepackage{caption}

\usepackage{booktabs}
\usepackage{longtable}
\usepackage{array}

\newcommand{\notegroup}[1]{%
  \addlinespace[0.6ex]%
  \multicolumn{2}{@{}l@{}}{\bfseries\scshape #1}%
  \\[-0.2ex]%
  \addlinespace[0.2ex]%
}

\usepackage{titletoc}

\usepackage{pgfplots}
\pgfplotsset{compat=1.18}

\title{A Theory of Random Graph Shift in Truncated-Spectrum vRKHS}

\author[1]{Zhang Wan}
\author[1]{Tingting Mu}
\author[1,2,3]{Samuel Kaski}
\affil[1]{%
    Department of Computer Science\\
    University of Manchester, UK
}
\affil[2]{%
    Department of Computer Science\\
    Aalto University, Finland
}
\affil[3]{%
    ELLIS Institute Finland
}
  
\begin{document}
\maketitle

\begin{abstract}
This paper develops a theory of graph classification under domain shift through a random-graph generative lens, where we consider intra-class graphs sharing the same random graph model (RGM) and the domain shift induced by changes in RGM components.
While classic domain adaptation (DA) theories have well-underpinned existing techniques to handle graph distribution shift, the information of graph samples, which are itself structured objects, is less explored.
The non-Euclidean nature of graphs and specialized architectures for graph learning further complicate a fine-grained analysis of graph distribution shifts.
In this paper, we propose a theory that assumes RGM as the data generative process, exploiting its connection to hypothesis complexity in function space perspective for such fine-grained analysis.
Building on a vector-valued reproducing kernel Hilbert space (vRKHS) formulation, we derive a generalization bound whose shift penalty admits a factorization into (i) a domain discrepancy term, (ii) a spectral-geometry term summarized by the accessible truncated spectrum, and (iii) an amplitude term that aggregates convergence and construction-stability effects.
%
%
We empirically verify the insights on these terms in both real data and simulations.
%
\end{abstract}

\section{Introduction}

Learning on graphs has been receiving increasing attention, as graphs serve as universal data structures to represent structural knowledge, e.g., of molecules \citep{stokes2020deep,jumper2021highly,ingraham2023illuminating}, chemical processes \citep{reiser2022graph}, images and texts \citep{chen2024survey,huang2025hl,wu2023graph}, social networks \citep{posfai2016network,abbe2018community}, and in combinatorial optimization problems \citep{cappart2023combinatorial}.

However, graph learning systems are frequently deployed under distribution shift: training and testing graphs may differ in structure, attributes, and size.
Developing techniques and establishing theories that account for distribution shift and support domain adaptation (DA), by taking into account special characteristics of graphs, is not trivial and is of critical importance \citep{shi2025domain}.
Existing graph DA methods often exploit generic DA ideas \citep{you2023graph} and propose heuristics for specific shift types, e.g., attribute or edge changes \citep{liu2023structural,luo2024gala}, however the theoretical guarantees considering \emph{graph-native} shifts and modern message-passing models remain unexplored.

While classic DA theory has thoroughly addressed the relation between target and source error through hypothesis complexity and a domain divergence, extending such analyses to graphs is nevertheless trivial.
This is because a graph sample is itself a structured object and the learning architecture is also specialized to such data structure.
However, the mixed effect of non-Euclidean nature of graphs and specialized model architectures complicate a fine-grained analysis of graph distribution shifts.
Hence, an open research question remains:
\emph{How to study distribution shift bespoke to graphs and account for specific graph learning models?}


Recent advances in modeling graphs as samples of random graphs shed light on this question.
%
%
Beyond descriptive modeling capabilities to provide principled abstraction for complex networks \citep{bollobas1998random}, random graphs also present a useful \emph{analytical lens} for graph learning: by relating a discrete graph neural network (GNN) to its continuous counterpart defined by the corresponding random graph model (RGM), one obtains convergence and generalization guarantees for graph convolution filters and GNNs under i.i.d. sampling from a fixed RGM \citep{keriven2020convergence,ruiz2020graphon,maskey2022generalization}.

In this work, we address the aforementioned open research question for graph classification. We adopt the random graph generative lens and take one step further to explicitly study RGM shifts.
Specifically, we consider class-wise RGM as a structured parameterization of a graph distribution and define domain shift as a change in RGM components.
This allows a result that pulls back domain divergence into RGM latent space, suggesting use of the latent Wasserstein distance as a domain shift indicator. 
Moreover, we also build on a vector-valued reproducing kernel Hilbert space (vRKHS) formulation of multi-class domain adaptation for a fine-grained analysis over hypothesis complexity, studying the graph learning models, i.e., GNNs, from the function space perspective.
By introducing a finite-spectrum assumption over vRKHS, we show a uniform upper bound of vRKHS norm via its spectrum geometry and infinity norm, while the latter further reveals how the assumed data-generative process and hypothesis complexity affect transferability.

\textbf{Contributions.}
This paper presents a generalization bound for domain adaptation in graph classification, from random graphs perspective.
By formalizing the graph domain shift via RGM component changes, we derive a graph-native DA bound in truncated-spectrum vRKHS, which yields a factorization of transfer penalty into domain discrepancy, spectral geometry, and amplitude components.
Empirically, we examine the qualitative implications of these factors on both real-world and synthetic data, suggesting a domain divergence proxy and results on the finite-rank behavior of hypothesis, and also revealing effects of graph-size and stratified structure of GNNs.

\textbf{Related Works.}
Existing graph DA theories are restricted to GCNs, building on graph spectral theory \citep{meng2023transfer}. 
With random graph theory, \citet{keriven2020convergence,keriven2021universality} define a continuous GCN, prove the discrete GCN converges to a limit object, and develop universality and stability results for GCNs.
\citet{maskey2022generalization} then extend such analysis to the generic message-passing neural networks (MPNNs) and develop a generalization bound for graph classification under the in-distribution setting.
There however remains a gap in the study of how generic MPNNs behave under explicit RGM shifts, especially from the function space perspective.
By considering hypothesis residing in RKHS, \citet{redko2017theoretical} propose the first DA bound using Wasserstein distance as domain divergence, and the RKHS norm as measure of complexity, yet simply assumed upper bounded by $1$ without further fine-grained analysis.
%
Thus, the missing analysis limits potential insights on how hypothesis complexity affects generalization.
Instead, this paper presents a finite-spectrum viewpoint that allows factorizing vRKHS norm into spectrum geometry and amplitude terms, where the latter accommodates RGM and MPNN related quantities, yielding insights on how complexities of data and model interact.
More discussion in Appendix~\ref{app:relatedwork}.

\vspace{-1.0\baselineskip}


\section{Preliminaries}
\label{sec:preliminaries}

We briefly introduce the notations and refer details to Appendix~\ref{app:notations}. 
Let $[N]=\{1,\ldots,N\}$.
A graph $G=(V, E, Z, A)$ contains a set of nodes $V = [N]$, edges $E \subseteq V \times V$, feature vectors $\left\{z_i \in \mathbb{R}^{F}\right\}_{i\in V}$  that characterize the nodes and form  the feature matrix $Z \in \mathbb{R}^{N \times F}$ (often called a graph signal), and the adjacency matrix $A \in \mathbb{R}^{N \times N}$ with each element $a_{ij}$ denoting the edge weight.
We sometimes simplify the graph notation to $G=(Z, A)$.
The degree of a node means the number of edges that connect the node  to the other nodes in the graph.
We denote the Euclidean norm by $\left\lVert \cdot \right\rVert$,   Frobenius norm by $\left\lVert \cdot \right\rVert_F$, the general $p$-norm by $\left\lVert \cdot \right\rVert_p$, and infinity norm by $\lVert \cdot \rVert_{\infty}$. 
The infinity norm of a function $f: \mathcal{X} \to \mathbb{R}^F$ is defined as $\left\lVert f \right\rVert_\infty = \sup_{x \in \mathcal{X}} \left\lvert f(x) \right\rvert$.
Lipschitz continuity of a function $f$ is defined with respect to a norm, given the existence of a Lipschitz constant.
Let $(\mathcal{X}, d)$ be a compact metric space for $\mathcal{X} \subseteq \mathbb{R}^D$. 
Its covering number is defined as the minimum number of balls with radius $\epsilon$ required to cover the space $\mathcal{X}$ under metric $d$, denoted by $\mathcal{N}(\mathcal{X}, \epsilon, d)$.
Given two distributions $P$ and $Q$ defined on $\mathcal{X}$, their Wasserstein $p$-distance is defined by
$\mathcal{W}_p^p (P, Q) = \inf_{\gamma \sim \Pi(P, Q)} \int_\mathcal{X} c(x,y)^p \, d\gamma(x,y)  $,
where $\Pi$ is the set of all couplings of $P$ and $Q$,  and $c(x,y)$ is the cost of moving $x$ to $y$. 
We adopt a specific form of RGM used by \citet{maskey2022generalization} to model the graph classification problem, formally defined below.
\begin{definition}[Random Graph Model]
\label{def:rgm}
    Given a compact metric space $(\mathcal{X}, d)$,   an RGM is a triplet $\Gamma = (W, P, f)$, containing a symmetric kernel function $W: \mathcal{X} \times \mathcal{X} \to \mathbb{R}$, a probability distribution $P$ over $\mathcal{X}$, and a measurable bounded function $f: \mathcal{X} \to \mathbb{R}^F$. 
    It generates a random graph with $N$ nodes by the following process: $\forall~i, j \leq N$,  
    \begin{align}
        \textmd{Sample: } & x_i \overset{\underset{\textmd{i.i.d}}{}}{\sim} P , \\
        \textmd{Compute: } & a_{ij} = W(x_i,x_j), z_i = f(x_i).
    \end{align}
\end{definition}

To generate a graph from the RGM, a set of $N$ latent variables $\{x_i\}_{i=1}^{N}$ are drawn from a latent space $\mathcal{X}$ following $P$, each corresponding to a graph node $i$.
The graph structure is then determined, in terms of the adjacency weights, by applying the kernel function $W$   to the latent variables.
The latent variable $x_i$ is then mapped to an observed feature vector $z_i \in \mathbb{R}^{F}$ by $f$.
To generate a graph dataset $\mathcal{D} = \{ (G_i^j,y_i^j)_{i=1}^{m_j}  \}_{j=1}^{C}$ with $C$ classes, a set of RGMs $ \{\Gamma^j \}_{j=1}^C$ sharing the same space $\mathcal{X}$ is used.
For each class $j$, a two-step process is repeated $m_j$ times: 
(1) Draw graph size $N \sim \nu$ from a measure $\nu$ defined on $\mathbb{N}^+$.
(2) Sample a graph $G \sim \Gamma^j$  from the RGM $\Gamma^j$ and label it as $y=j$.



\textbf{Problem Formulation.} In graph classification, given an input graph $G = (Z, A)$, a hypothesis function $h: \mathbb{R}^{N \times F} \to \mathbb{R}^{C}$ maps the graph signal $Z$ to a vector of scores, e.g., the membership logits belonging to classes.
In graph DA, a domain $D = (\mu_D, g_D)$ is a pair of probability distribution $\mu_D$ and a labeling function $g_D: \mathbb{R}^{N \times F} \to \mathbb{R}^{C}$ that maps the graph signal $Z$ to its ground truth class.
We study how a hypothesis function trained in a source domain $D_S=(\mu_S, g_{D})$ performs in a target domain $D_T=(\mu_T, g_{D})$, where the ground-truth labeling function $g_D$ does not change across domains, i.e., is domain-invariant.
A common way for assessing a hypothesis $h$ is to measure its disagreement with the labeling function $g_D$ through a loss function \citep{redko2017theoretical,redko2019advances}.
We aim to develop a \textit{domain adaptation error theory for multi-class graph classification}, establishing a theoretical understanding of the key factors that affect the classification error under domain shifts.

\textbf{Problem Setting.} To model the graph distribution,  we  use different RGMs to model different classes, denoted by $\{ \Gamma_D^j = (W_D^j, P_D^j, f_D^j) \}_{j=1}^C$, for a domain $D$ and a set of $C$ classes.
This results in a graph distribution equivalent to the product measure of the push-forward probability measures of different classes.
In the classification context, sampling graphs becomes sampling graph signals from the push-forward distribution, i.e., $Z \sim \mu_{D}$. 
This enables us to define an expected error risk, e.g., based on the $L_1$-norm,  as
\begin{equation}
\label{error_function}
    \epsilon_D(h, g_D) = \mathbb{E}_{Z \sim \mu_D} \left[  \| h(Z)- g_D(Z) \|_1  \right].
\end{equation}
It is associated with a vector-valued loss mapping function,  defined as $\ell_{h,g_D}(Z) = \lvert h(Z) - g_D(Z) \rvert$ through an element-wise operation.
We simplify notation to $\epsilon_D(h) := \epsilon_D(h,g_D)$.

We focus on a generic class of hypothesis functions enabled by  a composition of a $T$-layer MPNN feature extractor $\Bar{h}_{G}: \mathbb{R}^{N \times F} \to \mathbb{R}^{F_T}$ and an $L$-layer MLP classifier $h_{cls}: \mathbb{R}^{F_T} \to \mathbb{R}^C$, i.e., $h = h_{cls} \circ \Bar{h}_{G}$.
Since the labeling function also resides in our hypothesis class, it also follows the same decomposition, i.e., $g_D = g_{{cls}_D} \circ \bar{g}_{G_D}$.
We consider MLP classifiers with output vectors of bounded lengths, i.e., $\|h_{cls}\|\leq 1$, achievable through choice of activation function and weight scaling in the last layer.
We consider MPNN for feature extraction as it is a de facto architecture in graph learning \citep{jegelka2022theory,morris2024position}.
We consider mean aggregation (MA) for constructing MPNN pooling layers,  which is extendable to other aggregations \citep{cordonnier2024convergence}.
The formal definitions of MPNN and its continuous counterpart induced by  RGMs are provided in Appendix \ref{app:definitions}.
An MPNN converges to its continuous counterpart as the limit object when graph size increases \citep{keriven2020convergence,ruiz2020graphon,maskey2022generalization}.

We consider two sets of RGMs for generating the source and target graphs, $\{\Gamma_S^j = (W_S^j, P_S^j, f) \}_{j=1}^C$ and $\{\Gamma_T^j = (W_T^j, P_T^j, f)\}_{j=1}^C$.
We consider domain shift introduced through the latent distribution $P$ and kernel function $W$, assuming $f$ to be class and domain-invariant.
%
However, our result can be extended to accommodate domain shifts induced by all the three factors of $P$, $W$ and $f$.
With the imposed RGM structure, the source and target distributions $\mu_S $ and $\mu_T$ become products of push-forward measures, enabling to assess the source and target error risks, denoted by $\epsilon_S(h,g_D)$ and  $\epsilon_T(h,g_D)$,  via Eq. (\ref{error_function}) under $\mu_S$ and $\mu_T$, respectively.

\textbf{Model Assumptions.} 
For theory development, we assume well-accepted RGM properties  widely used by existing works on RGMs for graph learning \citep{keriven2020convergence,keriven2021universality,maskey2022generalization}. For instance, the latent space $\mathcal{X}$ has bounded diameter and bounded covering number, i.e., $\mathcal{N}(\mathcal{X}, \epsilon, d) \leq C_\mathcal{X} \epsilon^{- D_\mathcal{X}}$, $\forall~\epsilon > 0$, given constants $C_\mathcal{X}, D_\mathcal{X} \geq 0$, and the function $f$ is H{\"o}lder continuous. 
We also assume that the message and update functions of MPNNs, enabled by MLPs, have Lipschitz constants \citep{khromov2024some, belkin2006manifold,fiedler2023lipschitz}.
Appendix \ref{app:assumptions} provides a detailed summary of the used assumptions on  characteristics of RGMs and MPNNs, with a discussion on Lipschitz and H{\"o}lder continuity.


%
With formulation and model assumptions properly set up, we are now ready to present our first result, whose form may be familiar yet is extended to multi-class classification.
\begin{proposition} [\textbf{Domain Adaptation Generalization Error}]
\label{prop:general_DA}
    Given a source domain $D_S=(\mu_S, g_{D})$,  a target domain $D_T=(\mu_T, g_{D})$,  a hypothesis $h$,  and the source and target error risks $\epsilon_S(h,g_D)$ and $\epsilon_T(h,g_D)$ assessed by Eq. (\ref{error_function}). 
    Then, the following holds: 
    {
    \small
    \begin{equation}
    \label{eq:starting_res}
        \epsilon_T(h,g_D) \leq \epsilon_S(h,g_D) + \underbrace{\left\lVert \ell_{h,g_D} \right\rVert_{\mathcal{H}_{K_\ell}}}_{\text{smoothness}} \cdot \underbrace{\mathcal{W}_1(\mu_S, \mu_T)}_{\text{divergence}},
    \end{equation}
    }
    
    where $\mathcal{H}_{K_\ell}$ is the vRKHS of the hypothesis function $h$, the labeling function $g_D$, and the loss mapping function $\ell_{h,g_D}$. 
\end{proposition}

%

This result considers hypothesis functions, i.e., our graph neural networks, in a function space, i.e., vRKHS, characterized by its associated reproducing kernel.
While similar results can be retained via a global Lipschitz assumption \citep{shen2018wasserstein}, the function space perspective matters since it allows us to study property of hypothesis via studying the vRKHS it resides in, and eventually, the associated kernel.
More details on vRKHS are in Appendix~\ref{app:RKHS}.
The technical gist of the proof lies in a divide-and-conquer trick that decomposes vector-valued functions using orthonormal basis and handles single-dimension case therein via the reproducing property.
See proofs in Appendix~\ref{app:DA_thms:foundation}.

%
%

\section{Main Results}
\label{sec:main}

We now  present our main result the graph domain adaptation bound and a corollary on when our bound is tight.
Its key feature is a factorizatio of transfer penalty into \emph{discrepancy} $\times$ \emph{geometry} $\times$ \emph{amplitude}.
%
We defer the technical results to derive the three terms in the next section.
%

\begin{theorem}[Main Theorem]
\label{thm:DA}
Assume truncated-spectrum vRKHS in Assumption~\ref{assump:truncated_vrkhs} and reachability in Assumption~\ref{ass:reachability} hold.
Then, with probability at least $1-3\rho$,
\begin{small}
\begin{equation}
\label{thm:DA:eq}
\epsilon_T(h) \le \epsilon_S(h)
+ \underbrace{\sqrt{\Delta_D}}_{\text{divergence}}
\cdot
\underbrace{\sqrt{\frac{C}{\lambda_r}}
\Big( K' \Xi + K'' \Big)}_{\text{geometry$\times$amplitude}},
\end{equation}
\end{small}

where $\Xi := \sqrt{\Delta_N} + (\Delta_{\Gamma,\Theta}+\delta_{\mathrm{opt}}) + \varepsilon_3 + \varepsilon_4$, and $\Delta_D$, $\Delta_N$, $\Delta_{\Gamma,\Theta}$ are defined in Eq.~\eqref{lem:wasserstein:goal}, \eqref{eq:DeltaN_compact}, \eqref{thm:perturb_bound:goal}, respectively.
\end{theorem}

Eq.~\eqref{thm:DA:eq} decomposes the transfer penalty into three factors that admit \emph{operational proxies} from data and trained models:
(i) a \textbf{domain discrepancy} term $\Delta_D$ (Proposition~\ref{prop:dom_div}),
(ii) a \textbf{spectral-geometry} term summarized by the truncated eigenvalue $\lambda_r$ (Proposition~\ref{prop:norm2diff}),
and
(iii) an \textbf{amplitude} term $K'\Xi+K''$, where $\Xi$ aggregates convergence/optimization contributions (Eq.~\eqref{norm2diff_new1}).
We highlight Eq.~\eqref{thm:DA:eq} not as a numerically tight error estimator without the further conditions in Corollary~\ref{coro:sufficient} holding, but as a set of \emph{qualitative predictions} for which we design one-to-one experiments to verify.
%
%
We now illustrate our bound by attending to these factors.

\textbf{Implication I (domain discrepancy): latent Wasserstein distance as a proxy for $\Delta_D$.}
Proposition~\ref{prop:dom_div} upper bounds $\Delta_D$ by the maximal class-wise Wasserstein distance between latent distributions across domains, suggesting that larger latent shift implies a larger transfer error.
Empirically, we shall compute a latent Wasserstein proxy over estimated RGM measures and verify that it aligns with known shift patterns on real data and correlates with predictive losses.

\textbf{Implication II (spectral geometry): truncated-spectrum $\lambda_r$ across hypothesis classes.}
Assumption~\ref{assump:truncated_vrkhs} posits the integral operator associated to vRKHS is effectively finite-rank in the sense of rapidly decaying empirical spectrum, which supports deriving a uniform upper bound in Proposition~\ref{prop:norm2diff}.
Empirically, we estimate the spectra on multiple benchmark datasets and quantify the truncated-dimension $r_\varepsilon$ via tail-energy thresholding.
We then compare spectrum profiles across hypothesis classes to assess how spectrum geometry, summarized by $\lambda_r$, varies with the model family.
%

\textbf{Implication III (amplitude): graph-size effect on $\Delta_N$ and layerwise stability in $\Delta_{\Gamma,\Theta}$.}
The convergence contribution $\Delta_N$ in Eq.~\eqref{eq:DeltaN_compact} predicts a monotone trend: increasing the sampled graph size $N$ tightens the convergence component and hence decreases $\Xi$.
Moreover, the optimization contribution $\Delta_{\Gamma,\Theta}$ in Eq.~\eqref{thm:perturb_bound:goal} admits a layerwise product structure, indicating that early layer-wise stability could be magnified, thus motivating a more targeted regularizer to reduce $\Delta_{\Gamma,\Theta}$.
Empirically, we shall validate both:
simulation results confirm the monotone decrease of target loss with larger $N$, and real-data transfers show that the \emph{non-uniform} regularizer improves over uniform $\ell_2$ regularization.

While Theorem~\ref{thm:DA} holds with high probability, such uncertainty is attributed to empirical process fluctuations.
In the following, we thus further derive a sufficient condition for our bound to be tight by ensuring a large enough sample size $M$. See its formal presentation in Corollary~\ref{proof:bound_tightness}.

\begin{corollary}
\label{coro:sufficient}
Let $M$ be the larger of the total number of nodes in the source and target datasets. Suppose
\begin{equation}
    M \geq \log_{\frac{D_{\mathcal{X}}}{4}} \left( \frac{1 + \log(1/\rho)^{1/4}}{0.1 \times 27^{\frac{D_{\mathcal{X}}}{4}}} \right),
\end{equation}
and other conditions dependent upon $\xi > 0$ in Appendix~\ref{app:sufficient_condition} hold. Then, with probability at least $1-3\rho$,
\begin{equation}
    \epsilon_T(h)\leq (1+\xi) \; \epsilon_S(h).
\end{equation}
\end{corollary}


In the following, we present more details on how we obtain the Eq.~\eqref{thm:DA:eq}.
In Subsection~\ref{sec:dom_div}, we first introduce the latent Wasserstein distance as a domain divergence for which \citet{keriven2023entropic} has given consistency guarantee for empirical estimation.
In Subsection~\ref{sec:truncated_vrkhs}, we discuss the truncated-spectrum assumption and, based on it, a uniform upper bound of RKHS norm.
We finally incorporate RGM convergence and stability effects into our bound in Subsection~\ref{sec:disagreement_analysis}.

\subsection{Pullback Domain Divergence}
\label{sec:dom_div}

Building upon Eq. (\ref{eq:starting_res}), we further analyze the Wasserstein 1-distance $\mathcal{W}_1(\mu_S, \mu_T)$, aiming at revealing a clear structure of graph distribution shifts caused by distribution shifts in the latent space $\mathcal{X}$ of RGMs.
We achieve this by  upper bounding $\mathcal{W}_1(\mu_S, \mu_T)$ using the Wasserstein 2-distance $\mathcal{W}_2(\mu_S, \mu_T)$, based on  H{\"o}lder's inequality (Remark 6.6, \cite{villani2009optimal}), i.e., $\mathcal{W}_p \leq \mathcal{W}_q$ if $p \leq q$.
To derive a domain divergence bound, we exploit properties of Wasserstein distance for product push-forward measures \citep{panaretos2019statistical}, and the asymptotic convergence results of empirical Wasserstein distance \citep{weed2019sharp}. 

\begin{proposition}
\label{prop:dom_div}
Consider the data generation and model assumptions above. There exists a constant $L'$, dependent on these regularity quantities, such that $\mathcal{W}_2^2 (\mu_S, \mu_T) \leq \Delta_D$ holds with a probability at least $1-\rho$, where
\begin{equation}
    \label{lem:wasserstein:goal}
    \Delta_D := 2 \; C^2 \; L' \; \max_{j=1}^C \mathcal{W}_2 \left(P_S^j, P_T^j\right).
\end{equation}
%
%
\end{proposition}

The above result pulls back Wasserstein distances (WD) between observed graph signals to latent RGM probability measures.
%
This indicates that WD between RGM latent distributions serves as domain shift indicator.
We present the complete $\Delta_D$ and proofs in Lemma \ref{proof:reverse2latent}, where, with results on empirical Wasserstein approximation, we also show that the sample complexity involves the total number of nodes, e.g., on source dataset $\sum_{j=1}^C N_S \cdot m_S^j$, for maximum graph size $N_S$ and the class-wise number of graphs $m_S^j$.


\subsection{Spectrum-Amplitude Factorization}
\label{sec:truncated_vrkhs}

The function norm term $\lVert \ell_{h,g_D} \rVert_{\mathcal{H}_{K_\ell}}$ usually appears in kernel ridge regression (KRR) that punishes hypothesis complexity, whose solution, by representer theorem, admits the form $f = \sum_i w_i k(x_i,\cdot)$ when learning from a finite dataset $\{x_i,y_i\}_{i=1}^n$.
Although this certainly also applies to our situation, we hope to incorporate more priors in graph learning. We instead choose to further transform the vRKHS norm term into \textit{spectrum geometry} and \textit{amplitude}.
Compared with standard bounds that assume a uniform RKHS norm radius, we express the truncated vRKHS norm via the $r$-th eigenvalue $\lambda_r$ and a hypothesis–labeling disagreement, which further decomposes into sub-terms, each controlled by existing results or assumed regularity quantities.

\begin{assumption}[Truncated-spectrum vRKHS]
\label{assump:truncated_vrkhs}
Let $\mu$ be a probability measure on the graph-signal space $\mathcal Z$.
Consider the operator-valued reproducing kernel $K_\ell:\mathcal Z\times\mathcal Z\to \mathcal L(\mathbb R^C)$ that induces the vRKHS $\mathcal H_{K_\ell}$.
Define the associated integral operator $T_{K_\ell}:L_2(\mu;\mathbb R^C)\to L_2(\mu;\mathbb R^C)$ by
\[
(T_{K_\ell}f)(Z) \;=\; \int_{\mathcal Z} K_\ell(Z,Z')\,f(Z')\,d\mu(Z').
\]
Assume $T_{K_\ell}$ is positive, self-adjoint, and compact. Hence it admits an eigen-decomposition $T_{K_\ell}f = \sum_{i\ge1} \lambda_i \langle f, \psi_i \rangle_{L_2(\mu;\mathbb R^C)} \psi_i$ with eigenpairs $\{(\lambda_i,\psi_i)\}_{i\ge1}$ satisfying $\lambda_1 \ge \lambda_2 \ge \cdots>0$ and $\{\psi_i\}_{i\ge1}$ orthonormal in $L_2(\mu;\mathbb R^C)$.
Fix $r\ge1$ and denote the top-$r$ subspace
\[
\mathcal U_r \;:=\; \mathrm{span}\{\psi_1,\ldots,\psi_r\}\subseteq L_2(\mu;\mathbb R^C).
\]
We assume $\ell_{h,g_D}\in \mathcal H_{K_\ell}^{(r)} := \mathcal H_{K_\ell}\cap \mathcal U_r$, or equivalently
\[
\mathcal H_{K_\ell}^{(r)} := \{f \in H_{K_\ell} \mid f := \sum_{i=1}^r \langle f,\psi_i\rangle_{L_2(\mu)}\,\psi_i\}.
\]
%
\end{assumption}


For brevity, we keep using $\mathcal{H}_{K_\ell}$ in the rest of this paper.
Truncated spectrum assumption essentially ignores effects of (extremely) high-frequency components of hypothesis functions that are potentially harmful for generalization.
Since the eigenvalues of $T_{K_\ell}$ is dependent on kernel $K$ and measure $\mu$, it is fairly reasonable to assume such finite-spectrum hypothesis since there always exists an index $r>0$ on both source and target data (fixed $\mu$) under the considered hypothesis space (fixed $K_\ell$).
We shall empirically verify this assumption later.
With above assumption, one is now able to factorize the RKHS norm into a geometry and an amplitude term.


%
\begin{proposition}
\label{prop:norm2diff}
Under Assumption~\ref{assump:truncated_vrkhs},
{
    \small
    \begin{equation}
    \label{eq:smoothness}
        \lVert \ell_{h,g_D} \rVert_{\mathcal{H}_{K_\ell}} \leq \sqrt{\frac{C}{\lambda_r}} \cdot \left\| \ell_{h,g_D} \right\|_{\infty},
    \end{equation}
}

and under model assumptions in Appendix~\ref{app:assumptions},
{
    \small
    \begin{equation}
    \label{eq:architecture}
        \exists~K',K''>0, \;\; \left\| \ell_{h,g_D} \right\|_{\infty} \leq K' \left\lVert \Bar{h}_G - \Bar{g}_{G_D} \right\rVert_{\infty} + K''.
    \end{equation}
}


\end{proposition}

\begin{remark}
    $K'>0$ depends only on the MLP classifier regularities, e.g., activation Lipschitz and norm of weight matrices, quantifying the intrinsic amplification induced by the architecture choice.
    $K'' \geq 0$ depends on the same architecture constants but also on a parameter-mismatch term measuring the relative deviation between the current and the ground-truth MLP classifier, e.g. $\max_{l} \| \Delta W_l \|_F / \|W_l\|_F$.
    $K''=0$ if $\Delta W_l=0$ for all $l$.
    See proofs in Appendix~\ref{app:DA_thms:truncated_vrkhs}.
\end{remark}

Proposition~\ref{prop:norm2diff} factorizes the hard-to-interpret vRKHS norm
$\|\ell_{h,g_D}\|_{\mathcal H_{K_\ell}}$ into (i) the spectral geometry captured by $\lambda_r$, and (ii) the amplitude captured by an $\ell_\infty$ disagreement.
The factor $\lambda_r$ is a \emph{spectral-geometry} term: it depends only on the eigen-spectrum of the integral operator $T_{K_\ell}$.
%
Eq.~\eqref{eq:smoothness} reveals that a smaller $\lambda_r$ makes variations along the $r$-th spectral mode more expensive in $\mathcal H_{K_\ell}$, hence the same point-wise discrepancy, i.e., $\left\| \ell_{h,g_D} \right\|_{\infty}$, yields a larger RKHS complexity.
%
%
In contrast, the quantity $\| \bar{h}_G-\bar{g}_{G_D} \|_{\infty}$ is part of an \emph{amplitude} term that measures the worst-case magnitude of the hypothesis--labeler disagreement over $\mathcal Z$.
Importantly, it can be further controlled by model-specific regularity (e.g., Lipschitz, norm constraints, or approximation properties of the hypothesis class), and by optimization-related factors that reduce such disagreement via training.
%


This factorization is practically and conceptually useful for our domain adaptation analysis.
The geometry term $\lambda_r^{-1}$ is determined solely by the eigenspectrum of the kernel integral operator,
and can be viewed as a worst-direction scale on the truncated subspace:
in the eigencoordinates, the vRKHS norm weights the coefficient along $\psi_i$ by $1/\lambda_i$, so the smallest retained eigenvalue $\lambda_r$ controls the largest amplification.
%
%
In particular, the truncation level $r$ mediates a bias--complexity trade-off.
Increasing $r$ enriches the representation but typically decreases $\lambda_r$ and enlarges the geometry penalty.
Moreover, the \emph{shape} of the spectrum affects where this trade-off becomes active.
For fast-decaying spectra, $\lambda_r$ can become very small as $r$ grows, so the geometry penalty $\lambda_r^{-1}$ may increase sharply when the truncation includes high-frequency modes.
It yet remains unclear how to analyze $\left\| \ell_{h,g_D} \right\|_{\infty}$ even though under architecture assumptions we further transform it into Eq.~\eqref{eq:architecture} form, because the disagreement term $\lVert \Bar{h}_G - \Bar{g}_{G_D} \rVert_{\infty}$ contains an assumed ground-truth labeling function $\Bar{g}_{G_D}$.
In next section, we show that it can be decomposed into analyzable sub-terms.


\subsection{Hypothesis Disagreement Analysis}
\label{sec:disagreement_analysis}

%

The standard way to estimate a virtual ground-truth feature extractor  $\Bar{h}_G$ is to construct a hypothesis feature extractor, e.g., defined as an MPNN, and to optimize it by learning its weights.
We refer to this as the \textit{discrete hypothesis construction}.
Recent advances \citep{keriven2020convergence,ruiz2020graphon,maskey2022generalization} have demonstrated new ways of studying  GNNs by lifting them up to continuous spaces through RGMs.
Following this, we develop continuous counterparts for the feature extractors of the hypothesis and labeling functions, denoted by $\Bar{h}_{W,P}(f; \Theta)$ and $\Bar{g}^c_{G_D}$, respectively.
We refer to this as the \textit{continuous hypothesis construction}, where $\Bar{h}_{W,P}(f; \Theta)$ is constructed by RGM $\Gamma=(W,P,f)$ and the neural networks with parameter $\Theta=\{\Psi,\Phi\}$ which we drop and denote $\Bar{h}_{W,P}(f)$ for brevity.
%

We also consider the \textit{best-in-family hypothesis} in our analysis for finer understanding of error, defined as the best hypothesis from a given hypothesis class that minimizes the expected error risk \citep{bottou2007tradeoffs,brown2024biasvariance}. 
We denote the best-in-family hypothesis by $\Bar{h}^*_G(Z)$ for the discrete hypothesis construction and by $\Bar{h}_{W,P}^{*}(f)$ for the continuous hypothesis construction.
%
For multi-class graph classification, we interpret the disagreement analysis \emph{conditionally on a fixed class} $j\in[C]$ and a fixed domain $D\in\{S,T\}$.
That is, $Z$ is drawn from the class-conditional graph distribution $\mu_{D}^j$ induced by an RGM $\Gamma_{D}^j$.
The continuous hypothesis construction below should be understood with super-/sub-scripts, which we suppress for readability whenever no ambiguity arises.

Building upon these, we propose to analyze the hypothesis-labeling function disagreement $\lVert \Bar{h}_G(Z) - \Bar{g}_{G_D}(Z) \rVert_{\infty}$ through a decomposition that considers the best-in-family hypothesis and works with a hypothesis class enabled by the continuous hypothesis construction.
Specifically, we decompose the targeted disagreement into sub-terms:
\begin{equation}
\label{norm2diff_new1}
\|\Bar{h}_G(Z)-\Bar{g}_{G_D}(Z)\|_{\infty}
\le \varepsilon_1+\varepsilon_2+\varepsilon_3+\varepsilon_4,
\end{equation}
where
\begin{align}
\varepsilon_1 &:= \|\Bar{h}_G(Z)-\Bar{h}_{W,P}(f)\|_{\infty},
&&\text{(convergence)} \nonumber \\
\varepsilon_2 &:= \|\Bar{h}_{W,P}(f)-\Bar{h}^{*}_{W,P}(f)\|_{\infty},
&&\text{(optimization)} \nonumber \\
\varepsilon_3 &:= \|\Bar{h}^{*}_{W,P}(f)-\bar{g}^{c}_{G_D}(f)\|_{\infty},
&&\text{(approximation)} \nonumber \\
\varepsilon_4 &:= \|\bar{g}^{c}_{G_D}(f)-\bar{g}_{G_D}(Z)\|_{\infty}.
&&\text{(labeling)} \nonumber
\end{align}

The approximation term $\varepsilon_3$ measures the discrepancy between the best-in-family continuous predictor $\Bar{h}^{*}_{W,P}(f)$ and the (ideal) continuous ground truth labeling rule $\bar g^{c}_{G_D}(f)$.
Its magnitude depends on the richness of the induced continuous hypothesis class and on the nature of the ground-truth.

When different classes are modeled by different RGMs, learning amounts to identifying class-specific RGM parameters together with network weights that explain the observed graph data.
Since our adopted RGM family is identifiable up to Euclidean transformations of latent positions \citep{allman2009identifiability,allman2011parameter}, the fitted parameters are only defined up to an observationally equivalent class.
Thus, under a realizability assumption, there exist RGM parameters within this equivalence class such that $\Bar{h}^{*}_{W,P}(f)$ can approximate $\bar g^{c}_{G_D}(f)$ well, yielding a small $\varepsilon_3$.

The labeling term $\varepsilon_4$ captures the mismatch between the ideal continuous labeling rule $\bar g^{c}_{G_D}(f)$ and the observed labels $\bar g_{G_D}(f)$, for example due to annotation noise, calibration error, or discretization/quantization effects in the labeling pipeline.
We keep $\varepsilon_4$ explicit to accommodate such label imperfections, although it may be small under standard low-noise or well-calibrated labeling assumptions.



It remains to bound the convergence term $\varepsilon_1$ and the optimization term $\varepsilon_2$ under our RGM assumption and the perturbation model induced by $T$.

\textbf{Bounding Convergence Error $\varepsilon_1$.}
We adapt an existing result to bound the discrepancy between the discrete MPNN output and its continuous counterpart.

\begin{theorem}[Class-conditional convergence \citet{maskey2022generalization}]
\label{thm:convergence}
Fix a domain $D\in\{S,T\}$ and a class $j\in[C]$, and draw $Z\sim \mu_{D}^{j}$ induced by $\Gamma_{D}^{j}$.
Under model assumptions in Appendix~\ref{app:assumptions}, the following holds with probability at least $1-2\rho$,
\begin{equation}
    \label{thm:convergence:goal}
    \left\| \Bar{h}_G(Z) - \Bar{h}_{W,P}(f)\right\|_{\infty} \leq \Delta_N,
\end{equation}
where
\begin{equation}
    \label{eq:DeltaN_compact}
    \Delta_N
    :=
    \mathcal{O}\!\left(\frac{\log N}{N^{\frac{1}{D_{\mathcal{X}}+1}}} + N^{-1}\right),
\end{equation}
with omitted constant dependent upon class-wise RGM regularities and upon neural network constants.
\end{theorem}
We drop the $(D,j)$ in Eq.~\eqref{thm:convergence:goal} and \eqref{eq:DeltaN_compact} for readability, yet should bear in mind that the discrete hypothesis converges to its continuous counterpart constructed by certain RGM, and here we let it be the ground-truth $\Gamma_D^j$. It's however interesting to analyze convergence error to other classes of RGMs, potentially tied to learnability, since under certain conditions it's information-theoretically impossible to robustly distinguish from which RGM the graphs are drawn \citep{bangachev2025sandwiching}, and thus $\Delta_N$ would be high for all classes. We however leave this to future work.

The hidden constants in big oh notation depend polynomially on the Lipschitz constants of the message/update functions across layers, as well as RGM-dependent constants such as $\|f\|_\infty$, $L_f$, and $\|W\|_{\infty}$. The dependence on $D_{\mathcal{X}}$ reflects the intrinsic dimension of the latent space supporting $P$.
The main takeaway is that the hypothesis learned on larger graphs converges to certain limit, as revealed by previous works by \citet{keriven2020convergence,ruiz2020graphon,maskey2022generalization}.
See proofs in Thm.~\ref{proof:convergence}.  


\textbf{Bounding Optimization Error $\varepsilon_2$.}
The optimization term measures an \emph{algorithmic suboptimality} gap in the continuous hypothesis construction:
even under a fixed data-generating mechanism (up to observational equivalence), the constructed predictor $\bar h_{W,P}(f)$ may deviate from a best-in-family predictor within the considered hypothesis family.
To bound this gap without referring to any unknown ground-truth weights, we model a \emph{construction change} by an operator $T\in\mathcal T$ (Def.~\ref{def:perturbation}), induced jointly by
(1) an RGM deformation $\tau$ acting on $(W,P)$ and
(2) a bounded parameter update $\Delta\Theta$ acting on the MPNN message/update MLPs in $\Phi,\Psi$.
This yields a perturbed predictor $T(\bar h_{W,P}(f))=\tilde{\bar h}_{W_\tau,P_\tau}(f)$, whose deviation from $\bar h_{W,P}(f)$ can be controlled by a stability bound (Thm.~\ref{thm:opt_main}).

\begin{assumption}[Optimization reachability]
\label{ass:reachability}
Let $\mathcal T$ denote the family of construction-change operators in Def.~\ref{def:perturbation}.
We assume the best-in-family continuous predictor is (approximately) reachable from the constructed one:
there exists $T^\star\in\mathcal T$ such that
\[
\bigl\|\bar h^*_{W,P}(f)-T^\star(\bar h_{W,P}(f))\bigr\|_\infty \le \delta_{\mathrm{opt}},
\]
where $\delta_{\mathrm{opt}}\ge 0$ captures residual algorithmic suboptimality not explained by our perturbation model.
\end{assumption}
Importantly, the perturbation magnitude manifests itself as the size of an \emph{algorithmic update}, e.g., finite-step training/fine-tuning or trust-region updates.

\begin{theorem}[Optimization term bound via construction stability]
\label{thm:opt_main}
Let $T\in\mathcal T$ be as in Def.~\ref{def:perturbation}, induced by an RGM deformation $\tau$ (Ass.~\ref{ass:tau_Pconstants}, \ref{ass:tau_rgm}) and MPNN parameter perturbations $\Delta\Theta$ (Ass.~\ref{ass:msg_update}).
Under Assumptions~\ref{ass:tau_Pconstants}, \ref{ass:tau_rgm}, \ref{ass:msg_update}
(and the standing assumptions in Appendix~\ref{app:assumptions}), then:
\begin{equation}
\label{eq:opt_term_bound}
    \|\bar h_{W,P}(f)-T(\bar h_{W,P}(f))\|_\infty \le \Delta_{\Gamma,\Theta},
\end{equation}
where
%
\begin{small}
\begin{equation}
    \label{thm:perturb_bound:goal}
    \Delta_{\Gamma,\Theta} = \sum_{l=1}^{L} C_{1}^{(l)} \prod_{l'=l+1}^{L} C_{2}^{(l')} + \tilde{C}_1^{(L)}  C_{\nabla w} \lVert \nabla \tau \rVert_{\infty} + \tilde{C}_2^{(L)} N_{P_\tau}.
\end{equation}
\end{small}
Consequently, under Assumption~\ref{ass:reachability}:
\begin{equation}
\label{eq:eps2_bound}
    \|\bar h_{W,P}(f)-\bar h^*_{W,P}(f)\|_\infty
    \le \Delta_{\Gamma,\Theta}+\delta_{\mathrm{opt}}.
\end{equation}
\end{theorem}

The bound $\Delta_{\Gamma,\Theta}$ provides a stability view of the optimization contribution:
it increases with the deformation strength (e.g., through $C_{\nabla w}$, $\|\nabla\tau\|_\infty$ and $N_{P_\tau}$),
and with the magnitude of parameter updates encoded in $\Delta\Theta$ (via layer-wise Lipschitz $C_\Box^{(l)}$ and perturbation factors $\tilde{C}_\Box^{(L)}$).
This suggests controlling the adaptation step size (e.g., weight regularization, early stopping, trust-region updates)
and improving latent-space alignment can reduce the optimization term.
We emphasize that $\tau$ models a construction-level mismatch within an observationally equivalent RGM class,
rather than the domain divergence term handled separately by $\Delta_D$.
See Thm. \ref{proof:optimization} in Appendix \ref{app:optimization} for complete results and proofs.

\section{Experiments}
\label{sec:experiments}






We design three experiments that correspond one-to-one to the three empirically verifiable terms discussed above \footnote{All codes are anonymously available: \href{https://anonymous.4open.science/r/uai2026-2BC4}{[Here]}}.



\paragraph{Exp 1: Latent Wasserstein Distance as Domain Divergence ($\Delta_D$).}
Proposition~\ref{prop:dom_div} attributes domain shift to a change of the class-wise RGM latent distribution $P$ and predicts $\Delta_D$ is captured by Wasserstein distance between latent positions.
%
%
To link $\Delta_D$ to learning, we test on the \textsc{PTC} dataset \citep{helma2001predictive} which contains binary carcinogenicity labels of chemical structures for four groups of rodents: male mice (MM), male rats (MR), female mice (FM) and female rats (FR).
We train our 3-layer GIN on one group and test on the other three, aiming to show that the estimated latent WD correlates with test losses.
%
%
We estimate the class-wise latent positions using the latent space model (LSM) \citep{hoff2002latent}, however due to computation considerations, via vanilla maximum likelihood estimation without Markov Chain Monte Carlo (MCMC).
Then, we calculate the same-class WD $\mathcal{W}_2(\hat{P}_S^j, \hat{P}_T^j)$ and sum over all classes as the metric to report.
See details in Appendix~\ref{app:exp1}.



\begin{figure}[ht]
    \centering
    \subfigure{%
        \includegraphics[width=0.49\linewidth,trim=8pt 8pt 8pt 6pt,clip]{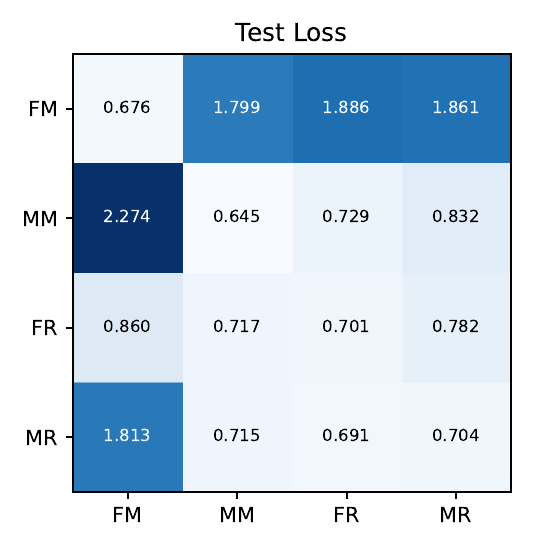}}
    \hfill
    \subfigure{%
        \includegraphics[width=0.49\linewidth,trim=8pt 8pt 8pt 6pt,clip]{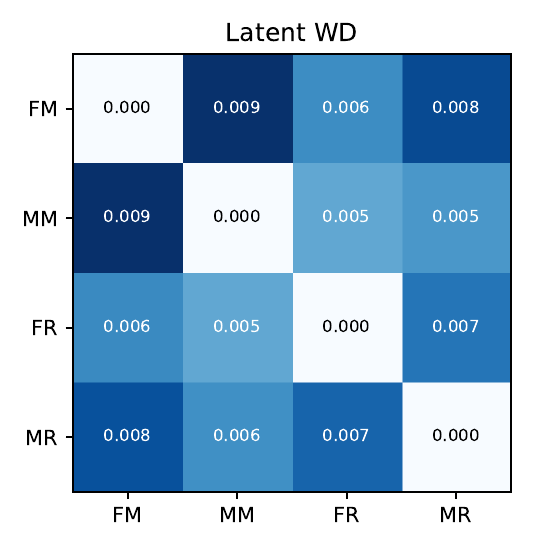}}
    \caption{Estimated latent WD correlates with test losses.}
    \label{fig:exp1_ptc}
\end{figure}

\textbf{Findings.}
%
%
%
Figure~\ref{fig:exp1_ptc} reports cross-group transfer matrices: entry $(i,j)$ is the test loss for the model trained on group $i$ and tested on group $j$ (left), and the corresponding latent WD between groups $i$ and $j$ (right).
We report the Pearson (PCC) and the Spearman correlation coefficients (SRC) for non-diagonal entries, respectively examining the linear and monotonic correlation between estimated latent WD and the test loss.
Our result show that the estimated shift correlates with performance degradation (PCC $0.726$, $p=0.007$; SRC $0.769$, $p=0.003$), supporting latent WD as a practical proxy for $\Delta_D$ revealed by our bound.
%

\begin{table}[ht]
\centering
\scriptsize
\setlength{\tabcolsep}{3pt}
\renewcommand{\arraystretch}{1.15}

\begin{adjustbox}{max width=\columnwidth}
\begin{tabular}{lccccc}
\toprule
Model & $h=1$ & $h=2$ & $h=3$ & $h=4$ & $h=5$ \\
\midrule
\multicolumn{6}{l}{\textbf{IMDB-MULTI} (1K)}\\
1-WL & \mstd{89.0}{2.45} & \mstd{243.0}{6.81} & \mstd{256.0}{6.81} & \mstd{260.8}{7.14} & \mstd{263.4}{6.95} \\
GIN  & \mstd{12.0}{0.63} & \mstd{10.2}{0.75} & \mstd{7.4}{0.49}  & \mstd{5.6}{1.36}  & \mstd{3.8}{0.75} \\
PPGN & \mstd{4.0}{0.00}  & \mstd{9.4}{1.02}  & \mstd{15.0}{1.26} & \mstd{19.8}{2.64} & \mstd{26.4}{1.74} \\
\midrule
\multicolumn{6}{l}{\textbf{NCI1} (2K)}\\
1-WL & \mstd{20.6}{0.49} & \mstd{239.2}{5.71} & \mstd{814.8}{10.17} & \mstd{1196.0}{9.10} & \mstd{1398.8}{8.93} \\
GIN  & \mstd{10.2}{1.17} & \mstd{19.4}{1.02}  & \mstd{20.2}{0.75}  & \mstd{15.4}{7.20}   & \mstd{15.4}{5.75} \\
PPGN & \mstd{17.4}{1.85} & \mstd{51.8}{5.19}  & \mstd{91.4}{8.78}  & \mstd{130.0}{12.26} & \mstd{144.0}{14.89} \\
\midrule
\multicolumn{6}{l}{\textbf{PROTEINS} (1K)}\\
1-WL & \mstd{36.8}{0.40} & \mstd{554.2}{2.40} & \mstd{730.6}{1.96} & \mstd{790.6}{1.96} & \mstd{821.0}{1.79} \\
GIN  & \mstd{7.6}{0.49}  & \mstd{13.8}{1.33}  & \mstd{19.0}{3.22}  & \mstd{19.2}{3.97}  & \mstd{25.0}{10.83} \\
PPGN & \mstd{15.6}{3.83} & \mstd{68.2}{18.00} & \mstd{74.4}{41.14} & \mstd{3.8}{1.17}   & \mstd{1.6}{1.20} \\
\bottomrule
\end{tabular}
\end{adjustbox}

\caption{Estimated truncation rank $r_{\varepsilon=1\%}$ (mean$\pm$std).}
\label{tab:rank-energy}
\end{table}

\paragraph{Exp 2: Truncated-Spectrum and Geometry ($\lambda_r$).}
We next empirically examine whether the truncated-spectrum Assumption~\ref{assump:truncated_vrkhs} holds and how spectrum geometry varies with practical model class choices.
We consider three hypothesis classes spanning different expressiveness: 1-WL subtree kernel (nonparametric), GIN (1-WL), and PPGN (3-WL), on three graph classification datasets IMDB-MULTI, NCI1, and PROTEINS, where missing node features are replaced by degree-binned categorical features.
For each method, we compute the empirical spectrum by eigendecomposing the diagonally-normalized Gram matrix built from $n$ sampled graphs, sweeping depth $h\in\{1,\ldots,5\}$ and random seeds.
For GIN/PPGN, we form the dot-product kernel $K(G,G')=\langle \phi(G),\phi(G')\rangle$ using learned embeddings (GIN: $d=256$, PPGN: $d=512$), with standard training on an 80/20 train/val split (GIN: Adam, 200 epochs, batch size 64; PPGN: default settings \citep{maron2019provably}) and no explicit embedding normalization.
We summarize the truncation by reporting $r_\varepsilon:=\min\{r:(\sum_{i>r}\lambda_i/\sum_i\lambda_i) \le \varepsilon\}$ with tail energy taking up $\varepsilon=1\%$ of total.

\begin{table}[h]
\centering
\scriptsize
\setlength{\tabcolsep}{3pt}
\renewcommand{\arraystretch}{1.15}

\begin{adjustbox}{max width=\columnwidth}
\begin{tabular}{lccccc}
\toprule
Dataset (4K) & $d=64$ & $d=128$ & $d=256$ & $d=512$ & $d=1024$ \\
\midrule
IMDB-MULTI & 4.8 $\pm$ 0.75 & 3.2 $\pm$ 0.40 & 2.6 $\pm$ 0.49 & 2.0 $\pm$ 0.00 & 2.0 $\pm$ 0.00 \\
NCI1 & 6.0 $\pm$ 0.63 & 8.8 $\pm$ 0.98 & 6.2 $\pm$ 0.75 & 5.6 $\pm$ 1.50 & 3.0 $\pm$ 0.89 \\
PROTEINS & 9.4 $\pm$ 2.24 & 7.8 $\pm$ 1.94 & 3.8 $\pm$ 0.98 & 2.0 $\pm$ 0.00 & 1.8 $\pm$ 0.40 \\
\bottomrule
\end{tabular}
\end{adjustbox}

\caption{Estimated truncation rank $r_{\varepsilon=1\%}$ for GIN ($h=5$).}
\label{tab:gin-rank-energy-e2}
\vspace{-0.6em}
\end{table}


\textbf{Findings.}
We observe in Table~\ref{tab:rank-energy} that the spectra concentrate in a small number of directions, across datasets and hypothesis classes, and the empirical truncated dimension $r_{\varepsilon=1\%} \ll \min(n,d)$ is much smaller than its theoretical maximum.
This supports the truncated-spectrum assumption.
Figure~\ref{fig:spec_v1} in Appendix~\ref{app:exp2} further shows a near-exponential eigenvalue decay, which is consistent with predictions given by the approximation theory \citep{belkin2018approximation}.
Finally, for GIN with the dot-product kernel $K(G,G')=\langle \phi(G),\phi(G')\rangle$, the Gram spectrum coincides with the empirical covariance spectrum of the learned embeddings;
and the fact that Table~\ref{tab:gin-rank-energy-e2} shows that a larger embedding dimension $d$ yields smaller $r_{\varepsilon=1\%}$ actually indicates that most additional dimensions carry negligible variance and the representation is effectively low-dimensional under our setup.
More results in Appendix~\ref{app:exp2}.

\paragraph{Exp 3: Amplitude ($K'\Xi+K''$).}
The amplitude factor in our main theorem aggregates (i) a learning-hardness component $\Xi$, shrinking with convergence errors $\Delta_N$, and (ii) a construction-stability component controlled by $\Delta_{\Gamma,\Theta}$.
Here we provide empirical evidence for both aspects, focusing on qualitative trends rather than numerical tightness.

\textbf{(A) Convergence error $\Delta_N$ decreases with larger graph size $N$ on simulation.}
Theorem~\ref{thm:convergence} suggests that training on larger size of graphs leads to lower $\Delta_N$ and thus transfer penalty.
We then conduct synthetic multi-class graph classification under controlled shifts.
Keeping the latent shift level fixed, we increase the graph size $N$ while holding the target-domain evaluation size fixed.
Fig.~\ref{fig:test_loss_variation} (a) shows that target test loss decreases monotonically as $N$ grows, consistent with the theorem prediction that larger $N$ reduces the convergence component $\Delta_N$, hence shrinking $\Xi$.
%
%
Implementation details and discussions are provided in Appendix~\ref{app:exp3:sim}

\textbf{(B) Layerwise stratified structure in $\Delta_{\Gamma,\Theta}$ on real data.}
Theorem~\ref{thm:opt_main} yields a \emph{downstream layerwise product}
$P_\ell := C_{1}^{(l)}\prod_{l'=l+1}^{L} C_{2}^{(l')}$,
whose multiplicative structure implies a \emph{layerwise stratified effect} through the downstream Lipschitz product. 
%
%
This implies that layerwise contributions to overall stability differ, and that non-uniform regularization could change the layerwise multiplicative structure, which should be visible from the trained weights.
We therefore propose two simple non-uniform regularizers, front-heavy and back-heavy, with stronger regularization penalty on earlier and later layers respectively, and ask the following questions:
\textbf{Q1:} compared to \texttt{L2}, do front-/back-heavy regularizers systematically change layerwise downstream amplification proxies?
\textbf{Q2:} which layer is more fragile to perturbation, and does non-uniform regularization change this fragility pattern?
\textbf{Q3:} under a fixed budget, does a non-uniform allocation of regularization strength provide practical performance headroom over the uniform choice?
%


\begin{figure}[ht]
    \centering
    \includegraphics[width=0.8\linewidth]{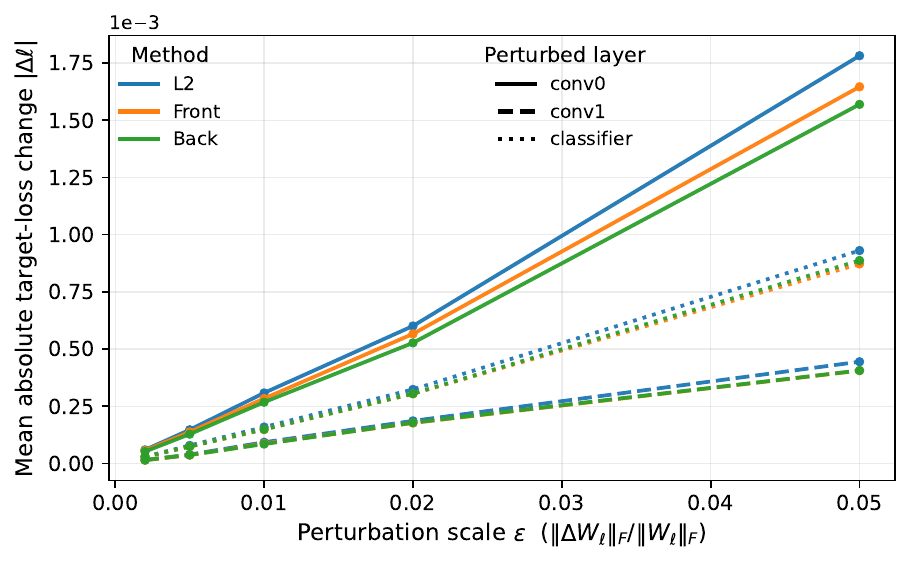}
    \caption{\textbf{Layerwise perturbation sensitivity.}
    Early-layer sensitivity (\texttt{conv0}) is consistently the largest. Both non-uniform schemes reduce early-layer fragility relative to \texttt{L2}.}
    \label{fig:exp3_B}
\end{figure}

\vspace{-0.5\baselineskip}

We compare these methods on \textsc{Mutagenicity} dataset \citep{morris2020tudataset} containing four groups partitioned via edge density $\mathrm{M0}$--$\mathrm{M3}$.
We train our GCN with two convolution layers (\texttt{conv0}, \texttt{conv1}) and one linear layer (\texttt{classifier}) on one group and test on the others, aiming to observe the behaviors.
By considering Frobenius norm as $C_{\Box}^{(l)}$ proxies, we conclude that:
(Q1) non-uniform regularizer significantly changes downstream proxy as revealed by layerwise product structure;
(Q2) the early layer (i.e. \texttt{conv0} in Figure~\ref{fig:exp3_B}) is more sensitive, while both front-/back-scheme can reduce early layer sensitivity;
(Q3) non-uniform strength allocation under fixed budget shows higher headrooms compared to uniform scheme.
%
%
%
%
%
%
%
Full experimental settings and results are discussed in Appendix~\ref{sec:exp3:mechanism}.

\section{Conclusion}
\label{sec:conclusion}

This paper presents a graph domain adaptation generalization theory for graph classification, filling up an important unexplored gap of studying graph distribution shifts.
Our framework contributes a random graph formulation for graph distributions via a function perspective in vRKHS, laying a foundation for future theoretical advances.

\begin{acknowledgements} 

    The authors would like to acknowledge the assistance given by Research IT and the use of the Computational Shared Facility at The University of Manchester.
    The authors would also like to thank Dmitry and Jaeyoung for their detailed comments on the manuscript and for insightful discussions that improved the presentation.
\end{acknowledgements}

\bibliography{uai2026-template}

\newpage

\onecolumn

\title{A Theory of Random Graph Shift in Truncated-Spectrum vRKHS\\(Supplementary Material)}
\maketitle

\appendix


\startcontents[app]

\printcontents[app]{}{0}{%
  \section*{Table of Contents}
  \setcounter{tocdepth}{2} 
}

\clearpage
\section{Notations}
\label{app:notations}

\begin{center}
\small
\setlength{\tabcolsep}{6pt}
\renewcommand{\arraystretch}{1.15}
\begin{longtable}{>{\raggedright\arraybackslash}p{0.26\linewidth} p{0.68\linewidth}}
\toprule
Notation & Description \\
\midrule
\endfirsthead

\toprule
Notation & Description \\
\midrule
\endhead

\midrule
\multicolumn{2}{r}{\emph{Continued on next page}}\\
\endfoot

\bottomrule
\endlastfoot

\notegroup{Graphs and basic objects}
$G=(V,E,Z,A)$ & A graph with node set $V=(1,\ldots,N)$, edge set $E\subseteq V\times V$, node features $Z\in\mathbb{R}^{N\times F}$ (rows $z_i\in\mathbb{R}^F$), and adjacency/weight matrix $A\in\mathbb{R}^{N\times N}$ (entries $a_{ij}$).\\
$N,\ F$ & Number of nodes; node feature dimension.\\
$d_i$ & Degree of node $i$ (number of incident edges).\\
$G=(Z,A)$ & Simplified notation when node/edge sets are implicit.\\

\notegroup{Norms, Lipschitzness, and function spaces}
$\|\cdot\|,\ \|\cdot\|_F,\ \|\cdot\|_p,\ \|\cdot\|_\infty$ & Euclidean norm, Frobenius norm, $p$-norm, and $\infty$-norm.\\
$\|f\|_\infty$ & Supremum norm of $f:\mathcal{X}\to\mathbb{R}^F$, $\|f\|_\infty=\sup_{x\in\mathcal{X}}|f(x)|$.\\
$L_f^{\square}$ & Lipschitz constant of $f$ w.r.t.\ an input norm $\|\cdot\|_\square$: $\|f(x)-f(y)\|_\infty\le L_f^{\square}\|x-y\|_\square$.\\
$[x,y]$ & Concatenation of vectors $x$ and $y$ (used for two-input Lipschitz bounds, e.g., kernels).\\
$L^2(\mu;\mathbb{R}^C)$ & Square-integrable vector-valued functions: $\{g:\mathcal{Z}\to\mathbb{R}^C\mid \int_{\mathcal{Z}}\|g(z)\|_2^2\,d\mu(z)<\infty\}$.\\

\notegroup{Random Graph Model and geometry}
$\Gamma=(W,P,f)$ & Random graph model with kernel $W$, latent distribution $P$, and feature map $f:\mathcal{X}\to\mathbb{R}^F$.\\
$(\mathcal{X},d),\ \mathcal{X}\subseteq\mathbb{R}^D$ & Compact latent metric space and its ambient dimension $D$.\\
$\mathcal{N}(\mathcal{X},\varepsilon,d)$ & Covering number: minimum number of $d$-balls of radius $\varepsilon$ covering $\mathcal{X}$.\\

\notegroup{Optimal transport and measures}
$P,Q$ & Probability distributions on $\mathcal{X}$.\\
$\Pi(P,Q)$ & Set of couplings (joint distributions) with marginals $P$ and $Q$.\\
$\gamma$ & A coupling random variable / measure in $\Pi(P,Q)$.\\
$c(x,y)$ & Transportation cost function.\\
$W_p(P,Q)$ & $p$-Wasserstein distance: $W_p^p(P,Q)=\inf_{\gamma\in\Pi(P,Q)}\int c(x,y)^p\,d\gamma(x,y)$.\\
$f_{\sharp}P$ & Push-forward of $P$ under $f$: for measurable $B\in\mathcal{B}(\mathbb{R}^F)$, $(f_{\sharp}P)(B)=P(f^{-1}(B))$.\\
$\mathcal{B}(\mathbb{R}^F)$ & Borel $\sigma$-algebra on $\mathbb{R}^F$.\\

\notegroup{Domain Adaptation}
$\{\Gamma_D^j=(W_D^j,P_D^j,f_D^j)\}_{j=1}^C$ & Class-wise RGMs for domain $D$ (e.g., $D\in\{S,T\}$).\\
$\mu_D$ & Graph distribution induced by class-wise pushforwards: $\mu_D=\bigotimes_{j=1}^C f_{D\sharp}^j P_D^j$.\\
$\mu_S,\ \mu_T$ & Source and target graph distributions.\\
%
$\epsilon_S(h),\ \epsilon_T(h)$ & Source/target risks (as defined in the main text; used in the main bound).\\
$\rho$ & Failure probability parameter in high-probability statements (e.g., prob.\ at least $1-3\rho$).\\

\notegroup{Main-theorem quantities}
$\Delta_D$ & Domain discrepancy term (main text Eq./Lemma reference).\\
$\lambda_r,\ r$ & Truncated-spectrum eigenvalue at rank $r$ and truncation rank.\\
$\Xi$ & Aggregated amplitude term (collecting convergence/optimization/approximation components).\\
$\Delta_N$ & Convergence term (typically decreases with graph size / sample size).\\
$\Delta_{\Gamma,\Theta}$ & Construction-stability / optimization term due to RGM and parameter perturbations.\\
$\delta_{\mathrm{opt}},\ \varepsilon_3,\ \varepsilon_4$ & Optimization suboptimality; approximation and label/noise (as defined in the paper).\\
$C,\ K',\ K''$ & Positive constants in the bound (dependent on model/hypothesis class; see main text / appendices).\\

\notegroup{Sufficiency Condition Quantities}
$m_S,\ m_T$ & Max number of graph instances per class in source/target domains (Appendix I).\\
$M$ & Hybrid sample-size proxy combining graph count and node count (Appendix I; e.g., $M=\max(N_S m_S, N_T m_T)$).\\
$\widehat P_S^j,\ \widehat P_T^j$ & Empirical (estimated) class-wise latent distributions for class $j$ in source/target.\\
$W_2$ & Maximal class-wise (2-)Wasserstein shift, e.g., $\max_{j\in[C]} W_2(\widehat P_S^j,\widehat P_T^j)$.\\
$L_\sigma$ & Max Lipschitz constant across MLP activations, $L_\sigma=\max_{l\le L} L_{\sigma_l}$.\\
$\lambda_M$ & Max singular-value proxy across MLP weight matrices, $\lambda_M=\max_{l\le L}\lambda_l$.\\
$L_P$ & Max Lipschitz proxy for message/update maps, $L_P=\max_{t\le T}\max(L_{\Phi^{(t)}},L_{\Psi^{(t)}})$.\\
$\Delta_M$ & Relative magnitude of MLP weight perturbations, e.g., $\max_l \|\Delta W_l\|_F/\|W_l\|_F$.\\
$\Delta_\Theta,\ D_\Theta$ & Relative/absolute perturbation proxies for MPNN parameters across layers (Appendix I).\\

\notegroup{Auxiliary constants in proofs}
$C^{(T)}_{\cdot},\ \widetilde C^{(T)}_{\cdot},\ U_{\cdot}$ & Collections of positive constants introduced to simplify intermediate bounds; each is explicitly defined where first used (e.g., in Appendices F--I/G) and depends only on fixed problem/model parameters (RGM regularity, Lipschitz/spectral norms, etc.).\\

\end{longtable}
\end{center}

In this section, we provide a more detailed explanation of the used notations across both the main paper and appendices. 
A graph $G=(V, E, Z, A)$ contains a set of nodes $V = (1, \cdots, N)$, edges $E \subseteq V \times V$, feature vectors $\left\{z_i \in \mathbb{R}^{F}\right\}_{i\in V}$  that characterize the nodes and form  the feature matrix $Z \in \mathbb{R}^{N \times F}$ (often called a graph signal), and the adjacency matrix $A \in \mathbb{R}^{N \times N}$ with each element $a_{ij}$ denoting the edge weight.
We sometimes simplify the graph notation to $G=(Z, A)$.
For a given graph, we denote the degree of each node $i$ by $d_i$, i.e., the number of edges that connect the node $i$ to the other nodes in the graph.

In general, we denote the Euclidean norm by $\left\lVert \cdot \right\rVert$,   Frobenius norm by $\left\lVert \cdot \right\rVert_F$, the general $p$-norm by $\left\lVert \cdot \right\rVert_p$, and infinity norm by $\lVert \cdot \rVert_{\infty}$. 
The infinity norm of a function $f: \mathcal{X} \to \mathbb{R}^F$ is defined as $\left\lVert f \right\rVert_\infty = \sup_{x \in \mathcal{X}} \left\lvert f(x) \right\rvert$.
Lipschitz continuity of a function $f$ is defined with respect to a norm, given the existence of a Lipschitz constant.
Specifically, given an  norm $\lVert \cdot \rVert_{\Box}$ of interest, the function $f$ is  $L_f^{\Box}$-Lipschitz continuous if, $\forall x, y \in \mathcal{X}$, there exists a constant $L_f^{\Box} > 0$ such that $\left\lVert f(x) - f(y) \right\rVert_{\infty} \leq L_f^{\Box} \lVert x-y \rVert_{\Box}$. 
We only specify the used input norm when it is not Euclidean norm, by where the symbol $\Box$ is, e.g., $\lVert \cdot \rVert_{\infty}$ for $L_f^{\infty}$-Lipschitz continuous.
For a function with two input vectors, e.g., the RGM kernel, its $L_f^{\infty}$-Lipschitz continuity satisfies $\left\lVert f(x,y) - f(x',y') \right\rVert_{\infty} \leq L_f^{\infty}\lVert [x, y]-[x',y'] \rVert_{\infty} = L_f^{\infty}\max \{ \lVert x - x' \rVert_{\infty}, \lVert y - y' \rVert_{\infty} \}$, where $[x,y]$ is a concatenation of the two vectors $x$ and $y$.
For a probability measure $\mu$ defined on space $\mathcal{Z}$, we denote $L_2(\mu;\mathbb{R}^C) := \{ f: \mathcal{Z} \to \mathbb{R}^C \mid \int_{\mathcal{Z}} \left\| f(Z) \right\|_2^2 \, d\mu(Z) \}$.

The RGM of our interest is defined over a compact metric space $(\mathcal{X}, d)$, where $\mathcal{X} \subseteq \mathbb{R}^D$. 
Its covering number is defined as the number of balls with radius $\epsilon$ that are required to cover the whole space $\mathcal{X}$ under a metric $d$, denoted by $\mathcal{N}(\mathcal{X}, \epsilon, d)$.
Given two probability distributions $P$ and $Q$ defined on $\mathcal{X}$, their Wasserstein $p$-distance is defined by
\begin{equation}
    \mathcal{W}_p^p (P, Q) = \inf_{\gamma \sim \Pi(P, Q)} \int_\mathcal{X} c(x,y)^p \, d\gamma(x,y) = \inf_{\gamma \sim \Pi(P, Q)} \mathbb{E}_{(x,y) \sim \gamma} c(x,y)^p.
\end{equation}
where $\Pi$ is the set of all couplings of $P$ and $Q$,  and $c(x,y)$ is the cost of moving $x$ to $y$. 
The push-forward measure $f\sharp  P$ of a distribution $P$ defined over $\mathcal{X} $ is obtained by transferring $P$ using the mapping function $f: \mathcal{X} \to \mathbb{R}^F$ to $f\sharp  P$ defined over $\mathbb{R}^F$.
By a more formal definition, for a measurable set  $B \in \mathfrak{B}\left( \mathbb{R}^F \right)$, which is an element of  the  $\sigma$-Borel algebra $\mathfrak{B}(\mathbb{R}^F)$ in space $\mathbb{R}^F$,  the push-forward is defined to be $f\sharp P(B) = P\left( f^{-1}(B) \right)$. 
When using different RGMs to model different classes, denoted by $\{ \Gamma_D^j = (W_D^j, P_D^j, f_D^j) \}_{j=1}^C$, the graph distribution is equivalent to the product measure of the push-forward probability measures of the different classes, i.e., $\mu_{D} = \otimes_{j=1}^C f^j_D \sharp P^j_D$.
In the DA setting, the source and target graph distributions, which are computed from the RGMs used for generating the source and target graphs, are denoted by $\mu_{S} = \otimes_{j=1}^C f^j_S \sharp P^j_S$ and $\mu_{T} = \otimes_{j=1}^C f^j_T \sharp P^j_T$, respectively.

\section{Definitions}
\label{app:definitions}

In this section, we formalize various definitions, operations and properties relevant to RGMs and the used neural network family. 
These are used across both the main paper and appendices. 

\subsection{On RGM and Its Deformation}

A sampling operator $S_X F$ defined over an RGM is used across the proof. 
The two versions of kernel degree are used in convergence and optimization error analysis, and their bounds play a role.
The concept of RGM deformation plays a key role in optimization error analysis.

\begin{definition}[\textbf{Function Sampling Operator}]
\label{def:sampling}
    Given an RGM $\Gamma = (W, P, f)$, sample  a set of points $X=\left\{x_i\right\}_{i=1}^N$ from $P$ in space $\mathcal{X}$. For an arbitrary mapping function $F: \mathcal{X} \rightarrow \mathbb{R}^{ F}$, define its sampling operator $S_X F \in \mathbb{R}^{N \times F}$ as a matrix,  of which the $i$-th row is $F(x_i)$, denoted by   $\left(S_X F\right)_i$.
\end{definition}

\begin{definition}[\textbf{Discrete Kernel Degree}]
\label{def:dx}
   Given an RGM $\Gamma = (W, P, f)$ with the kernel function $W: \mathcal{X} \times \mathcal{X} \to \mathbb{R}$,  the discrete kernel degree of $x\in\mathcal{X} $ is defined  over a set of points $X=\{x_i\}_{i=1}^{N}$ sampled from the probability distribution $P$ in the latent space $\mathcal{X}$, given by:
    \begin{equation}
        d_X(x) = \frac{1}{N} \sum_{i=1}^{N} W(x, x_i).
    \end{equation}
\end{definition}

\begin{definition}[\textbf{Continuous Kernel Degree}]
\label{def:dw}
Given an RGM $\Gamma = (W, P, f)$ with the kernel function $W: \mathcal{X} \times \mathcal{X} \to \mathbb{R}$,  the continuous kernel degree of $x\in\mathcal{X} $ is defined over the  probability distribution $P$ in the latent space $\mathcal{X}$, given by: 
    \begin{equation}
        d_W(x) = \int_{\mathcal{X}} W(x, y) \, dP(y).
    \end{equation}
\end{definition}

\begin{definition}[\textbf{RGM Deformation}]
\label{def:rgm_perturbation}
    An RGM deformation is defined as a diffeomorphism mapping $\tau: \mathcal{X} \to \mathcal{X}$ that causes  spatial deformation of an RGM, supporting deformation forms like $W_\tau := (Id - \tau) W$, $P_\tau := (Id - \tau) \sharp P$ and  $f_\tau := (Id - \tau) f$, where  $Id$ denotes an identity function. The kernel deformation needs to satisfy $\ W_\tau(x, y) = W(x - \tau(x), y - \tau(y))$, $\forall x, y \in \mathcal{X}$.
\end{definition}

\subsection{On Neural Network Families and Perturbation}
The advantage of RGM is to allow the definition of a continuous version of a GNN,  which acts on latent space $\mathcal{X}$ \citep{keriven2020convergence,maskey2022generalization}. 
Instead of propagating a specific input graph signal $Z$, the continuous version propagates a function $f$  that is defined within the RGM and can generate  node signals. 
We formalize  descriptions of mean aggregation (MA) and architectures of the MPNN feature extractor based on MA,  in terms of both the discrete and continuous versions;  explain the MLP architecture that is used to construct the   classifier and the MPNN message and update functions;  and define the perturbation mapping for an input cMPNN.

\begin{definition}[\textbf{Discrete Mean Aggregation}]
\label{def:agg_x}
Given an RGM $\Gamma = (W, P, f)$ and  a message function $\Phi(f(x_i), f(x_j)): \mathbb{R}^F \times \mathbb{R}^F \to \mathbb{R}^H$ that creates messages from node $j$ to $i$, the discrete MA is defined  over a set of points $X=\{x_i\}_{i=1}^{N}$ sampled from the probability distribution $P$ in the latent space $\mathcal{X}$, given by: 
    \begin{equation}
    \label{eq:aggX} 
        M_X^{\Phi,f}(\cdot)  = \frac{1}{N} \sum_{i=1}^{N} \frac{W(\cdot, x_i)}{d_X(\cdot)} \Phi(f(\cdot), f(x_i)).
    \end{equation}
\end{definition}

\begin{definition}[\textbf{Continuous Mean Aggregation}]
\label{def:agg_w}
    Given an RGM $\Gamma = (W, P, f)$ and a message function $\Phi(f(x_i), f(x_j)): \mathbb{R}^F \times \mathbb{R}^F \to \mathbb{R}^H$ that creates messages from node $j$ to $i$, the continuous MA is defined over the probability distribution $P$ in the latent space $\mathcal{X}$, given by:
    \begin{equation}
        M_W^{\Phi,f}(\cdot)= \int_\mathcal{X} \frac{W(\cdot, y)}{d_W(\cdot)} \Phi(f(\cdot), f(y)) \, dP(y).
    \end{equation}
\end{definition}
In the above Def. \ref{def:agg_x} and Def. \ref{def:agg_w}, the superscript $\Phi,f$ highlights the impact of $\Phi$ and $f$. In later proofs, we sometimes abuse the notation by using $M_X(\cdot) = M_X^{\Phi,f}(\cdot)$ and $M_W(\cdot) = M_W^{\Phi,f}(\cdot)$ for the convenience of writing.

\begin{definition}[\textbf{Discrete Feature Extractor MPNN}]
\label{def:mpnn}
    A discrete  MPNN feature extractor $\Bar{h}_{G}: \mathbb{R}^{N \times F} \to \mathbb{R}^{F_T}$ is defined as 
    a $T$-layer MPNN  followed by a pooling layer based on MA.
   The  MPNN mapping function $h_{G}^{(T)}: \mathbb{R}^{N \times F} \to \mathbb{R}^{N \times F_T}$ is expressed layer-wise by
     \begin{equation}
        h_{G}^{(T)} = h_{\Theta_G}^{(T)} \circ h_{\Theta_G}^{(T-1)} \circ \cdots \circ h_{\Theta_G}^{(t)} \circ \cdots \circ  h_{\Theta_G}^{(1)},
    \end{equation}
    where   $h_{\Theta_G}^{(t)}: Z^{(t-1)} \to Z^{(t)}$ denotes the feature mapping function and  $Z^{(t)} \in \mathbb{R}^{N \times F_{t}}$ denotes the computed features, both at layer $t$.
    The last layer of $h_{G}^{(T)}$ returns the feature matrix $Z^{(T)}$.

     The MPNN mapping is defined using a sequence of functions $ \{ \Phi^{(t)}, \Psi^{(t)} \}_{t=1}^T $.
    The mapping function $h_{\Theta_G}^{(t)}$ at each layer is formulated using the message function  $\Phi^{(t)}: \mathbb{R}^{2F_{t-1}} \to \mathbb{R}^{H_t}$ and  the update function $\Psi^{(t)}: \mathbb{R}^{F_{t-1} + {H_{t}}} \to \mathbb{R}^{F_t}$, 
where $F_t,H_t \in \mathbb{N}^+$ denote the layer-wise dimensions of the node and message representations. This results in the following  forward passing for $h_{\Theta_G}^{(t)}$: 
    %
    \begin{align}
    \label{mpnn_msg}
        \textmd{Message: } & m_i^{(t)} =  \frac{1}{N} \sum_{j=1}^{N} a_{ij} \Phi^{(t)} \left(z_i^{t-1}, z_j^{t-1}\right), \\
        \textmd{Update: } & z_i^{(t)} =  \Psi^{(t)} \left(z_i^{(t-1)}, m_i^{(t)}\right).
    \end{align}
    A pooling layer is appended in the end, obtaining the final  graph representation vector $\Bar{h}_G(Z) \in \mathbb{R}^{F_T}$ through MA: 
    \begin{equation}
        \Bar{h}_G(Z) = \frac{1}{N} \sum_{i=1}^{N} z^{(T)}_i.
    \end{equation}
    Working with normalized adjacency weights, i.e., $a_{ij}=\frac{W(x_i,x_j)}{d_X(x_i)}$,  the  message function in Eq. (\ref{mpnn_msg}) has the same form as $ M_X^{\Phi,f}$ in Eq. (\ref{eq:aggX}) from Def. \ref{def:agg_x}.
\end{definition}


\begin{definition}[\textbf{Continuous Feature Extractor cMPNN }]
\label{def:cmpnn}
    Given an RGM $\Gamma = (W, P, f)$, it induces a continuous formulation for the MPNN feature extractor  described in Def. \ref{def:mpnn}, which is referred to as a cMPNN and denoted by $\Bar{h}_{W,P}(f): \mathcal{X} \to \mathbb{R}^{F_T}$.
    It contains a continuous MPNN mapping function  denoted by $h_W^{(T)}(f): \mathcal{X} \to \mathbb{R}^{F_T}$, and a continuous pooling layer appended afterwards.  
    The continuous MPNN mapping function is expressed layer-wise by
 \begin{equation}
        h_{W}^{(T)} = h_{\Theta_W}^{(T)} \circ h_{\Theta_W}^{(T-1)} \circ \cdots \circ h_{\Theta_W}^{(t)} \circ \cdots \circ h_{\Theta_W}^{(1)},
    \end{equation}
where $h_{\Theta_W}^{(t)}: f^{(t-1)} \to f^{(t)}$ denotes the continuous feature mapping function at   layer $t$, and   $h_W^{(T)}(f) = f^{(T)}$.
The corresponding forward passing for $h_{\Theta_W}^{(t)}$ is given by
    \begin{align}
        \textmd{Message: } & g^{(t)}(x) = M_W^{\Phi^{(t)},f^{(t-1)}}(x), \\
        \textmd{Update: } & f^{(t)}(x) = \Psi^{(t)}\left(f^{(t-1)}(x), g^{(t)}(x)\right).
    \end{align}
In the end, the continuous pooling layer computes  the final output by
    \begin{equation}
        \Bar{h}_{W,P}(f) = \int_{\mathcal{X}} h_W^{(T)}(f)(x) \, dP(x).
    \end{equation}
\end{definition}

\begin{definition}[\textbf{MLP Architecture}]
\label{def:layermapping}
    An $L$-layer MLP  $f_{\textmd{MLP}}: \mathbb{R}^{F_1} \to \mathbb{R}^{F_2}$  is  defined as the following composite function:
    \begin{equation}
        f_{\textmd{MLP}}(W; x) := \circ_{l=1}^L (\sigma_l \circ W_l(x)) =  \sigma_L(W_L(\sigma_{L-1}W_{L-1}(\cdots W_1(x)))),
    \end{equation}
    where $W_l(x)$ is   a linear transformation applied to the input vector $x$ using the weight matrix $W_l$, and $\sigma_l$ is the  activation function. 
    Different $L$-layer MLPs are used to construct the classifier $h_{cls}$, as well as the MPNN message and update functions in the set $ \{ \Phi^{(t)}, \Psi^{(t)} \}_{t=1}^T $.
\end{definition}

\begin{definition}[\textbf{cMPNN Perturbation}]
\label{def:perturbation}
    Given  an RGM $\Gamma = (W, P, f)$ and its induced cMPNN $\Bar{h}_{W,P}(f)$, the  construction change used to formulate the perturbation mapping is defined by simultaneously (1) changing the neural network weights of the MPNN message and update functions,  and (2) deforming the RGM through perturbing the latent distribution $P$ and the kernel $W$ by a deformation $\tau$ satisfying Def. \ref{def:rgm_perturbation}.
    This results in 
    \begin{equation}
        T\left(\Bar{h}_{W,P}(f)\right) = \tilde{\Bar{h}}_{W_\tau,P_\tau}(f).
    \end{equation}
    In order  to obtain $ \tilde{\Bar{h}}$ from $\bar{h}$, the change of neural network weights  is defined as an additive  modification of the  weights,  generally denoted by $\Theta$, at each layer of the message and update function, given by
    \begin{equation}
       \tilde{\Theta}_{\Phi_l^{(t)}}   =  \Theta_{\Phi_l^{(t)}} +  \Delta \Theta_{\Phi_l^{(t)}}, \; \textmd{and } \tilde{\Theta}_{\Psi_l^{(t)}} =  \Theta_{\Psi_l^{(t)}} + \Delta \Theta_{\Psi_l^{(t)}}, 
    \end{equation}
for $t=1,2,\ldots, T$ and $l=1,2,\ldots, L$.
Here, the message function $\Phi$ and the update function $\Psi$  at the $t$-th MPNN layer  are both $L$-layer MLPs, thus each  weight matrix is distinguished using superscript $(t)$ for MPNN layers  while subscript $l$ for MLP layers in the notation.
\remark{(Interpretation of $\Delta\Theta$)}
The perturbation $\Delta\Theta$ encodes a construction/algorithmic update (e.g., finite-step training or adaptation)
and is not defined relative to any unknown ground-truth parameterization.
\end{definition}

\subsection{On Graph Distance}

We define a pseudo-metric that measures the distance between two graph signals $Z_1,Z_2 \in \mathcal{Z}$ to facilitate Assumption \ref{ass:rkhs}, which relies on a permutation mapping function $\sigma: \mathcal{I} \to \mathcal{I}$ defined over the index space $\mathcal{I}$.

\begin{definition}[\textbf{Pseudo Metric for Graph Signal Distance}]
\label{def:pseudo_metric}
Let $\Sigma$ denote the set of all permutation mappings in the index space $\mathcal{I}$, the permutation of interest minimizes the average row-wise Euclidean distance between two graph signals,
such that
\begin{equation}
    \sigma^* = \arg\min_{\sigma\in \Sigma} \frac{1}{N} \sum_{i=1}^N \lVert Z_{\sigma(i)} - Z'_i \rVert.
\end{equation}
We define a distance measure $d_{\Sigma}$  using $\sigma^*$, by
\begin{equation}
\label{def:d_Sigma}
    d_{\Sigma}(Z, Z') := \frac{1}{N} \sum_{i=1}^N \lVert Z_{\sigma^*(i)} - Z'_i \rVert.
\end{equation}
It can be easily proved that $d_\Sigma$ is a pseudo-metric, shown as follows.

\begin{proof}

It suffices to show $d_\Sigma$ satisfies non-negativity, symmetry, and triangle inequality.

\textbf{Non-negativity}. Since the Euclidean norm $\lVert \cdot \rVert$ is always non-negative, the average also induces non-negative results.

\textbf{Symmetry}. Given a permutation mapping $\sigma^*$  that minimizes the average row-wise Euclidean distance between $Z_1$ and $Z_2$, there always exists a matching mapping $\sigma'$  such that $d_\Sigma(Z_1,Z_2) = \frac{1}{N} \sum_{i=1}^N \lVert {Z_1}_{\sigma^*(i)} - {Z_2}_i \rVert = \frac{1}{N} \sum_{i=1}^N \lVert {Z_1}_i - {Z_2}_{\sigma'(i)} \rVert = \frac{1}{N} \sum_{i=1}^N \lVert {Z_2}_{\sigma'(i)} - {Z_1}_i \rVert = d_\Sigma(Z_2,Z_1)$.

\textbf{Triangle Inequality}. Let $\sigma^*$ denote the permutation applied to $Z_1$ to obtain  $d_\Sigma(Z_1,Z_2)$, and let $\tau'$ denote the permutation applied to $Z_3$ to obtain  $d_\Sigma(Z_2,Z_3)$. 
Therefore, it has
\begin{align}
\nonumber
    d_\Sigma(Z_1,Z_2) + d_\Sigma(Z_2,Z_3)
    & =  \sum_{i=1}^n \left( \left\lVert {Z_1}_{\sigma*(i)} - {Z_2}_i \right\rVert + \left\lVert {Z_2}_{i} - {Z_3}_{\tau'(i)} \right\rVert \right) \\  
    & \geq \sum_{i=1}^n \left\lVert {Z_1}_{\sigma*(i)} - {Z_3}_{\tau'(i)} \right\rVert \geq d_\Sigma(Z_1,Z_3).
\end{align}
This completes the pseudo-metric proof.
\end{proof}
\end{definition}

\section{Model Assumptions}
\label{app:assumptions}

For theory development, we assume well-accepted RGM properties  widely used by existing works on RGMs for graph learning \citep{keriven2020convergence,keriven2021universality,maskey2022generalization}, 
Lipschitz continuity on MPNNs well-accepted by deep learning community without compromising the usefulness of its practical insight \citep{khromov2024some}, 
healthy conditions of MLPs and MPNNs regarding to weights changes in order to enable the use of the  existing result in Thm \ref{thm:MLP} from \citet{bernstein2020distance};  basic properties of vRKHS reproducing kernel  following similar  assumptions made by Thm 4.6 of \citet{fiedler2023lipschitz}; 
and identifiable ground-truth labeling function commonly assumed in machine learning practice.

\begin{assumption} [\textbf{RGM Latent Space $\mathcal{X}$}]
\label{ass:space}
    (i) The RGM latent space $\mathcal{X}$ has upper bounded diameter, i.e., $\textmd{diam}(\mathcal{X}) = \sup_{x,y \in \mathcal{X}} \|x-y\|_2 \leq 1$; (ii) there exist $C_\mathcal{X}, D_\mathcal{X} \geq 0$ s.t. $\mathcal{N}(\mathcal{X}, \epsilon, d) \leq C_\mathcal{X} \; \epsilon^{- D_\mathcal{X}}$, $\forall~\epsilon > 0$.
\end{assumption}
%
\begin{assumption}[\textbf{RGM Kernel $W$}]
\label{ass:kernel}
    The RGM kernel function:
    (i) is upper bounded, i.e., $ W(x,x) \leq  W_{\textmd{max}}$  $\forall x \in \mathcal{X}$;
    (ii) has lower bounded kernel degree, i.e.,  $d_W (x) \geq  d_{\min}>0$;
    (iii) is $L_W^{\infty}$-Lipschitz continuous. 
\end{assumption}
As a result of Assumption~\ref{ass:kernel}, the infinity norm of RGM kernel function is also upper bounded, i.e., $\left\lVert W \right\rVert_\infty \leq W_{\max}$.
\begin{assumption} [\textbf{Translation Invariant RGM Kernel}]
\label{ass:tau_Wconstants}
     The RGN kernel is translation invariant, i.e., $W(x, y) = w(x - y)$, with $C_{\nabla w} = \sup_{x \in \mathcal{X}} \int_{\mathcal{X}} \left\lVert \nabla w \left( \frac{x - y}{2} \right) \right\rVert   \lVert x - y \rVert \, dP(y) < \infty$.
     %
\end{assumption}

\begin{assumption} [\textbf{RGM Mapping function $f$}]
\label{ass:contractive}
     Assume H{\"o}lder continuous RGM mapping function $f: \mathcal{X} \to \mathbb{R}^F$, i.e. there exist constants $L_f > 0$ and $\alpha \in (0, 1]$ so that $\left\lVert f(x) - f(y) \right\rVert \leq L_f  \left\lVert x - y \right\rVert^{\alpha}$ holds, $  \forall x, y \in \mathcal{X}$.
\end{assumption}
Define the Radon–Nikodym quantity $q_\tau(x)  = \frac{dP_\tau}{dP}(x)$  for measuring the deformation of the latent distribution $P$ and  $N_{P_\tau} = \lVert q_\tau - 1 \rVert_{\infty}$, and denote the Jacobian matrix of the deformation function by $\nabla \tau$.
\begin{assumption} [\textbf{RGM Deformation $\tau$}]
\label{ass:tau_Pconstants}
    $q_\tau(x), q_\tau(x)^{-1} \leq C_{P_\tau} < \infty$ and $\lVert \nabla \tau \rVert_{\infty} \leq \frac{1}{2}$.
\end{assumption}

\begin{assumption} [\textbf{Deformed RGM Kernel}]
\label{ass:tau_rgm}
    Before and after deformation the RGM kernel remains output bounds unchanged, i.e.,  $0 <  d_{\textmd{min}} \leq W_{\tau}(x,x) \leq  W_{\textmd{max}}$  $\forall x \in \mathcal{X}$. 
\end{assumption}
As a result of Assumption~\ref{ass:tau_rgm}, the maximum norm upper bound  and the kernel degree lower bound of the deformed kernel remain the same, i.e., $\left\lVert W_{\tau} \right\rVert_\infty \leq W_{\max}$ and $d_{W_\tau} \geq d_{\min}$.

%

\begin{assumption}[\textbf{MPNN Lipschitz Continuity}]
\label{ass:mpnn}
    For each MPNN layer $t$, the message, update, and output functions $\Phi^{(t)}$, $\Psi^{(t)}$, and $f^{(t)}$ are Lipschitz continuous with Lipschitz constants $L_{\Phi}^{(t)}$, $L_{\Psi}^{(t)}$, and $L_{f}^{(t)}$ regarding infinity norm, respectively.
\end{assumption}

For each layer of MLP classifier $h_{cls}$, the condition numbers of weight matrix $W$, its change $\Delta_{W}$, and the new matrix $\tilde{W} = W+\Delta_{W}$ are obviously upper-bounded, and we let $\kappa$ be the maximum of the three.

\begin{assumption}[\textbf{MLP Classifier and Its Pertrubation}]
\label{ass:mlp_classifier}
   (i) The activation functions of MLP classifier $h_{cls}$ has bounded nonlinearity, i.e., $\alpha   \lVert x \rVert \leq \lVert \sigma_l(x) \rVert \leq \beta   \lVert x \rVert$ and $\alpha   \lVert x -y \rVert \leq \lVert \sigma_l(x) - \sigma_l(y) \rVert \leq \beta  \lVert x-y \rVert$;
   (ii) its final output vector is upper bounded by $1$, i.e., $\lVert h_{cls}(W; x) \rVert \leq 1$. 
\end{assumption}

For each layer $l$ of the MLP used for constructing the message and update functions at each layer $t$ of the MPNN, i.e., $\Phi_t^{(l)}$ and $\Psi_t^{(l)}$, the condition numbers of their weight matrices, their changes and the new matrices are upper bounded by $\kappa$; 

\begin{assumption}[\textbf{MPNN Perturbation}]
\label{ass:msg_update}
   The activation functions for MPNN MLPs have the same bounded nonlinearity as described in Assumption \ref{ass:mlp_classifier}. 
\end{assumption}
\begin{assumption}[\textbf{vRKHS Reproducing Kernel}]
\label{ass:rkhs}
    (i) The reproducing kernel in the vRKHS $\mathcal{H}_{K_\ell}$ is bounded,  i.e., $0 < K_{\min} < \lVert K(\cdot,\cdot) \rVert_{\infty} < K_{\max}$;
    (ii) given a set $U_Z \subseteq \mathcal{Z}$,   for all $Z \in \mathcal{Z}$ and $Z_1, Z_2 \in U_Z$, a function $L_\beta: \mathcal{Z} \to \mathbb{R}_{\geq 0}$ exists such that 
    \begin{equation}
    \label{ass:kernel_smoothness}
    \left\lVert K(Z_1, Z)(\mathbf{1}) - K(Z_2, Z)(\mathbf{1}) \right\rVert_{\infty} \leq L_\beta(Z) \cdot d_{\Sigma}(Z_1, Z_2)^{\beta};
    \end{equation}
    (iii) and $L_2(Z) \leq L_2$ for $\beta=2$.
\end{assumption}

\begin{assumption}[\textbf{Ground-truth Labeling Function}]
\label{ass:labeling_function}
    The labeling function $g_D$ is realizable by the hypothesis family $h(Z) = h_{cls} \circ \bar h_G(Z)$,
    i.e., it can be written as $g_D(Z)=g_{cls_D} \circ \bar g_{G_D}(Z)$ for some \emph{unknown} reference classifier $g_{cls_D}$ within the same architecture family.
\end{assumption}
Moreover, we model the learned classifier $h_{cls}$ as an additive update from this reference parameterization,
capturing iterative training/fine-tuning effects and enabling the perturbation analysis in the proofs.
It is a model choice instead of an assumption since any two neural networks of the same architecture can be obtained from each other via weight changes.

\textbf{Discussion on Lipschitz and H{\"o}lder Continuity.} We assume Lipschitz or H{\"o}lder continuity over: a) neural network functions, and (b) RGM kernel and feature functions, as well as the reproducing kernel.
For (a), it is a desired property in ML for model robustness \citep{khromovsome2024fundamental}, widely used in both practical and theoretical ML works \citep{arjovsky2017wasserstein,bartlett2017spectrally,he2024gradual} including GNNs \citep{garg2020generalization}, and for the fundamental assumption of the active research area on Lipschitz constant estimation \citep{virmaux2018lipschitz,fazlyab2019efficient,kim2021lipschitz,delattre2023efficient,castin2024smooth}. Moreover, a trained network can be made Lipschitz continuous, e.g. through implicit Lipschitz regularization \citep{khromovsome2024fundamental}.
For (b), this assumption is widely used in RGM \citep{keriven2020convergence,maskey2022generalization} and RKHS studies \citep{fiedler2023lipschitz}, with the practical advantages of improving RGM learning through implicit Lipschitz regularization of kernel and feature functions (e.g. \citep{kyng2015algorithms}).
Overall, the wide usage of Lipschitz continuity encourages our assumption choices. Moreover, as such assumptions are favored by robust ML models, our theoretical results directly impact an important group of ``good'' models.
In our proof, we used this assumption to reveal how Lipschitz constants of functions of interest affect errors. As a by-product, the rate of error changes over the inflation of Lipschitz constants can offer insights for ``bad'' situations where continuity assumption is violated. In Appendix \ref{app:bound}, we show how such assumptions affect generalization error, indicating the ``good'' models.

\section{On Vector-Valued Reproducing Kernel Hilbert Space}
\label{app:RKHS}

Function complexity of neural networks can be effectively analyzed in function space \citep{cortes2011domain,neyshabur2015norm,jacot2018neural,liu2024learning}. 
By Moore-Aronszajn Theorem \citep{aronszajn1950theory}, reproducing kernel Hilbert space (RKHS) and reproducing kernels have one-to-one correspondence. Thus,  analysis of  function spaces can be naturally transformed into the study of kernel properties.
RKHS \citep{aronszajn1950theory,aizerman1964theoretical} has been the most studied function space in machine learning \citep{ghojogh2021reproducing}, which dates back to the well-known kernel support vector machine \citep{boser1992training,cortes1995support}.
To extend, vRKHS theory \citep{micchelli2005learning, carmeli2006vector, caponnetto2008universal, carmeli2010vector} has been developed, and lately been applied in machine learning \citep{li2024towards}. 
To be self-contained, we briefly explain vRKHS basics relevant to this work below, following the description of \citet{minh2016unifying}. We also refer reader to \citep{micchelli2005learning, carmeli2006vector, caponnetto2008universal, carmeli2010vector} for more detailed results of vRKHS.

Consider a vector space $\mathcal{Y^X}$ of functions $f: \mathcal{X} \to \mathcal{Y}$ each defined to map from a nonempty set $\mathcal{X}$ to a separable Hilbert space $\mathcal{Y}$ with inner product $\langle \cdot, \cdot \rangle_{\mathcal{Y}}$.
Let $\mathcal{L(Y)}$ be the Banach space of bounded linear operators on $\mathcal{Y}$. Define an operator-valued kernel $K: \mathcal{X} \times \mathcal{X} \to \mathcal{L(Y)}$. 
The positive definite property of a kernel requires $\sum_{i,j=1}^N \left\langle y_i, K(x_i, x_j)(y_j) \right\rangle_{\mathcal{Y}} \geq 0$ to hold   $\forall \{ x_i \in \mathcal{X} \}_{i=1}^N$, $\{ y_i \in \mathcal{Y} \}_{i=1}^N$ and $N \in \mathbb{N}^+$.
A unique RKHS $\mathcal{H}_K$ of functions $f: \mathcal{X} \to \mathcal{Y}$ with reproducing property exists that admits the reproducing kernel $K$, with the reproducing property described by:  
\begin{equation}
\label{eq:rkhs_reproducing}
     \langle f(x), y \rangle_{\mathcal{Y}} = \langle f, K_{x}y \rangle_{\mathcal{H}_K},
\end{equation}
$ \forall f \in \mathcal{H}_K$  and for a function $K_{x}y: \mathcal{X} \to \mathcal{Y} \in \mathcal{Y^X}$.
A specific way to construct the function $K_{x}y$ is by  $K(\cdot, x)(y)$, e.g., given $z \in \mathcal{X}$  and  $K_{x}y \in \mathcal{Y^X}$, it has $K_{x}y(z) = K(z, x)(y)$.

The kernel $K(x, t) = \phi(x)^{*} \phi(t)$ can be induced by a feature function $\phi(\cdot): \mathcal{X} \to \mathcal{B}(\mathcal{Y};\mathcal{H})$ defined to map from the input space $\mathcal{X}$ to a space of bounded functions from $\mathcal{Y}$ to the Hilbert space $\mathcal{H}$.
Considering the linear span of the set $\mathcal{H}_0=\textmd{span}\{ K_{x}y \mid x \in \mathcal{X}, y \in \mathcal{Y} \}$, completing $\mathcal{H}_0$ by assigning the limits of all the Cauchy sequences gives the Hilbert space $\mathcal{H}_K$.
For two functions $f = \sum_{i=1}^N K_{x_i} w_i$ and $g = \sum_{i=1}^N K_{z_i} y_i $ from $\mathcal{H}_K$ with $ x_i, z_i \in \mathcal{X}$, $w_i, y_i \in \mathcal{Y}$, and $N\in \mathbb{N}^+$, the inner product of $\mathcal{H}_K$ is defined by
\begin{equation}
\label{def:rkhs_inner_product}
    \left\langle f, g \right\rangle_{\mathcal{H}_K} = \sum_{i,j=1}^N \left\langle w_i, K(x_i, z_j) (y_j) \right\rangle_{\mathcal{Y}},
\end{equation}
and the  norm of $\mathcal{H}_K$ is defined by  $\left\lVert f \right\rVert_{\mathcal{H}_K} = \left\langle f, f \right\rangle_{\mathcal{H}_K}$, which is interpreted as the function complexity or smoothness of $f$.

\begin{assumption}[Truncated-spectrum vRKHS]
\label{assump:trunc_vrkhs}
Let $\mu$ be a probability measure on the graph-signal space $\mathcal{Z}$.
Consider the operator-valued reproducing kernel $K_\ell:\mathcal{Z}\times\mathcal{Z}\to \mathcal{L}(\mathbb{R}^C)$
that induces the vRKHS $\mathcal{H}_{K_\ell}$.
Define the associated integral operator $T_{K_\ell}: L_2(\mu;\mathbb{R}^C)\to L_2(\mu;\mathbb{R}^C)$ by
\begin{equation}
  (T_{K_\ell} f)(Z) \;:=\; \int_{\mathcal{Z}} K_\ell(Z,Z') f(Z') \, d\mu(Z').
  \label{eq:integral_operator_TKl}
\end{equation}
Assume $T_{K_\ell}$ is a positive, self-adjoint, compact operator, by the spectral theorem it admits an eigen-decomposition $T_{K_\ell}f = \sum_{i \geq 1} \lambda_i \langle f, \psi_i \rangle_{L_2(\mu)} \psi_i$ with eigenvalues and eigenfunctions $\{(\lambda_i,\psi_i)\}_{i\ge 1}$ such that $\lambda_1 \ge \lambda_2 \ge \cdots \ge 0$ and $\{\psi_i\}_{i\ge 1}$ forms an orthonormal system in $L_2(\mu;\mathbb{R}^C)$.
Fix an integer $r\ge 1$ and denote by $\Pi_r$ the orthogonal projection onto $\mathrm{span}\{\psi_1,\dots,\psi_r\}\subseteq L_2(\mu;\mathbb{R}^C)$.
We assume $\lambda_r \ge 0$ and the loss-mapping function $\ell_{h,g_D}\in \mathcal{H}_{K_\ell}$ is $r$-band-limited:
\begin{equation}
  \ell_{h,g_D} \;=\; \Pi_r\, \ell_{h,g_D}.
  \label{eq:bandlimited_loss}
\end{equation}
%
\end{assumption}

\section{Domain Adaptation Generalization Errors}
\label{app:DA_thms}

\subsection{Foundation Theorem }
\label{app:DA_thms:foundation}

We present below a more detailed version of Prop. \ref{prop:general_DA} with the added expression of $\mathcal{W}_1(\mu_S, \mu_T) $, and provides its proof.

\begin{proposition}[Thm. \ref{prop:general_DA}]
\label{proof:general_DA}
    Given a source domain $D_S=(\mu_S, g_{D})$,  a target domain $D_T=(\mu_T, g_{D})$,  and a hypothesis $h$,  assess the source and target error risks $\epsilon_S(h,g_D)$ and $\epsilon_T(h,g_D)$ by Eq. (\ref{error_function}). 
    Then, the following holds: 
    \begin{equation}
    \label{proof:general_DA:goal}
        \epsilon_T(h,g_D) \leq \epsilon_S(h,g_D) + \underbrace{\left\lVert \ell_{h,g_D} \right\rVert_{\mathcal{H}_{K_\ell}}}_{\text{complexity}} \cdot \underbrace{\mathcal{W}_1(\mu_S, \mu_T)}_{\text{divergence}},
    \end{equation}
    where  
    \begin{equation}
        \mathcal{W}_1(\mu_S, \mu_T) = \inf_{\pi \in \Pi(\mu_S, \mu_T)} \int_{\mathcal{Z} \times \mathcal{Z}} \left\lVert K_Z \mathbf{1} - K_{Z'} \mathbf{1} \right\rVert_{\mathcal{H}_{K_\ell}} \, d \pi(Z,Z'),
    \end{equation}
    and $\mathcal{H}_{K_\ell}$ is the vRKHS space containing the hypothesis function $h$, the ground-truth labeling function $g_D$, and the loss mapping function $\ell_{h,g_D}$
\end{proposition}

\begin{proof}
We start from expanding and transforming the left-hand side of Eq. (\ref{proof:general_DA:goal}) using Eq. (\ref{error_function}), by
\begin{align}
\label{proof:general_DA:0}
    \epsilon_T(h, g_D)
    & = \epsilon_T(h, g_D) + \epsilon_S(h, g_D) - \epsilon_S(h, g_D) \nonumber \\
    & = \epsilon_S(h, g_D) + \underbrace{\mathbb{E}_{Z \sim \mu_T} \left[ \| h(Z)- g_D(Z) \|_1 \right] - \mathbb{E}_{Z \sim \mu_S} \left[ \| h(Z)- g_D(Z) \|_1 \right]}_{\textmd{diff}},
\end{align}
and focus on analyzing the term $\textmd{diff}$. 
Introducing the standard basis by $\left\{e_j\in \mathbb{R}^C\right\}_{j=1}^C$, e.g. $e_1= [1, 0, \cdots, 0]^\intercal$, then it has
\begin{align}
\label{proof:general_DA:1}
    \textmd{diff}
    & = \mathbb{E}_{Z \sim \mu_T} \left[ \sum_{j=1}^C \left(\lvert h(Z) - g_D(Z) \rvert\right)_j \right] - \mathbb{E}_{Z \sim \mu_S} \left[ \sum_{j=1}^C \left(\lvert h(Z) - g_D(Z) \rvert\right)_j \right] \\
    \label{proof:general_DA:2}
    & = \mathbb{E}_{Z \sim \mu_T} \left[ \sum_{j=1}^C \langle \ell_{h,g_D}(Z), e_j \rangle_{\mathbb{R}^C} \right] - \mathbb{E}_{Z \sim \mu_S} \left[ \sum_{j=1}^C \langle \ell_{h,g_D}(Z), e_j \rangle_{\mathbb{R}^C} \right] \\
    \label{proof:general_DA:3}
    & = \mathbb{E}_{Z \sim \mu_T} \left[ \sum_{j=1}^C \langle \ell_{h,g_D}, K_Z e_j \rangle_{\mathcal{H}_{K_\ell}} \right] - \mathbb{E}_{Z \sim \mu_S} \left[ \sum_{j=1}^C \langle \ell_{h,g_D}, K_Z e_j \rangle_{\mathcal{H}_{K_\ell}} \right] \\
    \label{proof:general_DA:4}
    & = \mathbb{E}_{Z \sim \mu_T} \left[ \langle \ell_{h,g_D}, K_Z \mathbf{1} \rangle_{\mathcal{H}_{K_\ell}} \right] - \mathbb{E}_{Z \sim \mu_S} \left[ \langle \ell_{h,g_D}, K_Z \mathbf{1} \rangle_{\mathcal{H}_{K_\ell}} \right] \\
    \label{proof:general_DA:5}
    & = \langle \ell_{h,g_D}, \mathbb{E}_{Z \sim \mu_T} \left[ K_Z \mathbf{1} \right] - \mathbb{E}_{Z \sim \mu_S} \left[ K_Z \mathbf{1} \right] \rangle_{\mathcal{H}_{K_\ell}} \\
    \label{proof:general_DA:6}
    & \leq \lVert \ell_{h,g_D} \rVert_{\mathcal{H}_{K_\ell}} \cdot \left\lVert \mathbb{E}_{Z \sim \mu_T} \left[ K_Z \mathbf{1} \right] - \mathbb{E}_{Z \sim \mu_S} \left[ K_Z \mathbf{1} \right] \right\rVert_{\mathcal{H}_{K_\ell}}.
\end{align}
In the above, Eq. (\ref{proof:general_DA:1}) is resulted from the definition of $L^1$-norm, Eq. (\ref{proof:general_DA:2}) is resulted from orthonormal basis decomposition, Eq. (\ref{proof:general_DA:3}) is obtained by reproducing property, Eq. (\ref{proof:general_DA:4}) is resulted from the linear property of inner product and summing over $ \sum_{j=1}^C K_Z e_j = K_Z \mathbf{1}$,  Eq. (\ref{proof:general_DA:5}) is obtained by linear property of inner product,  Eq. (\ref{proof:general_DA:6}) is derived by applying Cauchy–Schwarz inequality.

We continue to expand the second factor in Eq. (\ref{proof:general_DA:6}):
\begin{align}
    & \left\lVert \mathbb{E}_{Z \sim \mu_T} \left[ K_Z \mathbf{1} \right] - \mathbb{E}_{Z \sim \mu_S} \left[ K_Z \mathbf{1} \right] \right\rVert_{\mathcal{H}_{K_\ell}} \nonumber \\
    = \; & \left\lVert \int_{\mathcal{Z}} K_Z \mathbf{1} \, d \mu_T - \int_{\mathcal{Z}} K_Z \mathbf{1} \, d \mu_S \right\rVert_{\mathcal{H}_{K_\ell}}
     = \left\lVert \int_{\mathcal{Z}} K_Z \mathbf{1} \, d(\mu_T-\mu_S) \right\rVert_{\mathcal{H}_{K_\ell}} \nonumber \\
    = \; &  \left\lVert \int_{\mathcal{Z} \times \mathcal{Z}} \left( K_Z \mathbf{1} - K_{Z'} \mathbf{1}\right) \, d \pi(Z,Z') \right\rVert_{\mathcal{H}_{K_\ell}}  
    \leq  \int_{\mathcal{Z} \times \mathcal{Z}} \left\lVert   K_Z \mathbf{1} - K_{Z'} \mathbf{1}  \right\rVert_{\mathcal{H}_{K_\ell}} \, d \pi(Z,Z') \nonumber \\
    \label{proof:general_DA:9}
    \leq \; & \inf_{\pi \in \Pi(\mu_S, \mu_T)} \int_{\mathcal{Z} \times \mathcal{Z}} \left\lVert K_Z \mathbf{1} - K_{Z'} \mathbf{1} \right\rVert_{\mathcal{H}_{K_\ell}} \, d \pi(Z,Z') = \mathcal{W}_1(\mu_S, \mu_T),
\end{align}
where the used cost function for the Wasserstein 1-distance is  
\begin{equation}
\label{def:feature_cost_function}
    c(Z,Z') = \left\lVert  K_Z \mathbf{1} - K_{Z'} \mathbf{1}  \right\rVert_{\mathcal{H}_{K_\ell}}.
\end{equation}
Substituting Eq. (\ref{proof:general_DA:9}) into Eq. (\ref{proof:general_DA:6}), and substituting the term $\textmd{diff}$ into Eq. (\ref{proof:general_DA:0}), the result in Eq. (\ref{proof:general_DA:goal}) is obtained.
This completes the proof.
\end{proof}

\subsection{Bounding Domain Divergence}
\label{app:DA_thms:dom_div}

Below we present and prove the key result on bounding domain divergence, to derive the domain divergence upper bound $\Delta_D$.

\begin{lemma}[Wasserstein 2-Domain Divergence]
\label{proof:reverse2latent}
Follow the problem setting explained in the beginning of Section \ref{sec:main} and suppose Assumptions \ref{ass:space}, \ref{ass:contractive} and \ref{ass:rkhs} are satisfied. 
   Sample i.i.d source graphs from the source distribution $\mu_S$ and target graphs from the target distribution $ \mu_T$ to obtain  source and target datasets $\mathcal{D}_S$ and   $\mathcal{D}_T$, which contain  $m_S^j$ source graphs and $m_T^j$ target graphs in each class $j$, with  the minimum node numbers $N_S$ and $N_T$.
   Define $\hat{P}_S^j$ and $\hat{P}_T^j$ as the resulting empirical estimations of $ P_S^j$ and $P_T^j$  through Monte Carlo sampling and $B = \max(1, C_{\mathcal{X}}^{-D_{\mathcal{X}}})$.
   Then, the following holds with a probability at least $1-\rho$ with $0<\rho<1$:
    \begin{equation}
    \mathcal{W}_2^2 (\mu_S, \mu_T) \leq \Delta_D,
    \end{equation}
    where
        \begin{align}
            \Delta_D
            & = 2C L_{2}  L_f^2 \sum_{j=1}^C \bigg[ \mathcal{W}_2 \left(\hat{P}_S^j, \hat{P}_T^j\right) + B\left( \left(N_S \cdot m^j_S\right)^{-\frac{1}{D_{\mathcal{X}}}} + \left(N_T \cdot m^j_T\right)^{-\frac{1}{D_{\mathcal{X}}}} \right) \nonumber \\
            \label{proof:reverse2latent:goal}
            & + 2B \cdot 27^{\frac{D_{\mathcal{X}}}{4}} + B \cdot \log(1/\rho)^{\frac{1}{4}} \left( \left(N_S \cdot m^j_S\right)^{-\frac{1}{4}} + \left(N_T \cdot m^j_T\right)^{-\frac{1}{4}} \right) \bigg]^{2 \alpha}.
        \end{align}
\end{lemma}

\begin{proof}
Applying the result of $\mathcal{W}_2^2 \left(\otimes_{i=1}^n \mu_i, \otimes_{i=1}^n \nu_i\right) = \sum_{i=1}^n \mathcal{W}_2^2 (\mu_i, \nu_i)$ (Chapter 2,  \citet{panaretos2019statistical}), it has
\begin{equation}
\label{proof:reverse2latent:1}
    \mathcal{W}_2^2 \left(\mu_S, \mu_T\right)
     = \mathcal{W}_2^2 \left( \otimes_{j=1}^C f^j_S \sharp P^j_S, \otimes_{j=1}^C f^j_T \sharp P^j_T\right)  = \sum_{j=1}^C \mathcal{W}_2^2 \left(f \sharp P^j_S, f \sharp P^j_T\right).
\end{equation}
We start from analyzing $\mathcal{W}_2^2 \left(f \sharp P_S^j,  f \sharp P_T^j\right)$, and  temporarily drop the superscription $j$ for writing convenience.
To study the push-forward probability measure through the feature function $f: \mathcal{X} \to \mathbb{R}^F$, we denote the $\sigma$-Borel algebra on $\mathbb{R}^F$ by $\mathfrak{B}(\mathbb{R}^F)$.
Then, $\forall B \in \mathfrak{B}\left( \mathbb{R}^F \right)$, it has 
\begin{equation}
     \ f_{\sharp }P_S(B) = P_S\left( f^{-1}(B) \right), \ f_{\sharp }P_T(B) = P_T\left( f^{-1}(B) \right).
\end{equation}
As a result, for all couplings $\pi \in \Pi(P_S, P_T)$,  their push-forward couplings $\pi' = (f \times f) \sharp \pi$ through the function $(f \times f)(x, y) = ((f(x), f(y))$ satisfies
\begin{align}
    \pi' (A \times \mathbb{R}^F)
    & = \pi \left( (f \times f)^{-1}(A \times \mathbb{R}^F) \right) = \pi(f^{-1}(A) \times \mathbb{R}^D) \nonumber \\
    & = P_S(f^{-1}(A)) = f_{\sharp }P_S(A), \\
    \pi' (\mathbb{R}^F \times B)
    & = \pi \left( (f \times f)^{-1}(\mathbb{R}^F \times B) \right) = \pi(\mathbb{R}^D \times f^{-1}(B)) \nonumber \\
    & = P_T(f^{-1}(B)) = f_{\sharp} P_T(B),
\end{align}
thus,
\begin{equation}
    \pi' \in \Pi\left( f_{\sharp }P_S, f_{\sharp }P_T\right),
\end{equation}
for all $A, B \in \mathfrak{B}(\mathbb{R}^F)$. 
Subsequently, it has 
\begin{align}
\label{proof:reverse2latent:pushforward_coupling}
       \mathcal{W}_{2}^{2}(f_{\sharp }P_S, f_{\sharp }P_T)
       & = \inf_{\pi' \in \Pi(f_{\sharp }P_S, f_{\sharp }P_T)} \iint_{\mathbb{R}^F \times \mathbb{R}^F} c(f(x), f(y))^{2} d\pi' \nonumber \\
       & = \inf_{\pi \in \Pi(P_S, P_T)} \iint_{\mathbb{R}^D \times \mathbb{R}^F} c(f(x), f(y))^{2} \, d\pi . 
\end{align}
Incorporating Eq. (\ref{proof:reverse2latent:pushforward_coupling}) with the superscription $j$ restored and the cost function   Eq. (\ref{def:feature_cost_function}) into Eq. (\ref{proof:reverse2latent:1}), it has
\begin{equation}
\label{eq:w2}
    \mathcal{W}_2^2 \left(\mu_S, \mu_T\right)  = \sum_{j=1}^C\inf_{\pi \in \Pi(P_S^j, P_T^j)} \iint_{\mathbb{R}^D \times \mathbb{R}^F} \left\lVert K_Z \mathbf{1} - K_{Z'} \mathbf{1} \right\rVert^2_{\mathcal{H}_{K_\ell}} \, d\pi(Z, Z').
\end{equation}

Next we focus on analyzing the cost function $\left\lVert K_Z \mathbf{1} - K_{Z'} \mathbf{1} \right\rVert^2_{\mathcal{H}_{K_\ell}}$.
It is upper bounded by  
\begin{align}
    \nonumber
    &\left\lVert K_Z \mathbf{1} - K_{Z'} \mathbf{1} \right\rVert^2_{\mathcal{H}_{K_\ell}} \\
    \nonumber
    = \; &  \langle K_Z \mathbf{1} - K_{Z'} \mathbf{1}, K_Z \mathbf{1} - K_{Z'} \mathbf{1} \rangle_{\mathcal{H}_{K_\ell}} \\
    \nonumber
    = \; &   \langle K_Z \mathbf{1}, K_Z \mathbf{1} \rangle_{\mathcal{H}_{K_\ell}} - 2   \langle K_Z \mathbf{1}, K_{Z'} \mathbf{1} \rangle_{\mathcal{H}_{K_\ell}} + \langle K_{Z'} \mathbf{1}, K_{Z'} \mathbf{1} \rangle_{\mathcal{H}_{K_\ell}}\\
    \label{proof:reverse2latent:rewrite1}
    \leq \; &  \left| \langle K_Z \mathbf{1}, K_Z \mathbf{1} \rangle_{\mathcal{H}_{K_\ell}} -   \langle K_Z \mathbf{1}, K_{Z'} \mathbf{1} \rangle_{\mathcal{H}_{K_\ell}}\right|  + \left|\langle K_Z \mathbf{1}, K_{Z'} \mathbf{1} \rangle_{\mathcal{H}_{K_\ell}} - \langle K_{Z'} \mathbf{1}, K_{Z'} \mathbf{1} \rangle_{\mathcal{H}_{K_\ell}} \right|.
\end{align}
Applying the kernel reproducing property in Eq. (\ref{eq:rkhs_reproducing}) for each item in Eq. (\ref{proof:reverse2latent:rewrite1}), e.g., $\langle K_Z \mathbf{1}, K_Z \mathbf{1} \rangle_{\mathcal{H}_{K_\ell}}  = \langle K(Z, Z)(\mathbf{1}), \mathbf{1} \rangle_{\mathbb{R}^C} $ and $\langle K_Z \mathbf{1}, K_{Z'} \mathbf{1} \rangle_{\mathcal{H}_{K_\ell}} = \langle K(Z, Z')(\mathbf{1}), \mathbf{1} \rangle_{\mathbb{R}^C}$,  it has
\begin{align}
    & \left\lVert K_Z \mathbf{1} - K_{Z'} \mathbf{1} \right\rVert^2_{\mathcal{H}_{K_\ell}} \nonumber \\
    \leq \; & \left|\langle K(Z, Z)(\mathbf{1}), \mathbf{1} \rangle_{\mathbb{R}^C} - \langle K(Z, Z')(\mathbf{1}), \mathbf{1} \rangle_{\mathbb{R}^C}\right| + \left|\langle K(Z, Z')(\mathbf{1}), \mathbf{1} \rangle_{\mathbb{R}^C}- \langle K(Z', Z')(\mathbf{1}), \mathbf{1} \rangle_{\mathbb{R}^C}\right| \nonumber \\
    = \; & \left\lvert \sum_{j=1}^C \left( K(Z,Z)(\mathbf{1})_j - K(Z,Z')(\mathbf{1})_j \right) \right\rvert + \left\lvert \sum_{j=1}^C \left( K(Z,Z')(\mathbf{1})_j - K(Z',Z')(\mathbf{1})_j \right) \right\rvert \nonumber \\
    \label{proof:reverse2latent:rewrite2}
    \leq \; & C \left\lVert K(Z,Z)(\mathbf{1}) - K(Z,Z')(\mathbf{1}) \right\rVert_{\infty} + C \left\lVert K(Z,Z')(\mathbf{1}) - K(Z',Z')(\mathbf{1}) \right\rVert_{\infty}.
\end{align}
Applying Eq. (\ref{ass:kernel_smoothness}) in Assumption \ref{ass:rkhs} with $\beta=2$ and Assumption \ref{ass:contractive}, it has 
\begin{align}
    \nonumber
    \left\lVert K_Z \mathbf{1} - K_{Z'} \mathbf{1} \right\rVert^2_{\mathcal{H}_{K_\ell}}     
    \leq \; & C (L_2(Z) + L_2(Z')) \cdot d_\Sigma(Z, Z')^2   \leq 2CL_{1} \cdot d_\Sigma(Z, Z') \nonumber \\
    = \; & 2CL_{2} \cdot \frac{1}{N} \sum_{i=1}^N \lVert Z_{\sigma(i)} - Z'_i \rVert^2 \nonumber \\
    \label{proof:reverse2latent:new3}
    =\; & 2CL_{2} \cdot \frac{1}{N} \sum_{i=1}^N \lVert f\left(x_{\sigma(i)}\right) - f\left(x_i'\right) \rVert^2 \nonumber \\
    \leq \; & 2CL_{2}  L_f^2 \cdot \frac{1}{N} \sum_{i=1}^N \lVert x_{\sigma(i)} - x_i' \rVert^{2\alpha}.  
\end{align}
Define the latent cost function

\begin{equation}
\label{def:latent_cost}
    d_{\max}(Z, Z') = \max_{i,j=1}^N\lVert x_{i} - x_j' \rVert,
\end{equation}

where $f(x_i) = Z_{i}$ and $f(x'_i) = Z'_{i}$.   
Eq. (\ref{proof:reverse2latent:new3}) results in 
\begin{equation}
    \label{proof:reverse2latent:new4}
  \left\lVert K_Z \mathbf{1} - K_{Z'} \mathbf{1} \right\rVert^2_{\mathcal{H}_{K_\ell}}     
      \leq   2CL_{2}  L_f^2 \cdot d_{\max}^{2\alpha}(Z, Z').
\end{equation}

Finally, incorporating Eq. (\ref{proof:reverse2latent:new4}) into Eq. (\ref{eq:w2}), it has 
\begin{align}
\nonumber
    \mathcal{W}_2^2 \left(\mu_S, \mu_T\right)  \leq\; & 2CL_{2}  L_f^2 \sum_{j=1}^C\inf_{\pi \in \Pi(P_S^j, P_T^j)} \iint_{\mathbb{R}^D \times \mathbb{R}^F}  d_{\max}^{2\alpha}(Z, Z') \, d\pi(Z, Z') \nonumber \\
    \leq\; & 2CL_{2}  L_f^2 \sum_{j=1}^C \left(\inf_{\pi \in \Pi(P_S^j, P_T^j)} \iint_{\mathbb{R}^D \times \mathbb{R}^F}  d_{\max}^{2 }(Z, Z') \, d\pi(Z, Z')  \right)^{\alpha} \nonumber \\
    \label{proof:reverse2latent:new5}
    = \; & 2CL_{2}  L_f^2 \sum_{j=1}^C \mathcal{W}_2^{2\alpha} \left(P_S^j, P_T^j\right),
\end{align}
for which the cost function used for the Wasserstein distance $\mathcal{W}_2^{2} \left(P_S^j, P_T^j\right)$ is $d_{\max}(Z, Z')$.

It's now left for us to analyze $\mathcal{W}_2(P_S^j,P_T^j)$ by considering the empirical probability estimations $\hat{P}_S^j$ and $\hat{P}_T^j$, i.e.,
\begin{equation}
\label{eq:w2P}
    \mathcal{W}_2^{2\alpha} \left(P_S^j, P_T^j\right) \leq \left( \mathcal{W}_2 \left(P_S^j, \hat{P}_S^j\right) + \mathcal{W}_2 \left(\hat{P}_S^j, \hat{P}_T^j\right) + \mathcal{W}_2 \left(P_T^j, \hat{P}_T^j\right) \right)^{2 \alpha}.
\end{equation}
We bound the two quantities $  \mathcal{W}_2 \left(P_S^j, \hat{P}_S^j\right)$ and $ \mathcal{W}_2 \left(P_T^j, \hat{P}_T^j\right)$  by using directly a recent result on empirical Wasserstein distance convergence presented in Theorem \ref{thm:Wasserstein}, i.e., 
\begin{equation}
   \mathcal{W}_2 \left(P_{\Box}^j, \hat{P}_{\Box}^j\right) \leq  B \left( {{n_\Box^j} }^{-\frac{1}{D_{\mathcal{X}}}} + \left( 27^{\frac{D_{\mathcal{X}}}{4}} + \log(1/\rho)^{\frac{1}{4}} {{n_\Box^j}}^{-\frac{1}{4}} \right) \right) := \epsilon_{n_\Box}. 
\end{equation}
Here and below, we use $\Box$ to refer to $S$ or $T$ depending on the domain.
The sampled dataset $\mathcal{D}_\Box = \left\{\left(G_i^j, y_i^j\right)_{i=1}^{m^j_\Box}\right\}_{j=1}^{C}$ contains $C$ classes and $m^j_\Box$ samples for each class $j$.
The number of samples $n^j_\Box$ for estimating $P_{\Box}^j$ satisfies $n^j_\Box \geq N_\Box \cdot m^j_\Box$.
Finally, incorporating the expressions of $\epsilon_{n_\Box}$ and  the lower bound of $n^j_\Box$ into Eqs. (\ref{eq:w2P})  and  (\ref{proof:reverse2latent:new5}), we obtain Eq. (\ref{proof:reverse2latent:goal}). This completes the proof.

\end{proof}

\subsection{Bounding Function Norm under Truncated-Spectrum vRKHS}
\label{app:DA_thms:truncated_vrkhs}

Below we present and prove the key result on bounding loss function complexity in Lem. \ref{proof:norm2diff}. 
Before that, we present Lem. \ref{proof:MLP_Lip}, which is a supporting lemma providing results on Lipschitz constant and output perturbation bound by weight changes for an MLP classifier, to be used by Lem \ref{proof:norm2diff} and Thm. \ref{proof:perturb_weights}.

\begin{lemma}[MLP Lipschitz Constant]
\label{proof:MLP_Lip}
Given an $L$-layer MLP classifier $f_{cls}: \mathbb{R}^{F_T} \to \mathbb{R}^C$  as in Def. \ref{def:layermapping} satisfying Assumption \ref{ass:mlp_classifier}, denote the Lipschitz constant of the activation function  by $L_{\sigma_l}$  and   the maximum eigenvalue of the weight matrix $W_l$ by  $\lambda_l$. 
The Lipschitz constant of $ f_{cls}(x)$ with respect to Euclidean norm satisfies 
\begin{equation}
   \label{mlp_lipschitz}
    L_{cls}  \leq    \prod_{l=1}^L L_{\sigma_l} \lambda_l .
\end{equation} 
Moreover, applying additive weight changes $\Delta W = \{\Delta W_l\}_{l=1}^L$,  the following holds
\begin{equation}
\label{mlp_lipschitz2}
    \lVert f_{cls}(W+\Delta W; x) - f_{cls}(W; x) \rVert \leq \left( \frac{\beta}{\alpha} \kappa^2 \right)^L   \left( \prod_{l=1}^L \left( 1 + \frac{\lVert \Delta W_l \rVert_F}{\lVert W_l \rVert_F} \right) - 1 \right). 
\end{equation}
 
\end{lemma}

\begin{proof}
 Given $x, x' \in \mathbb{R}^{F_T}$, the linear transformation $W_l(x)$ results in
\begin{equation}
     \lVert W_l(x) - W_l(x') \rVert = \lVert W_lx - W_lx' \rVert = \lVert W_l(x-x') \rVert \leq \lambda_l \lVert x-x' \rVert. 
\end{equation}
Based on the above result, we expand the following in a lay-wise fashion:
\begin{align}
    \nonumber
    \lVert f_{cls}(W; x+\Delta x) - f_{cls}(W; x) \rVert = \; & \lVert \circ_{l=1}^L (\sigma_l\circ W_l(x+ \Delta x)) - \circ_{l=1}^L (\sigma_l\circ W_l(x))  \rVert  \\
    \nonumber
    \leq\; & L_{\sigma_L}  \lambda_L  \lVert \circ_{l=1}^{L-1} (\sigma_l\circ W_l(x+ \Delta x)) - \circ_{l=1}^{L-1} (\sigma_l\circ W_l(x))  \rVert     \leq \cdots   \\
     \leq \; &    \left( \prod_{l=1}^L L_{\sigma_l}\lambda_l \right) \lVert x + \Delta x - x \rVert  =     \left( \prod_{l=1}^L L_{\sigma_l}\lambda_l \right)  \lVert \Delta x \rVert.
\end{align}

Substituting $x + \Delta x = x'$, Eq. (\ref{mlp_lipschitz}) is resulted.
Under the assumption $\lVert f_{cls}(W; x) \rVert \leq 1$, Eq. (\ref{mlp_lipschitz2}) is a direct result  of  applying Theorem 1 of  \citet{bernstein2020distance}, which is re-written as  Theorem \ref{thm:MLP} in our appendix.
This completes the proof.
\end{proof}

\begin{proposition}[Truncated-spectrum vRKHS norm upper bound]
\label{prop:trunc_vrkhs_upper}
Under Assumption~\ref{assump:trunc_vrkhs}, let $f\in \mathcal{H}_{K_\ell}$ satisfy $f=\Pi_r f$.
Then the vRKHS norm of $f$ admits the following bound:
\begin{equation}
\label{eq:trunc_vrkhs_bound}
    \|f\|_{\mathcal{H}_{K_\ell}}^2
    \;\le\;
    \frac{1}{\lambda_r}\, \|f\|_{L_2(\mu;\mathbb{R}^C)}^2
    \;\le\;
    \frac{C}{\lambda_r}\, \|f\|_\infty^2,
\end{equation}
where $\lambda_r$ is the $r$-th eigenvalue of $T_{K_\ell}$, and $\|f\|_\infty := \sup_{Z\in\mathcal{Z}}\|f(Z)\|_\infty$.

In particular, applying \eqref{eq:trunc_vrkhs_bound} to $f=\ell_{h,g_D}$ yields
\begin{equation}
\label{eq:trunc_loss_smoothness}
    \|\ell_{h,g_D}\|_{\mathcal{H}_{K_\ell}}^2
    \;\le\;
    \frac{C}{\lambda_r}\, \|\ell_{h,g_D}\|_\infty^2,
\end{equation}
which isolates a \emph{spectral-geometry factor} $\lambda_r^{-1}$ and an \emph{amplitude factor} $\|\ell_{h,g_D}\|_\infty$.
\end{proposition}

\begin{proof}
Since $f=\Pi_r f$, we can expand it in the first $r$ eigenfunctions:
$f(Z)=\sum_{i=1}^r a_i \psi_i(Z)$ with coefficients $a_i\in\mathbb{R}$.
On the truncated space $\mathcal{H}_{K_{\ell}}^{(r)}=\mathrm{span}\{\psi_1,\dots,\psi_r\}$, for any $f,g \in \mathcal{H}_{K_{\ell}}^{(r)}$ we have $f=\sum_{i=1}^r a_i \psi_i$ and $g=\sum_{i=1}^r a_i \psi_i$.
We endow the inner product $\langle f, g \rangle_{\mathcal{H}_{K_\ell}} = \langle \sum_i a_i\psi_i,\sum_i b_i\psi_i\rangle_{\mathcal H_{K_\ell}} :=\sum_{i=1}^r a_i b_i/\lambda_i$.
%
%
Therefore $\|f\|_{\mathcal{H}_{K_\ell}}^2 = \langle f, f \rangle_{\mathcal{H}_{K_\ell}} = \sum_{i=1}^r a_i^2/\lambda_i$.

By the standard characterization of the (vector-valued) RKHS norm via the eigensystem of $T_{K_\ell}$,
\begin{equation}
  \|f\|_{\mathcal{H}_{K_\ell}}^2
  \;=\;
  \sum_{i=1}^r \frac{a_i^2}{\lambda_i}
  \;\le\;
  \frac{1}{\lambda_r}\sum_{i=1}^r a_i^2.
  \label{eq:Hnorm_by_coeff}
\end{equation}
On the other hand, orthonormality of $\{\psi_i\}$ in $L_2(\mu;\mathbb{R}^C)$ implies
\begin{equation}
  \|f\|_{L_2(\mu;\mathbb{R}^C)}^2
  \;=\;
  \langle f, f \rangle
  \;=\;
  \langle \sum_{i=1}^r a_i \psi_i, \sum_{i=1}^r a_i \psi_i \rangle
  \;=\;
  \sum_{i=1}^r \sum_{j=1}^r a_i a_j \langle \psi_i, \psi_j \rangle
  \;=\;
  \sum_{i=1}^r a_i^2.
  \label{eq:L2_by_coeff}
\end{equation}
Combining \eqref{eq:Hnorm_by_coeff} and \eqref{eq:L2_by_coeff} gives $\|f\|_{\mathcal{H}_{K_\ell}}^2 \le \lambda_r^{-1}\|f\|_{L_2(\mu;\mathbb{R}^C)}^2$.

Finally, since $\mu$ is a probability measure,
\[
  \|f\|_{L_2(\mu;\mathbb{R}^C)}^2
  = \int_{\mathcal{Z}}\|f(Z)\|_2^2\, d\mu(Z)
  \le \sup_{Z\in\mathcal{Z}}\|f(Z)\|_2^2
  \le C \sup_{Z\in\mathcal{Z}}\|f(Z)\|_\infty^2
  = C\|f\|_\infty^2,
\]
where we used $\|v\|_2^2 \le C\|v\|_\infty^2$ for $v\in\mathbb{R}^C$.
This proves \eqref{eq:trunc_vrkhs_bound}, and \eqref{eq:trunc_loss_smoothness} follows by taking $f=\ell_{h,g_D}$.
\end{proof}

\begin{lemma}[Hypothesis-Labeling Function Disagreement]
\label{proof:norm2diff}
Follow the problem setting explained in the beginning of Section \ref{sec:main} and  suppose  Assumptions \ref{ass:mlp_classifier}, \ref{ass:rkhs}, \ref{ass:labeling_function}, and \ref{assump:trunc_vrkhs} hold.
Then, the following holds:
\begin{equation}
    \label{eq:function_complexity}
        \lVert \ell_{h,g_D} \rVert_{\mathcal{H}_{K_\ell}}^2 \leq \frac{C}{\lambda_r} \left( L_{NN}  \left\lVert \Bar{h}_G(Z) - \Bar{g}_{G_D}(Z) \right\rVert_{\infty} + G_{NN} \right)^2,
\end{equation}
where $\mathcal{H}_{K_\ell}$ is the vRKHS space containing the hypothesis function $h$, the ground-truth labeling function $g_D$, and the loss mapping function $\ell_{h,g_D}$, also $L_{NN} =   \prod_{l=1}^L L_{\sigma_l}\lambda_l$ and $G_{NN} = \left( \frac{\beta}{\alpha} \kappa^2 \right)^L \left( \prod_{l=1}^L \left( 1 + \frac{\lVert \Delta W_l \rVert_F}{\lVert W_l \rVert_F} \right) - 1 \right)$ are MLP-specific constants.
\end{lemma}

\begin{proof}

Given Proposition~\ref{prop:trunc_vrkhs_upper}, it suffices to analyze only $\lVert \ell_{h,g_D}(Z) \rVert_{\infty}$.
We now derive its upper bound by decomposing the hypothesis and labeling functions into feature extractors and classifiers. 
Under the assumptions on the hypothesis family and labeling function, it has
\begin{align}
    & \lVert \ell_{h,g_D}(Z) \rVert_{\infty} \nonumber \\
    = \; &\left\lVert h_{cls}(\Bar{h}_G(Z)) - g_{cls_D}(\Bar{g}_{G_D}(Z)) \right\rVert_{\infty} \nonumber \\
    = \; & \left\lVert h_{cls}(\Bar{h}_G(Z)) - h_{cls}(\Bar{g}_{G_D}(Z)) + h_{cls}(\Bar{g}_{G_D}(Z)) - g_{cls_D}(\Bar{g}_{G_D}(Z)) \right\rVert_{\infty} \nonumber \\
    \leq \; &  L_{cls}  \left\lVert \Bar{h}_G(Z) - \Bar{g}_{G_D}(Z) \right\rVert_{\infty} + \lVert h_{cls}(\Bar{g}_{G_D}(Z)) - g_{cls_D}(\Bar{g}_{G_D}(Z)) \rVert_{\infty} \nonumber\\
    \label{proof:norm2diff:9}
    = \; & L_{cls}  \left\lVert \Bar{h}_G(Z) - \Bar{g}_{G_D}(Z) \right\rVert_{\infty} + \lVert g_{cls_D}(W+\Delta W; \Bar{g}_{G_D}(Z)) - g_{cls_D}(W; \Bar{g}_{G_D}(Z)) \rVert_{\infty},
\end{align}
where $L_{cls}$ is the Lipschitz constant of the MLP classifier.
Applying result from Lemma \ref{proof:MLP_Lip}, i.e., Eqs. (\ref{mlp_lipschitz}) and (\ref{mlp_lipschitz2}), it has
\begin{equation}
    \label{proof:norm2diff:Cmax}
    \lVert \ell_{h,g_D}(Z) \rVert_{\infty} \leq   \left( \prod_{l=1}^L L_{\sigma_l} \lambda_l \right) \left\lVert \Bar{h}_G(Z) - \Bar{g}_{G_D}(Z) \right\rVert_{\infty} + \left( \frac{\beta}{\alpha} \kappa^2 \right)^L \left( \prod_{l=1}^L \left( 1 + \frac{\lVert \Delta W_l \rVert_F}{\lVert W_l \rVert_F} \right) - 1 \right).
\end{equation}
Finally, incorporating Eq. (\ref{proof:norm2diff:Cmax}) into Eq. (\ref{eq:trunc_loss_smoothness}), Eq. (\ref{eq:function_complexity}) is resulted. This completes the proof.

\end{proof}






\section{Bounding Convergence Error}
\label{app:convergence}

In this section, we prove major lemmas that enable the proof of convergence error in Section \ref{sec:disagreement_analysis}, which uses a basic result on recurrence inequality presented in Lem. \ref{lem:recursive}. Our primal idea is to use existing concentration inequalities to prove that the discrete samples converge to the continuous function from which they are drawn. In Lem. \ref{lem:agg}, we prove that under the defined mean aggregation scheme, the difference between outputs of the discrete and continuous MPNNs is bounded. Then, we prove a layer-wise upper bound between the outputs of MPNN and cMPNN. The further two lemmas show bounds for outputs of MPNN and cMPNN with $l$-layer, respectively. Altogether, the above lemmas ultimately leads to Thm. \ref{proof:convergence}, providing the convergence error bound $\Delta_N$ as in Eq. (\ref{thm:convergence:goal}).

\begin{lemma}[Recurrence Inequality]
\label{lem:recursive}
    Given a sequence of real numbers $\{ \Delta_t \in \mathbb{R} \ | \ t=0,1,\cdots,T \}$, when $\Delta_{t+1} \leq A^{(t+1)} \Delta_t + B^{(t+1)}$, the following holds 
    \begin{equation}
        \Delta_T \leq \sum_{t=1}^T B^{(t)} \prod_{t'=t+1}^T A^{(t')} + \Delta_0 \cdot \prod_{t=1}^T A^{(t)},
    \end{equation}
where we define $\prod_{t=T+1}^T A^{(t')}= 1$.
\end{lemma}

\begin{proof}

We prove  by induction. Denote the following  statement by $P(n)$:
\begin{equation}
    \Delta_{n} \leq \sum_{t=1}^n B^{(t)} \prod_{t'=t+1}^n A^{(t')} + \Delta_0 \cdot \prod_{t=1}^n A^{(t)}.
\end{equation}

\textbf{1) Base Case.} When $n=1$,  $P(1)$ holds, i.e.,
\begin{equation}
    \Delta_1 \leq A^{(1)} \Delta_0 + B^{(1)} = \sum_{t=1}^{n=1} B^{(t)} \prod_{t'=1+1}^1 A^{(t')} + \Delta_0 \cdot \prod_{t=1}^{n=1} A^{(t)},
\end{equation}
by recurrence relation $\Delta_1 \leq A^{(1)} \Delta_0 + B^{(1)}$  and by definition $\prod_{t'=1+1}^1 A^{(t')} = 1$.

\textbf{2) Induction Hypothesis.} Suppose $P(k)$ is correct, namely,
\begin{equation}
    \Delta_{k} \leq \sum_{t=1}^k B^{(t)} \prod_{t'=t+1}^k A^{(t')} + \Delta_0 \cdot \prod_{t=1}^k A^{(t)}.
\end{equation}

\textbf{3) Induction Step.} We aim at showing  $P(k+1)$ is correct, for which it has
\begin{align}
    \Delta_{k+1}
    \label{proof:recurrence:1}
    & \leq A^{(t+1)} \Delta_k + B^{(k+1)} \\
    \label{proof:recurrence:2}
    & \leq A^{(t+1)} \left( \sum_{t=1}^k B^{(t)} \prod_{t'=t+1}^k A^{(t')} + \Delta_0 \cdot \prod_{t=1}^k A^{(t)} \right) + B^{(k+1)} \\
    \label{proof:recurrence:3}
    & = \left( \sum_{t=1}^k B^{(t)} \prod_{t'=t+1}^k A^{(t')} + B^{(k+1)} \cdot 1 \right) + \Delta_0 \cdot \prod_{t=1}^{k+1} A^{(t)} \\
    \label{proof:recurrence:4}
    & = \left( \sum_{t=1}^k B^{(t)} \prod_{t'=t+1}^k A^{(t')} + B^{(k+1)} \prod_{t'=(k+1)+1}^{k+1} A^{(t')} \right) + \Delta_0 \cdot \prod_{t=1}^{k+1} A^{(t)} \\
    \label{proof:recurrence:5}
    & = \sum_{t=1}^{k+1} B^{(t)} \prod_{t'=t+1}^{k+1} A^{(t')} + \Delta_0 \cdot \prod_{t=1}^{k+1} A^{(t)},
\end{align}
where Eq. (\ref{proof:recurrence:1}) holds by recurrence relation, Eq. (\ref{proof:recurrence:2}) holds by induction hypothesis, Eq. (\ref{proof:recurrence:3})  results from re-organizing terms, Eq. (\ref{proof:recurrence:4}) is obtained  from $\prod_{t'=(k+1)+1}^{k+1} A^{(t')} = 1$, which altogether leads to Eq. (\ref{proof:recurrence:5}) and shows that $P(k+1)$ is correct.
So far we have proven by induction that if $P(1)$ is correct, then $\forall n \in \mathbb{N}^{*}$, $P(n)$ is correct.
This completes the proof.

\end{proof}

\begin{lemma}
\label{lem:agg}
    Given an RGM $\Gamma = (W, P, f)$  and the discrete and continuous versions of an MPNN as in Def. \ref{def:mpnn} and Def. \ref{def:cmpnn},   suppose Assumptions \ref{ass:space}, \ref{ass:kernel}, \ref{ass:contractive},  and \ref{ass:mpnn} hold. 
    For $0<\rho <1$, define 
    \begin{align}
        C_W(\rho) =\; & L_W^{\infty} \left(\sqrt{\log(C_{\mathcal{X}})} + \sqrt{D_\mathcal{X}}\right) + \left(\sqrt{2}\left\lVert W \right\rVert_\infty + L_W^{\infty} \right) \sqrt{\log\left(\frac{2}{\rho}\right)} , \\
        C_{\Phi;f} =\; & L_{\Phi}  \lVert f \rVert_{\infty} + \left\lVert \Phi(0,0) \right\rVert_{\infty}.
    \end{align}
When a sufficiently large number of points is sampled following  $P$ in $\mathcal{X}$, i.e., $\sqrt{N} > \frac{2C_W(\rho)}{d_{\min}}$, the following holds with a probability at least $1-2\rho$:
    \begin{align}
    \label{lem:agg:goal}
          \left\lVert  M_X(\cdot) - M_W(\cdot) \right\rVert_{\infty} \leq \; & \frac{2 C_W(\rho) C_{\Phi;f} \left\lVert W \right\rVert_{\infty}}{d_{\min}^2\sqrt{N}}   + N^{-\frac{1}{2(D_{\mathcal{X}}+1)}} \left( \frac{2\left(\lVert W \rVert_{\infty} L_{\Phi} L_f^{\infty} + L_W^{\infty} C_{\Phi;f}\right)}{d_{\min}} + \right. \nonumber \\
        & \left. \frac{\lVert W \rVert_{\infty} C_{\mathcal{X}} C_{\Phi;f}}{\sqrt{2}d_{\min}} \sqrt{\log(C_{\mathcal{X}}) + \frac{D_{\mathcal{X}}\log(N)}{2(D_{\mathcal{X}+1})}  + \log\left(\frac{2}{\rho}\right)} \right).
    \end{align}
    
\end{lemma}

\begin{proof}
Applying the definitions of discrete and continuous mean aggregation operators in \ref{def:agg_x} and \ref{def:agg_w}, we bound their difference by adding and removing the term $-\frac{1}{N} \sum_{i=1}^N \frac{W(\cdot,x_i)}{d_W(\cdot)}\Phi(f(\cdot), f(x_i))$, and this gives
\begin{align}
    \left\lVert M_X(\cdot) - M_W(\cdot)  \right\rVert_{\infty}
    & = \left\lVert \frac{1}{N} \sum_{i=1}^{N} \frac{W(\cdot, x_i)}{d_X(\cdot)} \Phi(f(\cdot), f(x_i)) - \int_\mathcal{X} \frac{W(\cdot, x)}{d_W(\cdot)} \Phi(f(\cdot), f(x)) \, dP(x) \right\rVert_{\infty} \nonumber \\
    & \leq \left\lVert \frac{1}{N} \sum_{i=1}^{N} \frac{W(\cdot, x_i)}{d_X(\cdot)} \Phi(f(\cdot), f(x_i)) - \frac{1}{N} \sum_{i=1}^{N} \frac{W(\cdot, x_i)}{d_W(\cdot)} \Phi(f(\cdot), f(x_i)) \right\rVert_{\infty} \nonumber \\
    & + \left\lVert \frac{1}{N} \sum_{i=1}^{N} \frac{W(\cdot, x_i)}{d_W(\cdot)} \Phi(f(\cdot), f(x_i)) - \int_\mathcal{X} \frac{W(\cdot, x)}{d_W(\cdot)} \Phi(f(\cdot), f(x)) \, dP(x) \right\rVert_{\infty} \nonumber \\
    \label{lem:agg:1}
    & \leq \underbrace{\left\lVert \frac{1}{N} \sum_{i=1}^N W(\cdot,x_i) \Phi(f(\cdot),f(x_i))\right\rVert_{\infty}}_{T_1} \underbrace{\left\lVert \frac{1}{d_X(\cdot)}-\frac{1}{d_W(\cdot)} \right \rVert_{\infty}}_{T_2} \\
    \label{lem:agg:2}
    & + \underbrace{\left\lVert \frac{1}{N} \sum_{i=1}^{N} \tilde{W}(\cdot, x_i) \Phi(f(\cdot), f(x_i)) - \int_\mathcal{X} \tilde{W}(\cdot, x) \Phi(f(\cdot), f(x)) \, dP(x) \right\rVert_{\infty}}_{T_3},
\end{align}
where $\tilde{W}(y,x) = \frac{W(y, x)}{d_W(y)}$. 
We analyze below each term $T_1$, $T_2$ and $T_3$,
separately.

We first bound $\left\lVert \Phi(f(x),f(y))\right\rVert_{\infty}$ (with simplified notation as $\left\lVert \Phi(f,f)\right\rVert_{\infty}$) using its Lipschitz constant $L_{\Phi}$ from Assumption \ref{ass:mpnn}, resulting in 
\begin{align}
\label{phiformalbias}
     \lVert \Phi(f(x),f(y)) \rVert_{\infty}
     = \; & \lVert \Phi(f(x),f(y)) - \Phi(0,0) + \Phi(0,0) \rVert_{\infty} \nonumber \\
     \leq \; & \left\lVert \Phi(f(x),f(y)) - \Phi(0,0) \right\rVert_{\infty} + \left\lVert \Phi(0,0) \right\rVert_{\infty} \nonumber \\
     \leq \; & L_{\Phi} \left\lVert [f(x),f(y)] - [0,0]\right\rVert_{\infty} + \left\lVert \Phi(0,0) \right\rVert_{\infty} \nonumber \\
     \leq \; & L_{\Phi}  \lVert f \rVert_{\infty} + \left\lVert \Phi(0,0) \right\rVert_{\infty} := C_{\Phi;f} .
\end{align}
Applying the definition of infinity norm and Eq. (\ref{phiformalbias}), it then has
\begin{equation}
\label{lem:agg:1:1:a}
    T_1 \leq \left\lVert W \right\rVert_{\infty} \left\lVert \Phi(f,f)\right\rVert_{\infty} \leq \left\lVert W \right\rVert_{\infty} C_{\Phi;f}.
\end{equation}

Now, we analyze $T_2$. 
According to Assumption \ref{ass:kernel}, the RGM kernel degree is lower bounded by $d_W(\cdot) \geq d_{\min}$, which results in
\begin{equation}
\label{eq:T2}
    T_2 = \left\lVert \frac{d_W(\cdot) - d_X(\cdot)}{d_X(\cdot)   d_W(\cdot)} \right\rVert_{\infty} \leq    \frac{d_{\min}^{-1}\left\lVert  d_X(\cdot) - d_W(\cdot) \right\rVert_{\infty} }{\min_{x} |d_X(x)|}, 
\end{equation}
and also
\begin{align}
\label{eq:dx_dw}
    |d_X(\cdot)| & = |d_X(\cdot)-d_W(\cdot) + d_W(\cdot)| \geq  | d_W(\cdot)| - |d_X(\cdot)-d_W(\cdot)| \nonumber \\ & \geq d_{\min} - \left\lVert d_X(\cdot)-d_W(\cdot) \right\rVert_{\infty}. 
\end{align} 
Applying Lemma \ref{lem:convergence_0} under Assumptions \ref{ass:space} and \ref{ass:kernel} with $f=1$, the following holds with a probability at least $1-\rho$:
\begin{align}
\label{eq:dx-dw}
    \lVert d_X(\cdot) - d_W(\cdot) \rVert_{\infty}
    & \lesssim \frac{ L_W^{\infty} \left(\sqrt{\log(C_{\mathcal{X}})} + \sqrt{D_\mathcal{X}}\right) + \left(\sqrt{2}\left\lVert W \right\rVert_\infty + L_W^{\infty}\right) \sqrt{\log\left(\frac{2}{\rho}\right)}}{\sqrt{N}} \\
    & := \frac{C_W(\rho)}{\sqrt{N}},
\end{align}
which is defined as event $\mathcal{E}_1$.
When the sample size is sufficiently large, i.e.,  
\begin{equation}
\label{eq:sample_size}
    N > 4\left(\frac{ C_W(\rho)}{d_{\min}}\right)^2,
\end{equation}
  $\lVert d_X(\cdot) - d_W(\cdot) \rVert_{\infty}  \leq \frac{d_{\min}}{2}$ holds, which, when being combined with Eq. (\ref{eq:dx_dw}), further results in 
  \begin{equation}
  \label{eq:d_x_bound}
    | d_X(\cdot)   |\geq \frac{d_{\min}}{2}.  
  \end{equation} 
Combining the above equation with Eqs. (\ref{eq:T2}) and (\ref{eq:dx-dw}), it has
\begin{equation}
\label{eq:T2_final}
    T_2   \leq 2d_{\min}^{-2} N^{-\frac{1}{2}} C_W(\rho) , 
\end{equation}

Next, we proceed to analyzing $T_3$. By applying Lemma \ref{lem:uniform_bound} under Assumptions \ref{ass:space} and \ref{ass:contractive} with $F_y(\cdot) = \tilde{W}(y,\cdot) \Phi(f(y), f(\cdot))$, the following holds with a probability at least $1-\rho$, as
\begin{equation}
     \label{lem:agg:2:1}
    T_3 \leq N^{-\frac{1}{2(D_{\mathcal{X}}+1)}} \left( 2 L_{F_y} + \frac{C_{\mathcal{X}} \lVert F_y \rVert_{\infty}}{\sqrt{2}} \sqrt{\log(C_{\mathcal{X}}) + \frac{D_{\mathcal{X}}\log(N)}{2(D_{\mathcal{X}+1})}  + \log\left(\frac{2}{\rho}\right)} \right),
\end{equation}
which is defined as the event $\mathcal{E}_2$.
It now boils down to the calculation of Lipschitz constant of $F_y(\cdot)$ and its infinity norm $\lVert F_y(\cdot) \rVert_{\infty}$. 
We analyze the Lipschitz constant $L_{\tilde{W}}^{\infty}$ of $\tilde{W}(y,x)$ with respect to $x$, as below
\begin{align}
    \left|\tilde{W}(y, x) - \tilde{W}(y, x')\right| & = \left|\frac{W(y, x)}{d_W(x)} - \frac{W(y, x')}{d_W(y)} \right| \leq d_{\min}^{-1} |W(y, x) -W(y, x')| \nonumber \\ & \leq d_{\min}^{-1}L_W^{\infty}\|x-x'\|_{\infty},
\end{align}
therefore we have $L_{\tilde{W}}^{\infty} \leq d_{\min}^{-1}L_W^{\infty}$, and 
for $\forall x \in \mathcal{X}, \ \left\lVert \tilde{W}(\cdot, x) \right\rVert_{\infty} = \left\lVert \frac{W(\cdot, x)}{d_W(\cdot)} \right\rVert_{\infty} \leq d_{\min}^{-1} \left\lVert W \right\rVert_{\infty}$
Then, with Eq. (\ref{phiformalbias}), $\forall x, x' \in \mathcal{X}$, we have:

\begin{align}
\nonumber
    & \lVert F_y(x) - F_y(x') \rVert_{\infty} \\
    = \;& \left \lVert \tilde{W}(y,x)\Phi(f(y),f(x)) - \tilde{W}(y,x')\Phi(f(y),f(x')) \right \rVert_{\infty} \nonumber \\
    \leq \;& \left\lVert \tilde{W}(y,x)\Phi(f(y),f(x)) - \tilde{W}(y,x)\Phi(f(y),f(x')) \right\rVert_{\infty} \nonumber \\
    + \; & \left\lVert \tilde{W}(y,x)\Phi(f(y),f(x')) - \tilde{W}(y,x')\Phi\left(f(y),f(x'))\right) \right\rVert_{\infty} \nonumber \\
    \leq \; & \left\lVert \tilde{W} \right \rVert_{\infty} L_{\Phi} L_f^{\infty} \|x-x'\|_{\infty} + L_{\tilde{W}}^{\infty} \|x-x'\|_{\infty} \left\lVert \Phi(f, f) \right \rVert_{\infty} \nonumber \\
    \label{eq:fx_max_norm}
    \leq \; & \left(d_{\min}^{-1}\lVert W \rVert_{\infty} L_{\Phi} L_f^{\infty} + d_{\min}^{-1}L_W^{\infty} C_{\Phi;f} \right)\|x-x'\|_{\infty}.
\end{align}
Eq. (\ref{eq:fx_max_norm}) results from $\left\lVert \tilde{W} \right\rVert_{\infty} \leq d_{\min}^{-1} \left\lVert W \right\rVert_{\infty}$, $L_{\tilde{W}}^{\infty} \leq d_{\min}^{-1} L_W^{\infty}$, and Eq. (\ref{phiformalbias}). It leads to 
\begin{equation}
\label{eq:fx_lip}
    L_{F_y}  \leq \frac{\lVert W \rVert_{\infty} L_{\Phi} L_f^{\infty} + L_W^{\infty} C_{\Phi;f}}{d_{\min}}.
\end{equation}
Again, applying Eq. (\ref{phiformalbias}), for the infinity norm, we have 
\begin{equation}
\label{eq:fx_norm}
    \lVert F_y(\cdot) \rVert_{\infty} = \left\lVert \tilde{W}(y,\cdot) \Phi(f(y),f(\cdot)) \right\rVert_{\infty} \leq \left\lVert \tilde{W} \right\rVert_{\infty} \lVert \Phi(f,f) \rVert_{\infty} \leq \frac{\lVert W \rVert_{\infty} C_{\Phi;f}}{d_{\min}}.
\end{equation}

Finally, substituting Eqs. (\ref{eq:fx_lip}) and (\ref{eq:fx_norm}) into Eq. (\ref{lem:agg:2:1}), and combining the result with Eqs. (\ref{lem:agg:1:1:a}), (\ref{eq:T2_final}) to expand $\left\lVert   M_X(\cdot) - M_W(\cdot)  \right\rVert_{\infty} \leq T_1T_2 +T_3$ as in Eqs. (\ref{lem:agg:1}) and (\ref{lem:agg:2}), it results in Eq. (\ref{lem:agg:goal}). 
Note that we use concentration inequalities twice for concluding the result. The event $\mathcal{E}_1$ is regarding the number of nodes $N$ such that the approximation error of non-normalized kernel is uniformly bounded. Likewise, the event $\mathcal{E}_2$   is regarding the choice of $X_N$ such that the approximation error of Monte-Carlo estimation towards the covering balls is uniformly bounded. Since each event holds independently with probability at least $1-\rho$, the final result holds with probability $(1-\rho)(1-\rho) = 1-2\rho+\rho^2 \geq 1-2\rho$.

\end{proof}

\begin{lemma}
\label{lem:boundedf}
Given an RGM $\Gamma = (W, P, f)$  and the discrete and continuous versions of an MPNN as in Def. \ref{def:mpnn} and Def. \ref{def:cmpnn},  suppose Assumptions \ref{ass:kernel} and \ref{ass:mpnn}  hold. 
Then, the output infinity norm and the Lipschitz constant  of the cMPNN are bounded layer-wise by
    \begin{align}
        \label{lem:contlayernorm:goal}
        & \left\lVert f^{(l+1)} \right\rVert_{\infty}  \leq C_1^{(l+1)} + C_2^{(l+1)} \lVert f \rVert_{\infty}, \\
        \label{lem:contlayerlip:goal}
        & L_{f^{(l)}} \leq D_1^{(l)} + D_2^{(l)} \lVert f \rVert_{\infty} + D_3^{(l)} L_f^{\infty},
    \end{align}
    where the constants are defined as follows:
    \begin{align}
        \label{constant:C1}
          C_1^{(l+1)} := \; & \sum_{t=1}^{l+1} \left(
L_{\Psi^{(t)}} \lVert \Phi^{(t)}(0,0) \rVert_\infty+ \lVert \Psi^{(t)}(0,0) \rVert_\infty \right) \prod_{l'=t+1}^{l+1}  L_{\Psi^{(l')}} \left( 1 + L_{\Phi^{(l')}} \right), \\
        \label{constant:C2}
         C_2^{(l+1)} := \; & \prod_{t=1}^{l+1} L_{\Psi^{(t)}} \left( 1 +  L_{\Phi^{(t)}} \right), \\
        \label{constant:D1}
        D_1^{(l)} := \; & \sum_{t=1}^{l} \Bigg[\left( L_{\Psi^{(t)}}  L_W^{\infty}d_{\min}^{-1} \lVert \Phi^{(t)}(0,0) \rVert_\infty + L_{\Psi^{(t)}} \lVert \Phi^{(t)}(0,0) \rVert_\infty  L_W^{\infty} d_{\min}^{-2} \right), \nonumber \\
        & + C_1^{(t-1)} \left( L_{\Psi^{(t)}} L_{\Phi^{(t)}}  L_W^{\infty}d_{\min}^{-1} + L_{\Psi^{(t)}} L_{\Phi^{(t)}} \lVert W \rVert_{\infty}  L_W^{\infty} d_{\min}^{-2} \right) \Bigg] \nonumber \\ 
        & \prod_{l' = t+1}^{l} L_{\Psi^{(l')}}  \left( 1 + \lVert W \rVert_{\infty} d_{\min}^{-1}L_{\Phi^{(l')}}  \right), \\
        \label{constant:D2}
        D_2^{(l)} := \; & \sum_{t=1}^{l} C_2^{(t-1)} \left( L_{\Psi^{(t)}} L_{\Phi^{(t)}}  L_W^{\infty} d_{\min}^{-1} + L_{\Psi^{(t)}} L_{\Phi^{(t)}} \lVert W \rVert_{\infty}  L_W^{\infty} d_{\min}^{-2}\right) \nonumber \\
        &\prod_{l' = t+1}^{l} L_{\Psi^{(l')}} \left(1+ \lVert W \rVert_{\infty} d_{\min}^{-1} L_{\Phi^{(l')}} \right), \\
        \label{constant:D3}
        D_3^{(l)} := \; & 1 + \lVert W \rVert_{\infty} d_{\min}^{-1} L_{\Phi^{(t)}}.
    \end{align}
\end{lemma}

\begin{proof}

We use result of Eq. (\ref{phiformalbias}), but layer-wise, i.e.,  
\begin{equation}
\label{phiformalbias_l}
     \left\lVert \Phi^{(l)}\left(f^{(l-1)}(\cdot),f^{(l-1)}(\cdot)\right) \right\rVert_{\infty}    \leq L_{\Phi^{(l)}}  \left\lVert f^{(l-1)} \right\rVert_{\infty} + \left\lVert \Phi^{(l)}(0,0) \right\rVert_{\infty} = C_{\Phi^{(l)};f^{(l-1)}}.
\end{equation}
By the definition of cMPNN and Eq. (\ref{phiformalbias_l}), we expand $\left\lVert f^{(l+1)} \right\rVert_\infty$ as follows:
\begin{align}
    & \left\lVert f^{(l+1)} \right\rVert_\infty \nonumber \\
    & = \left\lVert \Psi^{(l+1)} \left(f^{(l)}(\cdot), M_W^{\Phi^{(l+1)},f^{(l)}}(\cdot) \right) \right\rVert_\infty \nonumber \\
    & \leq \left\lVert \Psi^{(l+1)} \left( f^{(l)}(\cdot), M_W^{\Phi^{(l+1)},f^{(l)}}(\cdot)\right) - \Psi^{(l+1)}(0,0) \right\lVert_\infty + \left\lVert \Psi^{(l+1)}(0,0) \right\rVert_\infty \nonumber \\
    & \leq L_{\Psi^{(l+1)}} \left( \left\| f^{(l)} \right\|_\infty
    + \left\lVert M_W^{\Phi^{(l+1)},f^{(l)}}(\cdot) \right\rVert_\infty \right) + \left\lVert \Psi^{(l+1)}(0,0)\right\rVert_\infty \nonumber \\
    & = L_{\Psi^{(l+1)}} \left(\left\| f^{(l)} \right\|_\infty
    + \left\lVert \int_\mathcal{X} \frac{W(\cdot, y)}{d_W(\cdot)} \Phi^{(l+1)} \left(f^{(l)}(\cdot), f^{(l)}(y) \right) \, dP(y) \right\rVert_\infty \right) + \left\lVert \Psi^{(l+1)}(0,0)\right\rVert_\infty \nonumber \\
    & \leq L_{\Psi^{(l+1)}} \left( \left\lVert f^{(l)} \right\rVert_\infty
    + \left\lVert \int_\mathcal{X} \frac{W(\cdot, y)}{d_W(\cdot)} \, dP(y) \right\rVert_\infty \cdot \left\lVert \Phi^{(l+1)} \left(f^{(l)}(\cdot), f^{(l)}(\cdot) \right) \right\rVert_{\infty} \right) + \left\lVert \Psi^{(l+1)}(0,0)\right\rVert_\infty \nonumber \\
    & \leq L_{\Psi^{(l+1)}} \left(\left \lVert f^{(l)} \right\rVert_\infty
    + L_{\Phi^{(l+1)}} \left \lVert f^{(l)} \right\rVert_{\infty} + \left\lVert \Phi^{(l+1)}(0,0) \right\rVert_{\infty} \right) + \left\lVert \Psi^{(l+1)}(0,0)\right\rVert_\infty \nonumber \\
    \label{lem:contlayernorm:final}
    & = L_{\Psi^{(l+1)}} \left( 1 + L_{\Phi^{(l+1)}} \right) \left\lVert f^{(l)} \right\rVert_{\infty} + L_{\Psi^{(l+1)}} \left\lVert \Phi^{(l+1)}(0,0) \right\rVert_{\infty} + \left\lVert \Psi^{(l+1)}(0,0)\right\rVert_\infty.
\end{align}
Applying Lemma \ref{lem:recursive} with 
\begin{align}
    \Delta_{l+1} =\; & \left\lVert f^{(l+1)} \right\rVert_\infty,\\
    A^{(l+1)} =\; &  L_{\Psi^{(l+1)}} \left( 1 + L_{\Phi^{(l+1)}} \right) ,\\
    B^{(t+1)} =\; & L_{\Psi^{(l+1)}} \left\lVert \Phi^{(l+1)}(0,0) \right\rVert_{\infty} + \left\lVert \Psi^{(l+1)}(0,0)\right\rVert_\infty,
\end{align}
Eq. (\ref{lem:contlayernorm:goal}) is  obtained.
We proceed to proving (\ref{lem:contlayerlip:goal}). For $k=0,1,\cdots,l-1$ and $\forall x,x' \in \mathcal{X}$, applying Lipschitz constant definition, we have:
\begin{align}
    & \left\lVert f^{(k+1)}(x)-f^{(k+1)} \left(x'\right) \right\rVert_{\infty} \nonumber \\
    = \; & \left\lVert \Psi^{(k+1)} \left(f^{(k)}(x), M_W^{\Phi^{(k+1)},f^{(k)}}(x)\right) - \Psi^{(k+1)}\left(f^{(k)}\left(x'\right), M_W^{\Phi^{(k+1)},f^{(k)}}(x')\right) \right\rVert_{\infty} \nonumber \\
    \leq \; & L_{\Psi^{(k+1)}} \left(\left\lVert f^{(k)}(x)-f^{(k)}\left(x'\right) \right\rVert_{\infty} + \left\lVert M_W^{\Phi^{(k+1)},f^{(k)}}(x) -M_W^{\Phi^{(k+1)},f^{(k)}}(x')\right\rVert_{\infty}\right) \nonumber \\
    \label{lem:contlayerlip:1}
    \leq \; & L_{\Psi^{(k+1)}}\left(L_{f^{(k)}} \left\|x- x'\right\|_{\infty}+  \left\lVert M_W^{\Phi^{(k+1)},f^{(k)}}(x) -M_W^{\Phi^{(k+1)},f^{(k)}}(x')\right\rVert_{\infty}\right) \nonumber \\
    \leq \; & L_{\Psi^{(k+1)}} L_{f^{(k)}}\left\|x- x'\right\|_{\infty}  +  L_{\Psi^{(k+1)}} \underbrace{\left\lVert M_W^{\Phi^{(k+1)},f^{(k)}}(x) -M_W^{\Phi^{(k+1)},f^{(k)}}(x')\right\rVert_{\infty}}_{T}  
\end{align}
We focus on the second term, and it has
\begin{align}
    T = \; & \left\|\int_\mathcal{X}\left[ \frac{W(x, y)}{d_W(x)} \Phi^{(k+1)}\left(f^{(k)}(x), f^{(k)}(y)\right)-\frac{W\left(x', y\right)}{d_W\left(x'\right)} \Phi^{(k+1)}\left(f^{(k)}\left(x'\right), f^{(k)}(y)\right) \right]\, dP(y)\right\|_{\infty} \nonumber \\
    \leq \; &\underbrace{\int_\mathcal{X}\left\|\frac{W(x, y)}{d_W(x)} \Phi^{(k+1)}\left(f^{(k)}(x), f^{(k)}(y)\right)-\frac{W\left(x', y\right)}{d_W(x)} \Phi^{(k+1)}\left(f^{(k)}(x), f^{(k)}(y)\right)\right\|_{\infty} \, dP(y)}_{(A)} \nonumber \\
    + \; & \underbrace{\int_\mathcal{X}\left\|\frac{W\left(x', y\right)}{d_W(x)} \Phi^{(k+1)}\left(f^{(k)}(x), f^{(k)}(y)\right)-\frac{W\left(x', y\right)}{d_W(x)} \Phi^{(k+1)}\left(f^{(k)}\left(x'\right), f^{(k)}(y)\right)\right\|_{\infty} \, dP(y)}_{(B)} \nonumber \\
    \label{lem:contlayerlip:2}
    + \; & \underbrace{\int_\mathcal{X}\left\|\frac{W\left(x', y\right)}{d_W(x)} \Phi^{(k+1)}\left(f^{(k)}\left(x'\right), f^{(k)}(y)\right)-\frac{W\left(x', y\right)}{d_W\left(x'\right)} \Phi^{(k+1)}\left(f^{(k)}\left(x'\right), f^{(k)}(y)\right)\right\|_{\infty} \, dP(y) }_{(C)}, 
\end{align}
for which we bound $(A)$, $(B)$ and $(C)$ separately as below.
For $(A)$:
\begin{align}
    (A) &   =\int_\mathcal{X} \frac{\left|W(x, y)-W\left(x', y\right)\right|}{d_W(x)}\left\|\Phi^{(k+1)}\left(f^{(k)}(x), f^{(k)}(y)\right)\right\|_{\infty} \, dP(y) \nonumber \\
    & \leq \frac{L_W^{\infty} \|x-x'\|_{\infty}}{d_{\min}} \int_\mathcal{X}\left\|\Phi^{(k+1)}\left(f^{(k)}(x), f^{(k)}(y)\right)\right\|_{\infty} \, dP(y) \nonumber\\
    &\leq \frac{L_W^{\infty}\|x-x'\|_{\infty}}{d_{\min}} \left\|\Phi^{(k+1)}\left(f^{(k)}(\cdot), f^{(k)}(\cdot)\right)\right\|_{\infty}\int_\mathcal{X}   dP(y)\nonumber \\
    & \leq  L_W^{\infty} d_{\min}^{-1}\left(\left\|\Phi^{(k+1)}(0,0)\right\|_{\infty}+L_{\Phi^{(k+1)}}\left\|f^{(k)}\right\|_{\infty}\right)\|x-x'\|_{\infty}  \nonumber \\
     & \leq L_W^{\infty} d_{\min}^{-1} C_{\Phi^{(k+1)};f^{(k)}}\|x-x'\|_{\infty}.
\end{align}
For $(B)$:
\begin{align}
    (B) &  =\int_\mathcal{X} \left|\frac{W\left(x', y\right)}{d_W(x)}\right|\left\|\Phi^{(k+1)}\left(f^{(k)}(x), f^{(k)}(y)\right)-\Phi^{(k+1)}\left(f^{(k)}\left(x'\right), f^{(k)}(y)\right)\right\|_{\infty} \, dP(y) \nonumber \\
    & \leq  \|W\|_{\infty} d_{\min}^{-1} L_{\Phi^{(k+1)}} \int_\mathcal{X}\left\|\left[f^{(k)}(x), f^{(k)}(y)\right]-\left[f^{(k)}\left(x'\right), f^{(k)}(y)\right]\right\|_{\infty} \, dP(y) \nonumber \\
    & \leq \|W\|_{\infty} d_{\min}^{-1} L_{\Phi^{(k+1)}} \left\| f^{(k)}(x)-f^{(k)}\left(x'\right) \right\|_{\infty} \leq \|W\|_{\infty} d_{\min}^{-1} L_{\Phi^{(k+1)}} L_{f^{(k)}}\|x-x'\|_{\infty}.
\end{align}
For $(C)$:
\begin{align}
    (C) & = \left \| \frac{1}{d_W(x)}-\frac{1}{d_W(x')}\right\|_{\infty} \int_\mathcal{X} \left\lvert W(x', y)\right\rvert \left\lVert \Phi^{(k+1)} \left(f^{(k)}(x'), f^{(k)}(y)\right) \right\rVert_{\infty} \, d P(y) \nonumber \\
    & \leq  \left\| \frac{d_W(x')-d_W(x)}{d_W(x)d_W(x')}\right\|_{\infty} \|W\|_{\infty} C_{\Phi^{(k+1);f^{(k)}}}   \nonumber \\
    &\leq \|W\|_{\infty}  C_{\Phi^{(k+1);f^{(k)}}} d_{\min}^{-2} \int_{\mathcal{X}} \left\lvert W(x',y) - W(x,y) \right\rvert \, dP(y) \nonumber \\
    & \leq \|W\|_{\infty}  C_{\Phi^{(k+1);f^{(k)}}} d_{\min}^{-2} \int_{\mathcal{X}} L_W^{\infty} \|x-x'\|_{\infty} \, dP(y) \nonumber \\
    & \leq \|W\|_{\infty} C_{\Phi^{(k+1)};f^{(k)}}  d_{\min }^{-2}L_W^{\infty}\|x-x'\|_{\infty}.
\end{align}
We can then substitute $(A)$, $(B)$ and $(C)$ into Eq. (\ref{lem:contlayerlip:1}), and analyze the Lipschitz constant of $f^{(k+1)}(x)$.  
\begin{align}
     & \left\lVert f^{(k+1)}(x)-f^{(k+1)} \left(x'\right) \right\rVert_{\infty} \leq   L_{\Psi^{(k+1)}}  L_{f^{(k)}}\|x-x'\|_{\infty} + L_{\Psi^{(k+1)}}((A) + (B) + (C)) \nonumber \\
     = \;&  L_{\Psi^{(k+1)}} \left( L_{f^{(k)}} +  L_W^{\infty} d_{\min}^{-1} C_{\Phi^{(k+1)};f^{(k)}} +  \|W\|_{\infty} d_{\min}^{-1} L_{\Phi^{(k+1)}} L_{f^{(k)}} + \right.\nonumber \\
     & \left.\|W\|_{\infty} C_{\Phi^{(k+1)};f^{(k)}}  L_W^{\infty} d_{\min }^{-2} \right)\|x-x'\|_{\infty} \nonumber \\
      =\; & \left[L_{\Psi^{(k+1)}} \left( 1 + \|W\|_{\infty} d_{\min}^{-1} L_{\Phi^{(k+1)}} \right) L_{f^{(k)}} +\right. \nonumber\\
      & \left. L_{\Psi^{(k+1)}} C_{\Phi^{(k+1)};f^{(k)}}   L_W^{\infty}d_{\min}^{-1}  (1 + d_{\min}^{-1}\|W\|_{\infty} )\right]\|x-x'\|_{\infty}, 
\end{align}
which results in
{
    \small
    \begin{equation} 
    L_{f^{(k+1)}} \leq L_{\Psi^{(k+1)}} \left( 1 + \|W\|_{\infty} d_{\min}^{-1} L_{\Phi^{(k+1)}} \right) L_{f^{(k)}} + L_{\Psi^{(k+1)}} C_{\Phi^{(k+1)};f^{(k)}}   L_W^{\infty}d_{\min}^{-1}  (1 + d_{\min}^{-1}\|W\|_{\infty} ).
\end{equation}
}
Applying Lemma \ref{lem:recursive} with 
\begin{align}
    \Delta_{l+1} =\; & L_{f^{(l+1)}},\\
    A^{(l+1)} =\; &  L_{\Psi^{(l+1)}} \left( 1 + \|W\|_{\infty} d_{\min}^{-1} L_{\Phi^{(l+1)}} \right) ,\\
    B^{(t+1)} =\; & L_{\Psi^{(l+1)}} C_{\Phi^{(l+1)};f^{(l)}}   L_W^{\infty}d_{\min}^{-1}  (1 + d_{\min}^{-1}\|W\|_{\infty} ) ,
\end{align}
and re-organizing the  formulation, Eq. (\ref{lem:contlayerlip:goal}) is resulted.
This completes the proof.

\end{proof}

\begin{lemma}
\label{lem:layerwise}
    Given an RGM $\Gamma = (W, P, f)$  and the discrete and continuous versions of an MPNN as in Def. \ref{def:mpnn} and Def. \ref{def:cmpnn},   suppose Assumptions \ref{ass:space}, \ref{ass:kernel}, \ref{ass:contractive},  and \ref{ass:mpnn} hold. 
    Assess the layer-wise difference between the discrete and continuous versions of the MPNN through mean squared error (MSE), computed as
    \begin{equation}
        MSE_X\left( Z^{(l+1)}, f^{(l+1)} \right) = \left( \frac{1}{N} \sum_{i=1}^N \left\lVert Z_i^{(l+1)} - \left(S_X f^{(l+1)}\right)_i \right\rVert_{\infty}^2 \right)^{\frac{1}{2}}.
    \end{equation}
    Considering a graph signal sampled from the RGM, i.e., $G \sim \Gamma$, to be used as the input of MPNN, the above difference is bounded by the following quantity with a probability at least $1-2\rho$:
    \begin{align}
        \label{lem:layerwise:goal}
        MSE_X\left( Z^{(l+1)}, f^{(l+1)} \right)
        & \leq N^{-\frac{1}{2}} \left( K_1 + K_2 \sqrt{\log\left(\frac{2}{\rho}\right)} + K_3 \lVert f \rVert_{\infty} + K_4 \lVert f \rVert_{\infty} \sqrt{\log{\left(\frac{2}{\rho}\right)}} \right) \nonumber \\
        & + N^{-\frac{1}{2(D_{\mathcal{X}+1})}} (K_5 + K_6 \lVert f \rVert_{\infty} + K_7 L_f) \nonumber \\
        & + N^{-\frac{1}{2(D_{\mathcal{X}+1})}} (K_8 + K_9 \lVert f \rVert_{\infty}) \sqrt{\log(C_{\mathcal{X}}) + \frac{D_{\mathcal{X}}\log(N)}{2(D_{\mathcal{X}+1})}  + \log\left(\frac{2}{\rho}\right)}.
    \end{align}
    where the used  constants are defined by
\begin{align}
    \label{constant:K}
    & K_1 := \frac{2L_W^{\infty}\lVert W\rVert_{\infty}\left(\sqrt{\log(C_{\mathcal{X}})} + \sqrt{D_\mathcal{X}}\right)}{d_{\min}^2} \sum_{l=1}^{T}  L_{\Psi^{(l)}}   \left(L_{\Phi^{(l)}}C_1^{(l-1)} + \lVert
     \Phi^{(l)}(0,0) \rVert_{\infty}\right)\prod_{l'=l+1}^{T} A^{(l')}, \nonumber\\
    & K_2 := \frac{2\left(\sqrt{2}\left\lVert W \right\rVert_\infty + L_W^{\infty}\right)\lVert W \rVert_{\infty}}{d_{\min}^2} \sum_{l=1}^{T}    L_{\Psi^{(l)}} \left(L_{\Phi^{(l)}}C_1^{(l-1)} + \lVert
     \Phi^{(l)}(0,0) \rVert_{\infty}\right) \prod_{l'=l+1}^{T} A^{(l')}, \nonumber \\
    & K_3 := \frac{2 L_W^{\infty}\left(\sqrt{\log(C_{\mathcal{X}})} + \sqrt{D_\mathcal{X}}\right)\lVert W \rVert_{\infty}}{d_{\min}^2}   \sum_{l=1}^{T}      L_{\Psi^{(l)}}  L_{\Phi^{(l)}}C_2^{(l-1)}\prod_{l'=l+1}^{T} A^{(l')} , \nonumber \\
    & K_4 := \frac{2  \left(\sqrt{2}\left\lVert W \right\rVert_\infty + L_W^{\infty}\right)\lVert W \rVert_{\infty} }{d_{\min}^2 \sqrt{N}}  \sum_{l=1}^{T}     L_{\Psi^{(l)}}  L_{\Phi^{(l)}}C_2^{(l-1)} \prod_{l'=l+1}^{T} A^{(l')} , \nonumber \\
    & K_5 := \frac{2 }{d_{\min}} \sum_{l=1}^{T} L_{\Psi^{(l)}} \left(  2\lVert W \rVert_{\infty}L_{\Phi^{(l)}}D_1^{(l-1)}  +   L_W^{\infty}L_{\Phi^{(l)}}C_1^{(l-1)} +    L_W^{\infty} \lVert \Phi^{(l)}(0,0) \rVert_{\infty} \right) \prod_{l'=l+1}^{T} A^{(l')}, \nonumber \\
    & K_6 := \frac{2 }{d_{\min}} \sum_{l=1}^{T} L_{\Psi^{(l)}}  \left(  \lVert W \rVert_{\infty} L_{\Phi^{(l)}}C_1^{(l-1)} +  L_W^{\infty}L_{\Phi^{(l)}}C_2^{(l-1)} \right) \prod_{l'=l+1}^{T} A^{(l')}, \nonumber \\
    & K_7 :=  \frac{2 \lVert W \rVert_{\infty} }{d_{\min}} \sum_{l=1}^{T}  L_{\Psi^{(l)}}    L_{\Phi^{(l)}}D_3^{(l-1)}  \prod_{l'=l+1}^{T} A^{(l')}, \nonumber \\
    & K_8 := \frac{C_{\mathcal{X}} \lVert W \rVert_{\infty} }{\sqrt{2}d_{\min}} \sum_{l=1}^{T}   L_{\Psi^{(l)}}   \left (L_{\Phi^{(l)}}C_1^{(l-1)} + \lVert \Phi^{(l)}(0,0) \rVert_{\infty} \right) \prod_{l'=l+1}^{T} A^{(l')}, \nonumber \\
    & K_9 :=\frac{C_{\mathcal{X}} \lVert W \rVert_{\infty} }{\sqrt{2}d_{\min}} \sum_{l=1}^{T} L_{\Psi^{(l)}} L_{\Phi^{(l)}}C_2^{(l-1)} \prod_{l'=l+1}^{T} A^{(l')},
\end{align}
where $C_1^{(l-1)}$, $C_2^{(l-1)}$ are defined in Eqs. (\ref{constant:C1}) and (\ref{constant:C2}), respectively; $D_1^{(l-1)}$, $D_2^{(l-1)}$ and $D_3^{(l-1)}$ are defined in Eqs. (\ref{constant:D1}), (\ref{constant:D2}) and (\ref{constant:D3}), respectively; and $A^{(l)}$ is defined as
\begin{equation}
\label{lem:layerwise:A}
    A^{(l)} = L_{\Psi^{(l)}} \sqrt{1 + \frac{8\lVert W \rVert_{\infty}^2L_{\Phi^{(l)}}^2}{d_{\min}^2}  }.
\end{equation}

\end{lemma}

\begin{proof}
Applying the layer-wise structure of MPNN, it has $Z^{(l+1)} = h_{\Theta_G}^{(l+1)}\left(Z^{(l)}\right)$ and $f^{(l+1)} = h_{\Theta_W}^{(l+1)}\left(f^{(l)}\right)$, based on which we analyze MSE error as follows:
\begin{align}
      MSE_X\left( Z^{(l+1)}, f^{(l+1)} \right)    = \; & \left( \frac{1}{N} \sum_{i=1}^N \left\lVert \left( h_{\Theta_G}^{(l+1)}\left(Z^{(l)}\right) \right)_i  - \left(S_X h_{\Theta_W}^{(l+1)}\left(f^{(l)}\right)\right)_i  \right\rVert_{\infty}^2 \right)^{\frac{1}{2}} \nonumber \\
    \label{lem:layerwise:1}
    \leq \; & \left( \underbrace{\frac{1}{N} \sum_{i=1}^N \left\lVert  \left(h_{\Theta_G}^{(l+1)}\left(Z^{(l)}\right) \right)_i - \left(h_{\Theta_G}^{(l+1)} \left(S_X f^{(l)} \right) \right)_i \right\rVert_{\infty}^2}_{T_1} \right)^{\frac{1}{2}} \nonumber \\
    + \; & \left( \underbrace{\frac{1}{N} \sum_{i=1}^N \left\lVert \left(h_{\Theta_G}^{(l+1)}\left(S_X f^{(l)}\right)\right)_i - \left(S_X h_{\Theta_W}^{(l+1)}\left(f^{(l)}\right)\right)_i \right\rVert_{\infty}^2}_{T_2} \right)^{\frac{1}{2}},
\end{align}
where we use $(\cdot)_i$ to denote the $i$-th row of the input matrix. We summarize the result above as  
\begin{equation}
\label{eq:mse_bound}
    MSE_X\left( Z^{(l+1)}, f^{(l+1)} \right) \leq \sqrt{T_1} +  \sqrt{T_2}.
\end{equation}
When using a graph signal sampled from the RGM, i.e., $G \sim \Gamma$, as the MPNN input, the adjacency weight satisfies $a_{ij} = W(x_i, x_j)$ and the message signal in Eq. (\ref{mpnn_msg}) becomes
%
%
\begin{equation}
\label{eq:messangeG}
    m_i^{(l+1)} = \frac{1}{N} \sum_{j=1}^{N} \frac{W(x_i, x_j)}{d_X(x_i)} \Phi^{(l+1)}\left(z_i^{(l)}, z_j^{(l)}\right),
\end{equation}
which is denoted by $M_G^{ \Phi^{(l+1)},Z^{(l)}}(x_i)$ to be consistent with the notation $M_X^{\Phi,f}(\cdot)$ in Eq. (\ref{eq:aggX}).
Using the above expression of the MPNN message signal, we further bound the terms $T_1$ and $T_2$ separately as below.

By applying the Lipschitz continuity of $\Psi^{(l+1)}$ and a simple norm inequality  for a concatenated vector $z=[x,y] \in \mathbb{R}^{p+q}$ with $x\in \mathbb{R}^{p}$ and $y\in \mathbb{R}^{q}$, i.e.,  $\lVert z \rVert_{\infty}     \leq \lVert x \rVert_{\infty} + \lVert y \rVert_{\infty}$, we have
\begin{align}
   T_1 &    = \frac{1}{N} \sum_{i=1}^N \left\lVert \Psi^{(l+1)} \left( Z_i^{(l)}, M_G^{ \Phi^{(l+1)},Z^{(l)}}(x_i) \right)   - \Psi^{(l+1)} \left( \left(S_X f^{(l)}\right)_i, M_G^{ \Phi^{(l+1)},S_X f^{(l)}}(x_i) \right) \right\rVert_{\infty}^2 \nonumber \\
    & \leq \frac{1}{N} L_{\Psi^{(l+1)}}^2 \left( \sum_{i=1}^N \left\lVert Z_i^{(l)} - \left(S_X f^{(l)}\right)_i \right\rVert_{\infty}^2  + \sum_{i=1}^N \left\lVert M_G^{ \Phi^{(l+1)},Z^{(l)}}(x_i)  - M_G^{ \Phi^{(l+1)},S_X f^{(l)}}(x_i) \right\rVert_{\infty}^2 \right) \nonumber \\
    \label{lem:layerwise:1:1}
    & = L_{\Psi^{(l+1)}}^2 \left( MSE_X\left(Z^{(l)}, f^{(l)}\right)^2  + \frac{1}{N} \sum_{i=1}^N \underbrace{\left\lVert M_G^{ \Phi^{(l+1)},Z^{(l)}}(x_i)  - M_G^{ \Phi^{(l+1)},S_X f^{(l)}}(x_i) \right\rVert_{\infty}^2}_{T_3} \right).
\end{align}
Now we  analyze   $T_3$ below by expanding it using Eq. (\ref{eq:messangeG}), and applying Eq. (\ref{eq:d_x_bound}) and Lipschitz continuity of $\Psi^{(l+1)}$:
\begin{align}
 T_3 = \; &\left\lVert \frac{1}{N} \sum_{j=1}^N \frac{W(x_i, x_j)}{d_X(x_i)} \Phi^{(l+1)}\left( Z_i^{(l)}, Z_j^{(l)}\right) - \frac{1}{N} \sum_{j=1}^N \frac{W(x_i, x_j)}{d_X(x_i)} \Phi^{(l+1)}\left( \left(S_X f^{(l)}\right)_i, \left(S_X f^{(l)}\right)_j \right) \right\rVert_{\infty}^2 \nonumber \\
  \leq  \;& \frac{1}{N^2} \sum_{j=1}^N \left( \frac{W(x_i,x_j)}{d_X(x_i)} \right)^2 \sum_{j=1}^N \left\lVert \Phi^{(l+1)}\left( Z_i^{(l)}, Z_j^{(l)} \right) - \Phi^{(l+1)}\left( \left(S_X f^{(l)}\right)_i, \left(S_X f^{(l)}\right)_j \right) \right\rVert_{\infty}^2 \nonumber \\
      \leq \; & \frac{4\lVert W \rVert_{\infty}^2}{Nd_{\min}^2}   \sum_{j=1}^N L_{\Phi^{(l+1)}}^2 \left( \left\lVert Z_i^{(l)} - \left( S_X f^{(l)} \right)_i \right\rVert_{\infty}^2 + \left\lVert Z_j^{(l)} - \left( S_X f^{(l)} \right)_j \right\rVert_{\infty}^2 \right) \nonumber \\
    \label{lem:layerwise:1:2}
     = \; & \frac{4\lVert W \rVert_{\infty}^2L_{\Phi^{(l+1)}}^2}{d_{\min}^2}    \left( \left\lVert Z_i^{(l)} - \left( S_X f^{(l)} \right)_i \right\rVert_{\infty}^2 + MSE_X\left(Z^{(l)}, f^{(l)}\right)^2 \right).
\end{align}
Substituting (\ref{lem:layerwise:1:2}) back to (\ref{lem:layerwise:1:1}) and defining the quantity $A^{(l)}$ in Eq. (\ref{lem:layerwise:A}), we upper bound $T_1$ in terms of the convergence error $\left(Z^{(l)}, f^{(l)}\right)$, as below: 
\begin{align}
    T_1 \leq\; &  L_{\Psi^{(l+1)}}^2 MSE_X\left(Z^{(l)}, f^{(l)}\right)^2 + \nonumber \\
    & \frac{4\lVert W \rVert_{\infty}^2L_{\Psi^{(l+1)}}^2 L_{\Phi^{(l+1)}}^2 }{d_{\min}^2} \cdot \frac{1}{N}\sum_{i=1}^N\left( \left\lVert Z_i^{(l)} - \left( S_X f^{(l)} \right)_i \right\rVert_{\infty}^2 + MSE_X\left(Z^{(l)}, f^{(l)}\right)^2 \right) \nonumber \\
     =\; &  L_{\Psi^{(l+1)}}^2 MSE_X\left(Z^{(l)}, f^{(l)}\right)^2 + \nonumber \\
     & \frac{4\lVert W \rVert_{\infty}^2L_{\Psi^{(l+1)}}^2 L_{\Phi^{(l+1)}}^2 }{d_{\min}^2} \left( MSE_X(Z^{(l)}, f^{(l)})^2 + MSE_X(Z^{(l)}, f^{(l)})^2 \right) \nonumber \\
     \label{lem:layerwise:1:final}
    =\; & L_{\Psi^{(l+1)}}^2 \left( 1 + \frac{8\lVert W \rVert_{\infty}^2L_{\Phi^{(l+1)}}^2}{d_{\min}^2}  \right) MSE_X\left(Z^{(l)}, f^{(l)}\right)^2   = \left(A^{(l+1)} MSE_X\left(Z^{(l)}, f^{(l)}\right)\right)^2.
\end{align}

Next, we upper bound $T_2$ as follows:
\begin{align}
    T_2 = \;&   \frac{1}{N} \sum_{i=1}^N \bigg\lVert \Psi^{(l+1)} \left( \left(S_X f^{(l)}\right)_i, M_G^{ \Phi^{(l+1)},S_X f^{(l)}}(x_i) \right)  - \left( S_X \Psi\left( f^{(l)}(x_i), M_W^{ \Phi^{(l+1)}, f^{(l)}}(x_i) \right) \right)_i \bigg\rVert_{\infty}^2  \nonumber \\
      = \; &  \frac{1}{N} \sum_{i=1}^N \bigg\lVert \Psi^{(l+1)}\left( f^{(l)}(x_i), M_X^{ \Phi^{(l+1)}, f^{(l)}}(x_i)\right)   - \Psi\left( f^{(l)}(x_i), M_W^{ \Phi^{(l+1)}, f^{(l)}}(x_i) \right) \bigg\rVert_{\infty}^2   \nonumber \\
      \label{lem:layerwise:2:final0}
     \leq \; &   \frac{L_{\Psi^{(l+1)}}^2}{N} \sum_{i=1}^N    \left\lVert M_X^{ \Phi^{(l+1)}, f^{(l)}}(x_i) - M_W^{ \Phi^{(l+1)}, f^{(l)}}(x_i) \right\rVert_{\infty}^2 \nonumber \\
     \leq \; & L_{\Psi^{(l+1)}}^2 \left\lVert M_X^{ \Phi^{(l+1)}, f^{(l)}}( \cdot) - M_W^{ \Phi^{(l+1)}, f^{(l)}}(\cdot) \right\rVert_{\infty}^2.   
     \end{align}
Applying the result in Lemma {\ref{lem:agg}} with their mean aggregation operators $M_X(\cdot)$ and $M_W(\cdot)$ implemented for the specific functions $\Phi=\Phi^{(l+1)}$ and $f=f^{(l)}$, the following holds with a probability at least $1-2\rho$:
\begin{align}
    \label{lem:layerwise:2:final}
       \sqrt{T_2} \leq\; &      \frac{2L_{\Psi^{(l+1)}} C_W(\rho) C_{\Phi^{(l+1)};f^{(l)}} \left\lVert W \right\rVert_{\infty}}{d_{\min}^2\sqrt{N}}   +   \nonumber \\
       &   L_{\Psi^{(l+1)}}N^{-\frac{1}{2(D_{\mathcal{X}}+1)}}\left( \frac{2\lVert W \rVert_{\infty} L_{\Phi^{(l+1)}} L_{f^{(l+1)}}^{\infty} + 2L_W^{\infty} C_{\Phi^{(l+1)};f^{(l+1)}} }{d_{\min}}  + \right. \nonumber \\
     & \left. \frac{ \lVert W \rVert_{\infty} C_{\mathcal{X}} C_{\Phi^{(l+1)};f^{(l)}}}{\sqrt{2}d_{\min}} \sqrt{\log(C_{\mathcal{X}}) + \frac{D_{\mathcal{X}}\log(N)}{2(D_{\mathcal{X}+1})}  + \log\left(\frac{2}{\rho}\right)} \right)   =B^{(l+1)},
\end{align}
where the derived upper bound of $\sqrt{T_2}$ is  denoted as $B^{(l+1)}$.

Finally, substituting Eqs. (\ref{lem:layerwise:1:final}) and (\ref{lem:layerwise:2:final}) into Eq. (\ref{eq:mse_bound}), we have:
\begin{equation}
    MSE_X(Z^{(l+1)}, f^{(l+1)})  \leq    A^{(l+1)}  MSE_X\left(Z^{(l)}, f^{(l)}\right) + B^{(l+1)}.
\end{equation}
Applying Lemma \ref{lem:recursive} with $\Delta_{l+1}=MSE_X(Z^{(l+1)}, f^{(l+1)})$, it has 
\begin{equation}
\label{lem:layerwise:final}
    MSE_X(Z^{(T)}, f^{(T}) \leq \sum_{l=1}^{T} B^{(l)} \prod_{l'=l+1}^{T} A^{(l')},
\end{equation}
where $MSE_X(Z^{(0)}, f^{(0)}) = 0$ since $Z^{(0)}$ is directly sampled from $f^{(0)}$ using operator $S_X$.
As a result the following holds   with 
a probability at least $1-2\rho$:
\begin{equation}
\label{eq:mse}
    MSE_X\left(Z^{(T)}, f^{(T)}\right) \leq \sum_{l=1}^{T} L_{\Psi^{(l)}} C^{(l)} \prod_{l'=l+1}^{T} A^{(l')}, 
\end{equation}
where 
\begin{align}
    C^{(l)} = \; &  \frac{2 C_W(\rho) \left(L_{\Phi^{(l)}}   \lVert f^{(l-1)} \rVert_{\infty} + \lVert \Phi^{(l)}(0,0) \rVert_{\infty}\right) \left\lVert W \right\rVert_{\infty}}{d_{\min}^2\sqrt{N}}  \\
    + \; & N^{-\frac{1}{2(D_{\mathcal{X}}+1)}} \left[  \frac{2\lVert W \rVert_{\infty}L_{\Phi^{(l)}} L_{f^{(l-1)}} }{d_{\min}} + \frac{2L_W^{\infty}}{d_{\min}} \left(L_{\Phi^{(l)}}   \lVert f^{(l-1)} \rVert_{\infty} + \lVert \Phi^{(l)}(0,0) \rVert_{\infty}\right)   \right. \nonumber \\
    + \; & \left. \frac{\lVert W \rVert_{\infty} C_{\mathcal{X}}}{\sqrt{2}d_{\min}} \left(L_{\Phi^{(l)}}  \lVert f^{(l-1)} \rVert_{\infty} + \lVert \Phi^{(l)}(0,0) \rVert_{\infty}\right)  \sqrt{\log(C_{\mathcal{X}}) + \frac{D_{\mathcal{X}}\log(N)}{2(D_{\mathcal{X}+1})}  + \log\left(\frac{2}{\rho} \right)} \right].  \nonumber
\end{align}
Applying the result in Lemma \ref{lem:boundedf} for upper bounding $\lVert f^{(l)} \rVert_{\infty}$ and $L_{f^{(l)}}$ and expanding  $C_W(\rho)$ based on its definition in Lemma {\ref{lem:agg}}, the term $C^{(l)} $ is further bounded by
\begin{align}
    C^{(l)} \leq \; & \frac{2 C_W(\rho) \left(L_{\Phi^{(l)}}  \left(C_1^{(l-1)} + C_2^{(l-1)} \lVert f \rVert_{\infty}\right) + \lVert \Phi^{(l)}(0,0) \rVert_{\infty}\right) \left\lVert W \right\rVert_{\infty}}{d_{\min}^2\sqrt{N}} + \nonumber \\
    &    N^{-\frac{1}{2(D_{\mathcal{X}}+1)}} \left(   \frac{2\lVert W \rVert_{\infty}L_{\Phi^{(l)}}}{d_{\min}}  \left(D_1^{(l-1)} + D_2^{(l-1)} \lVert f \rVert_{\infty} + D_3^{(l-1)} L_f\right)  + \right.\nonumber \\
        & \frac{2L_W^{\infty}}{d_{\min}} \left(L_{\Phi^{(l)}}  \left(C_1^{(l-1)} + C_2^{(l-1)} \lVert f \rVert_{\infty}\right) + \lVert \Phi^{(l)}(0,0) \rVert_{\infty}\right) +\nonumber \\
       & \left. \frac{\lVert W \rVert_{\infty} C_{\mathcal{X}}}{\sqrt{2}d_{\min}} \left(L_{\Phi^{(l)}}  \left(C_1^{(l-1)} + C_2^{(l-1)} \lVert f \rVert_{\infty}\right) + \lVert \Phi^{(l)}(0,0) \rVert_{\infty}\right)   \right. \nonumber \\
     &\left.\cdot \sqrt{\log(C_{\mathcal{X}}) + \frac{D_{\mathcal{X}} \log(N)}{2(D_{\mathcal{X}+1})} + \log\left(\frac{2}{\rho}\right)} \right) \nonumber \\
    \leq \; & \frac{\left( 2L_W^{\infty}\left(\sqrt{\log(C_{\mathcal{X}})} + \sqrt{D_\mathcal{X}}\right) + \left(\sqrt{2}\left\lVert W \right\rVert_\infty + L_W^{\infty}\right) \sqrt{\log\left(\frac{2}{\rho}\right)} \right)}{d_{\min}^2\sqrt{N}} \nonumber \\
      & \cdot  \frac{\left(L_{\Phi^{(l)}}  \left(C_1^{(l-1)} + C_2^{(l-1)} \lVert f \rVert_{\infty}\right) + \lVert \Phi^{(l)}(0,0) \rVert_{\infty}\right) \left\lVert W \right\rVert_{\infty}}{d_{\min}^2\sqrt{N}} + \nonumber \\
      & N^{-\frac{1}{2(D_{\mathcal{X}}+1)}} \left[  \frac{2\lVert W \rVert_{\infty}L_{\Phi^{(l)}}}{d_{\min}}  \left(D_1^{(l-1)} + D_2^{(l-1)} \lVert f \rVert_{\infty} + D_3^{(l-1)} L_f\right) + \right. \nonumber \\
       & \frac{2L_W^{\infty}}{d_{\min}} \left(L_{\Phi^{(l)}}  \left(C_1^{(l-1)} + C_2^{(l-1)} \lVert f \rVert_{\infty}\right) + \lVert \Phi^{(l)}(0,0) \rVert_{\infty}\right) + \nonumber \\
        &   \frac{\lVert W \rVert_{\infty} C_{\mathcal{X}}}{\sqrt{2}d_{\min}} \left(L_{\Phi^{(l)}}  \left(C_1^{(l-1)} + C_2^{(l-1)} \lVert f \rVert_{\infty}\right) + \lVert \Phi^{(l)}(0,0) \rVert_{\infty}\right)  \nonumber \\
        &  \cdot\left.  \sqrt{\log(C_{\mathcal{X}}) + \frac{D_{\mathcal{X}} \log(N)}{2(D_{\mathcal{X}+1})}  + \log\left(\frac{2}{\rho}\right)} \right].
\end{align}
Finally, we  substitute the above bound of $C^{(l)}$ into Eq. (\ref{eq:mse}), re-organize terms, and obtain the following:
\begin{align}
        MSE_X\left( Z^{(T)}, f^{(T)} \right)
        & \leq \frac{1}{\sqrt{N}} \left( K_1 + K_2 \sqrt{\log\left(\frac{2}{\rho}\right)} + K_3 \lVert f \rVert_{\infty} + K_4 \lVert f \rVert_{\infty} \sqrt{\log{\left(\frac{2}{\rho}\right)}} \right) \nonumber \\
        & + N^{-\frac{1}{2(D_{\mathcal{X}+1})}} (K_5 + K_6 \lVert f \rVert_{\infty} + K_7 L_f) \nonumber \\
        & + N^{-\frac{1}{2(D_{\mathcal{X}+1})}} (K_8 + K_9 \lVert f \rVert_{\infty}) \sqrt{\log(C_{\mathcal{X}}) + \frac{D_{\mathcal{X}} \log(N)}{2(D_{\mathcal{X}+1})}  + \log\left(\frac{2}{\rho}\right)}.
    \end{align}
Let $T =l+1$, Eq. (\ref{lem:layerwise:goal}) is resulted. This completes the proof.

\end{proof}




\begin{theorem}[Convergence Error]
\label{proof:convergence}
     Given an RGM $\Gamma = (W, P, f)$  and the discrete and continuous versions of an MPNN as in Def. \ref{def:mpnn} and Def. \ref{def:cmpnn}, suppose Assumptions \ref{ass:space}, \ref{ass:kernel}, \ref{ass:contractive},  and \ref{ass:mpnn} hold.   
     Then the following holds with a probability at least $1-2\rho$:
  \begin{equation}
      \left\lVert \Bar{h}_G(Z) - \Bar{h}_{W,P}(f) \right\rVert_{\infty}^2   \leq \Delta_N,
  \end{equation}
 where
    \begin{align}
    \label{eq:delta_N_def}
       \Delta_N  =\;& \frac{R_1 + R_2 \lVert f \rVert_{\infty}^2}{N} + \frac{S_1 + S_2 \lVert f \rVert_{\infty}^2 + S_3 L_f^2 +(T_1 + T_2 \lVert f \rVert_{\infty}^2) \log(N)}{N^{\frac{1}{D_{\mathcal{X}}+1}}}   \\
       \nonumber
       &+ \left( \frac{R_3 + R_4 \lVert f \rVert_{\infty}^2}{N} + \frac{S_4 + S_5 \lVert f \rVert_{\infty}^2}{N^{\frac{1}{D_{\mathcal{X}}+1}}} \right)  \log\left(\frac{2}{\rho}\right), 
    \end{align}
   along with  the following constants computed from the constants used in Lem. \ref{lem:layerwise}:  
   \begin{align}
    \label{constant:RST}
    & R_i := 14 K_i^2, \textmd{ for } i=1, 2, 3, 4, \nonumber \\
    & S_1 := 14 K_5^2 + 14 K_8^2 \log(C_{\mathcal{X}}) + 56 \left(D_1^{(T)}\right)^2 + 7 C_{\mathcal{X}}^2 \left(C_1^{(T)}\right)^2 \log(C_{\mathcal{X}}), \nonumber \\
    & S_2 := 14 K_6^2 + 14 K_9^2 \log(C_{\mathcal{X}}) + 56 \left(D_3^{(T)}\right)^2 + 7 C_{{\mathcal{X}}}^2 \left(C_2^{(T)}\right)^2 \log(C_{\mathcal{X}}), \nonumber \\
    & S_3 := 14 K_7^2 + 56 \left(D_3^{(T)}\right)^2 L_f^2, \nonumber \\
    & S_4 := 14 K_8^2 + 7 C_{\mathcal{X}}^2 \left(C_1^{(T)}\right)^2, \nonumber \\
    & S_5 := 14 K_9^2 + 7 C_{\mathcal{X}}^2 \left(C_2^{(T)}\right)^2, \nonumber \\
    & T_1 := \left(14 K_8^2 + 7 C_{\mathcal{X}}^2 \left(C_1^{(T)}\right)^2\right) \frac{D_{\mathcal{X}}}{2(D_{\mathcal{X}}+1)}, \nonumber \\
    & T_2 := \left(14 K_9^2 + 7 C_{\mathcal{X}}^2 \left(C_2^{(T)}\right)^2 \right) \frac{D_{\mathcal{X}}}{2(D_{\mathcal{X}}+1)}.
\end{align}
\end{theorem}

\begin{proof}
Applying the definition of discrete and continuous pooling operations, it has 
\begin{align}
    &\left\lVert \Bar{h}_G(Z) - \Bar{h}_{W,P}(f) \right\rVert_{\infty}\nonumber \\
      \leq\;& \left\lVert \frac{1}{N} \sum_{i=1}^N \left(h_G^{(T)}(Z)\right)_i - \int_{\mathcal{X}} h_W^{(T)}\circ f(x) \, dP(x) \right\rVert_{\infty} \nonumber \\
    \leq\;& \left\lVert \frac{1}{N} \sum_{i=1}^N \left(h_G^{(T)}(Z)\right)_i - \frac{1}{N} \sum_{i=1}^N \left( S_X h_W^{(T)}(f) \right)_i \right\rVert_{\infty}  + \nonumber \\
    & \left\lVert \frac{1}{N} \sum_{i=1}^N \left( S_X h_W^{(T)}(f) \right)_i - \int_{\mathcal{X}} h_W^{(T)}\circ f(x) \, dP(x) \right\rVert_{\infty} \nonumber \\
      \leq \; & \frac{1}{N} \sum_{i=1}^N \left\lVert Z^{(T)}_i - \left( S_X f^{(T)}\right)_i \right\rVert_{\infty} + \left\lVert \frac{1}{N} \sum_{i=1}^N   h_W^{(T)}\circ f(x_i) - \int_{\mathcal{X}} h_W^{(T)}\circ f(x) \, dP(x) \right\rVert_{\infty} \nonumber \\
    \label{thm:convergence_final:1}
      =\; & MSE_X\left(Z^{(T)}, f^{(T)}\right) + \underbrace{\left\lVert \frac{1}{N} \sum_{i=1}^N h_W^{(T)}\circ f(x_i) - \int_{\mathcal{X}} h_W^{(T)}\circ f(x) \, dP(x) \right\rVert_{\infty}}_T.
\end{align}
We apply  Lem. \ref{lem:uniform_bound} to bound $T$ by letting $F(\cdot) = h_W^{(T)}(f) = h_W^{(T)}\circ f(\cdot) $, and apply Lem. \ref{lem:boundedf} to obtain $\lVert h_W^{(T)}(f) \rVert_{\infty} \leq C_1^{(T)} + C_2^{(T)} \lVert f \rVert_{\infty}$ and $L_{h_W^{(T)}(f)} \leq D_1^{(T)} + D_3^{(T)} \lVert f \rVert_{\infty} + D_3^{(T)} L_f$. 
As a result, this gives
\begin{align}
    T \leq \; & N^{-\frac{1}{2(D_{\mathcal{X}}+1)}} \bigg[ 2 \left( D_1^{(T)} + D_3^{(T)} \lVert f \rVert_{\infty} + D_3^{(T)} L_f \right) \nonumber \\
    \label{thm:convergence_final:2}
    + \; & \frac{C_{\mathcal{X}}}{\sqrt{2}} \left( C_1^{(T)} + C_2^{(T)} \lVert f \rVert_{\infty} \right) \sqrt{\log(C_{\mathcal{X}}) + \frac{D_{\mathcal{X}} \log(N)}{2(D_{\mathcal{X}+1})} + \log\left(\frac{2}{\rho}\right)} \bigg].
\end{align}
Combining Eq. (\ref{thm:convergence_final:2}) and Lem. \ref{lem:layerwise} that bounds $MSE_X\left(Z^{(T)}, f^{(T)}\right)$, we derive as below an upper bound of $\left\lVert \Bar{h}_G(Z) - \Bar{h}_{W,P}(f) \right\rVert_{\infty}$ based on Eq. (\ref{thm:convergence_final:1}), by further re-organizing and merging constant terms   and  by highlighting the number of nodes $N$ and probability terms such as $\log(\frac{2}{\rho})$. 
This results in the following:
\begin{align}
    \label{thm:convergence_final:3}
    &\left\lVert \Bar{h}_G(Z) - \Bar{h}_{W,P}(f) \right\rVert_{\infty} \nonumber \\
    \leq \; &  N^{-\frac{1}{2}} \left( K_1 + K_2 \sqrt{\log\left(\frac{2}{\rho}\right)} + K_3 \lVert f \rVert_{\infty} + K_4 \lVert f \rVert_{\infty} \sqrt{\log\left(\frac{2}{\rho}\right)} \right) \nonumber \\
    + \; &  N^{-\frac{1}{2(D_{\mathcal{X}+1})}} (K_5 + K_6 \lVert f \rVert_{\infty} + K_7 L_f) \nonumber \\
    + \; & N^{-\frac{1}{2(D_{\mathcal{X}+1})}} (K_8 + K_9 \lVert f \rVert_{\infty}) \sqrt{\log(C_{\mathcal{X}}) + \frac{D_{\mathcal{X}}\log(N)}{2(D_{\mathcal{X}+1})}  + \log\left(\frac{2}{\rho}\right)} \nonumber \\
    + \; & N^{-\frac{1}{2(D_{\mathcal{X}}+1)}} \bigg[ 2 \left( D_1^{(T)} + D_3^{(T)} \lVert f \rVert_{\infty} + D_3^{(T)} L_f \right) \nonumber \\
    + \; & \frac{C_{\mathcal{X}}}{\sqrt{2}} \left( C_1^{(T)} + C_2^{(T)} \lVert f \rVert_{\infty} \right) \sqrt{\log(C_{\mathcal{X}}) + \frac{D_{\mathcal{X}}\log(N)}{2(D_{\mathcal{X}+1})}  + \log\left(\frac{2}{\rho}\right)} \bigg].
\end{align}
Finally, we upper bound $\left\lVert \Bar{h}_G(Z) - \Bar{h}_{W,P}(f) \right\rVert_{\infty}^2$ by applying one version of Cauchy-Schwartz inequality $\left( \sum_{i=1}^N a_i \right)^2 \leq N  \sum_{i=1}^N a_i^2$, i.e., to compute the sum of the squared additive quantities  from the right side of Eq. (\ref{thm:convergence_final:3}). This results in 
\begin{align}
    & \left\lVert \Bar{h}_G(Z) - \Bar{h}_{W,P}(f) \right\rVert_{\infty}^2 \nonumber \\
    \leq \;& \frac{14\left(K_1^2 + K_3^2 \lVert f \rVert_{\infty}^2\right)}{N} + \frac{14\left(K_2^2 + K_4^2 \lVert f \rVert_{\infty}^2\right)\log\left(\frac{2}{\rho}\right)}{N} + \frac{14 \left(K_5^2 + K_6^2 \lVert f \rVert_{\infty}^2 + K_7^2 L_f^2\right)}{N^{\frac{1}{D_{\mathcal{X}}+1}}} \nonumber \\
    + \; & \frac{14\left(K_8^2 + K_9^2 \lVert f \rVert_{\infty}^2\right)\left( \log(C_{\mathcal{X}}) + \frac{D_{\mathcal{X}}\log(N)}{2(D_{\mathcal{X}}+1)}  + \log\left(\frac{2}{\rho}\right)\right)}{N^{\frac{1}{D_{\mathcal{X}}+1}}}  \nonumber \\
    + \; & \frac{56\left(\left(D_1^{(T)}\right)^2 + \left(D_3^{(T)}\right)^2 \lVert f \rVert_{\infty}^2 + \left(D_3^{(T)}\right)^2 L_f^2\right)}{N^{\frac{1}{D_{\mathcal{X}}+1}}} \nonumber \\
    + \; & \frac{7 C_{\mathcal{X}}^2\left(\left(C_1^{(T)}\right)^2 + \left(C_2^{(T)}\right)^2 \lVert f \rVert_{\infty}^2\right)\left( \log(C_{\mathcal{X}}) + \frac{D_{\mathcal{X}}\log(N)}{2(D_{\mathcal{X}}+1)}  +\log\left(\frac{2}{\rho}\right)\right)}{N^{\frac{1}{D_{\mathcal{X}}+1}}} := \Delta_N.
\end{align}
The above gives rise to the definition of $\Delta_N$ as in Eq. (\ref{eq:delta_N_def}), after re-organizing terms. 
This completes the proof.
\end{proof}

\section{Bounding Optimization Error}
\label{app:optimization}

This section presents proofs and supporting results for bounding the optimization error. Lem. \ref{proof:kernel_perturb} bounds the output change  of cMPNN by perturbing the RGM kernel $W$, while Lem. \ref{proof:prob_perturb} perturbs the RGM distribution $P$.  Together,  Lem. \ref{proof:kernel_perturb} and Lem. \ref{proof:prob_perturb}  form the RGM perturbation result in Thm.  \ref{proof:perturb_RGM}. Then, Thm. \ref{proof:perturb_weights} presents the weight perturbation results. 
Finally, the optimization error results is the combination of the RGM and weight  perturbation results, presented in Thm. \ref{proof:optimization}.

\begin{lemma} [RGM Kernel Perturbation]
\label{proof:kernel_perturb}
Given an RGM $\Gamma = (W, P, f)$ and its perturbation $\Gamma' = (W_{\tau}, P, f')$ by kernel deformation $\tau$,   and the discrete and continuous versions of an MPNN as in Def. \ref{def:mpnn} and Def. \ref{def:cmpnn}, suppose Assumptions \ref{ass:space}-\ref{ass:tau_Wconstants} and \ref{ass:tau_Pconstants}-\ref{ass:mpnn} hold. Define the following constants:
\begin{align}
        C_3^{(T)} & := \frac{d_{\min}+W_{\max}}{d_{\min}^2}  \sum_{l=1}^T L_{\Psi^{(l)}} \left( L_{\Phi^{(l)}}   \left( C_1^{(l-1)} + C_2^{(l-1)} \lVert f \rVert_{\infty} \right) + \left\lVert \Phi^{(l)}(0,0) \right\rVert_{\infty} \right) \nonumber \\
        & \cdot \prod_{l'=l+1}^T L_{\Psi^{(l')}} \left( 1 + \frac{W_{\max}}{d_{\min}} \cdot L_{\Phi^{(l')}} \right), \\
        C_4^{(T)} & := \prod_{l=1}^T L_{\Psi^{(l)}}  \left( 1 + \frac{W_{\max}}{d_{\min}} \cdot L_{\Phi^{(l)}} \right).
    \end{align}
Then, the  hypothesis change induced by kernel deformation is bounded by the following:
    \begin{equation}
    \label{proof:kernel_perturb:goal}
        \left\lVert \Bar{h}_{W,P}(f) - \Bar{h}_{W_\tau,P}(f') \right\rVert_{\infty} \leq C_3^{(T)} C_{\nabla w} \lVert \nabla \tau \rVert_{\infty} + C_4^{(T)} \lVert f - f' \rVert_{\infty},
    \end{equation}    
where $C_{\nabla w}$ is defined in Assumption \ref{ass:tau_Wconstants}.
\end{lemma}

\begin{proof}

Applying the continuous pooling and cMPNN structure in Def. \ref{def:cmpnn}, we expand the following:
\begin{align}
    \left\lVert \Bar{h}_{W,P}(f) - \Bar{h}_{W_\tau,P}(f') \right\rVert_{\infty}
    & = \left\lVert \int_{\mathcal{X}} h_W^{(T)}\circ f(x) \, dP(x) - \int_{\mathcal{X}} h_{W_\tau}^{(T)}\circ f'(x) \, dP(x) \right\rVert_{\infty}  \nonumber \\
    \label{proof:kernel_perturb:subgoal}
    & \leq \int_{\mathcal{X}} \left\lVert h_W^{(T)}\circ f(x) - h_{W_\tau}^{(T)}\circ f'(x) \right\rVert_{\infty} \, dP(x)  \leq  \left\lVert f^{(T)} - f^{(T)'}\right\rVert_{\infty}
\end{align}
It now suffices to bound the quantity $\lVert f^{(l+1)} - f^{(l+1)'} \rVert_{\infty}$ as below:
\begin{align}
    \left\lVert f^{(l+1)} - f^{(l+1)'} \right\rVert_{\infty}
    & = \left\lVert h_{W,P}\left(f^{(l)}\right) - h_{W_\tau,P}\left(f^{(l)'}\right) \right\rVert_{\infty} \nonumber \\
    \label{proof:kernel_perturb:1}
    & \leq \underbrace{\left\lVert h_{W,P}\left(f^{(l)}\right) - h_{W_\tau,P}\left(f^{(l)}\right) \right\rVert_{\infty}}_{T_a} + \underbrace{\left\lVert h_{W_\tau,P}\left(f^{(l)}\right) - h_{W_\tau,P}\left(f^{(l)'}\right) \right\rVert_{\infty}}_{T_b}.
\end{align}
for which we will bound below the two terms $T_a$ and $T_b$ separately.

Applying the definition of cMPNN in Def. \ref{def:cmpnn} and Lipschitz continuity of $\Phi$ and $\Psi$, it has

\begin{align}
    T_a = \; &   \left\lVert \Psi^{(l+1)} \left( f^{(l)}, M_W^{ \Phi^{(l+1)}, f^{(l)}} \right) - \Psi^{(l+1)} \left( f^{(l)}, M_{W_{\tau}}^{ \Phi^{(l+1)}, f^{(l)}} \right) \right\rVert_{\infty}  \nonumber\\
    \leq \; &  L_{\Psi^{(l+1)}}   \left\lVert M_W^{ \Phi^{(l+1)}, f^{(l)}} -  M_{W_{\tau}}^{ \Phi^{(l+1)}, f^{(l)}} \right\rVert_{\infty} \nonumber \\
   = \; &  L_{\Psi^{(l+1)}}  \bigg\lVert \int_{\mathcal{X}} \frac{W(x,y)}{d_W(x)}\Phi^{(l+1)}\left(f^{(l)}(x), f^{(l)}(y)\right) \, dP(y) - \nonumber \\ 
    & \int_{\mathcal{X}} \frac{{W_\tau}(x,y)}{d_{W_\tau}(x)}\Phi^{(l+1)}\left(f^{(l)}(x), f^{(l)}(y)\right) \, dP(y) \bigg\rVert_{\infty} \nonumber \\
    \leq\;  & L_{\Psi^{(l+1)}}  \int_{\mathcal{X}} \left\lVert \frac{W(x,y)}{d_W(x)} - \frac{{W_\tau}(x,y)}{d_{W_\tau}(x)} \right\rVert_{\infty} \cdot \left\lVert \Phi^{(l+1)}\left(f^{(l)}(x), f^{(l)}(y)\right) \right\rVert_{\infty} \, dP(y) \nonumber \\
    \label{proof:kernel_perturb:1:1}
     \leq \; & L_{\Psi^{(l+1)}} \left( L_{\Phi^{(l+1)}}   \left\lVert f^{(l)} \right\rVert_{\infty} + \left\lVert \Phi^{(l+1)}(0,0) \right\rVert_{\infty} \right)  \underbrace{\int_{\mathcal{X}} \left\lVert \frac{W(x,y)}{d_W(x)} - \frac{{W_\tau}(x,y)}{d_{W_\tau}(x)} \right\rVert_{\infty} \, dP(y)}_{T_c},
\end{align}
where we recall that $\left\lVert \Phi^{(l+1)} \right\rVert_{\infty} \leq L_{\Phi^{(l+1)}}   \lVert f^{(l)} \rVert_{\infty} + \left\lVert \Phi^{(l+1)}(0,0) \right\rVert_{\infty}$ as shown by Eq. (\ref{phiformalbias}) in the proof of Lem. \ref{lem:agg}. 
It now suffices to upper bound the last term $T_c$ in (\ref{proof:kernel_perturb:1:1}), which we expand below:
\begin{align}
    T_c
    & \leq \int_{\mathcal{X}} \left\lVert \frac{W(x,y)}{d_W(x)} - \frac{W_\tau(x,y)}{d_{W}(x)} \right\rVert_{\infty} \, dP(y) + \int_{\mathcal{X}} \left\lVert \frac{W_\tau(x,y)}{d_{W}(x)} - \frac{{W_\tau}(x,y)}{d_{W_\tau}(x)} \right\rVert_{\infty} \, dP(y) \nonumber \\
    \label{proof:kernel_perturb:1:2}
    & = \left\lVert \frac{1}{d_{W}(x)} \right\rVert_{\infty}  \int_{\mathcal{X}} \left\lVert W(x,y) - W_\tau(x,y) \right\rVert_{\infty} \, dP(y)  + \lVert W_\tau \rVert_{\infty}  \left\lVert \frac{1}{d_W(x)} - \frac{1}{d_{W_\tau}(x)} \right\rVert_{\infty}.
\end{align}
Applying a direct result from \citet{keriven2020convergence} used for deriving their Eq. (31), it has
\begin{equation}
    \int_{\mathcal{X}} \left\lVert W(x,y) - W_\tau(x,y) \right\rVert_{\infty} \, dP(y)
     \leq C_{\nabla w} \lVert \nabla \tau \rVert_{\infty}.
\end{equation}

Regarding to the second term in Eq. (\ref{proof:kernel_perturb:1:2}), it has
\begin{align}
    \left\lVert \frac{1}{d_W(x)} - \frac{1}{d_{W_\tau}(x)} \right\rVert_{\infty}
    & = \left\lVert \frac{d_{W_\tau}(x) - d_W(x)}{d_W(x)   d_{W_\tau}(x)} \right\rVert_{\infty} \nonumber \\
    & \leq \left\lVert \frac{1}{d_W(x)  d_{W_\tau}(x)} \right\rVert_{\infty} \cdot \int \left\lVert W_\tau(x,y) - W(x, y) \right\rVert_{\infty} \, dP(y)  \nonumber \\
    & \leq \frac{C_{\nabla w}   \lVert \nabla \tau \rVert_{\infty}}{d_{\min}^2}.
\end{align}
Substituting these back to Eq. (\ref{proof:kernel_perturb:1:2}), it has
\begin{equation}
\label{proof:kernel_perturb:1:2_Tc}
    T_c \leq \left\lVert \frac{1}{d_{W}(x)} \right\rVert_{\infty} C_{\nabla w} \lVert \nabla \tau \rVert_{\infty} +  \frac{C_{\nabla w}   \lVert \nabla \tau \rVert_{\infty}\lVert W_\tau \rVert_{\infty}}{d_{\min}^2} \leq   \frac{C_{\nabla w} \lVert \nabla \tau \rVert_{\infty}}{d_{\min}} +  \frac{C_{\nabla w}   \lVert \nabla \tau \rVert_{\infty} W_{\max}}{d_{\min}^2}.
\end{equation}
Substitute further Eq. (\ref{proof:kernel_perturb:1:2_Tc}) back to (\ref{proof:kernel_perturb:1:1}) and applying the result of Lem. \ref{lem:boundedf} for bounding $\left\lVert f^{(l)} \right\rVert_{\infty}$, we conclude the bound of $T_a$ in Eq. (\ref{proof:kernel_perturb:1}):
\begin{align}
    T_a \leq \; & L_{\Psi^{(l+1)}} \left( L_{\Phi^{(l+1)}}  \left\lVert f^{(l)} \right\rVert_{\infty} + \left\lVert \Phi^{(l+1)}(0,0) \right\rVert_{\infty} \right)  \frac{C_{\nabla w}  \lVert \nabla \tau \rVert_{\infty}}{d_{\min}} \cdot \bigg( 1 + \frac{W_{\max}}{d_{\min}} \bigg) \nonumber \\
    \label{proof:kernel_perturb:1:final}
      \leq \; & L_{\Psi^{(l+1)}} \left( L_{\Phi^{(l+1)}}  \left( C_1^{(l)} + C_2^{(l)} \lVert f \rVert_{\infty} \right) + \left\lVert \Phi^{(l+1)}(0,0) \right\rVert_{\infty} \right)  C_{\nabla w} \lVert \nabla \tau \rVert_{\infty}  \left( \frac{1}{d_{\min}} + \frac{W_{\max}}{d_{\min}^2} \right).
\end{align}
 We now seek to bound the second term $T_b$ in Eq. (\ref{proof:kernel_perturb:1}) by expanding based on definition of cMPNN in Def. \ref{def:cmpnn}:
\begin{align}
    T_b = \; & \bigg\lVert \Psi^{(l+1)} \left( f^{(l)}, M_{W_{\tau}}^{ \Phi^{(l+1)}, f^{(l)}}  \right)  - \Psi^{(l+1)} \left( f^{(l)'}, M_{W_{\tau}}^{ \Phi^{(l+1)}, f^{(l)'}}  \right) \bigg\rVert_{\infty} \nonumber \\
    \label{proof:kernel_perturb:2}
     \leq \; & L_{\Psi^{(l+1)}}   \left( \left\lVert f^{(l)} - f^{(l)'} \right\rVert_{\infty}   + \underbrace{\left\lVert M_{W_{\tau}}^{ \Phi^{(l+1)}, f^{(l)}} - M_{W_{\tau}}^{ \Phi^{(l+1)}, f^{(l)'}}  \right\rVert_{\infty}}_{T_d} \right)
\end{align}
It now suffices to bound the second term $T_d$ in (\ref{proof:kernel_perturb:2}) as below:
\begin{align}
   T_d \leq\; & \bigg\lVert \int_{\mathcal{X}} \frac{W_\tau(x,y)}{d_{W_\tau}(x)}\Phi^{(l+1)}\left(f^{(l)}(x), f^{(l)}(y)\right) \, dP(y) - \nonumber \\
   &   \int_{\mathcal{X}} \frac{{W_\tau}(x,y)}{d_{W_\tau}(x)}\Phi^{(l+1)}\left(f^{(l)'}(x), f^{(l)'}(y) \right) \, dP(y) \bigg\rVert_{\infty} \nonumber \\
   \label{proof:kernel_perturb:2:1}
      \leq \; & \frac{W_{\max}}{d_{\min}}   \left\lVert \Phi^{(l+1)}\left(f^{(l)},f^{(l)}\right) - \Phi^{(l+1)}\left(f^{(l)'},f^{(l)'}\right) \right\rVert_{\infty} \leq \frac{W_{\max}L_{\Phi^{(l+1)}}}{d_{\min}}    \left\lVert f^{(l)} - (f^{(l)})' \right\rVert_{\infty}.
\end{align}
Substituting the above back to Eq. (\ref{proof:kernel_perturb:1}), it has
\begin{align}
\label{proof:kernel_perturb:2:final}
    T_b \leq\;  & L_{\Psi^{(l+1)}} \left( \left\lVert f^{(l)} - f^{(l)'} \right\rVert_{\infty} + \frac{W_{\max}L_{\Phi^{(l+1)}}}{d_{\min}}    \left\lVert f^{(l)} - f^{(l)'} \right\rVert_{\infty} \right)   \nonumber \\
    =\; &  L_{\Psi^{(l+1)}}   \left( 1 + \frac{W_{\max}L_{\Phi^{(l+1)}} }{d_{\min}}  \right) \left\lVert f^{(l)} - f^{(l)'} \right\rVert_{\infty}.
\end{align}

Combining (\ref{proof:kernel_perturb:1:final}) and (\ref{proof:kernel_perturb:2:final}) into Eq. (\ref{proof:kernel_perturb:1}), we get the following recursion form:
\begin{align}
    \label{proof:kernel_perturb:recursive}
    \left\lVert f^{(l+1)} - f^{(l+1)'}  \right\rVert_{\infty}
    & \leq L_{\Psi^{(l+1)}} \left( L_{\Phi^{(l+1)}}   \left( C_1^{(l)} + C_2^{(l)} \lVert f \rVert_{\infty} \right) + \left\lVert \Phi^{(l+1)}(0,0) \right\rVert_{\infty} \right) \nonumber\\ 
    & \cdot C_{\nabla w} \lVert \nabla \tau \rVert_{\infty}   \bigg( \frac{1}{d_{\min}} + \frac{W_{\max}}{d_{\min}^2} \bigg) \nonumber \\
    & + L_{\Psi^{(l+1)}} \left( 1 + \frac{W_{\max}L_{\Phi^{(l+1)}}}{d_{\min}}  \right)  \left\lVert f^{(l)} - f^{(l)} \right\rVert_{\infty} 
\end{align}
Applying the recursion Lem. \ref{lem:recursive} with $\Delta_t =\left\lVert f^{(t)} - f^{(t)'}  \right\rVert_{\infty}$,  an upper bound for $\left\lVert f^{(T)} - f^{(T)'}\right\rVert_{\infty}$ is obtained, which, together with Eq. (\ref{proof:kernel_perturb:subgoal}), results in Eq. (\ref{proof:kernel_perturb:goal}).
This completes the proof.

\end{proof}

\begin{lemma} [RGM Distribution Perturbation]
\label{proof:prob_perturb}
Given an RGM $\Gamma = (W, P, f)$ and its perturbation $\Gamma' = (W, P_{\tau}, f')$ by distribution deformation $\tau$,   and the discrete and continuous versions of an MPNN as in Def. \ref{def:mpnn} and Def. \ref{def:cmpnn}, suppose Assumptions \ref{ass:space}-\ref{ass:tau_Wconstants} and \ref{ass:tau_Pconstants}-\ref{ass:mpnn} hold.  
    Define the following constants:
    \begin{align}
        C_4^{(T)} & := \prod_{l=1}^T L_{\Psi^{(l)}}  \left( 1 + \frac{W_{\max}L_{\Phi^{(l)}} }{d_{\min}} \right), \nonumber \\
        \label{constant:C5}
        C_5^{(T)} & := C_1^{(T)} + C_2^{(T)} \lVert f \rVert_{\infty}.
    \end{align}
    Then, the  hypothesis change induced by distribution deformation is bounded by the following:
    \begin{equation}
        \label{proof:prob_perturb:goal}
        \left\lVert \Bar{h}_{W,P}(f) - \Bar{h}_{W,P_\tau}(f') \right\rVert_{\infty} \leq C_5^{(T)} N_{P_\tau} + C_4^{(T)} C_{P_\tau} \lVert f - f' \rVert_{\infty},
    \end{equation}
where $N_{P_\tau}$ and $C_{P_\tau}$ are defined in Assumption \ref{ass:tau_Pconstants}.
\end{lemma}

\begin{proof}
Applying  the continuous pooling operation and cMPNN structure in Def. \ref{def:cmpnn}, the fact that $N_{P_\tau} = \|q_\tau(x)-1\|_{\infty}$ and 
$q_\tau =\frac{d P_\tau}{d P} (x)\leq C_{P_\tau}  $ as in Assumption \ref{ass:tau_Pconstants}, and the result of Lem. \ref{lem:boundedf} for bounding $\left\lVert f^{(l)} \right\rVert_{\infty}$, we expand the following:
\begin{align}
    & \left\lVert \Bar{h}_{W,P}(f) - \Bar{h}_{W,P_\tau}(f') \right\rVert_{\infty} \nonumber \\
    \leq \; & \left\lVert \Bar{h}_{W,P}(f) - \Bar{h}_{W,P_\tau}(f) \right\rVert_{\infty} + \left\lVert \Bar{h}_{W,P_\tau}(f) - \Bar{h}_{W,P_\tau}(f') \right\rVert_{\infty} \nonumber \\
    = \; &  \left\lVert \int_{\mathcal{X}} h^{(T)}\circ f(x) \, dP(x) - \int_{\mathcal{X}} h^{(T)}\circ f(x) \, dP_\tau(x) \right\rVert_{\infty} \nonumber \\
    + \; & \left\lVert \int_{\mathcal{X}} h^{(T)}\circ f(x) \, dP_\tau(x) - \int_{\mathcal{X}} h^{(T)}\circ f'(x) \, dP_\tau(x) \right\rVert_{\infty} \nonumber \\
    \leq \; & \left\lVert \int_{\mathcal{X}} f^{(T)}(x)  \cdot \left(q_\tau(x)-1\right) \, dP(x) \right\rVert_{\infty}  +  \int_{\mathcal{X}} \left\lVert f^{(T)}(x) - f^{(T)'}(x)\right\rVert_{\infty} \, dP_\tau(x) \nonumber \\
    \leq  \; & \left\lVert q_\tau(x)-1\right\rVert_{\infty}\cdot \left\lVert f^{(T)} \right\rVert_{\infty}  +  \int_{\mathcal{X}} \left\lVert f^{(T)}(x) - f^{(T)'}(x)\right\rVert_{\infty} q_\tau \, dP(x) \nonumber \\
    \label{proof:prob_perturb:1}
    \leq \; & N_{P_\tau} \left( C_1^{(T)} + C_2^{(T)} \lVert f \rVert_{\infty} \right)  +    C_{P_\tau} \left\lVert f^{(T)} - f^{(T)'} \right\rVert_{\infty}.   
\end{align}
We have shown in the proof of Lemma \ref{proof:kernel_perturb} on bounding $\lVert f^{(l+1)} - (f^{(l+1)})' \rVert_{\infty}$, i.e., Eq. (\ref{proof:kernel_perturb:recursive}).
This then results in an upper bound for $\left\lVert f^{(T)} - f^{(T)'}\right\rVert_{\infty}$ by applying the recursion Lem. \ref{lem:recursive} with $\Delta_t =\left\lVert f^{(t)} - f^{(t)'}  \right\rVert_{\infty}$, i.e., 
\begin{equation}
    \left\lVert f^{(T)} - f^{(T)'} \right\rVert_{\infty}  \leq  C_4 \lVert f - f' \rVert_{\infty}.
\end{equation}
Substituting the above  into Eq. (\ref{proof:prob_perturb:1}), Eq. (\ref{proof:prob_perturb:goal}) is obtained. This completes the proof.

\end{proof}

\begin{theorem}[RGM Perturbation]
\label{proof:perturb_RGM}
    Given an RGM $\Gamma = (W, P, f)$ and its perturbation $\Gamma' = (W_{\tau}, P_{\tau}, f')$ by both kernel and distribution deformation,   and an MPNN under mean aggregation, suppose Assumptions \ref{ass:space}-\ref{ass:tau_Wconstants} and \ref{ass:tau_Pconstants}-\ref{ass:mpnn} hold.
    Then, the  hypothesis change induced by the deformation is bounded by the following:
    \begin{align}
        \label{proof:perturb_RGM:goal}
        \lVert \Bar{h}_{W,P}(f) - \Bar{h}_{W_\tau,P_\tau}(f') \rVert_{\infty}
        & \leq C_3  C_{\nabla w} \lVert \nabla \tau \rVert_{\infty} + C_4 (1+C_{P_\tau})  \lVert f - f' \rVert_{\infty} + C_5  N_{P_\tau}
    \end{align}    
\end{theorem}

\begin{proof}

Expanding the l.h.s. of Eq. (\ref{proof:perturb_RGM:goal}) and applying results from Lem. \ref{proof:kernel_perturb} and \ref{proof:prob_perturb}, it has
\begin{align}
    \left\lVert \Bar{h}_{W,P}(f) - \Bar{h}_{W_\tau,P_\tau}(f') \right\rVert_{\infty}
    & \leq \left\lVert \Bar{h}_{W,P}(f) - \Bar{h}_{W_\tau,P}(f') \right\rVert_{\infty} + \left\lVert \Bar{h}_{W_\tau,P}(f') - \Bar{h}_{W_\tau,P_\tau}(f') \right\rVert_{\infty} \nonumber \\
    & \leq C_3  C_{\nabla w} \lVert \nabla \tau \rVert_{\infty} + C_4  \lVert f - f' \rVert_{\infty} + C_5N_{P_\tau} +C_4  C_{P_\tau} \lVert f - f' \rVert_{\infty} \nonumber \\
    & =  C_3  C_{\nabla w} \lVert \nabla \tau \rVert_{\infty} + C_4 (1+C_{P_\tau})  \lVert f - f' \rVert_{\infty} + C_5  N_{P_\tau}.
\end{align}
This completes the proof.

\end{proof}

\begin{theorem}[Weight Perturbation]
\label{proof:perturb_weights} 
Given an RGM $\Gamma = (W, P, f)$ and the discrete and continuous versions of an MPNN as in Def. \ref{def:mpnn} and Def. \ref{def:cmpnn}, where the message and update functions  are instantiated as $L$-layer MLPs as in Def \ref{def:layermapping}. Perturb the MPNN weight matrices by Def. \ref{def:perturbation}. Suppose Assumption \ref{ass:msg_update} holds. 
Then, the hypothesis change induced by weight perturbation is bounded by 
    \begin{equation}
    \label{goal:perturb_weights}
        \left\lVert \Bar{h}_{W,P}(f) - \widetilde{\Bar{h}}_{W,P}(f) \right\rVert_{\infty} \leq \sum_{l=1}^{T} \left(  C_{T_2}^{(l)} + C_{T_4}^{(l)}\right) \prod_{l'=l+1}^{T}  \left(L_{\Psi^{(l')}} +C_{T_3}^{(l')}\right).
    \end{equation}
\end{theorem}
where the layer-specific constants are computed by
\begin{align}
\label{constant:C_phi}
   & C_{\Phi}^{(l+1)} :=  L_{\Phi^{(l+1)}} \left( C_1^{(l)} + C_2^{(l)} \left\lVert f \right\rVert_{\infty} \right) + \left\lVert \Phi^{(l+1)}(0,0) \right\rVert_{\infty}, \\
%
%
\label{constant:C1_tilde}
  & \tilde{C_1}^{(l)}  :=  \sum_{t=1}^{l} \left(
L_{\tilde{\Psi}^{(t)}} D_{\tilde{\Phi}^{(t)}(0,0)}  + D_{\tilde{\Psi}^{(t)}(0,0)} \right) \prod_{l'=t+1}^{l}  L_{\tilde{\Psi}^{(l')}} \left( 1 + L_{\tilde{\Phi}^{(l')}} \right), \\
\label{constant:C2_tilde}
   & \tilde{C_2}^{(l)}  :=   \prod_{t=1}^{l} L_{\tilde{\Psi}^{(t)}} \left( 1 +  L_{\tilde{\Phi}^{(t)}} \right), \\
\label{constant:C_phi_tilde}
   & \tilde{C}_{\Phi}^{(l+1)} :=  L_{\Phi^{(l+1)}} \left(\tilde{C_1}^{(l)} + \tilde{C_2}^{(l)} \lVert f \rVert_{\infty}\right) + \left\lVert \Phi^{(l+1)}(0, 0) \right\rVert_{\infty},\\
\label{proof:perturb_weights:T2}
    & C_{T_2}^{(l+1)}:=   L_{\Psi^{(l+1)}}\Delta_{\Phi}^{(l+1)} C_{\Phi}^{(l+1)} \frac{\lVert W \rVert_{\infty}}{d_{\min}},   \\
\label{proof:perturb_weights:T3}
    & C_{T_3}^{(l+1)} :=   L_{\Psi^{(l+1)}}   L_{\tilde{\Phi}^{(l)}} \frac{\lVert W \rVert_{\infty}}{d_{\min}}, \\
\label{constant:C5_tilde}
& \tilde{C_5}^{(l+1)}  := \max\left(\tilde{C_1}^{(l+1)} + \tilde{C_2}^{(l+1)} \lVert f \rVert_{\infty}, \frac{\left\lVert W \right\rVert_{\infty}}{d_{\min}} \left( \Delta_{\Phi}^{(l+1)} + 1 \right) \tilde{C}_{\Phi}^{(l+1)} \right), \\
\label{proof:perturb_weights:T4}
   & C_{T_4}^{(l+1)} :=   \Delta_{\Psi}^{(l+1)}   L_{\Psi^{(l+1)}}  \tilde{C_5}^{(l)} + \Delta_{\Psi}^{(l+1)} \left\lVert \Psi^{(l+1)}(0, 0) \right\rVert_{\infty}, 
\end{align}
with
\begin{align}
\label{constant:delta_phi}
   & \Delta_{\Phi}^{(l+1)}   =  \sqrt{H_{l+1}} \left( \frac{\beta}{\alpha} \kappa^2 \right)^L \left[ \prod_{k=1}^L \left( 1 + \frac{\left\lVert \Delta \Theta_{\Phi_k^{(l+1)}} \right\rVert_F}{\left\lVert \Theta_{\Phi_k^{(l+1)}} \right\rVert_F} \right) - 1 \right], \\
\label{constant:delta_psi}
   & \Delta_{\Psi}^{(l+1)} =  \sqrt{F_{l+1}} \left( \frac{\beta}{\alpha} \kappa^2 \right)^L \left[ \prod_{k=1}^L \left( 1 + \frac{\lVert \Delta \Theta_{\Psi_k^{(l+1)}} \rVert_F}{\lVert \Theta_{\Psi_k^{(l+1)}} \rVert_F} \right) - 1 \right]  , \\
   & L_{\tilde{\Psi}^{(l)}} =   L_{\sigma}^{L}   \prod_{t=1}^{L} \left( \left\lVert \Delta \Theta_{\Psi_t^{(l)}} \right\rVert_F + \left\lVert \Theta_{\Psi_t^{(l)}} \right\rVert_F \right), \\
   & L_{\tilde{\Phi}^{(l)}} = L_{\sigma}^{L}   \prod_{t=1}^{L} \left( \left\lVert \Delta \Theta_{\Phi_t^{(l)}} \right\rVert_F + \left\lVert \Theta_{\Phi_t^{(l)}} \right\rVert_F \right), \\
   & D_{\tilde{\Psi}^{(t)}(0,0)}  =  \left( \Delta_{\Psi}^{(t)} + 1 \right)\left\lVert \Psi^{(t)}(0, 0) \right\rVert_{\infty}, \\
   & D_{\tilde{\Phi}^{(t)}(0,0)}  =  \left( \Delta_{\Phi}^{(t)} + 1 \right)\left\lVert \Phi^{(t)}(0, 0) \right\rVert_{\infty}.
\end{align}

\begin{proof}

Applying the continuous pooling operation and cMPNN structure in Def. \ref{def:cmpnn}, we expand the following:
\begin{align}
\label{proof:perturb_weights:0}
    \left\lVert \Bar{h}_{W,P}(f) - \widetilde{\Bar{h}}_{W,P}(f) \right\rVert_{\infty}
    & = \left\lVert \int_{\mathcal{X}} h_{W,P}^{(T)}\circ f(x) \, dP(x) - \int_{\mathcal{X}} \tilde{h}_{W,P}^{(T)}\circ f(x) \, dP(x) \right\rVert_{\infty} \nonumber \\
    & = \left\lVert \int_{\mathcal{X}} \left( h_{W,P}^{(T)}\circ f(x) - \tilde{h}_{W,P}^{(T)}\circ f(x) \right) \, dP(x) \right\rVert_{\infty} \leq \left\lVert f^{(T)} - \tilde{f}^{(T)} \right\rVert_{\infty}
\end{align}
where we define $f^{(l+1)}$ be the output function of layer $l+1$.  
For the $(l+1)$-th layer, it has
\begin{align}
\label{proof:perturb_weights:1}
    \left\lVert f^{(l+1)} - \tilde{f}^{(l+1)} \right\rVert_{\infty}
    & = \Psi^{(l+1)} \left( f^{(l)}, M_W^{ \Phi^{(l+1)}, f^{(l)}} \right) - \tilde{\Psi}^{(l+1)} \left( \tilde{f}^{(l)}, M_W^{ \tilde{\Phi}^{(l+1)}, \tilde{f}^{(l)}} \right) \nonumber \\
    & \leq \underbrace{\left\lVert \Psi^{(l+1)}\left( f^{(l)}, M_W^{ \Phi^{(l+1)}, f^{(l)}}  \right) - \Psi^{(l+1)}\left( \tilde{f}^{(l)}, M_W^{ \Phi^{(l+1)}, f^{(l)}}  \right) \right\rVert_{\infty}}_{T_1} \nonumber \\
    & + \underbrace{\left\lVert \Psi^{(l+1)}\left( \tilde{f}^{(l)}, M_W^{ \Phi^{(l+1)}, f^{(l)}}  \right) - \Psi^{(l+1)}\left( \tilde{f}^{(l)}, M_W^{ \tilde{\Phi}^{(l+1)}, f^{(l)}}  \right) \right\rVert_{\infty}}_{T_2}\nonumber \\
    & + \underbrace{\left\lVert  \Psi^{(l+1)}\left( \tilde{f}^{(l)}, M_W^{ \tilde{\Phi}^{(l+1)}, f^{(l)}}  \right) - \Psi^{(l+1)}\left( \tilde{f}^{(l)}, M_W^{ \tilde{\Phi}^{(l+1)}, \tilde{f}^{(l)}}  \right) \right\rVert_{\infty}}_{T_3} \nonumber \\
    & + \underbrace{\left\lVert   \Psi^{(l+1)}\left( \tilde{f}^{(l)}, M_W^{ \tilde{\Phi}^{(l+1)}, \tilde{f}^{(l)}}  \right) - \tilde{\Psi}^{(l+1)}\left( \tilde{f}^{(l)}, M_W^{ \tilde{\Phi}^{(l+1)}, \tilde{f}^{(l)}}  \right) \right\rVert_{\infty}}_{T_4}.
\end{align}
Next, we bound each of the terms $\{T_i\}_{i=1}^4$ separately using results from Thm. \ref{thm:MLP}.

The second input variables of the two functions in $T_1$ are identical. Thus,  it has
\begin{equation}
\label{proof:perturb_weights:T1}
    T_1    \leq L_{\Psi^{(l+1)}} \left\lVert f^{(l)} - \tilde{f}^{(l)} \right\rVert_{\infty}.
\end{equation}
The first input variables of the two functions in $T_2$ are identical, thus we have
\begin{align}
    T_2   & \leq L_{\Psi^{(l+1)}} \left\lVert \int_{\mathcal{X}} \frac{W(x,y)}{d_W(x)}   \Phi^{(l+1)}\left( f^{(l)}(x), f^{(l)}(y) \right)\, dP(y) - \right. \nonumber \\
    & \left. \int_{\mathcal{X}} \frac{W(x,y)}{d_W(x)}\tilde{\Phi}^{(l+1)}\left( f^{(l)}(x), f^{(l)}(y) \right)  \, dP(y) \right\rVert_{\infty} \nonumber \\
    & \leq L_{\Psi^{(l+1)}} \frac{\lVert W \rVert_{\infty}}{d_{\min}} \int_{\mathcal{X}} \left\lVert \Phi^{(l+1)}\left( f^{(l)}(x), f^{(l)}(y) \right) - \tilde{\Phi}^{(l+1)}\left( f^{(l)}(x), f^{(l)}(y) \right) \right\rVert_{\infty} \, dP(y).
\end{align}
Applying norm inequality $\|v\|_{\infty} \leq \|v\|_2\leq \sqrt{d}\|v\|_{\infty} $ for $v\in \mathbb{R}^d$ and Thm. \ref{thm:MLP}, it has
%
\begin{align}
    T_2   
    & \leq L_{\Psi^{(l+1)}} \frac{\lVert W \rVert_{\infty}}{d_{\min}}\int_{\mathcal{X}} \left\lVert \Phi^{(l+1)}\left( f^{(l)}(x), f^{(l)}(y) \right) - \tilde{\Phi}^{(l+1)}\left( f^{(l)}(x), f^{(l)}(y) \right) \right\rVert  \, dP(y) \nonumber \\
    & \leq L_{\Psi^{(l+1)}} \frac{\lVert W \rVert_{\infty}}{d_{\min}}  \left( \frac{\beta}{\alpha} \kappa^2 \right)^L \left[ \prod_{k=1}^L \left( 1 + \frac{\left\lVert \Delta \Theta_{\Phi_k^{(l+1)}} \right\rVert_F}{\left\lVert \Theta_{\Phi_k^{(l+1)}} \right\rVert_F} \right) - 1 \right] \nonumber \\   &\cdot \int_{\mathcal{X}} \left\lVert \Phi^{(l+1)}\left( f^{(l)}(x), f^{(l)}(y) \right) \right\rVert \, dP(y) \nonumber \\
    & \leq L_{\Psi^{(l+1)}} \frac{\lVert W \rVert_{\infty}}{d_{\min}} \Delta_{\Phi}^{(l+1)} \int_{\mathcal{X}} \left\lVert \Phi^{(l+1)}\left( f^{(l)}(x), f^{(l)}(y) \right) \right\rVert_{\infty} \, dP(y) \nonumber \\
    & \leq L_{\Psi^{(l+1)}} \frac{\lVert W \rVert_{\infty}}{d_{\min}}\Delta_{\Phi}^{(l+1)} \int_{\mathcal{X}}\left( L_{\Phi^{(l+1)}} \left\lVert f^{(l)}(y) \right\rVert_{\infty} + \lVert \Phi^{(l+1)}(0,0) \rVert_{\infty} \right)\, dP(y).
\end{align}
Further applying the result from Lem. \ref{lem:boundedf} as in Eq. (\ref{lem:contlayernorm:goal}), it has
\begin{align}
\label{eq:error_bound_T2}
   T_2 & \leq L_{\Psi^{(l+1)}} \frac{\lVert W \rVert_{\infty}}{d_{\min}} \Delta_{\Phi}^{(l+1)} \int_{\mathcal{X}}\left( L_{\Phi^{(l+1)}} \left( C_1^{(l)} + C_2^{(l)} \lVert f \rVert_{\infty} \right) + \lVert \Phi^{(l+1)}(0,0) \rVert_{\infty} \right) \, dP(y) \nonumber \\
    & =   L_{\Psi^{(l+1)}}\lVert W \rVert_{\infty}\Delta_{\Phi}^{(l+1)} C_{\Phi}^{(l+1)} d_{\min}^{-1}  = C_{T_2}^{(l+1)}.
\end{align}

To bound $T_3$, a similar strategy is applied, and it has
\begin{align}
\label{proof:perturb_weights:T3:1}
    T_3 &  \leq L_{\Psi^{(l+1)}} \left\lVert \int_{\mathcal{X}} \frac{W(x,y)}{d_W(x)} \left[ \tilde{\Phi}^{(l+1)}\left( f^{(l)}(x), f^{(l)}(y) \right) - \tilde{\Phi}^{(l+1)}\left( \tilde{f}^{(l)}(x), \tilde{f}^{(l)}(y) \right) \right] \, dP(y) \right\rVert_{\infty} \nonumber \\
    & \leq L_{\Psi^{(l+1)}} \frac{\lVert W \rVert_{\infty}}{d_{\min}} L_{\tilde{\Phi}}^{(l+1)} \left\lVert f^{(l)} - \tilde{f}^{(l)} \right\rVert_{\infty}.
\end{align}
Note that $L_{\tilde{\Phi}^{(l+1)}}$  in Eq. (\ref{proof:perturb_weights:T3:1})  is unknown due to the change of neural network weights of $\Phi^{(l+1)}$. We now seek to bound this  new Lipschitz constant after weight perturbation. 
Both $\Psi$ and $\Phi$ are assumed to be $L$-layer MLPs described by    $\circ_{t=1}^L \left(\sigma_t \circ W_{t}(x)\right)$ as in Lemma \ref{proof:MLP_Lip}, where the Lipschitz constant of the used activation function is denoted by $L_{\sigma}$.
According to the result from Lemma \ref{proof:MLP_Lip} as in Eq. (\ref{mlp_lipschitz}), the Lipschitz constant of such an MLP is bounded by
\begin{equation}
    \label{mlp_lipschitz_weightperturb1}
    L_{\textmd{MLP}} = L_{\sigma}^L  \prod_{t=1}^L \lambda_t \leq L_{\sigma}^L  \prod_{t=1}^L \|W_t\|_F, 
\end{equation}
where the last inequality results from the fact that the maximum eigenvalue of a matrix is bounded by its Frobenius norm.
After the weight perturbation $\Delta W_t = \tilde{W}_{t} - W_{t}$,  it has
\begin{equation}
    \label{mlp_lipschitz_weightperturb2}
    \tilde{L}_{\textmd{MLP}} \leq L_{\sigma}^L  \prod_{t=1}^L \left\|\tilde{W}_t \right\|_F = L_{\sigma}^L  \prod_{t=1}^L \|W_t + \Delta W_t\|_F \leq L_{\sigma}^L  \prod_{t=1}^L \left(\|  \Delta W_t\|_F + \|W_t\|_F\right). 
\end{equation}
Applying the above result over the perturbed MLPs of $\Psi^{(l)} $ and $\Phi^{(l)}$, it has
\begin{align}
    \label{constant:L_psi_tilde}
    L_{\tilde{\Psi}^{(l)}} \leq \;& L_{\sigma}^{L}   \prod_{t=1}^{L} \left( \left\lVert \Delta \Theta_{\Psi_t^{(l)}} \right\rVert_F + \left\lVert \Theta_{\Psi_t^{(l)}} \right\rVert_F \right), \\
    \label{constant:L_phi_tilde}
    L_{\tilde{\Phi}^{(l)}} \leq\;& L_{\sigma}^{L}   \prod_{t=1}^{L} \left( \left\lVert \Delta \Theta_{\Phi_t^{(l)}} \right\rVert_F + \left\lVert \Theta_{\Phi_t^{(l)}} \right\rVert_F \right).
\end{align}
Incorporating Eq. (\ref{constant:L_phi_tilde}) into Eq. (\ref{proof:perturb_weights:T3:1}),   we can conclude an upper bound of $T_3$:
\begin{equation}
\label{eq:error_bound_T3}
    T_3 \leq L_{\Psi^{(l+1)}} \frac{\lVert W \rVert_{\infty}}{d_{\min}} L_{\sigma}^{L}   \prod_{t=1}^{L} \left( \left\lVert \Delta \Theta_{\Phi_t^{(l)}} \right\rVert_F + \left\lVert \Theta_{\Phi_t^{(l)}} \right\rVert_F \right) \left\lVert f^{(l)} - \tilde{f}^{(l)} \right\rVert_{\infty} = C_{T_3}^{(l+1)}\left\lVert f^{(l)} - \tilde{f}^{(l)} \right\rVert_{\infty}.
\end{equation}

We use Thm. \ref{thm:MLP}, as well as norm inequalities, to bound $T_4$:
\begin{align}
\label{proof:perturb_weights:T4:1}
    T_4 & \leq \left\lVert   \Psi^{(l+1)}\left( \tilde{f}^{(l)}, M_W^{ \tilde{\Phi}^{(l+1)}, \tilde{f}^{(l)}}  \right) - \tilde{\Psi}^{(l+1)}\left( \tilde{f}^{(l)}, M_W^{ \tilde{\Phi}^{(l+1)}, \tilde{f}^{(l)}}  \right) \right\rVert \nonumber \\
    & \leq  \Delta_{\Psi}^{(l+1)} \left\lVert \Psi^{(l+1)}\left( \tilde{f}^{(l)}, M_W^{ \tilde{\Phi}^{(l+1)}, \tilde{f}^{(l)}}  \right) \right\rVert_{\infty} \nonumber \\
    & \leq \Delta_{\Psi}^{(l+1)} \left(\left\lVert \Psi^{(l+1)}\left( \tilde{f}^{(l)}, M_W^{ \tilde{\Phi}^{(l+1)}, \tilde{f}^{(l)}}  \right) - \Psi^{(l+1)}(0, 0) \right\rVert_{\infty} + \left\lVert \Psi^{(l+1)}(0, 0) \right\rVert_{\infty} \right) \nonumber \\
    & \leq \Delta_{\Psi}^{(l+1)}  L_{\Psi^{(l+1)}} \max\left( \left\lVert \tilde{f}^{(l)} \right\rVert_{\infty} , \left\lVert \int_{\mathcal{X}} \frac{W(x,y)}{d_W(x)} \tilde{\Phi}^{(l+1)}\left( \tilde{f}^{(l)}(x), \tilde{f}^{(l)}(y) \right) \, dP(y) \right\rVert_{\infty} \right) \nonumber \\
    & + \Delta_{\Psi}^{(l+1)}\left\lVert \Psi^{(l+1)}(0, 0) \right\rVert_{\infty}  \nonumber \\
    & \leq \Delta_{\Psi}^{(l+1)}  L_{\Psi^{(l+1)}} \max\left( \left\lVert \tilde{f}^{(l)} \right\rVert_{\infty}, \frac{\left\lVert W \right\rVert_{\infty}}{d_{\min}} \left\lVert \tilde{\Phi}^{(l+1)}\left(\tilde{f}^{(l)}, \tilde{f}^{(l)}\right) \right\rVert_{\infty} \right) + \Delta_{\Psi}^{(l+1)} \left\lVert \Psi^{(l+1)}(0, 0) \right\rVert_{\infty}  ,
\end{align}
We further bound the term $\left\lVert \tilde{\Phi}^{(l+1)}\left(\tilde{f}^{(l)}, \tilde{f}^{(l)}\right) \right\rVert_{\infty}$ in the above equation as follows:
\begin{align}
\label{proof:perturb_weights:T4:2}
    & \left\lVert \tilde{\Phi}^{(l+1)}\left(\tilde{f}^{(l)}, \tilde{f}^{(l)}\right) \right\rVert_{\infty} \nonumber \\
  \leq\;  &  \left\lVert \tilde{\Phi}^{(l+1)}\left(\tilde{f}^{(l)}, \tilde{f}^{(l)}\right) - \Phi^{(l+1)}\left(\tilde{f}^{(l)}, \tilde{f}^{(l)}\right) \right\rVert_{\infty} + \left\lVert \Phi^{(l+1)}\left(\tilde{f}^{(l)}, \tilde{f}^{(l)}\right) \right\rVert_{\infty} \nonumber \\
   \leq\; &  \Delta_{\Phi}^{(l+1)} \left\lVert \Phi^{(l+1)}\left(\tilde{f}^{(l)}, \tilde{f}^{(l)}\right) \right\rVert_{\infty} + \left\lVert \Phi^{(l+1)}\left(\tilde{f}^{(l)}, \tilde{f}^{(l)}\right) \right\rVert_{\infty} \nonumber \\
 \leq\;   &  \left( \Delta_{\Phi}^{(l+1)} + 1 \right) \left(  L_{\Phi^{(l+1)}} \left\lVert \tilde{f}^{(l)} \right\rVert_{\infty} + \left\lVert \Phi^{(l+1)}(0, 0) \right\rVert_{\infty} \right).
\end{align}
Both Eq. (\ref{proof:perturb_weights:T4:1}) and (\ref{proof:perturb_weights:T4:2})   contain the term $\lVert \tilde{f}^{(l)} \rVert_{\infty}$, which is further bounded by applying  the result from Lem. \ref{lem:boundedf} as in Eq. (\ref{lem:contlayernorm:goal}), as
\begin{equation}
\label{upper_bound_f_tilde}
   \left \lVert \tilde{f}^{(l)} \right\rVert_{\infty} \leq \tilde{C_1}^{(l)} + \tilde{C_2}^{(l)}\left\lVert \tilde{f}^{(0)} \right\rVert_{\infty}  = \tilde{C_1}^{(l)} + \tilde{C_2}^{(l)}\left\lVert f \right\rVert_{\infty} ,
\end{equation}
where 
\begin{align}
    \tilde{C_1}^{(l)} &= \sum_{t=1}^{l} \left(
L_{\tilde{\Psi}^{(t)}} \left\lVert \tilde{\Phi}^{(t)}(0,0) \right\rVert_\infty+ \left\lVert \tilde{\Psi}^{(t)}(0,0) \right\rVert_\infty \right) \prod_{l'=t+1}^{l}  L_{\tilde{\Psi}^{(l')}} \left( 1 + L_{\tilde{\Phi}^{(l')}} \right), \\
    \tilde{C_2}^{(l)} & = \prod_{t=1}^{l} L_{\tilde{\Psi}^{(t)}} \left( 1 +  L_{\tilde{\Phi}^{(t)}} \right).
\end{align}
Applying Eq. (\ref{proof:perturb_weights:T4:2}), it has
\begin{equation}
\label{zero_phi_tilde}
     \left\lVert \tilde{\Phi}^{(t)}(0,0) \right\rVert_\infty \leq \left( \Delta_{\Phi}^{(t)} + 1 \right)\left\lVert \Phi^{(t)}(0, 0) \right\rVert_{\infty}.\\
\end{equation}
Similarly, we can derive the same result for $\Psi$, as 
\begin{equation}
\label{zero_psi_tilde}
     \left\lVert \tilde{\Psi}^{(t)}(0,0) \right\rVert_\infty \leq \left( \Delta_{\Psi}^{(t)} + 1 \right)\left\lVert \Psi^{(t)}(0, 0) \right\rVert_{\infty}.\\
\end{equation}
The constant $\tilde{C_1}^{(l)}$ can be inflated,   expressed in terms of the original network before perturbation:
Also, the Lipschitz constants $L_{\tilde{\Phi}^{(t)}}$ and $L_{\tilde{\Psi}^{(t)}}$ are upper bounded by Eq. (\ref{constant:L_phi_tilde}) and Eq. (\ref{constant:L_psi_tilde}) respectively. These result in further inflated expression of $\tilde{C_1}^{(l)}$ and $\tilde{C_2}^{(l)}$ as in the theorem body.
Substituting Eq. (\ref{upper_bound_f_tilde}) back to Eq. (\ref{proof:perturb_weights:T4:2}) and then Eq. (\ref{proof:perturb_weights:T4:1}), we obtain the following:   
\begin{equation}
\label{eq:error_bound_T4}
    T_4 \leq C_{T_4}^{(l+1)}.
\end{equation}

Finally, combining  the upper bounds of $\{T_i\}_{i=1}^4$, as in Eq. (\ref{proof:perturb_weights:T1}), (\ref{eq:error_bound_T2}), (\ref{eq:error_bound_T3}), (\ref{eq:error_bound_T4}), we can conclude the following upper bound for Eq. (\ref{proof:perturb_weights:1}):
\begin{align}
\label{proof:perturb_weights:2}
    \left\lVert f^{(l+1)} - \tilde{f}^{(l+1)} \right\rVert_{\infty}
    & \leq L_{\Psi^{(l+1)}} \left\lVert f^{(l)} - \tilde{f}^{(l)} \right\rVert_{\infty}+ C_{T_2}^{(l+1)}  +  C_{T_3}^{(l+1)} \left\lVert f^{(l)} - \tilde{f}^{(l)} \right\rVert_{\infty}  + C_{T_4}^{(l+1)} \nonumber\\
    &=     \left(L_{\Psi^{(l+1)}} +C_{T_3}^{(l+1)} \right) \left\lVert f^{(l)} - \tilde{f}^{(l)} \right\rVert_{\infty} + C_{T_2}^{(l+1)}  + C_{T_4}^{(l+1)}.
\end{align}
Applying the recursive result from Lem. \ref{lem:recursive} to the above Eq. (\ref{proof:perturb_weights:2}) and  recalling that $\lVert f^{(0)} - \tilde{f}^{(0)} \rVert_{\infty} = 0$, we have:
\begin{align}
\label{proof:perturb_weights:3}
    \left\lVert f^{(l+1)} - \tilde{f}^{(l+1)} \right\rVert_{\infty}
    & \leq \sum_{l=1}^{T} \left(  C_{T_2}^{(l)} + C_{T_4}^{(l)}\right) \prod_{l'=l+1}^{T}  \left(L_{\Psi^{(l')}} +C_{T_3}^{(l')}\right).
\end{align}
Substituting Eq. (\ref{proof:perturb_weights:3}) back to Eq. (\ref{proof:perturb_weights:0}), this completes the proof.

\end{proof}

\begin{theorem}[Optimization Error]
\label{proof:optimization}
Given an RGM $\Gamma = (W, P, f)$ and the discrete and continuous versions of an MPNN as in Def. \ref{def:mpnn} and Def. \ref{def:cmpnn}, where the message and update functions are instantiated as $L$-layer MLPs as in Def \ref{def:layermapping}. 
Perturb the RGM by both kernel and distribution deformation.
Perturb the MPNN weight matrices by Def. \ref{def:perturbation}. 
Suppose Assumptions  \ref{ass:space}-\ref{ass:tau_Wconstants}, \ref{ass:tau_Pconstants}-\ref{ass:mpnn} and  \ref{ass:msg_update} hold. 
Define an additional constant
\begin{equation}
\label{constant:C3_perturb}
     \tilde{C_3}^{(T)} := \bigg( \frac{1}{d_{\min}} + \frac{W_{\max}}{d_{\min}^2} \bigg) \sum_{l=1}^T L_{\tilde{\Psi}^{(l)}} \left( L_{\tilde{\Phi}^{(l)}}  \left( \tilde{C_1}^{(l-1)} + \tilde{C_2}^{(l-1)} \lVert f \rVert_{\infty} \right) + D_{\tilde{\Phi}^{(t)}(0,0)} \right). 
\end{equation}

%
Then, the hypothesis change induced by RGM deformation and weight perturbation is bounded by the following:
\begin{equation}
    \left\lVert \Bar{h}_{W,P}(f) - \tilde{\Bar{h}}_{W_\tau,P_\tau}(f) \right\rVert_{\infty} \leq \Delta_{\Gamma,\Theta} 
\end{equation}
where
{\small
    \begin{equation}
        \label{proof:perturb_bound:goal}
        \Delta_{\Gamma,\Theta}  =   \sum_{l=1}^{T}\left (C_{T_2}^{(l)} + C_{T_4}^{(l)}\right) \prod_{l'=l+1}^{T} L_{\Psi^{(l')}} \left( 1 +  C_{T_3}^{(l')}\right)   + \tilde{C_3}^{(T)}  C_{\nabla w} \lVert \nabla \tau \rVert_{\infty} +  \left(\tilde{C_1}^{(T)} + \tilde{C_2}^{(T)} \lVert f \rVert_{\infty}\right)N_{P_\tau}.  
    \end{equation}
}
\end{theorem}

\begin{proof}
Combining results from Theorems \ref{proof:perturb_RGM} and \ref{proof:perturb_weights}, we upper bound the hypothesis change induced by both RGM deformation and network weight perturbation, as below:
\begin{align}
    \label{proof:perturb_bound:1}
    &\left\lVert \Bar{h}_{W,P}(f) - \widetilde{\Bar{h}}_{W_\tau,P_\tau}(f) \right\rVert_{\infty} \\
    \leq \; & \left\lVert \Bar{h}_{W,P}(f) - \widetilde{\Bar{h}}_{W,P}(f) \right\rVert_{\infty} + \left\lVert \widetilde{\Bar{h}}_{W,P}(f) - \widetilde{\Bar{h}}_{W_\tau,P_\tau}(f) \right\rVert_{\infty} \nonumber \\
    \leq \; & \sum_{l=1}^{T} \left(  C_{T_2}^{(l)} + C_{T_4}^{(l)}\right) \prod_{l'=l+1}^{T}  \left(L_{\Psi^{(l')}} +C_{T_3}^{(l')}\right)  +  \tilde{C_3}^{(T)}  C_{\nabla w} \lVert \nabla \tau \rVert_{\infty} + \left(\tilde{C_1}^{(T)} + \tilde{C_2}^{(T)} \lVert f \rVert_{\infty}\right) N_{P_\tau}. \nonumber
\end{align}
where
\begin{equation}
 \tilde{C_3}^{(T)} := \bigg( \frac{1}{d_{\min}} + \frac{W_{\max}}{d_{\min}^2} \bigg) \sum_{l=1}^T L_{\tilde{\Psi}^{(l)}} \left( L_{\tilde{\Phi}^{(l)}}   \left( \tilde{C_1}^{(l-1)} + \tilde{C_2}^{(l-1)} \lVert f \rVert_{\infty} \right) + \left\lVert \Phi^{(l)}(0,0) \right\rVert_{\infty} \right).
\end{equation}
Applying the upper bounds derived for the Lipschitz constants  for the perturbed networks as in Eqs. (\ref{constant:L_psi_tilde}) and (\ref{constant:L_phi_tilde}), we  re-define the above constant  as in Eq. (\ref{constant:C3_perturb}). 
This completes the proof.
\end{proof}

\section{\texorpdfstring{Proof Theorem~\ref{thm:DA}}{Main Theorem: Proof}}
\label{app:main:DA}

We present our complete main result in the following, which is proved by combining key results from Prop. \ref{prop:general_DA}, Lem. \ref{proof:reverse2latent},  Lem. \ref{proof:norm2diff}, Thm. \ref{proof:convergence} and Thm. \ref{proof:optimization}.

\begin{theorem}[Main Result in Theorem \ref{thm:DA}]
\label{proof:DA}
    Following definitions in Appendix \ref{app:definitions} and suppose  assumptions in Appendix \ref{app:assumptions}    hold. Suppose Lem. \ref{proof:reverse2latent} holds with probability $1-\rho$, and events $\mathcal{E}_1, \mathcal{E}_2$ defined in Lemma \ref{lem:agg} each hold with probability at least $1-\rho$. Then, the following holds with a probability at least $1-3\rho$:
    \begin{align}
        \epsilon_T (h, g_D) \leq \epsilon_S(h, g_D) + \sqrt{\frac{C\Delta_D}{\lambda_{r}}}  \left( L_{NN} \left(\sqrt{\Delta_N} + \Delta_{\Gamma,\Theta} + \varepsilon_3 + \varepsilon_4\right)  + G_{NN} \right),
    \end{align}
    where $\Delta_D,\Delta_N,\Delta_{\Gamma,\Theta}$ are computed by Eq. (\ref{lem:wasserstein:goal}), (\ref{thm:convergence:goal}), (\ref{thm:perturb_bound:goal}).
\end{theorem}

\begin{proof}

By Prop. \ref{prop:general_DA} and $\mathcal{W}_1(\mu_S, \mu_T) \leq \mathcal{W}_2(\mu_S, \mu_T)$, we have:
\begin{equation}
    \label{proof:DA:1}
    \epsilon_T(h,g_D) \leq \epsilon_S(h,g_D) + \left\lVert \ell_{h,g_D} \right\rVert_{\mathcal{H}_{K_\ell}} \cdot \mathcal{W}_2(\mu_S, \mu_T).
\end{equation}
By Lem. \ref{proof:reverse2latent}, it has $\mathcal{W}_2^2 (\mu_S, \mu_T) \leq \Delta_D$, where the domain divergence upper bound $\Delta_D$ is  given in Eq. (\ref{proof:reverse2latent:goal}).
Substituting the above, together with the  Lem. \ref{proof:norm2diff} result in Eq. (\ref{eq:function_complexity}) and the error decomposition result in Eq. (\ref{norm2diff_new1}), back to Eq. (\ref{proof:DA:1}), it has 
\begin{align}
\label{proof:DA:5}
    \epsilon_T(h,g_D) \leq\;  & \epsilon_S(h,g_D) + \sqrt{\frac{C\Delta_D}{\lambda_{r}}} \left( L_{NN}  \underbrace{\left\| \Bar{h}_G(Z) -\Bar{h}_{W,P}(f)\right\|_{\infty}}_{\textmd{convergence error}} + \right. \nonumber\\
    & \left.  L_{NN} \underbrace{ \left\|  \Bar{h}_{W,P}(f) - \widetilde{\Bar{h}}_{W_\tau,P_\tau}(f)  \right\|_{\infty}}_{\textmd{optimization error}} + L_{NN}\varepsilon_3 + L_{NN}\varepsilon_4 + G_{NN} \right).
\end{align}
By substituting the upper bounds of the convergence error and optimization error derived in Thm. \ref{proof:convergence} and Thm. \ref{proof:optimization}, respectively, we conclude the final generalization DA bound.
This completes the proof.
\end{proof}

\section{\texorpdfstring{Proof of Corollary~\ref{coro:sufficient}}{Main Theorem: Sufficiency}}
\label{app:sufficient_condition}

In this section, we further analyze the result of Theorem \ref{thm:DA}, examining key factors that affect the bound tightness, for the case when the formal bias of the MPNN message and update functions is zero, i.e., $ \lVert \Phi^{(t)}(0,0) \rVert_\infty, \lVert \Psi^{(t)}(0,0) \rVert_\infty =0 $ and when approximation and label error is zero, i.e., $\varepsilon_3,\varepsilon_4=0$.
Our analysis  can be easily extended to cases with nonzero formal bias and to keep $\varepsilon_3 $ and $\varepsilon_4$ in the result.

Before proceeding further, we introduce a few  new notations.
Denote the maximum number of graph instances among all classes by $m_S$ for the source domain and  $m_T$ for the target domain, and define a hybrid sample size resulting from both graph number and node number as $M = \max(N_Sm_S, N_Tm_T)$.
Define the maximum domain shift by $  \mathcal{W}_2 =\arg\max_{j=1}^C   \mathcal{W}_2 \left(\hat{P}_S^j, \hat{P}_T^j\right)$.
For the MLP classifier, define two quantities relevant to its Lipschitz  constants of the activation functions and its maximum singular vectors of its weight matrices, i.e., $L_{\sigma } = \max_{l=1}^L L_{\sigma_l}$ and  $\lambda_M = \max_{l=1}^L\lambda_l$.
For the  MPNN message and update functions, define a quantity relevant to their Lipschitz  constants, i.e., $L_{P} =\max_{t=1}^T \max\left( L_{\Phi^{(t)}}, L_{\Psi^{(t)}}\right) $. 
For the labeling function that is modeled as a perturbed version of a given hypothesis,  define  the following quantities relevant to   the perturbation strength  achieved through changing neural network weights:
\begin{align}
    &\Delta_{M} = \max_{l=1}^L \frac{\lVert \Delta W_l \rVert_F}{\lVert W_l \rVert_F}, \\
    & \Delta_{\Theta } = \max_{t=1}^T\max_{l=1}^L \max\left(\frac{\left\lVert \Delta \Theta_{\Psi_l^{(t)}} \right\rVert_F}{\left\lVert \Theta_{\Psi_l^{(t)}} \right\rVert_F}, \frac{\left\lVert \Delta \Theta_{\Phi_l^{(t)}} \right\rVert_F}{\left\lVert \Theta_{\Phi_l^{(t)}} \right\rVert_F} \right), \\
    & D_{\Theta } = \max_{t=1}^T\max_{l=1}^L \max\left( \left\lVert  \Theta_{\Psi_l^{(t)}} \right\rVert_F, \left\lVert   \Theta_{\Phi_l^{(t)}} \right\rVert_F  \right).
\end{align}
To present the result in a neat fashion, define  the following set of quantities computed from properties of the RGM, MLP classifier and MPNN feature extractor:
\begin{align}
    & U_{L_{NN}} = L_{\sigma }^L\lambda_M^L, \\
    & U_{G_{NN}} =  \left( \frac{\beta}{\alpha} \kappa^2 \right)^L \left(\sum_{k=1}^L \binom{L}{k} \Delta_{M}^k \right), \\
    & U_{K_6} = 2TL_P^{T+2}(1+L_P)^T L_W^{\infty}  \left(  \sqrt{1 +  8\lVert W \rVert_{\infty}^2L_P^2 d_{\min}^{-2} }\right)^{T-1 }d_{\min}^{-1}, \nonumber \\
    & U_{K_7} = 2TL_P^2 \lVert W \rVert_{\infty}   \left(1 +  \lVert W \rVert_{\infty} d_{\min}^{-1}L_P\right) d_{\min}^{-1} \left(1 +  8\lVert W \rVert_{\infty}^2L_P^2 d_{\min}^{-2}\right)^{\frac{T-1}{2}}  , \nonumber \\
    & U_{K_9} = 0.8TL_P^{T+2} (1+L_P)^TC_{\mathcal{X}} \lVert W \rVert_{\infty} \left(  \sqrt{1 +  8\lVert W \rVert_{\infty}^2L_P^2 d_{\min}^{-2} }\right)^{T-1 }d_{\min}^{-1}, \\
    & U_{R_3} = 56T^2L_P^{2T+4} (1+L_P)^{2T} \left(  1 +  8\lVert W \rVert_{\infty}^2L_P^2 d_{\min}^{-2} \right)^{T-1 }   (L_W^{\infty})^2\left(\sqrt{\log(C_{\mathcal{X}})} + \sqrt{D_\mathcal{X}}\right)^2\lVert W \rVert_{\infty}^2 d_{\min}^{-4}, \\ 
    & U_{R_4} = 56TL_P^{2T+4}(1+L_P)^{2T}  \left(\sqrt{2}\left\lVert W \right\rVert_\infty + L_W^{\infty}\right)^2\lVert W \rVert_{\infty}^2  \left(   1 +  8\lVert W \rVert_{\infty}^2L_P^2 d_{\min}^{-2}  \right)^{T-1 } d_{\min}^{-4}, \\   
    & U_{S_2}= 14 U_{K_6}^2 + 14 U_{K_9}^2 \log(C_{\mathcal{X}}) + 56 \left(1 +  \lVert W \rVert_{\infty} d_{\min}^{-1}L_P\right)^2 + 7 L_P^{2l}(1+L_P)^{2l}C_{{\mathcal{X}}}^2 \log(C_{\mathcal{X}}),   \nonumber \\
    & U_{S_3} =  14 U_{K_7}^2 + 56 \left(1 +  \lVert W \rVert_{\infty} d_{\min}^{-1}L_P\right)^2 L_f^2,   \nonumber \\
    & U_{S_5} = 14 U_{K_9}^2 + 7 C_{\mathcal{X}}^2 L_P^{2l}(1+L_P)^{2l}, \nonumber \\
    & U_{T_2} =  \left(14 U_{K_9}^2 + 7 C_{\mathcal{X}}^2   L_P^{2l}(1+L_P)^{2l} \right) \frac{D_{\mathcal{X}}}{2(D_{\mathcal{X}}+1)},  \\
    & Q_1 = U_{T_2} \lVert f \rVert_{\infty}^2  + 1.5\left(U_{S_2} \lVert f \rVert_{\infty}^2 +U_{S_3} L_f^2\right), \\
%
    & Q_2 = 1.5 U_{S_5}  \lVert f \rVert_{\infty}^2 + U_{R_3} + 0.5U_{R_4} \lVert f \rVert_{\infty}^2, \\
    & Q_3 =  \max\left( L_{\sigma}^{L} D_{\Theta}^L \left( 1 +\Delta_{\Theta} \right)^L , \left(L_{\sigma}^{L} D_{\Theta}^L \left( 1 +\Delta_{\Theta} \right)^L\right)^T\right), \\
    & U_{\tilde{C_3}^{(T)}} = \frac{d_{\min}+W_{\max}}{d_{\min}^2} Q_3 \left(1+L_{\sigma}^{L} D_{\Theta}^L \left( 1 +\Delta_{\Theta} \right)^L\right)^{T+2}\lVert f \rVert_{\infty},  \\
    & Q^{(l)}   =  \left(L_{\sigma}^{L}  D_{\Theta}^L \left( 1 +\Delta_{\Theta} \right)^L\right)^l  \left(1+L_{\sigma}^{L} D_{\Theta}^L \left( 1 +\Delta_{\Theta} \right)^L\right)^l\lVert f \rVert_{\infty},\\
    & U_{C_{5}^{(l+1)}} =  Q^{(l)} \max\left( 1 , \frac{\left\lVert W \right\rVert_{\infty}}{d_{\min}} \left( \sqrt{H_{l+1}} \left( \frac{\beta\kappa^2}{\alpha}  \right)^L \left(   \sum_{k=1}^L \binom{L}{k} \Delta_{\Theta}^k \right) + 1 \right) L_P  \right). 
\end{align}

and 

\begin{align}
& U_{C_{T_2}^{(l+1)}} = \sqrt{H_{l+1}} \left( \frac{\beta\kappa^2}{\alpha}  \right)^L   L_P^{l+2}(1+L_P)^l\frac{\lVert W \rVert_{\infty}\lVert f \rVert_{\infty} }{d_{\min}},   \\
&
U_{C_{T_3}^{(l+1)}} =  L_PL_{\sigma}^{L} D_{\Theta}^L   \left( 1 +\Delta_{\Theta} \right)^L \frac{\lVert W \rVert_{\infty}}{d_{\min}}, \\
&U_{C_{T_4}^{(l+1)}} = \sqrt{F_{l+1}} \left( \frac{\beta\kappa^2}{\alpha}  \right)^L  L_P U_{C_{5}^{(l+1)}}, \\
& Q_4 = \sum_{l=1}^T \left(U_{C_{T_2}^{(l )}} + U_{C_{T_4}^{(l )}}\right)\prod_{l'=l+1}^{T} \left(1+ U_{C_{T_3}^{(l')}}\right),
\end{align}

In the following, we present conditions on graph size ($N$) and the difference between the given hypothesis and labeling function characterized through  $\Delta_{M}^k$, $\Delta_{\Theta}$ and $N_{P_{\tau}}$, and reveal connections how these conditions are linked to   properties of the RGM, MLP classifier and MPNN feature extractor.
\begin{corollary}[Sufficient Condition]
\label{proof:bound_tightness}
    Follow the same problem setting and assumptions of Theorem \ref{thm:DA}. Suppose $ M\geq  \log_{\frac{D_{\mathcal{X}}}{4} }\left(\frac{\left(1+\log(1/\rho)^{\frac{1}{4}}\right)}{0.1\times 27^{\frac{D_{\mathcal{X}}}{4}} }\right)$ and $D_{\mathcal{X}} > 1$. Then, it has
\begin{align}
\label{eq:bound1}
    &\sqrt{\Delta_D} \leq  \sqrt{2L_{2}} C L_f   \bigg( \mathcal{W}_2 \left(\hat{P}_S^m, \hat{P}_T^m\right) + 2B \cdot 1.1\times 27^{\frac{D_{\mathcal{X}}}{4}}  \bigg)^{ \alpha} = U_{\sqrt{\Delta_D}}, \\
\label{eq:bound2}
    &  \Delta_N \leq \frac{ \log(N)}{N^{ \frac{1}{D_{\mathcal{X}}+1}}}\left(Q_1 +  Q_2\log\left(\frac{2}{\rho}\right)\right) =U_{\Delta_N}, \\
\label{eq:bound3}
    &\Delta_{\Gamma,\Theta} \leq L_P  Q_3   C_{\nabla w} \lVert \nabla \tau \rVert_{\infty}\left(  \sum_{k=1}^L \binom{L}{k} \Delta_{\Theta}^k \right)  + U_{\tilde{C_3}^{(T)}}C_{\nabla w} \lVert \nabla \tau \rVert_{\infty} + Q^{(T)} N_{P_\tau} = U_{\Delta_{\Gamma,\Theta}}.
\end{align}

For $\xi>0$, define a quantity
\begin{equation}
   \eta= \left(\frac{\xi \epsilon_S(h,g_D)\sqrt{\lambda_{r}}}{\sqrt{2L_{2}} C^{1.5} L_f   \bigg( \mathcal{W}_2 \left(\hat{P}_S^m, \hat{P}_T^m\right) + 2B \cdot 1.1\times 27^{\frac{D_{\mathcal{X}}}{4}}  \bigg)^{ \alpha}}-U_{G_{NN}} \right)L_{\sigma }^{-L}\lambda_M^{-L}.
\end{equation}
Supposing the following holds, 
\begin{align}
\label{eq:final_condition1}
      & \sum_{k=1}^L \binom{L}{k} \Delta_{M}^k  \leq \frac{\xi \epsilon_S(h,g_D)\sqrt{\lambda_{r}}}{ \sqrt{2L_{2}} C^{1.5} L_f   \bigg( \mathcal{W}_2 \left(\hat{P}_S^m, \hat{P}_T^m\right) + 2B \cdot 1.1\times 27^{\frac{D_{\mathcal{X}}}{4}}  \bigg)^{ \alpha}  \left( \frac{\beta}{\alpha} \kappa^2 \right)^L }, \\
\label{eq:final_condition2}
           &\frac{ \log(N)}{N^{ \frac{1}{D_{\mathcal{X}}+1}}} \leq \frac{0.25\eta^2}{Q_1 +  Q_2\log\left(\frac{2}{\rho}\right) }, \\
\label{eq:final_condition3}           
           & L_P  Q_3   C_{\nabla w} \lVert \nabla \tau \rVert_{\infty}\left(  \sum_{k=1}^L \binom{L}{k} \Delta_{\Theta}^k \right)+ Q^{(T)} N_{P_\tau} \leq 0.5\eta, \\
\label{eq:final_condition4}  
 & C_{\nabla w} \lVert \nabla \tau \rVert_{\infty}   \leq \frac{0.5\eta -  Q^{(T)} N_{P_\tau}}{L_P\left(  \sum_{k=1}^L \binom{L}{k} \Delta_{\Theta}^k \right)  Q_3+ U_{\tilde{C_3}^{(T)}}}, 
\end{align}
the DA generalization bound in Theorem \ref{thm:DA} satisfies 
\begin{equation}
\label{eq:bound_tightneess}
     \sqrt{\frac{C\Delta_D}{\lambda_{r}}}  \left( L_{NN} \left(\sqrt{\Delta_N} + \Delta_{\Gamma,\Theta}    \right)  + G_{NN} \right)   \leq \xi \epsilon_S(h,g_D), 
\end{equation}
and as a result $\epsilon_T(h,g_D )\leq (1+\xi) \epsilon_S(h,g_D)$.
\end{corollary}

Condition in Eq. (\ref{eq:final_condition1}) requires a limited perturbation strength over the MLP classifier.
Condition in Eq. (\ref{eq:final_condition3}) requires a limited perturbation strength over the MPNN feature extractor through both network weighs and RGM deformation.
Condition in Eq. (\ref{eq:final_condition2}) imposes a sample complexity requirement on graph size.
The last condition in Eq. (\ref{eq:final_condition4}) requires properties on RGMs.

\begin{proof}

\textbf{First, we seek to bound $\Delta_D$, originally defined in Lem. \ref{proof:reverse2latent}.}

The assumption on $M$ and $D_{\mathcal{X}}$ results in 
\begin{equation}
    M^{-\frac{1}{D_{\mathcal{X}}}} + M^{-\frac{1}{4}} \cdot \log(1/\rho)^{\frac{1}{4}} \leq \left(1+\log(1/\rho)^{\frac{1}{4}}\right)M^{-\frac{D_{\mathcal{X}}}{4}} \leq 0.1\times 27^{\frac{D_{\mathcal{X}}}{4}} 
\end{equation}
It then has
\begin{align}
    \sqrt{\Delta_D} \leq\; & \sqrt{2L_{2}} C L_f \bigg[ \mathcal{W}_2 \left(\hat{P}_S^m, \hat{P}_T^m\right) + 2B \cdot \left(1+\log(1/\rho)^{\frac{1}{4}}\right)M^{-\frac{D_{\mathcal{X}}}{4}} + 2B \cdot 27^{\frac{D_{\mathcal{X}}}{4}}  \bigg]^{ \alpha} \nonumber \\
     \leq\; & \sqrt{2L_{2}} C L_f   \bigg( \mathcal{W}_2 \left(\hat{P}_S^m, \hat{P}_T^m\right) + 2B \cdot 1.1\times 27^{\frac{D_{\mathcal{X}}}{4}}  \bigg)^{ \alpha}.
\end{align}

\textbf{Second, we seek to analyze $\Delta_N$, originally defined in Thm. \ref{proof:convergence}.}

We approach to $L_{NN}$ and $G_{NN}$, originally defined in Lem. \ref{proof:norm2diff}. According to definition, it has 
\begin{align}
    L_{NN} \leq \; & L_{\sigma }^L\lambda_M^L = U_{L_{NN}}, \\
    G_{NN} \leq \; & \left( \frac{\beta}{\alpha} \kappa^2 \right)^L \left(( 1 +\Delta_M )^L - 1 \right) = \left( \frac{\beta}{\alpha} \kappa^2 \right)^L \left(\sum_{k=1}^L \binom{L}{k} \Delta_{M}^k \right) =U_{G_{NN}},
\end{align}

We now deal with $\{C_{i}^{(l)}\}_{i=1}^2,\{D_{i}^{(l)}\}_{i=1}^3$, originally defined in Eqs. (\ref{constant:C1}) - (\ref{constant:D3}).

Under the assumption of formal bias, it has $C_1^{(l)}, D_1^{(l)} =0$.
Again, according to definitions, it has
\begin{equation}
    C_2^{(l )} = \prod_{t=1}^{l } L_{\Psi^{(t)}} \left( 1 +  L_{\Phi^{(t)}} \right) \leq L_P^l(1+L_P)^l,
\end{equation}
also
\begin{align}
    D_2^{(l)}
    = \; & \sum_{t=1}^{l} C_2^{(t-1)} \left( L_{\Psi^{(t)}} L_{\Phi^{(t)}}  L_W^{\infty} d_{\min}^{-1} + L_{\Psi^{(t)}} L_{\Phi^{(t)}} \lVert W \rVert_{\infty}  L_W^{\infty} d_{\min}^{-2} \right) \nonumber \\
    & \cdot \prod_{l' = t+1}^{l} L_{\Psi^{(l')}} \left(1+ \lVert W \rVert_{\infty} d_{\min}^{-1} L_{\Phi^{(l')}} \right) \nonumber \\
    \leq \; & \sum_{t=1}^{l}  L_P^{l+2}(1+L_P)^l  L_W^{\infty}d_{\min}^{-1}\left(1+\lVert W \rVert_{\infty}d_{\min}^{-1}\right) \cdot \left( L_P + \lVert W \rVert_{\infty} d_{\min}^{-1}L_P^2  \right)^{l-t} \nonumber \\
    = \; & L_P^{2l+2}(1+L_P)^l  L_W^{\infty}d_{\min}^{-1}\left(1+\lVert W \rVert_{\infty}d_{\min}^{-1}\right)     \left( 1+ \lVert W \rVert_{\infty} d_{\min}^{-1}L_P   \right)^{l } \nonumber \\
    & \cdot \sum_{t=1}^{l} \left( L_P + \lVert W \rVert_{\infty} d_{\min}^{-1}L_P^2  \right)^{-t} \nonumber \\
    \leq \; & l \cdot L_P^{2l+2}(1+L_P)^l  L_W^{\infty}d_{\min}^{-1}\left(1+\lVert W \rVert_{\infty}d_{\min}^{-1}\right) \left( 1+ \lVert W \rVert_{\infty} d_{\min}^{-1}L_P   \right)^{l } \nonumber \\
    & \cdot \max\left(\left( L_P + \lVert W \rVert_{\infty} d_{\min}^{-1}L_P^2  \right)^{-1}, \left( L_P + \lVert W \rVert_{\infty} d_{\min}^{-1}L_P^2  \right)^{-l}\right),
\end{align}
and 
\begin{equation}
    D_3^{(l)} =    1 +  \lVert W \rVert_{\infty} d_{\min}^{-1} L_{\Phi^{(t)}}  \leq  1 +  \lVert W \rVert_{\infty} d_{\min}^{-1}L_P,
\end{equation}

We now seek to bound $\{K_{i}\}_{i=1}^9$, originally defined in Eqs. (\ref{constant:K}).

Originally defined in Eq. (\ref{lem:layerwise:A}), the constant $A^{(l)}$ satisfies
\begin{equation}
    A^{(l)} = L_{\Psi^{(l)}} \sqrt{1 + \frac{8\lVert W \rVert_{\infty}^2L_{\Phi^{(l)}}^2}{d_{\min}^2}  } \leq L_P \sqrt{1 +  8\lVert W \rVert_{\infty}^2L_P^2 d_{\min}^{-2} }, 
\end{equation}
thus
\begin{equation}
    \prod_{l'=l+1}^{T} A^{(l')} \leq \left(L_P \sqrt{1 +  8\lVert W \rVert_{\infty}^2L_P^2 d_{\min}^{-2} }\right)^{T-l}.
\end{equation}
And it has
\begin{align}
    \sum_{l=1}^{T}  C_2^{(l-1)}  \prod_{l'=l+1}^{T} A^{(l')} \leq\; & \sum_{l=1}^{T} L_P^{l}(1+L_P)^l \left(L_P \sqrt{1 +  8\lVert W \rVert_{\infty}^2L_P^2 d_{\min}^{-2} }\right)^{T-l} \nonumber \\
    \leq\; & L_P^{T}\left(  \sqrt{1 +  8\lVert W \rVert_{\infty}^2L_P^2 d_{\min}^{-2} }\right)^{T }\sum_{l=1}^{T} (1+L_P)^l \left(  \sqrt{1 +  8\lVert W \rVert_{\infty}^2L_P^2 d_{\min}^{-2} }\right)^{ -l} \nonumber \\
    \leq\; & TL_P^{T}(1+L_P)^T\left(  \sqrt{1 +  8\lVert W \rVert_{\infty}^2L_P^2 d_{\min}^{-2} }\right)^{T-1 }  ,
 \end{align}
 and 
\begin{align}
    \sum_{l=1}^{T}  D_3^{(l-1)}  \prod_{l'=l+1}^{T} A^{(l')} \leq\; &\left(1 +  \lVert W \rVert_{\infty} d_{\min}^{-1}L_P\right) \sum_{l=1}^{T}   \left(  \sqrt{1 +  8\lVert W \rVert_{\infty}^2L_P^2 d_{\min}^{-2} }\right)^{ T-l} \nonumber \\
    = \; &   T\left(1 +  \lVert W \rVert_{\infty} d_{\min}^{-1}L_P\right)  \left(1 +  8\lVert W \rVert_{\infty}^2L_P^2 d_{\min}^{-2}\right)^{\frac{T-1}{2}}
 \end{align}
 
As a result, the quantities $\{K_i\}_{i=1}^9$ satisfy
\begin{align}
      K_1 = \; & 0, \nonumber\\
      K_2  = \; &0, \nonumber \\
      K_3  \leq\; &    2TL_P^{T+2} (1+L_P)^T \left(  \sqrt{1 +  8\lVert W \rVert_{\infty}^2L_P^2 d_{\min}^{-2} }\right)^{T-1 }   L_W^{\infty}\left(\sqrt{\log(C_{\mathcal{X}})} + \sqrt{D_\mathcal{X}}\right)\lVert W \rVert_{\infty} d_{\min}^{-2} \nonumber \\
     K_4  \leq \; &
     \frac{2TL_P^{T+2}(1+L_P)^T  \left(\sqrt{2}\left\lVert W \right\rVert_\infty + L_W^{\infty}\right)\lVert W \rVert_{\infty}  \left(  \sqrt{1 +  8\lVert W \rVert_{\infty}^2L_P^2 d_{\min}^{-2} }\right)^{T-1 } d_{\min}^{-2} }{ \sqrt{N}} \nonumber \\
     K_5 = \; &0, \nonumber \\
     K_6  \leq \; &  2TL_P^{T+2}(1+L_P)^T L_W^{\infty}  \left(  \sqrt{1 +  8\lVert W \rVert_{\infty}^2L_P^2 d_{\min}^{-2} }\right)^{T-1 }d_{\min}^{-1}  =U_{K_6}   , \nonumber \\
      K_7   \leq \; &  2TL_P^2 \lVert W \rVert_{\infty}   \left(1 +  \lVert W \rVert_{\infty} d_{\min}^{-1}L_P\right) d_{\min}^{-1} \left(1 +  8\lVert W \rVert_{\infty}^2L_P^2 d_{\min}^{-2}\right)^{\frac{T-1}{2}}  =U_{K_7}  , \nonumber \\
     K_8 = \; & 0, \nonumber \\
     K_9   <\; &  0.8TL_P^{T+2} (1+L_P)^TC_{\mathcal{X}} \lVert W \rVert_{\infty} \left(  \sqrt{1 +  8\lVert W \rVert_{\infty}^2L_P^2 d_{\min}^{-2} }\right)^{T-1 }d_{\min}^{-1}  =U_{K_9} .
\end{align}
Based on the above, the following quantities satisfy
\begin{align}
    \label{constant:RST_condition}
    & R_1 = R_2 =0, \nonumber \\
    &R_3  \leq       56T^2L_P^{2T+4} (1+L_P)^{2T} \left(  1 +  8\lVert W \rVert_{\infty}^2L_P^2 d_{\min}^{-2} \right)^{T-1 }   (L_W^{\infty})^2\left(\sqrt{\log(C_{\mathcal{X}})} + \sqrt{D_\mathcal{X}}\right)^2\lVert W \rVert_{\infty}^2 d_{\min}^{-4} = U_{R_3},\nonumber \\
    &  R_4  \leq  
     \frac{56TL_P^{2T+4}(1+L_P)^{2T}  \left(\sqrt{2}\left\lVert W \right\rVert_\infty + L_W^{\infty}\right)^2\lVert W \rVert_{\infty}^2  \left(   1 +  8\lVert W \rVert_{\infty}^2L_P^2 d_{\min}^{-2}  \right)^{T-1 } d_{\min}^{-4} }{ N} = \frac{U_{R_4}}{N}, \nonumber \\
    & S_1 =S_4 =0, \nonumber \\
    & S_2 \leq  14 U_{K_6}^2 + 14 U_{K_9}^2 \log(C_{\mathcal{X}}) + 56 \left(1 +  \lVert W \rVert_{\infty} d_{\min}^{-1}L_P\right)^2 + 7 L_P^{2l}(1+L_P)^{2l}C_{{\mathcal{X}}}^2 \log(C_{\mathcal{X}}) =U_{S_2}, \nonumber \\
    & S_3 \leq 14 U_{K_7}^2 + 56 \left(1 +  \lVert W \rVert_{\infty} d_{\min}^{-1}L_P\right)^2 L_f^2 = U_{S_3}, \nonumber \\
    & S_5 \leq 14 U_{K_9}^2 + 7 C_{\mathcal{X}}^2 L_P^{2l}(1+L_P)^{2l} = U_{S_5}, \nonumber \\
    & T_1 =0, \nonumber \\
    & T_2 \leq \left(14 U_{K_9}^2 + 7 C_{\mathcal{X}}^2   L_P^{2l}(1+L_P)^{2l} \right) \frac{D_{\mathcal{X}}}{2(D_{\mathcal{X}}+1)} = U_{T_2},
\end{align}
where $\{R_i\}_{i=1}^4,\{S_i\}_{i=1}^5,\{T_i\}_{i=1}^2$ are originally defined in Eqs. (\ref{constant:RST}).

Assuming the graph size $N>2$ and given that $\frac{1}{\log(2)} < 1.5$ and $\frac{1}{N} < \frac{log(N)}{N^{\frac{1}{D_{\mathcal{X}}+1}}}$, the above results in the following:
{\small
\begin{align}
   \Delta_N
   =\; & \frac{  S_2 \lVert f \rVert_{\infty}^2 + S_3 L_f^2 + T_2 \lVert f \rVert_{\infty}^2  \log(N)}{N^{\frac{1}{D_{\mathcal{X}}+1}}}   + \left( \frac{R_3 + R_4 \lVert f \rVert_{\infty}^2}{N} + \frac{ S_5 \lVert f \rVert_{\infty}^2}{N^{\frac{1}{D_{\mathcal{X}}+1}}} \right)  \log\left(\frac{2}{\rho}\right) \nonumber \\
   \leq \; & \frac{ \log(N)}{N^{ \frac{1}{D_{\mathcal{X}}+1}}}\left[  U_{T_2} \lVert f \rVert_{\infty}^2 + \frac{1}{\log{N}}\left(U_{S_2} \lVert f \rVert_{\infty}^2 + U_{S_3} L_f^2 + \log\left(\frac{2}{\rho}\right) U_{S_5}  \lVert f \rVert_{\infty}^2\right) + log\left(\frac{2}{\rho}\right)\left( U_{R_3} + \frac{U_{R_4}\lVert f \rVert_{\infty}^2}{N} \right) \right] \nonumber \\
   < \; & \frac{\log(N)}{N^{ \frac{1}{D_{\mathcal{X}}+1}}} \left[  U_{T_2} \lVert f \rVert_{\infty}^2  + 1.5\left(U_{S_2} \lVert f \rVert_{\infty}^2 +U_{S_3} L_f^2 + \log\left(\frac{2}{\rho}\right) U_{S_5}  \lVert f \rVert_{\infty}^2 \right) + log\left(\frac{2}{\rho}\right) \left(U_{R_3} + 0.5U_{R_4} \lVert f \rVert_{\infty}^2  \right) \right] \nonumber \\
   = \; & \frac{\log(N)}{N^{ \frac{1}{D_{\mathcal{X}}+1}}} \left[  U_{T_2} \lVert f \rVert_{\infty}^2  + 1.5 \left(U_{S_2} \lVert f \rVert_{\infty}^2 +U_{S_3} L_f^2 \right) + log\left(\frac{2}{\rho}\right) \left(1.5 U_{S_5}  \lVert f \rVert_{\infty}^2 + U_{R_3} + 0.5U_{R_4} \lVert f \rVert_{\infty}^2  \right) \right].
\end{align}
}

Defining new quantities $Q_1 = U_{T_2} \lVert f \rVert_{\infty}^2  + 1.5\left(U_{S_2} \lVert f \rVert_{\infty}^2 +U_{S_3} L_f^2\right)$ and 
$Q_2 = 1.5 U_{S_5}  \lVert f \rVert_{\infty}^2 + U_{R_3} + 0.5U_{R_4} \lVert f \rVert_{\infty}^2$
, it then has 
\begin{equation}
     \Delta_N \leq \frac{ \log(N)}{N^{ \frac{1}{D_{\mathcal{X}}+1}}}\left(Q_1 +  Q_2\log\left(\frac{2}{\rho}\right)\right).
\end{equation}

\textbf{Third, we analyze optimization error bound $\Delta_{\Gamma,\Theta}$, originally defined in Thm. \ref{proof:optimization}.}

We further simplify the following quantities:
\begin{align}
        C_3^{(T)} = \; &     \frac{d_{\min}+W_{\max}}{d_{\min}^2}  \sum_{l=1}^T L_{\Psi^{(l)}} L_{\Phi^{(l)}}     C_2^{(l-1)} \lVert f \rVert_{\infty}       \prod_{l'=l+1}^T L_{\Psi^{(l')}} \left( 1 + \frac{W_{\max}}{d_{\min}} \cdot L_{\Phi^{(l')}} \right)  \nonumber \\
        \leq\; & \frac{d_{\min}+W_{\max}}{d_{\min}^2} L_P^{T+1}\lVert f \rVert_{\infty} \left( 1 + \frac{W_{\max}}{d_{\min}}L_P \right)^{T}\sum_{l=1}^T (1+L_P)^{l-1}   \left( 1 + \frac{W_{\max}}{d_{\min}}L_P \right)^{-l}\nonumber\\
        < \; & \frac{(d_{\min}+W_{\max}) \lVert f \rVert_{\infty}TL_P^{T+1}  (1+L_P)^{T-1}}{d_{\min}^2}  \left( 1 + \frac{W_{\max}}{d_{\min}}L_P \right)^{T-1}         
    \end{align}
and     
\begin{align}
          C_4^{(T)} = \; &    \prod_{l=1}^T L_{\Psi^{(l)}}  \left( 1 + \frac{W_{\max}}{d_{\min}} \cdot L_{\Phi^{(l)}} \right) \leq L_P^T\left( 1 + \frac{W_{\max}}{d_{\min}}L_P \right)^{T }, \\
          C_5^{(T)} = \; &   C_2^{(T)} \lVert f \rVert_{\infty} \leq L_P^T(1+L_P)^T\lVert f \rVert_{\infty}.
\end{align}

Then the following quantities, originally defined in Thm. \ref{proof:perturb_weights}, satisfy
{\small
    \begin{align}
    & \Delta_{\Phi}^{(l+1)}  =  \sqrt{H_{l+1}} \left( \frac{\beta\kappa^2}{\alpha}  \right)^L \left[ \prod_{k=1}^L \left( 1 + \frac{\left\lVert \Delta \Theta_{\Phi_k^{(l+1)}} \right\rVert_F}{\left\lVert \Theta_{\Phi_k^{(l+1)}} \right\rVert_F} \right) - 1 \right]  
     \leq \sqrt{H_{l+1}} \left( \frac{\beta\kappa^2}{\alpha}  \right)^L \left(   \left( 1 +\Delta_{\Theta} \right)^L - 1 \right), \\
    & \Delta_{\Psi}^{(l+1)}  =  \sqrt{F_{l+1}} \left( \frac{\beta \kappa^2}{\alpha} \right)^L \left[ \prod_{k=1}^L \left( 1 + \frac{\lVert \Delta \Theta_{\Psi_k^{(l+1)}} \rVert_F}{\lVert \Theta_{\Psi_k^{(l+1)}} \rVert_F} \right) - 1 \right]  
     \leq   \sqrt{F_{l+1}} \left( \frac{\beta\kappa^2}{\alpha}  \right)^L \left(   \left( 1 +\Delta_{\Theta} \right)^L - 1 \right)  , \\
  &   L_{\tilde{\Psi}^{(l)}} =  L_{\sigma}^{L}   \prod_{t=1}^{L} \left( \left\lVert \Delta \Theta_{\Psi_t^{(l)}} \right\rVert_F + \left\lVert \Theta_{\Psi_t^{(l)}} \right\rVert_F \right) \leq   L_{\sigma}^{L} D_{\Theta}^L   \left( 1 +\Delta_{\Theta} \right)^L, \\
  &  L_{\tilde{\Phi}^{(l)}} =    L_{\sigma}^{L}   \prod_{t=1}^{L} \left( \left\lVert \Delta \Theta_{\Phi_t^{(l)}} \right\rVert_F + \left\lVert \Theta_{\Phi_t^{(l)}} \right\rVert_F \right)\leq   L_{\sigma}^{L} D_{\Theta}^L   \left( 1 +\Delta_{\Theta} \right)^L, \\
&      D_{\tilde{\Psi}^{(t)}(0,0)}  =    \left( \Delta_{\Psi}^{(t)} + 1 \right)\left\lVert \Psi^{(t)}(0, 0) \right\rVert_{\infty} =0 , \\
&      D_{\tilde{\Phi}^{(t)}(0,0)}  =    \left( \Delta_{\Phi}^{(t)} + 1 \right)\left\lVert \Phi^{(t)}(0, 0) \right\rVert_{\infty} =0,
\end{align}
}

and the last two quantities result in $ \left\lVert \tilde{\Phi}^{(t)}(0,0) \right \rVert_\infty,  \left\lVert \tilde{\Psi}^{(t)}(0,0) \right\rVert_\infty =0$.
Based on these, the following quantities satisfy
\begin{align}
& \tilde{C_1}^{(l)} = 0, \\
&\tilde{C_2}^{(l)}= \prod_{t=1}^{l} L_{\tilde{\Psi}^{(t)}} \left( 1 +  L_{\tilde{\Phi}^{(t)}} \right) \leq \left(L_{\sigma}^{L} D_{\Theta}^L\left( 1 +\Delta_{\Theta} \right)^L\right)^l \left(1+L_{\sigma}^{L} D_{\Theta}^L \left( 1 +\Delta_{\Theta} \right)^L\right)^l,\\ 
&C_{\Phi}^{(l+1)} =L_{\Phi^{(l+1)}}  C_2^{(l)} \left\lVert f \right\rVert_{\infty} \leq  L_P^{l+1}(1+L_P)^l\lVert f \rVert_{\infty}, \\
& \tilde{C}_{\Phi}^{(l+1)} \leq  L_P  \left(L_{\sigma}^{L}  D_{\Theta}^L \left( 1 +\Delta_{\Theta} \right)^L\right)^l  \left(1+L_{\sigma}^{L} D_{\Theta}^L \left( 1 +\Delta_{\Theta} \right)^L\right)^l\lVert f \rVert_{\infty} . 
\end{align}
and the quantity, originally defined in Thm. \ref{proof:optimization}, satisfy
\begin{align}
\tilde{C_3}^{(T)} =\; & \bigg( \frac{1}{d_{\min}} + \frac{W_{\max}}{d_{\min}^2} \bigg) \sum_{l=1}^T L_{\tilde{\Psi}^{(l)}}  L_{\tilde{\Phi}^{(l)}}  \tilde{C_2}^{(l-1)} \lVert f \rVert_{\infty}  \nonumber \\
\leq \; &\bigg( \frac{1}{d_{\min}} + \frac{W_{\max}}{d_{\min}^2} \bigg) \left(1+L_{\sigma}^{L} D_{\Theta}^L \left( 1 +\Delta_{\Theta} \right)^L\right)^2 \lVert f \rVert_{\infty}\sum_{l=1}^T   \tilde{C_2}^{(l-1)} \nonumber \\
<\; & \bigg( \frac{1}{d_{\min}} + \frac{W_{\max}}{d_{\min}^2} \bigg) \left(1+L_{\sigma}^{L} D_{\Theta}^L \left( 1 +\Delta_{\Theta} \right)^L\right)^{T+2}\lVert f \rVert_{\infty} \sum_{l=1}^T   \left(L_{\sigma}^{L} D_{\Theta}^L \left( 1 +\Delta_{\Theta} \right)^L\right)^l \nonumber \\
\leq \; & T\bigg( \frac{1}{d_{\min}} + \frac{W_{\max}}{d_{\min}^2} \bigg) \left(1+L_{\sigma}^{L} D_{\Theta}^L \left( 1 +\Delta_{\Theta} \right)^L\right)^{T+2}\lVert f \rVert_{\infty} \cdot \nonumber \\
& \max\left( L_{\sigma}^{L} D_{\Theta}^L \left( 1 +\Delta_{\Theta} \right)^L , \left(L_{\sigma}^{L} D_{\Theta}^L \left( 1 +\Delta_{\Theta} \right)^L\right)^T\right) = U_{\tilde{C_3}^{(T)}}.
\end{align}
Then, the following can be derived:
\begin{align}
    C_{T_2}^{(l+1)}
    = \; &  L_{\Psi^{(l+1)}}\Delta_{\Phi}^{(l+1)} C_{\Phi}^{(l+1)} \frac{\lVert W \rVert_{\infty}}{d_{\min}} \nonumber \\ 
    \leq \; &  \sqrt{H_{l+1}} \left( \frac{\beta\kappa^2}{\alpha}  \right)^L L_P\left( \left( 1 +\Delta_{\Theta} \right)^L - 1 \right)L_P^{l+1}(1+L_P)^l\frac{\lVert W \rVert_{\infty}\lVert f \rVert_{\infty} }{d_{\min}} \nonumber \\ 
    = \; & \sqrt{H_{l+1}} \left( \frac{\beta\kappa^2}{\alpha}  \right)^L   L_P^{l+2}(1+L_P)^l\frac{\lVert W \rVert_{\infty}\lVert f \rVert_{\infty} }{d_{\min}} \left(  \sum_{k=1}^L \binom{L}{k} \Delta_{\Theta}^k \right) = U_{C_{T_2}^{(l+1)}}\left(  \sum_{k=1}^L \binom{L}{k} \Delta_{\Theta}^k \right),\\
    C_{T_3}^{(l+1)} = \; &    L_{\Psi^{(l+1)}}   L_{\tilde{\Phi}^{(l)}} \frac{\lVert W \rVert_{\infty}}{d_{\min}} \leq  L_PL_{\sigma}^{L} D_{\Theta}^L   \left( 1 +\Delta_{\Theta} \right)^L \frac{\lVert W \rVert_{\infty}}{d_{\min}} = U_{C_{T_3}^{(l+1)}},  
\end{align}
and
\begin{align}
 \tilde{C_5}^{(l+1)}   =\; & \max\left(  \tilde{C_2}^{(l+1)} \lVert f \rVert_{\infty}, \frac{\left\lVert W \right\rVert_{\infty}}{d_{\min}} \left( \Delta_{\Phi}^{(l+1)} + 1 \right) \tilde{C}_{\Phi}^{(l+1)} \right), \nonumber \\ 
 \leq\; & \max\left( \left(L_{\sigma}^{L} D_{\Theta}^L\left( 1 +\Delta_{\Theta} \right)^L\right)^l \left(1+L_{\sigma}^{L} D_{\Theta}^L \left( 1 +\Delta_{\Theta} \right)^L\right)^l \lVert f \rVert_{\infty}, \right. \nonumber \\ 
 &  \frac{\left\lVert W \right\rVert_{\infty}}{d_{\min}} \left( \sqrt{H_{l+1}} \left( \frac{\beta\kappa^2}{\alpha}  \right)^L \left(   \left( 1 +\Delta_{\Theta} \right)^L - 1 \right) + 1 \right)  \times  \nonumber \\ 
 &
 \left. L_P\left(L_{\sigma}^{L}  D_{\Theta}^L \left( 1 +\Delta_{\Theta} \right)^L\right)^l  \left(1+L_{\sigma}^{L} D_{\Theta}^L \left( 1 +\Delta_{\Theta} \right)^L\right)^l\lVert f \rVert_{\infty}\right) = U_{C_{5}^{(l+1)}}, \\
 C_{T_4}^{(l+1)}  =\; &   \Delta_{\Psi}^{(l+1)}   L_{\Psi^{(l+1)}}  \tilde{C_5}^{(l)} \leq  \sqrt{F_{l+1}} \left( \frac{\beta\kappa^2}{\alpha}  \right)^L \left(   \left( 1 +\Delta_{\Theta} \right)^L - 1 \right) L_P\tilde{C_5}^{(l)} \nonumber \\ 
  =\; & \sqrt{F_{l+1}} \left( \frac{\beta\kappa^2}{\alpha}  \right)^L  L_P U_{C_{5}^{(l+1)}}\left(   \sum_{k=1}^L \binom{L}{k} \Delta_{\Theta}^k \right) = U_{C_{T_4}^{(l+1)}}\left(   \sum_{k=1}^L \binom{L}{k} \Delta_{\Theta}^k \right).
\end{align}

The above inequalities result in
\begin{align} 
    \Delta_{\Gamma,\Theta}  = \; &  \sum_{l=1}^{T}\left (C_{T_2}^{(l)} + C_{T_4}^{(l)}\right) \prod_{l'=l+1}^{T} L_{\Psi^{(l')}} \left( 1 +  C_{T_3}^{(l')}\right)   + \tilde{C_3}^{(T)}  C_{\nabla w} \lVert \nabla \tau \rVert_{\infty} +    \tilde{C_2}^{(T)} \lVert f \rVert_{\infty}N_{P_\tau} \nonumber \\ 
    \leq \; & \sum_{l=1}^T \left(U_{C_{T_2}^{(l )}} + U_{C_{T_4}^{(l )}}\right)\left(  \sum_{k=1}^L \binom{L}{k} \Delta_{\Theta}^k \right) L_P \prod_{l'=l+1}^{T} \left(1+ U_{C_{T_3}^{(l')}}\right) + U_{\tilde{C_3}^{(T)}}C_{\nabla w} \lVert \nabla \tau \rVert_{\infty} +\nonumber \\ 
    & \left(L_{\sigma}^{L}  D_{\Theta}^L \left( 1 +\Delta_{\Theta} \right)^L\right)^T  \left(1+L_{\sigma}^{L} D_{\Theta}^L \left( 1 +\Delta_{\Theta} \right)^L\right)^T\lVert f \rVert_{\infty} N_{P_\tau} \nonumber \\ 
     = \; & L_P\left(  \sum_{k=1}^L \binom{L}{k} \Delta_{\Theta}^k \right) \sum_{l=1}^T \left(U_{C_{T_2}^{(l )}} + U_{C_{T_4}^{(l )}}\right)\prod_{l'=l+1}^{T} \left(1+ U_{C_{T_3}^{(l')}}\right)   C_{\nabla w} \lVert \nabla \tau \rVert_{\infty} + \nonumber \\ 
    &   U_{\tilde{C_3}^{(T)}}C_{\nabla w} \lVert \nabla \tau \rVert_{\infty} + \left(L_{\sigma}^{L}  D_{\Theta}^L \left( 1 +\Delta_{\Theta} \right)^L\right)^T  \left(1+L_{\sigma}^{L} D_{\Theta}^L \left( 1 +\Delta_{\Theta} \right)^L\right)^T\lVert f \rVert_{\infty} N_{P_\tau}.
\end{align}

Defining the following two new quantities,
\begin{align}
    &Q^{(l)}   =  \left(L_{\sigma}^{L}  D_{\Theta}^L \left( 1 +\Delta_{\Theta} \right)^L\right)^l  \left(1+L_{\sigma}^{L} D_{\Theta}^L \left( 1 +\Delta_{\Theta} \right)^L\right)^l\lVert f \rVert_{\infty}, \\
    & Q_3 = \sum_{l=1}^T \left(U_{C_{T_2}^{(l )}} + U_{C_{T_4}^{(l )}}\right)\prod_{l'=l+1}^{T} \left(1+ U_{C_{T_3}^{(l')}}\right),
\end{align}
it then has
\begin{align}
  &U_{C_{5}^{(l+1)}} =  Q^{(l)} \max\left( 1 ,   \frac{\left\lVert W \right\rVert_{\infty}}{d_{\min}} \left( \sqrt{H_{l+1}} \left( \frac{\beta\kappa^2}{\alpha}  \right)^L \left(   \sum_{k=1}^L \binom{L}{k} \Delta_{\Theta}^k \right) + 1 \right) L_P  \right) , \\
  &\Delta_{\Gamma,\Theta} \leq L_P\left(  \sum_{k=1}^L \binom{L}{k} \Delta_{\Theta}^k \right)  Q_3   C_{\nabla w} \lVert \nabla \tau \rVert_{\infty} + U_{\tilde{C_3}^{(T)}}C_{\nabla w} \lVert \nabla \tau \rVert_{\infty} + Q^{(T)} N_{P_\tau}.
\end{align}

\textbf{Finally, we are now ready to develop sufficient conditions for the bound}
\begin{equation}
    \sqrt{\frac{C\Delta_D}{\lambda_{r}}} \left( L_{NN} \left(\sqrt{\Delta_N} + \Delta_{\Gamma,\Theta} \right) + G_{NN} \right)
\end{equation}
to be no more than $\xi \epsilon_S(h,g_D)$ with $ \xi>0$, where the smaller $\xi$  is, the tighter the bound becomes.
Applying Eqs. (\ref{eq:bound1})-(\ref{eq:bound3}), we study sufficient conditions for the last inequality to hold 
\begin{align}
   & \sqrt{\frac{C\Delta_D}{\lambda_{r}}}  \left( L_{NN} \left(\sqrt{\Delta_N} + \Delta_{\Gamma,\Theta} + \right)  + G_{NN} \right) \nonumber \\ 
\leq\; & \sqrt{\frac{C}{\lambda_{r}}} U_{\sqrt{\Delta_D}} \left( U_{L_{NN}} \left(\sqrt{U_{\Delta_N}} + U_{\Delta_{\Gamma,\Theta}}  \right)  + U_{G_{NN}} \right)  \leq \xi \epsilon_S(h,g_D).
\end{align}
It requires 
\begin{equation}
      U_{L_{NN}} \left(\sqrt{U_{\Delta_N}} + U_{\Delta_{\Gamma,\Theta}}  \right)  + U_{G_{NN}}  \leq \frac{\xi \epsilon_S(h,g_D)\sqrt{\lambda_{r}}}{\sqrt{2L_{2}} C^{1.5} L_f   \bigg( \mathcal{W}_2 \left(\hat{P}_S^m, \hat{P}_T^m\right) + 2B \cdot 1.1\times 27^{\frac{D_{\mathcal{X}}}{4}}  \bigg)^{ \alpha}},
\end{equation}
which in turn requires the following:
\begin{align}
\label{eq:tight_condition1}
    & U_{G_{NN}}  \leq \frac{\xi \epsilon_S(h,g_D)\sqrt{\lambda_{r}}}{\sqrt{2L_{2}} C^{1.5} L_f   \bigg( \mathcal{W}_2 \left(\hat{P}_S^m, \hat{P}_T^m\right) + 2B \cdot 1.1\times 27^{\frac{D_{\mathcal{X}}}{4}}  \bigg)^{ \alpha}}, \\
\label{eq:tight_condition2}
    &\sqrt{U_{\Delta_N}} + U_{\Delta_{\Gamma,\Theta}}   \leq  \frac{\frac{\xi \epsilon_S(h,g_D)\sqrt{\lambda_{r}}}{\sqrt{2L_{2}} C^{1.5} L_f   \bigg( \mathcal{W}_2 \left(\hat{P}_S^m, \hat{P}_T^m\right) + 2B \cdot 1.1\times 27^{\frac{D_{\mathcal{X}}}{4}}  \bigg)^{ \alpha}}-U_{G_{NN}} }{U_{L_{NN}}} = \eta.
\end{align}
Applying the expression of $U_{G_{NN}}$, the first condition as in Eq. (\ref{eq:tight_condition1}) results in Eq. (\ref{eq:final_condition1}), requiring a limited perturbation strength over the MLP classifier, as 
\begin{equation}
    \sum_{k=1}^L \binom{L}{k} \Delta_{M}^k  \leq \frac{\xi \epsilon_S(h,g_D)\sqrt{\lambda_{r}}}{\sqrt{2L_{2}} C^{1.5} L_f   \bigg( \mathcal{W}_2 \left(\hat{P}_S^m, \hat{P}_T^m\right) + 2B \cdot 1.1\times 27^{\frac{D_{\mathcal{X}}}{4}}  \bigg)^{ \alpha}  \left( \frac{\beta}{\alpha} \kappa^2 \right)^L }.
\end{equation}
For the sake of convenience, we develop conditions for $\sqrt{U_{\Delta_N}} \leq 0.5\eta$ and $U_{\Delta_{\Gamma,\Theta}} \leq 0.5\eta $, separately. 
Applying Eq. (\ref{eq:bound2}), it requires
\begin{equation}
    \frac{ \log(N)}{N^{ \frac{1}{D_{\mathcal{X}}+1}}} \leq \frac{0.25\eta^2}{Q_1 +  Q_2\log\left(\frac{2}{\rho}\right) }.
\end{equation}
Applying Eq. (\ref{eq:bound3}), it requires 
\begin{equation}
    L_P\left(  \sum_{k=1}^L \binom{L}{k} \Delta_{\Theta}^k \right)  Q_3   C_{\nabla w} \lVert \nabla \tau \rVert_{\infty} + U_{\tilde{C_3}^{(T)}}C_{\nabla w} \lVert \nabla \tau \rVert_{\infty} + Q^{(T)} N_{P_\tau} \leq 0.5\eta.
\end{equation}
This in turn requires a limited perturbation strength over the MPNN feature extractor through both network weighs and RGM deformation, given as 
\begin{equation}
    \left[ L_P\left(  \sum_{k=1}^L \binom{L}{k} \Delta_{\Theta}^k \right)  Q_3 + U_{\tilde{C_3}^{(T)}} \right]  C_{\nabla w} \lVert \nabla \tau \rVert_{\infty}+ Q^{(T)} N_{P_\tau} \leq 0.5\eta,
\end{equation}
and conditions over the RGMs, as
\begin{equation}
   C_{\nabla w} \lVert \nabla \tau \rVert_{\infty}   \leq \frac{0.5\eta -  Q^{(T)} N_{P_\tau}}{L_P\left(  \sum_{k=1}^L \binom{L}{k} \Delta_{\Theta}^k \right)  Q_3+ U_{\tilde{C_3}^{(T)}}}.
\end{equation}
These correspond to the four conditions in Eqs. (\ref{eq:final_condition1})-(\ref{eq:final_condition4}), which together result in Eq. (\ref{eq:bound_tightneess}).

\end{proof}

\section{Third-Party Results}

To be self-contained, we re-state in this section the existing results that support our result development.

\begin{lemma}[Lemma 4, \citet{keriven2020convergence}]
\label{lem:convergence_0}
    Consider an RGM $\Gamma=(W, P, f)$ satisfying Ass. \ref{ass:space} and \ref{ass:kernel}, we draw a graph $G_N \sim \Gamma$ with $N$ nodes $x_1,\cdots,x_N$. Let $\rho \in (0,1)$. Then, with probability at least $1 - \rho$, the following holds:
    \begin{align}
        & \left\lVert \frac{1}{N}\sum_{i=1}^{N} W(\cdot,x_i)f(x_i) - \int W(\cdot,x)f(x) dP(x) \right\rVert_\infty \nonumber \\
        \lesssim \; & \frac{\left\lVert f \right\rVert_\infty \left(L_W^{\infty}(\sqrt{\log(C_{\mathcal{X}})} + \sqrt{D_\mathcal{X}}) + (\sqrt{2}\left\lVert W \right\rVert_\infty + L_W) \sqrt{\log\left(\frac{2}{\rho}\right)}\right)}{\sqrt{N}}.
    \end{align}
\end{lemma}

\begin{remark}
    Re-written by assuming $\mathcal{N}(\mathcal{X},\epsilon,d) \leq C_{\mathcal{X}} \cdot \epsilon^{-D_{\mathcal{X}}}$ and ignoring constants. Original results obtained using Hoeffding's inequality, Dudley's inequality (Thm 8.1.6, \citet{vershynin2018high}), and other results in (Sec 2, \citet{vershynin2018high}).
\end{remark}

\begin{lemma}[Lemma 3, \citet{maskey2022generalization}]
\label{lem:uniform_bound}
    Given a compact metric space $(\mathcal{X}, d)$ satisfying Ass. \ref{ass:space} and a $L_F$-Lipschitz continuous function $f: \mathcal{X} \to \mathbb{R}^F$. Suppose a set of points $X_N = \{ x_i \}_{i=1}^N$ are sampled from probability distribution $P$ in space $\mathcal{X}$. Then, the following holds with probability at least $1-\rho$:
    \begin{align}
        & \left\lVert \frac{1}{N} \sum_{i=1}^N F(x_i) - \int_{\mathcal{X}} F(x) \, dP(x) \right\rVert_{\infty} \nonumber \\
        \leq \; & N^{-\frac{1}{2(D_{\mathcal{X}}+1)}} \left( 2 L_F + \frac{C_{\mathcal{X}} \lVert F \rVert_{\infty}}{\sqrt{2}} \sqrt{\log(C_{\mathcal{X}}) + \frac{D_{\mathcal{X}}\log(N)}{2(D_{\mathcal{X}+1})}  + \log\left(\frac{2}{\rho}\right)} \right).
    \end{align}
\end{lemma}

\begin{remark}
    Obtained by exploiting covering balls structures in space $\mathcal{X}$. Hoeffding's inequality is used to bound the error of Monte Carlo approximation towards a single covering ball, i.e. points in $X_N$ that fall into a single ball can be regarded as an approximation of this ball which leads to an approximation error. A joint event, where such small approximation errors hold for all balls, is constructed with respect to choice of samples $X_N = \{ x_i \}_{i=1}^N$. As indicated in the lemma, such event occurs with probability at least $1-\rho$.
\end{remark}

\begin{theorem}[Theorem 7, \citet{keriven2020convergence}]
\label{thm:Wasserstein}
    Let $(\mathcal{X}, d)$ be a compact metric space with $diam(\mathcal{X}) \leq B$ and $\mathcal{N}(\mathcal{X}, \epsilon, d) \leq (B/\epsilon)^{D_\mathcal{X}}$. Let $P$ be a probability measure in $\mathcal{X}$ and $x_1,\cdots,x_n$ drawn i.i.d. from $P$ and define $\hat{P}=\frac{1}{n} \sum_i \delta_{x_i}$ as the empirical probability measure through Monte Carlo sampling. Let $\rho \in (0,1)$. Then, with probability $1-\rho$ the following holds:
    \begin{center}
        $\mathcal{W}_2\left(P,\hat{P}\right) \lesssim B \left( n^{-\frac{1}{D_\mathcal{X}}} + \left( 27^{\frac{D_\mathcal{X}}{4}} + \log(1/\rho)^{\frac{1}{4}}n^{-\frac{1}{4}} \right) \right)$.
    \end{center}
\end{theorem}

\begin{remark}
    Obtained by combining Prop. 5 and Prop. 20 in \citet{weed2019sharp} with $\epsilon'=1$, applicable to any cost function used for defining Wasserstein distance.
\end{remark}

\begin{theorem}[Theorem 1, \citet{bernstein2020distance}]
\label{thm:MLP}
    Let $f$ be two multi-layer perceptrons with nonlinearity $\varphi$ and $L$ weight matrices $\{W_l\}_{l=1}^L$. Define $\tilde{f}$ the multi-layer perceptron with the same architecture but different weight matrices $\{\tilde{W}_l\}_{l=1}^L$. Specifically, let $\{ \Delta W_l = \tilde{W}_l - W_l\}_{l=1}^L$ be the layer-wise perturbation. If there exists $\alpha,\beta > 0$ s.t. $\forall x, y$:
    \begin{equation}
        \label{ass:nonlinearity1}
        \alpha  \lVert x \rVert \leq \lVert \varphi(x) \rVert \leq \beta  \lVert x \rVert,
    \end{equation}
    \begin{equation}
        \label{ass:nonlinearity2}
        \alpha \lVert x -y \rVert \leq \lVert \varphi(x) - \varphi(y) \rVert \leq \beta  \lVert x-y \rVert,
    \end{equation}
    and all matrices $\{W_l\}_{l=1}^L$, $\{\tilde{W}_l\}_{l=1}^L$, and perturbations $\{ \Delta W_l \}_{l=1}^L$ have condition number, i.e. ratio of largest to smallest singular value, no larger than $\kappa$. Then, for all $x \in \mathbb{R}^D$, the following holds:
    \begin{equation}
        \frac{\lVert \tilde{f}(x) - f(x) \rVert}{\lVert f(x) \rVert} \leq \left( \frac{\beta}{\alpha} \kappa^2 \right)^L \left[ \prod_{k=1}^L \left( 1 + \frac{\lVert \Delta W_k \rVert_F}{\lVert W_k \rVert_F} \right) - 1 \right].
    \end{equation}
\end{theorem}

\section{Exp 1: Latent Wasserstein Distance as Domain Shift Indicator}
\label{app:exp1}

In this section, we present experimental details and results to assess latent Wasserstein distance as domain shift indicator.

\paragraph{Data and protocol.}
For data, we use the four sub-datasets of predictive toxicology challenge (PTC) \citep{helma2001predictive,morris2020tudataset} as domains:
\(D=\{\texttt{PTC\_FM},\texttt{PTC\_MM},\texttt{PTC\_FR},\texttt{PTC\_MR}\}\).
For each source domain \(S \in D\), we train a GIN classifier on \(S\) (train ratio \(0.8\)) and evaluate on all target domains \(T \in D\), producing a \(4\times4\) test-loss matrix \(\mathbf{L}\), where \(L_{S,T}\) is the mean test loss from source \(S\) to target \(T\).
All classifier results are averaged over seeds \(\{0,1,2\}\).

\begin{table}[ht]
\centering
\small
\setlength{\tabcolsep}{6pt}
\renewcommand{\arraystretch}{1.15}
\begin{tabular}{lccccccc}
\toprule
Subgroup & \#\, Graphs & mean nodes & mean edges & min nodes & max nodes & Feat. dim. & Label hist. (0/1) \\
\midrule
PTC\_FM & 349 & 14.112 & 28.968 & 2 & 64 & 18 & 206 / 143 \\
PTC\_MM & 336 & 13.973 & 28.643 & 2 & 64 & 20 & 207 / 129 \\
PTC\_FR & 351 & 14.558 & 30.006 & 2 & 64 & 19 & 230 / 121 \\
PTC\_MR & 344 & 14.288 & 29.384 & 2 & 64 & 18 & 192 / 152 \\
\bottomrule
\end{tabular}
\caption{Summary statistics of the four PTC subgroups.}
\label{tab:ptc_subgroup_stats}
\end{table}

\paragraph{Latent position estimation.}
We adopt the latent distance model of \citet{hoff2002latent}.
For each graph $G=(V,E)$ with binary adjacency $Y\in\{0,1\}^{|V|\times |V|}$ and dyad covariates $x_{ij}\in\mathbb{R}^p$, where $x_{ij}$ is a feature vector constructed from node attributes of the pair $(i,j)$.
In our implementation, we use the concatenation $x_{ij}=\textmd{concat}[|x_i-x_j|,\,(x_i-x_j)^{\odot 2}]$) where $\odot 2$ indicates element-wise square operation.
%
%
the model assigns each node $i\in V$ a latent position $z_i\in\mathbb{R}^d$ and assumes conditional dyadic independence:
\[
P(Y\mid Z,\alpha,\beta)=\prod_{i<j} \mathrm{Bern}\!\left(Y_{ij};\, p_{ij}\right),\qquad
\mathrm{logit}(p_{ij})=\alpha+\beta^\top x_{ij}-\|z_i-z_j\|_2,
\]
where $Z=\{z_i\}_{i\in V}$, $\alpha\in\mathbb{R}$ is an intercept and $\beta\in\mathbb{R}^p$ are covariate coefficients.
This parameterization makes tie probabilities decrease monotonically with latent distance.
The likelihood is invariant to global translation and orthogonal transforms of $Z$, so $Z$ is identifiable only up to such transformations.
For each domain $S\in D$ and class $c\in\{0,1\}$, we estimate $(\alpha,\beta)$ and per-graph latent coordinates $Z$ by maximum likelihood,
followed by post-processing (alignment/pooling) to obtain graph-level representations.

\paragraph{Graph-level representation and Wasserstein distance.}
For each graph $G$, let $\hat Z_G=\{\hat z_i\in\mathbb{R}^d\}_{i\in V}$ denote the post-processed latent positions, after Procrustes alignment, and after size-controlled subsampling ($n=16$ for each graph).
We then map the latent point cloud to a graph-level vector
\[
\phi(G)=\Big[\mu(\hat Z_G),\,\mathrm{Var}(\hat Z_G)\Big]\in\mathbb{R}^{2d},
\]
i.e., concatenated per-dimension mean and variance (\texttt{mean\_var}, $2d=4$).
For each domain $S\in D$ and class $c\in\{0,1\}$, this yields graph-level samples
\[
\Phi_{S,c}=\{\phi(G): G\in\mathcal{D}_{S,c}\},
\]
where $\mathcal{D}_{S,c}$ is the set of graphs in domain $S$ with label $c$.
We compute the class-wise entropic $2$-Wasserstein (Sinkhorn) distance
\[
W_{S,T}^{(c)} \;=\; W_{2,\varepsilon}\!\left(\Phi_{S,c},\Phi_{T,c}\right),
\]
with $\varepsilon=0.1$, 1000 Sinkhorn iterations, and a per-class sample cap of 2000, then aggregate
\[
W_{S,T}=\sum_{c\in\{0,1\}} W_{S,T}^{(c)}.
\]
The resulting WD matrix $\mathbf{W}$ is averaged over random seeds \(\{0,1,2\}\).

\begin{figure}[ht]
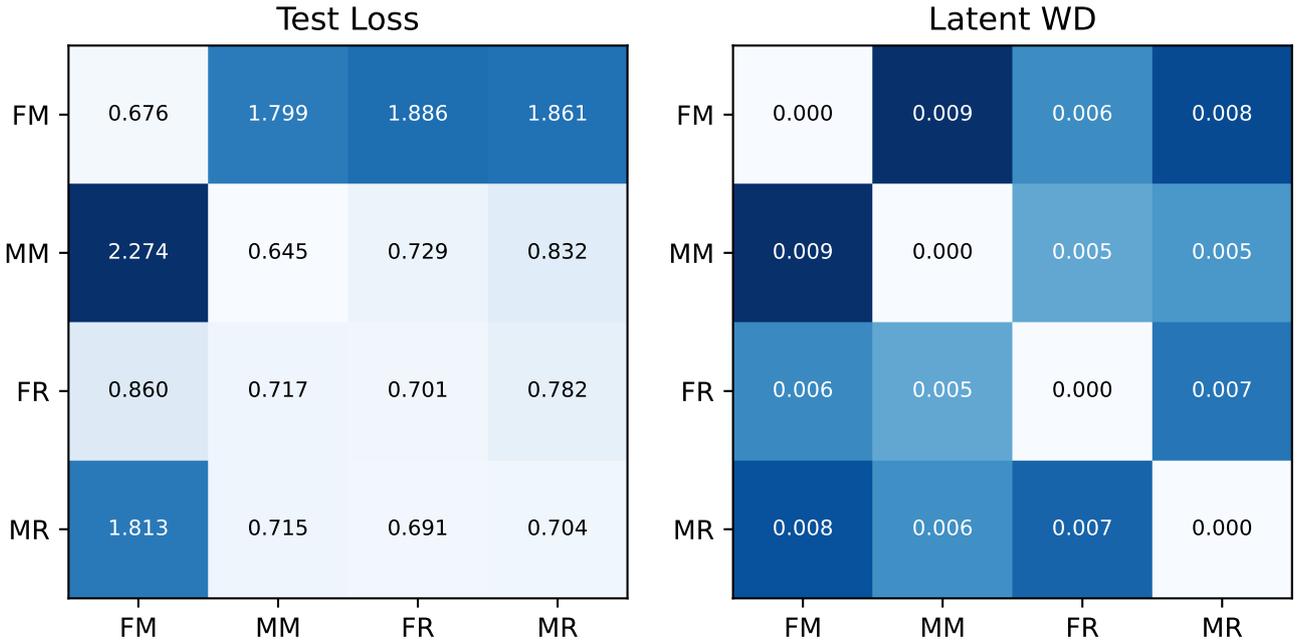

    \centering
    \subfigure{%
        \includegraphics[width=0.49\linewidth,trim=8pt 8pt 8pt 6pt,clip]{imgs/TestLoss.pdf}}
    \hfill
    \subfigure{%
        \includegraphics[width=0.49\linewidth,trim=8pt 8pt 8pt 6pt,clip]{imgs/LWD.pdf}}
    \caption{Estimated latent WD correlates with test losses.}
    \label{fig:exp1_ptc:app}
\end{figure}

\paragraph{Evaluation and Results.}
We evaluate whether the estimated shift tracks performance degradation by correlating $\{W_{S,T}\}_{S\neq T}$ with $\{L_{S,T}\}_{S\neq T}$ over the 12 off-diagonal directed domain pairs.
We report Pearson correlation coefficient (PCC) and Spearman rank correlation (SRC).
As shown in Figure~\ref{fig:exp1_ptc:app}, the off-diagonal graph-level latent Wasserstein values are in a narrow but structured range (\(\approx 0.005\)--\(0.009\)), while the test-loss matrix shows clear cross-domain degradation patterns.
The correlation between \(\mathbf{W}\) and \(\mathbf{L}\) is:
\[
\mathrm{PCC}=0.751\;(p=0.00488),\qquad
\mathrm{SRC}=0.580\;(p=0.0479).
\]
Both are significantly positive at conventional levels, indicating that graph-level latent WD is aligned with cross-domain test-loss degradation in this PTC setting.

\paragraph{Takeaway.}
Using graph-level latent summaries together with Procrustes alignment yields a statistically significant association between estimated shift and downstream degradation.

\section{Exp 2: Spectrum Geometry}
\label{app:exp2}

In this section, we introduce our experiment details and how we estimate empirical spectrum of kernel integral operator.

\subsection{Empirical Spectrum Estimation}

Let $k:\mathcal X\times\mathcal X\to\mathbb R$ be a positive semidefinite kernel and let $\mu$ be a probability measure on $\mathcal X$.
The associated kernel integral operator $T_k:L_2(\mu)\to L_2(\mu)$ is defined by
\begin{equation}
  (T_k f)(x)\;:=\;\int_{\mathcal X} k(x,x')\,f(x')\,d\mu(x').
  \label{eq:Tk_population}
\end{equation}
Given i.i.d.\ samples $\{x_i\}_{i=1}^n\sim\mu$, we form the empirical measure $\mu_n:=\frac1n\sum_{i=1}^n\delta_{x_i}$ and the empirical operator
\begin{equation}
  (T_{k,n} f)(x)\;:=\;\int_{\mathcal X} k(x,x')\,f(x')\,d\mu_n(x')
  \;=\;\frac1n\sum_{i=1}^n k(x,x_i)\,f(x_i).
  \label{eq:Tk_empirical}
\end{equation}
Restricting the eigen-equation $T_{k,n}\phi=\lambda\phi$ to the sample points yields
\begin{equation}
  \frac1n K v \;=\;\lambda v,
  \qquad K_{ij}:=k(x_i,x_j),\;\; v_i:=\phi(x_i),
  \label{eq:gram_eig}
\end{equation}
so the nonzero eigenvalues of $T_{k,n}$ coincide with those of $\frac1n K$.
Therefore, eigen-decomposition of the normalized Gram matrix provides a finite-sample spectral estimate of $T_k$.

The results of this estimation depend on:
\begin{itemize}
    \item \emph{sample size} $n$, which controls how well $\mu_n$ approximates $\mu$ and hence how close $T_{k,n}$ is to $T_k$;
    \item \emph{kernel choice and regularity}, which determine whether $T_k$ is compact and how rapidly its spectrum decays;
    \item \emph{rank constraints \& scaling effects} when $k(x,x')=\langle \phi(x),\phi(x')\rangle$ is induced by learned embeddings, since $K=\Phi\Phi^\top$ has $\mathrm{rank}(K)\le \min\{n,d\}$ and its spectral scale satisfies $\mathrm{Tr}(\tfrac1n K)=\tfrac1n \sum_{i=1}^n K_{ii} = \frac1n \sum_{i=1}^n \|\phi(x_i)\|_2^2$. If imposing normalization $\| \phi(x) \|_2=1$, then the total energy (sum of eigenvalues), which equals to trace, is 1, and increasing $d$ would spread energy into different directions; if embedding norm increases along with $d$ then the spectrum will be scaled.
\end{itemize}

Also note that despite we work with vRKHS, the definition of vRKHS should allow a trivialization to scalar case, and therefore could be estimated using above method.
One may also follow the canonical Mercer Theorem and definition of kernel integral operator in \citet{carmeli2006vector} for operator-valued kernel (see Proposition 3 \& 4 therein), however it's mathematically equivalent since they consider the vRKHS as a subspace of $L_2$ via an inclusion map.

\subsection{Settings}

We now verify the truncated-spectrum assumption and empirically probe eigenstructure for GNNs in different levels of depth and expressiveness on TUDataset \citep{morris2020tudataset}.
We aim to investigate the following questions:

Q1: \textit{does finite-spectrum assumption hold in practice?}

Q2: \textit{how does different levels of complexity, e.g., depth, expressiveness, embedding dimension, for hypothesis functions affect spectrum structures?}

\begin{table}[ht]
\centering
\small
\setlength{\tabcolsep}{6pt}
\begin{tabular}{lcccccccc}
\toprule
dataset & MUTAG & PTC & PROTEINS & NCI1 & NCI109 & COLLAB & IMDB-B & IMDB-M \\
\midrule
size        & 188  & 344  & 1113 & 4110 & 4127 & 5000 & 1000 & 1500 \\
classes     & 2    & 2    & 2    & 2    & 2    & 3    & 2    & 3    \\
avg node \# & 17.9 & 25.5 & 39.1 & 29.8 & 29.6 & 74.4 & 19.7 & 13   \\
\bottomrule
\end{tabular}
\caption{Dataset statistics. \citep{maron2019provably}}
\label{tab:datasets_app}
\end{table}

We choose three real-world data IMDB-MULTI, NCI1, PROTEINS to compare, whose basic statistics are shown in Table \ref{tab:datasets_app}. For those datasets, e.g., IMDB and COLLAB, without node feature, we use node degree binning as categorical feature for nodes.
Two factors affect empirical spectrum estimation: (1) number of data samples $n$, and (2) the dimension $d$ of feature map $\phi: \mathcal{G} \to \mathbb{R}^d$ (for GNNs). For (1), we randomly sample a subset with $n=1000$ (IMDB-MULTI), $n=2000$ (NCI1), $n=1000$ (NCI109). For (2), we use WL subtree kernel directly, $d=\mathcal{O}(hn)$; we use dot product kernel $K(G, G') = \langle \phi(G), \phi(G') \rangle_{\mathbb{R}^d}$ with embedding dimension $d=256$ for GIN and $d=512$ for PPGN.
To calculate the feature map $\phi$, we train GIN on $80\%:20\%$ train-val set split over $n$ (depends on dataset) samples, using Adam optimizer with epochs 200 and batch size 64; and we train PPGN with the same data split protocol yet defaults hyperparameters (e.g., learning rate, batch size, etc.) in the code base \citep{maron2019provably}. No embedding regularization (e.g. $\| \phi(x) \|_2 \leq 1$) are used for GIN and PPGN.
Finally, eigen-decomposition for Gram matrix after diagonal normalization is conducted to compute empirical spectrum (see last section). For all three methods, we traverse the depth of kernel $h\in\{1,2,3,4,5\}$ over seeds $0,1,2,3,4$.
Table~\ref{tab:rank-energy-full} reports the estimated truncation dimension $r_{\varepsilon} := \min\left\{r:\ \frac{\sum_{i>r}\lambda_i}{\sum_{i}\lambda_i}\le \varepsilon\right\}$, i.e., the smallest rank such that the tail eigenvalue mass beyond $r$ contributes at most $\varepsilon$ of total.

\subsection{Results}

As shown in Table~\ref{tab:rank-energy-full}, the two sub-tables reports dimension $r_{\varepsilon}$ under different conditions of $\varepsilon = \frac{\sum_{i>r} \lambda_i}{\sum_i \lambda_i}$, i.e., the residual eigenvalues beyond rank $r$ over the total eigenvalues: left for $\varepsilon=1\%$ and right for $\varepsilon = 0.1\%$. In plain words, the minimum rank index $r$ such that the energy (eigenvalues) beyond $r$ larger than $\epsilon \times \mathrm{total \; energy}$.
Usually a lower effective dimension $d_{\mathrm{eff}}$ denotes better generalisation. However, in our bound (see proposition above), a higher truncated eigenvalue (representinig the worst direction in vRKHS ball) indicates a better generalisation, meaning that the eigenvalue curve along rank $i$ should be as heavy-tail as possible, seemingly equivalent to a larger eigenvalue. This, at first glance is in contradiction, yet the nature of our assumption is controlling the worst direction and hoping that most eigenvalue concentrated on the top $r$ rank. We therefore report the $r_{\epsilon}$ above as a substitution to $d_{\mathrm{eff}}$. Observations discussed as follows:

\begin{table*}[ht]
\centering
\scriptsize
\setlength{\tabcolsep}{3pt}
\renewcommand{\arraystretch}{1.15}

\begin{adjustbox}{max width=\textwidth}
\begin{tabular}{lcccccccccc}
\toprule
 & \multicolumn{5}{c}{$r_{\varepsilon=10^{-2}}$ (mean$\pm$std)} & \multicolumn{5}{c}{$r_{\varepsilon=10^{-3}}$ (mean$\pm$std)} \\
\cmidrule(lr){2-6}\cmidrule(lr){7-11}
Model & $h=1$ & $h=2$ & $h=3$ & $h=4$ & $h=5$ & $h=1$ & $h=2$ & $h=3$ & $h=4$ & $h=5$ \\
\midrule
\multicolumn{11}{l}{\textbf{IMDB-MULTI}}\\
1-WL & \mstd{89.0}{2.45} & \mstd{243.0}{6.81} & \mstd{256.0}{6.81} & \mstd{260.8}{7.14} & \mstd{263.4}{6.95}
     & \mstd{161.8}{3.49} & \mstd{277.4}{6.95} & \mstd{279.8}{7.05} & \mstd{281.2}{6.71} & \mstd{281.8}{7.05} \\
GIN  & \mstd{12.0}{0.63} & \mstd{10.8}{0.98} & \mstd{8.2}{0.98} & \mstd{5.4}{1.02} & \mstd{3.4}{0.49}
     & \mstd{36.2}{2.04} & \mstd{37.0}{3.74} & \mstd{23.2}{5.08} & \mstd{14.8}{3.97} & \mstd{9.0}{0.63} \\
PPGN & \mstd{4.0}{0.00} & \mstd{9.4}{0.49} & \mstd{17.0}{1.26} & \mstd{20.0}{0.63} & \mstd{25.4}{3.26}
     & \mstd{7.2}{0.75} & \mstd{30.2}{2.40} & \mstd{61.8}{5.98} & \mstd{76.2}{3.82} & \mstd{95.2}{7.93} \\
\midrule
\multicolumn{11}{l}{\textbf{NCI1}}\\
1-WL & \mstd{19.8}{1.17} & \mstd{186.6}{4.67} & \mstd{498.0}{4.10} & \mstd{667.8}{4.96} & \mstd{751.6}{4.92}
     & \mstd{61.0}{3.58} & \mstd{515.0}{4.05} & \mstd{826.4}{5.00} & \mstd{899.0}{4.69} & \mstd{928.4}{4.84} \\
GIN  & \mstd{9.6}{1.74} & \mstd{12.2}{9.20} & \mstd{11.4}{13.29} & \mstd{20.2}{10.11} & \mstd{6.4}{6.97}
     & \mstd{26.2}{4.17} & \mstd{45.4}{33.92} & \mstd{41.0}{47.44} & \mstd{73.6}{35.87} & \mstd{27.8}{28.22} \\
PPGN & \mstd{12.6}{1.50} & \mstd{26.8}{2.56} & \mstd{34.8}{5.00} & \mstd{44.4}{11.94} & \mstd{53.8}{4.71}
     & \mstd{61.8}{8.26} & \mstd{180.4}{10.63} & \mstd{223.2}{12.32} & \mstd{248.0}{23.32} & \mstd{274.8}{10.19} \\
\midrule
\multicolumn{11}{l}{\textbf{PROTEINS}}\\
1-WL & \mstd{36.8}{0.40} & \mstd{554.2}{2.40} & \mstd{730.6}{1.96} & \mstd{790.6}{1.96} & \mstd{821.0}{1.79}
     & \mstd{85.0}{0.63} & \mstd{859.4}{1.50} & \mstd{915.4}{1.50} & \mstd{928.6}{1.85} & \mstd{935.4}{2.06} \\
GIN  & \mstd{7.6}{0.49} & \mstd{14.0}{1.10} & \mstd{17.4}{3.77} & \mstd{19.4}{2.50} & \mstd{23.2}{3.66}
     & \mstd{15.8}{0.75} & \mstd{52.0}{4.24} & \mstd{66.4}{15.62} & \mstd{74.4}{7.39} & \mstd{82.6}{10.03} \\
PPGN & \mstd{15.6}{3.44} & \mstd{51.4}{23.00} & \mstd{32.0}{12.15} & \mstd{13.0}{21.00} & \mstd{1.0}{0.00}
     & \mstd{70.6}{14.97} & \mstd{203.4}{72.86} & \mstd{164.4}{57.84} & \mstd{55.4}{100.81} & \mstd{1.0}{0.00} \\
\bottomrule
\end{tabular}
\end{adjustbox}
\caption{Estimated truncation rank $r_{\varepsilon}$ at which the normalized spectrum energy residual falls below $\varepsilon$.}
\label{tab:rank-energy-full}
\end{table*}

\textbf{Valid Finite Spectrum Assumption.} As shown in Table~\ref{tab:rank-energy-full}, all methods on all datasets presents a concentrated spectrum pattern such that most eigenvalues concentrate in top $r$ directions, indicating our assumption holds in practice. Also, this observation is valid since the $r_{\epsilon=1\%}$ are all smaller than $n$ (1000 or 2000), and thus the upper limit of $r$. Even if not, the finite embedding dimension $d$ for practical graph neural networks would also limit such truncation, and therefore assumptions always hold.
Table~\ref{tab:rank-energy-large} further shows scaled results on COLLAB (4K), NCI1 (4K), NCI109 (4K) which also demonstrates valid truncated spectrum assumption.
On both sets we observe higher depth $h$ yields higher intrinsic dimension.


\begin{table*}[t]
\centering
\scriptsize
\setlength{\tabcolsep}{3pt}
\renewcommand{\arraystretch}{1.15}
\begin{adjustbox}{max width=\textwidth}
\begin{tabular}{lccccc ccccc}
\toprule
 & \multicolumn{5}{c}{$r_{\varepsilon=10^{-2}}$ (mean$\pm$std)} & \multicolumn{5}{c}{$r_{\varepsilon=10^{-3}}$ (mean$\pm$std)} \\
\cmidrule(lr){2-6}\cmidrule(lr){7-11}
Dataset & $d=64$ & $d=128$ & $d=256$ & $d=512$ & $d=1024$ & $d=64$ & $d=128$ & $d=256$ & $d=512$ & $d=1024$ \\
\midrule
IMDB-MULTI & 4.2 $\pm$ 0.75 & 2.8 $\pm$ 0.40 & 2.6 $\pm$ 0.49 & 2.0 $\pm$ 0.00 & 2.0 $\pm$ 0.00 & 9.0 $\pm$ 0.89 & 6.4 $\pm$ 1.02 & 4.4 $\pm$ 0.49 & 2.8 $\pm$ 0.40 & 2.6 $\pm$ 0.49 \\
NCI1 & 5.8 $\pm$ 0.75 & 10.6 $\pm$ 1.02 & 10.0 $\pm$ 1.10 & 8.0 $\pm$ 1.41 & 3.8 $\pm$ 0.75 & 16.6 $\pm$ 1.02 & 35.2 $\pm$ 3.66 & 40.4 $\pm$ 3.77 & 28.4 $\pm$ 8.82 & 6.8 $\pm$ 1.33 \\
PROTEINS & 7.6 $\pm$ 1.36 & 7.0 $\pm$ 1.26 & 4.6 $\pm$ 1.20 & 3.2 $\pm$ 0.40 & 1.8 $\pm$ 0.40 & 21.6 $\pm$ 2.33 & 23.8 $\pm$ 4.17 & 15.2 $\pm$ 5.46 & 7.2 $\pm$ 1.60 & 4.4 $\pm$ 2.06 \\
\bottomrule
\end{tabular}
\end{adjustbox}
\caption{Estimated truncation rank $r_\varepsilon$ at which the normalized residual spectrum energy falls below $\varepsilon$, reported across embedding dimensions $d$ for model GIN.}
\label{tab:rank-energy-style}
\vspace{-0.6em}
\end{table*}

\begin{table*}[t]
\centering
\scriptsize
\setlength{\tabcolsep}{3pt}
\renewcommand{\arraystretch}{1.15}

\begin{adjustbox}{max width=\textwidth}
\begin{tabular}{lccccc ccccc}
\toprule
 & \multicolumn{5}{c}{$r_{\varepsilon=10^{-2}}$ (mean$\pm$std)} & \multicolumn{5}{c}{$r_{\varepsilon=10^{-3}}$ (mean$\pm$std)} \\
\cmidrule(lr){2-6}\cmidrule(lr){7-11}
Model & $h=1$ & $h=2$ & $h=3$ & $h=4$ & $h=5$ & $h=1$ & $h=2$ & $h=3$ & $h=4$ & $h=5$ \\
\midrule
\multicolumn{11}{l}{\textbf{COLLAB}}\\
1-WL & \mstd{303.6}{62.36} & \mstd{2099.8}{725.90} & \mstd{2303.4}{798.70} & \mstd{2397.6}{831.81} & \mstd{2454.8}{851.41} & \mstd{1224.8}{385.47} & \mstd{2631.0}{911.51} & \mstd{2690.2}{932.60} & \mstd{2715.2}{941.10} & \mstd{2729.8}{946.41} \\
GIN & \mstd{3.0}{0.00} & \mstd{2.4}{0.49} & \mstd{2.8}{0.40} & \mstd{3.0}{0.00} & \mstd{2.8}{0.40} & \mstd{6.8}{0.75} & \mstd{4.4}{0.49} & \mstd{3.8}{0.40} & \mstd{3.8}{0.40} & \mstd{4.0}{0.63} \\
PPGN & \mstd{5.8}{0.40} & \mstd{18.6}{6.37} & \mstd{46.4}{15.08} & \mstd{59.0}{24.12} & \mstd{45.8}{24.21} & \mstd{15.8}{2.93} & \mstd{105.6}{42.86} & \mstd{240.8}{72.27} & \mstd{271.2}{91.25} & \mstd{232.0}{81.25} \\
\midrule
\multicolumn{11}{l}{\textbf{NCI1}}\\
1-WL & \mstd{20.4}{0.49} & \mstd{287.6}{1.02} & \mstd{1295.8}{1.60} & \mstd{2107.4}{2.24} & \mstd{2563.0}{2.76} & \mstd{73.4}{0.49} & \mstd{1092.2}{1.17} & \mstd{2720.2}{2.64} & \mstd{3237.4}{2.65} & \mstd{3460.4}{2.24} \\
GIN & \mstd{10.4}{0.49} & \mstd{19.2}{2.56} & \mstd{23.0}{3.29} & \mstd{21.2}{1.94} & \mstd{17.8}{3.87} & \mstd{30.0}{1.79} & \mstd{66.4}{5.54} & \mstd{83.8}{7.52} & \mstd{84.2}{4.71} & \mstd{73.4}{10.33} \\
PPGN & \mstd{18.6}{2.15} & \mstd{100.2}{3.76} & \mstd{162.2}{5.23} & \mstd{165.2}{54.34} & \mstd{218.8}{15.78} & \mstd{106.4}{7.68} & \mstd{359.0}{3.63} & \mstd{416.0}{3.29} & \mstd{399.6}{67.04} & \mstd{445.4}{6.65} \\
\midrule
\multicolumn{11}{l}{\textbf{NCI109}}\\
1-WL & \mstd{20.0}{0.00} & \mstd{286.8}{1.17} & \mstd{1309.2}{2.64} & \mstd{2130.8}{2.23} & \mstd{2582.8}{2.32} & \mstd{73.8}{0.40} & \mstd{1098.0}{2.37} & \mstd{2743.4}{2.06} & \mstd{3243.4}{2.33} & \mstd{3465.2}{2.71} \\
GIN & \mstd{10.2}{0.40} & \mstd{21.2}{1.72} & \mstd{20.2}{1.17} & \mstd{17.2}{3.82} & \mstd{18.6}{3.72} & \mstd{30.0}{1.67} & \mstd{71.6}{4.22} & \mstd{76.2}{2.32} & \mstd{71.8}{11.44} & \mstd{77.0}{12.92} \\
PPGN & \mstd{20.8}{2.79} & \mstd{95.4}{11.62} & \mstd{143.4}{10.63} & \mstd{188.6}{9.48} & \mstd{214.8}{6.88} & \mstd{120.8}{7.98} & \mstd{351.6}{10.86} & \mstd{403.4}{8.48} & \mstd{431.0}{4.82} & \mstd{443.8}{3.60} \\
\bottomrule
\end{tabular}
\end{adjustbox}

\caption{(Scaled) Estimated truncation rank $r_\varepsilon$ at which the normalized spectrum energy residual falls below $\varepsilon$.}
\label{tab:rank-energy-large}
\vspace{-0.6em}
\end{table*}

\textbf{Nearly exponential decay.}
We observe in Figure~\ref{fig:spec_v1} that eigenvalues decay at a nearly exponential rate, which aligns with approximation-theoretic predictions \citep{belkin2018approximation}, but is less compatible with polynomial decay rates often assumed in classic analyses \citep{caponnetto2007optimal,kuo2008multivariate,fischer2020sobolev,chen2023infty}.
To interpret this decay for GIN/PPGN, recall that we use a dot-product kernel
$K(G,G')=\langle \phi(G),\phi(G')\rangle_{\mathbb{R}^d}$,
where $\phi(G)\in\mathbb{R}^d$ is the learned graph embedding.%
Let $\Phi\in\mathbb{R}^{n\times d}$ stack embeddings of $n$ graphs as rows, i.e., $\Phi_{i:}=\phi(G_i)^\top$.%
Then the Gram matrix used for eigendecomposition is simply
$K=\Phi\Phi^\top$, meaning $K_{ij}$ is the dot product between two embeddings.%
Equivalently, looking at $K$ is the same as looking at the ``energy distribution'' of the embedding coordinates: the nonzero eigenvalues of $K$ match those of $\Phi^\top\Phi$ (up to a scale), which is exactly the (uncentered) covariance/second-moment matrix of embeddings.%
Hence, a fast eigenvalue decay says that most embedding variance lies in only a few principal directions, while the remaining directions carry little energy.%
This is why we view small $r_\varepsilon\ll d$ as \emph{feature compressibility}: the learned representations behave as if they have a low intrinsic dimension.%

\textbf{Consistent feature compressibility and the effect of embedding dimension.}
We further probe compressibility by fixing GIN with depth $h=5$ and varying the embedding dimension $d\in\{64,128,256,512,1024\}$, training with an 80\%/20\% train/val split (hyperparameters selected on the validation set) and evaluating $r_{\varepsilon=1\%}$ on the full dataset; results are averaged over 5 seeds.
Table~\ref{tab:rank-energy-style} shows that $r_{\varepsilon}$ remains finite even as $\varepsilon$ decreases (e.g., $1\%\to 0.1\%$), supporting the truncated-spectrum assumption.
Moreover, we empirically observe that larger $d$ often yields smaller $r_{\varepsilon}$ (i.e., a more concentrated spectrum).
This trend is not a mathematical necessity, but we provide several explanations as follows:

(A) \emph{Low intrinsic task dimension:} the classification signal may lie in a $k\ll d$ subspace; increasing $d$ provides extra degrees of freedom that can remain unused, so energy concentrates on the same few principal directions.

(B) \emph{Implicit regularization/optimization bias:} even without explicit embedding normalization, training dynamics (e.g., Adam, early stopping) tend to amplify a small set of discriminative directions while leaving many coordinates near initialization, making the embedding covariance effectively low-rank as $d$ grows.

(C) \emph{Energy-thresholded truncation:} $r_\varepsilon$ is defined by a tail-energy criterion $\sum_{i>r}\lambda_i/\sum_i\lambda_i\le\varepsilon$; if additional dimensions mainly contribute near-zero eigenvalues, the relative tail energy decreases and the threshold is met at a smaller $r$.

Overall, these observations indicate that learned graph representations are highly compressible under our setup, and that increasing $d$ mainly adds low-variance directions rather than increasing the intrinsic complexity.


\begin{figure}[ht!]
  \centering
  \includegraphics[width=\textwidth]{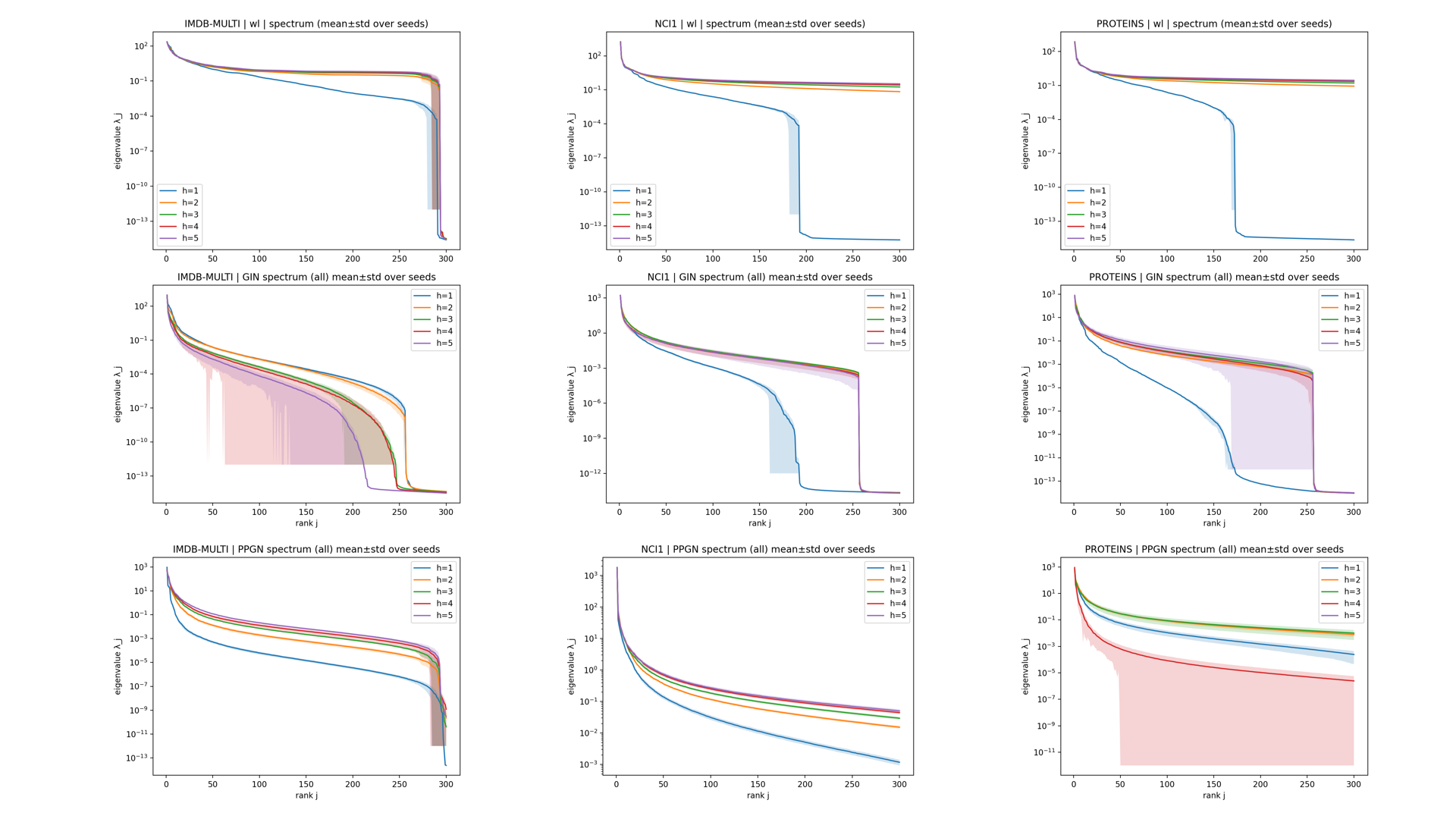}
  \caption{Eigenvalues on IMDB-MULTI (1k), NCI1 (2k), PROTEINS (1k).}
  \label{fig:spec_v1}
\end{figure}

\begin{figure}[ht!]
  \centering
  \includegraphics[width=\textwidth]{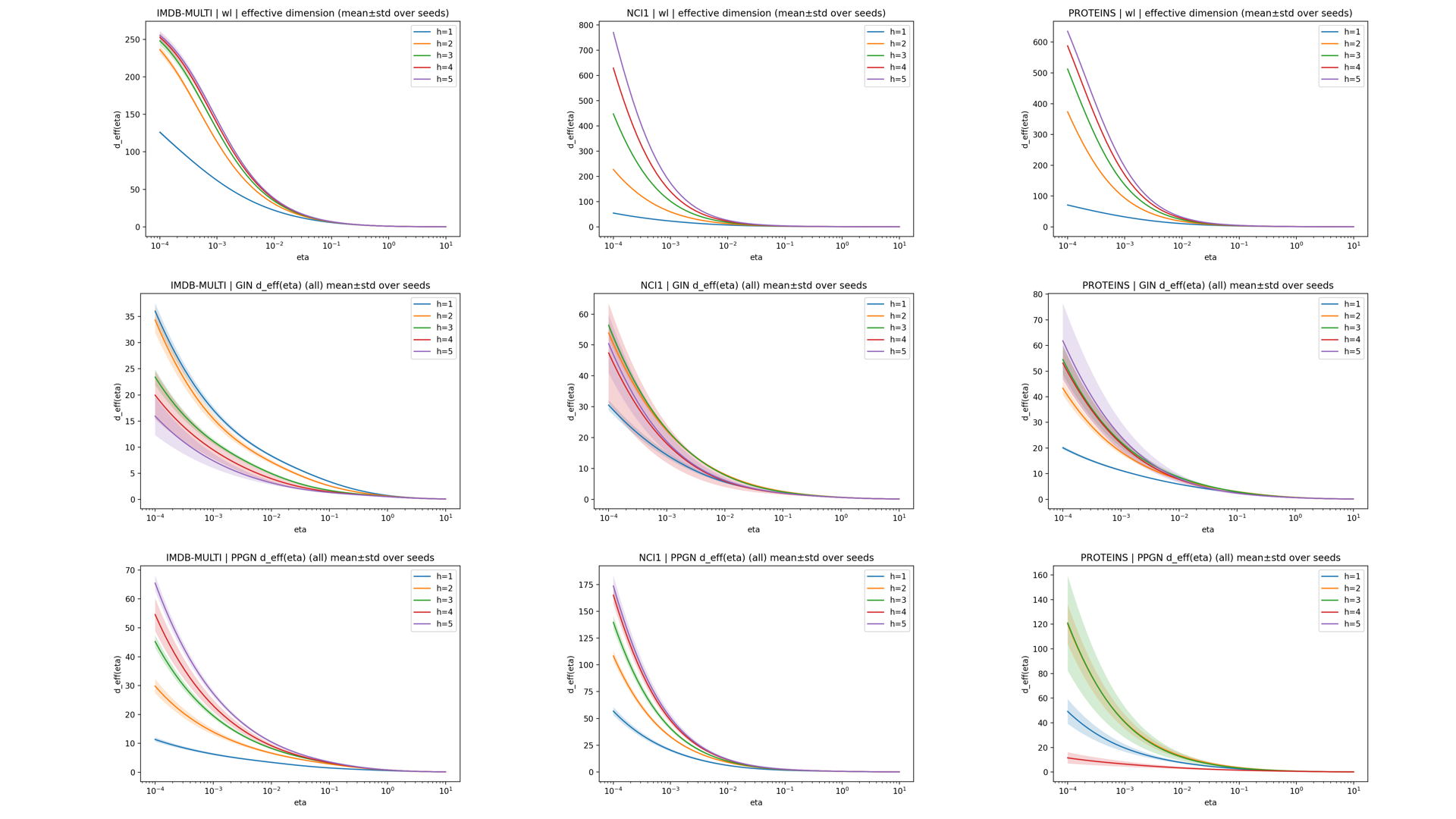}
  \caption{$d_{\mathrm{eff}}$ on IMDB-MULTI (1k), NCI1 (2k), PROTEINS (1k).}
  \label{fig:deff_v1}
\end{figure}

\begin{figure}[ht!]
  \centering
  \includegraphics[width=\textwidth]{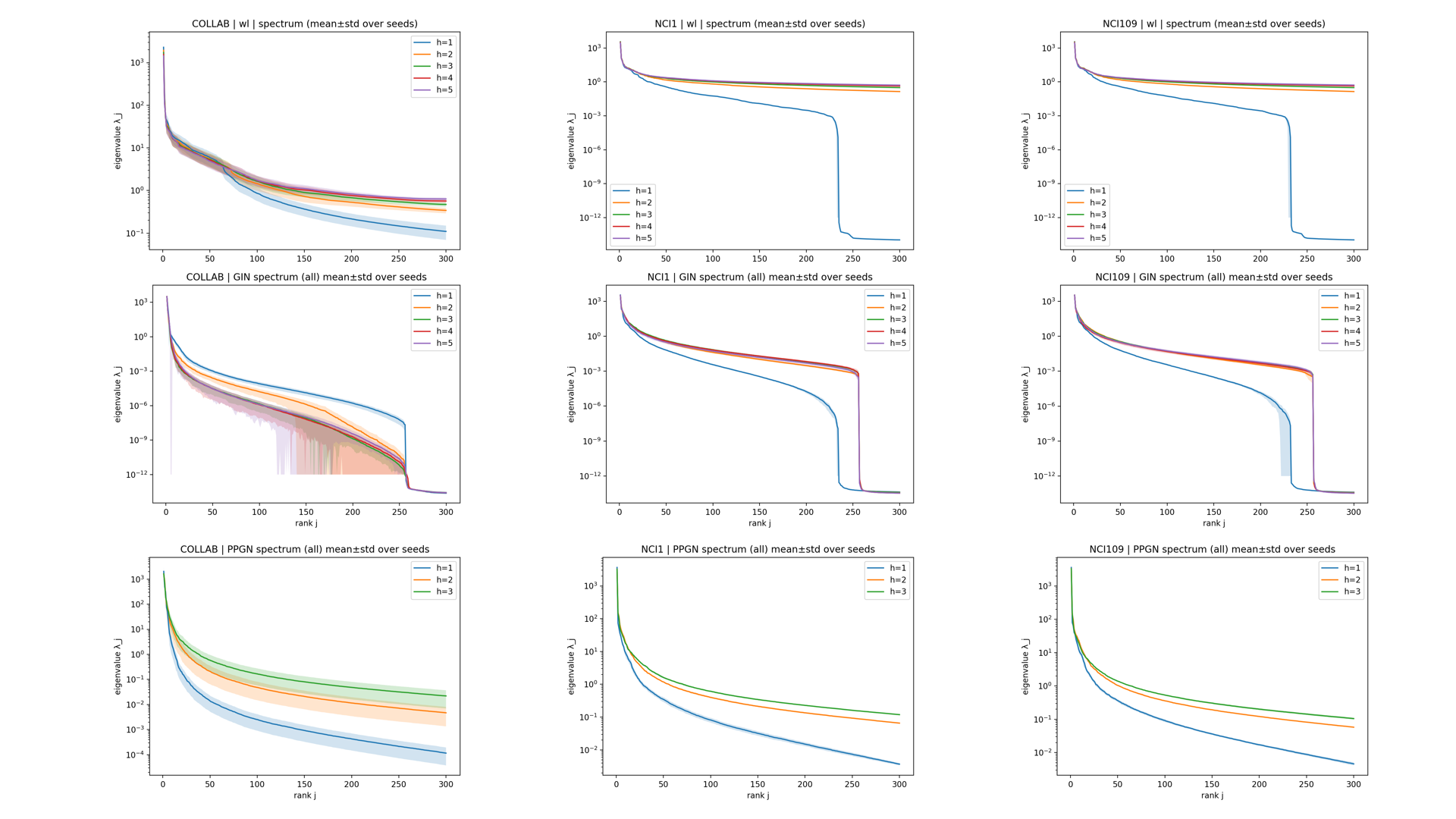}
  \caption{Eigenvalues on IMDB-MULTI (4k), NCI1 (4k), PROTEINS (4k).}
  \label{fig:spec_v3}
\end{figure}

\begin{figure}[ht!]
  \centering
  \includegraphics[width=\textwidth]{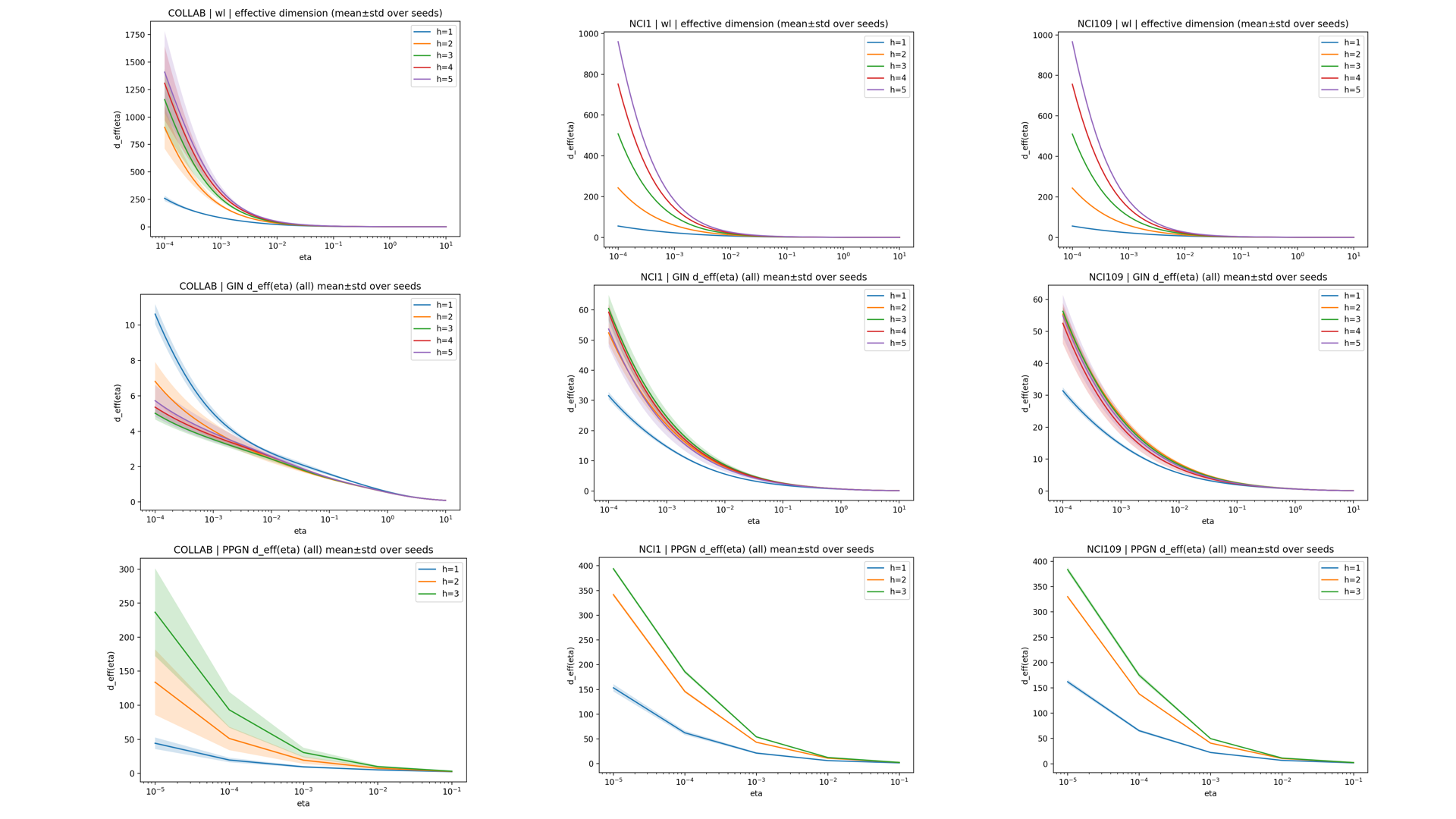}
  \caption{$d_{\mathrm{eff}}$ on IMDB-MULTI (4k), NCI1 (4k), PROTEINS (4k).}
  \label{fig:deff_v3}
\end{figure}

\subsection{Others}

We briefly discuss why we choose to include parameterization (WL-kernel) and why expressiveness as measure of complexity.

\begin{definition}[Definition 2 \& 4 in \citet{shervashidze2011weisfeiler}]
    Let K be any kernel for graphs, that we will call the base kernel. Then the Weisfeiler-Lehman kernel with $h$ iterations with the base kernel K is defined as
    \begin{equation}
        k_{WL}^{(h)}(G, G') = k(G_0,G_0') + k(G_1,G_1') + \cdots + k(G_h,G_h'),
    \end{equation}
    where $h$ is number of Weisfeiler-Lehman iterations and $\{G_0,\cdots,G_h\}$ and $\{G_0',\cdots,G_h'\}$ are the Weisfeiler-Lehman sequences of $G$ and $G'$ respectively. Equivalently, the Weisfeiler-Lehman subtree kernel on two graphs $G$ and $G'$ with $h$ iterations can be also defined as:
    \begin{equation}
        k_{WL}^{(h)}(G, G') = \langle \phi_{WLsubtree}^{(h)}(G), \phi_{WLsubtree}^{(h)}(G') \rangle,
    \end{equation}
    where $\phi_{WLsubtree}^{(h)}(G)$ defines the feature map obtained by $h$ Weisfeiler-Lehman iterations.
\end{definition}

We present a brief definition for WL kernel as above, however kindly refer to \citet{shervashidze2011weisfeiler}, especially Figure 2 for illustrative example of calculation.
In short, WL kernel aggregates neighbour labels and maintains for each node a multiset. For $h$-depth of kernel then $h$-hop neighbours’ labels would be aggregated into current node. This is essentially the mechanism of message-passing in graph neural network. Although GNNs are usually known to achieve better performance than WL kernel due to learnable mappings, WL kernel however still  serves as a strong baseline in certain datasets, and as a function with strict 1-WL expressiveness which is empirically \& theoretically proven useful for better architecture design. 

In recent years, there has been a prevalent research topic over higher-order GNN, i.e., extending 1-WL to 2-WL or 3-WL regarding the expressiveness, because MPNNs are known at most up to 1-WL and thus unable to distinguish isomorphism graphs. The major difference between higher-order $k$-WL lies in number of nodes to preserve: in 1-WL, we preserve only $n$ nodes and for each a multiset; yet in $k$-WL, we preserve a $k$-tuple of nodes and for each a multiset. For more details on Weisfeiler-Lehman and its stories in ML, we refer to \citet{morris2023weisfeiler}. Yet, the challenge primarily lies in computation since the complexity increases to $O(n^k)$ and for large graphs it’s barely possible. Thanks to Maron et al. [7,8,9], who proposed a series of work on networks allowing higher-order tensorization with higher-order expressivity guarantees, we now have PPGN \citep{maron2019provably} with up to 3-WL network.

Although higher-order expressiveness, by intuition, empowers GNNs with stronger fitting capabilities, it is still in ambiguity if and when are these higher-order GNNs guaranteed to achieve better performance and robustness under distribution shift. In particular, recent works \citep{herbst2025higher,li2025towards,maskey2025graph} probe this question on theory, and all of them inevitably rely upon certain distance metric between graphs. This however trigger another line of thinking: \textit{when are we guaranteed to distinguish graphs sampled from RGMs?} This question was preliminarily answered in the long version section below, where we propose to compare the kernel mean embedding of graphs, i.e., via the so-called metric maximum mean discrepancy. We hope to investigate more on illustrating relations between expressiveness and generalisation in RKHS, and how to connect reproducing kernel with graph kernel.

\section{Exp 3: Amplitude}
\label{app:exp3}

\subsection{Simulations}
\label{app:exp3:sim}

We conduct experiments to illustrate and verify our theoretical findings using synthetic graph data for multi-class classification. 
We conduct three sets of experiments to validate implications of our theoretical results, on sample complexity, domain divergence, and the effect of class number.


\textbf{Data Generation.} For each class, we use an RGM that employs a Gaussian distribution in a 4-dimensional latent space to sample nodes with distinct means $\{m_j\}_{j=1}^C$ and identity covariance matrix. 
Each RGM uses either an Erd{\H{o}}s-R{\'{e}}nyi or $\epsilon$-kernel for adjacency matrix, controlled by their kernel parameters. 
We vary the latent positions and kernel parameters for different classes so that graphs from different classes differ in both structures and node features. 
The mapping function $f$ computes an 8-dimensional feature by padding zeros for extra 4 dimensions.

\textbf{DA Setting.} To generate the source-domain data, we generate $50$ graphs for each class. To construct the target domain, we follow the same sampling protocol, but add a Gaussian mean shift in the latent space for each RGM, i.e., $m_j \to m_j + s$, creating shifted latent positions in the target domain.

\textbf{Model Implementation.} We implement MPNN consisting of 3 hidden layers with a hidden dimension of 16. For training, we tune the learning rate over [1e-3, 3e-4, 1e-4] with a batch size 4 and early stopping at 20 epochs. We observe the best target-domain loss. For each setting, we repeat the training-testing trial for 5 times, and report the mean and variance of target-domain losses.
%


\begin{figure}[t]
  \centering
  \begin{subfigure}
    \centering
    \includegraphics[width=0.46\linewidth]{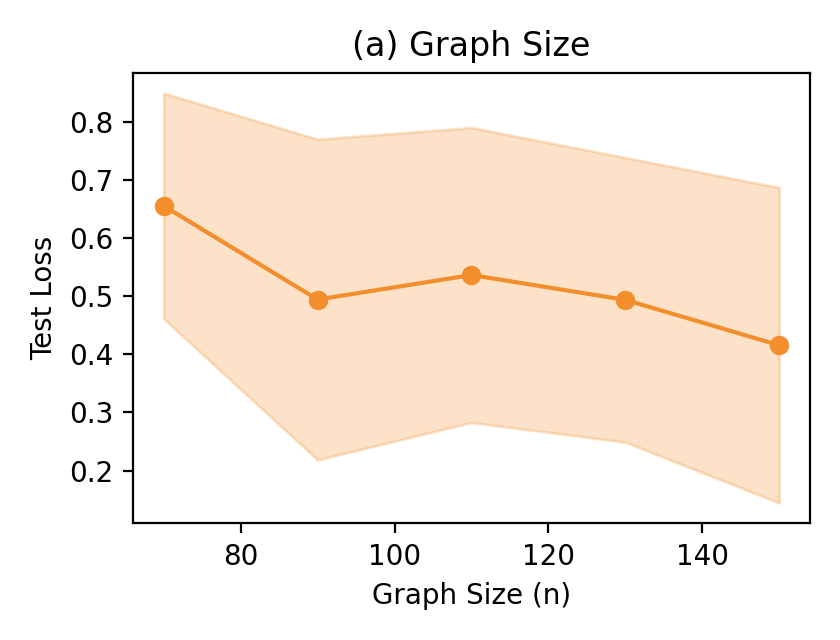}
    \label{fig:test_loss_variation:a}
  \end{subfigure}
  \hfill
  \begin{subfigure}
    \centering
    \includegraphics[width=0.46\linewidth]{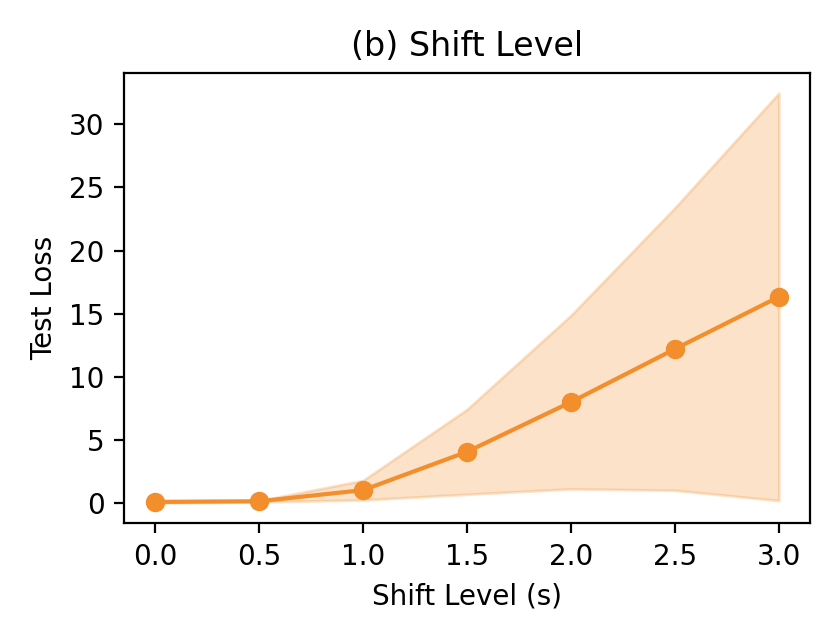}
    \label{fig:test_loss_variation:b}
  \end{subfigure}
  \caption{Target-domain test loss under different experimental conditions: \textbf{(a)} Graph size ($n$), \textbf{(b)} Shift level ($s$). Shaded areas indicate variance around the mean loss.}
  \label{fig:test_loss_variation}
\end{figure}


\textbf{Sample Complexity.} A main indication of our theoretical result is that DA error depends on graph size, and it reduces as the size increases. To validate this, we fix the class number as $C=3$, domain shit as $s=3.0$ for all classes, and use the Erd{\H{o}}s-R{\'{e}}nyi kernel. Following \citet{yehudai2021local}, we increase the node number when sampling graphs for training, e.g., from $n=50$ to $n=150$, and test in target domain with fixed graph size $150$. The target-domain loss is reported in Figure \ref{fig:test_loss_variation} (a). As $n$ increases, the decrease of the test loss in target domain coincides with our theoretical result.

\textbf{Domain Divergence.} Another indication is that DA error reduces as domain divergence in latent space decreases. Specifically, Eq. (\ref{lem:wasserstein:goal}) indicates an increase along with Wasserstein 2-distance between source and target latent distributions. Since the Wasserstein 2-distance between two Gaussians $\mu_1=\mathcal{N}(m_1,\Sigma_1)$ and $\mu_2=\mathcal{N}(m_2,\Sigma_2)$ has an analytical form, i.e., $\mathcal{W}_2(\mu_1,\mu_2) \propto \left\lVert m_1 - m_2 \right\rVert_2 \propto s$, the shift $s$ is thus proportional to Wasserstein distance. Therefore, to validate our domain divergence result, we fix the class number $C=3$, source and target domain graph size n=20, use $\epsilon$-kernel, while vary the shift level $s$. The result is reported in Figure \ref{fig:test_loss_variation} (b). We observe that a larger shift $s$ in latent space (equivalently a larger Wasserstein distance) leads to higher testing loss in target domain following a polynomial trend. This matches our theoretical results.

\subsection{Transfer Accuracy on Mutagenicity Data}

\paragraph{Dataset.}
Following general protocols \citep{yin2023coco,luo2024gala}, we use \textsc{Mutagenicity} and construct four domains $\{\mathrm{M0},\mathrm{M1},\mathrm{M2},\mathrm{M3}\}$ by splitting graphs into equally-sized subsets according to edge-density levels (from denser, e.g., M0, to sparser, e.g., M3).
The resulting shift is a controlled structural change: transferring from $\mathrm{M0}\rightarrow \mathrm{M3}$ corresponds to a larger density decrease than $\mathrm{M0}\rightarrow \mathrm{M1}$.
We evaluate all pairwise transfer directions among these domains (12 source$\rightarrow$target pairs).

\begin{table}[ht]
\centering
\resizebox{0.8\columnwidth}{!}{%
\begin{tabular}{lcccccc}
\toprule
Split & \#Graphs & Nodes (mean$\pm$std) & Edges (mean$\pm$std) & Density (mean$\pm$std) & Label 0 ratio & Label 1 ratio \\
\midrule
M0 & 1084 & 52.6153 $\pm$ 28.1772 & 104.6328 $\pm$ 34.6927 & 0.045216 $\pm$ 0.012618 & 0.4483 & 0.5517 \\
M1 & 1084 & 31.0268 $\pm$ 2.6188  & 65.3948  $\pm$ 7.6232  & 0.070414 $\pm$ 0.005084 & 0.6771 & 0.3229 \\
M2 & 1084 & 23.2094 $\pm$ 2.5559  & 47.7823  $\pm$ 6.8632  & 0.093360 $\pm$ 0.009128 & 0.6236 & 0.3764 \\
M3 & 1084 & 14.4437 $\pm$ 3.3299  & 28.3967  $\pm$ 7.5115  & 0.155783 $\pm$ 0.047771 & 0.4649 & 0.5351 \\
\bottomrule
\end{tabular}%
}
\caption{Mutagenicity M0–M3 split statistics.}
\label{tab:mutagenicity-splits}
\end{table}

\paragraph{Task.}
For each source$\rightarrow$target pair, we train a graph classifier using labeled graphs from the \emph{source} domain only, and evaluate the trained classifier on labeled graphs from the \emph{target} domain.
No target labels are used for training or adaptation; target labels are used only for evaluation.
%

\paragraph{Model: GCN.}
%
%
Given node features $x$, the model applies two GCNConv blocks with ReLU and dropout, followed by global mean pooling and a linear classifier:
\begin{align}
x &\leftarrow \mathrm{GCNConv}(x, E), \quad x\leftarrow \mathrm{ReLU}(x), \quad x\leftarrow \mathrm{Dropout}(x), \\
z_G &\leftarrow \mathrm{MeanPool}(x), \\
\hat y &\leftarrow \mathrm{Linear}(z_G).
\end{align}
%

\subsubsection{Methods compared}
We compare three training methods:
\begin{itemize}
\item \textbf{GCN (vanilla).} Cross-entropy loss training with no weight decay.
\item \textbf{GCN + L2.} Standard uniform $\ell_2$ regularization implemented as Adam weight decay (\texttt{weight\_decay=1e-4}).
\item \textbf{GCN + Front (front-heavy).} We add a layerwise-weighted quadratic penalty
\begin{equation}
\label{eq:app:front-heavy}
\mathcal{L}
=
\mathcal{L}_{\mathrm{CE}}
+
\lambda_{\mathrm{reg}} \sum_{k=0}^{K} \gamma^{-k}\|W_k\|_F^2,
\end{equation}
where $k$ indexes the sequence of trainable weight matrices in order (GCN blocks first, classifier last).
We set $\lambda_{\mathrm{reg}}=\texttt{1e-4}$. When $\gamma>1$, earlier layers receive stronger penalty (front-heavy).
\item \textbf{GCN + Back (back-heavy).} We use a layerwise-weighted quadratic penalty with coefficients that are the mirror of front-heavy:
\begin{equation}
\label{eq:app:back-heavy}
\mathcal{L}
=
\mathcal{L}_{\mathrm{CE}}
+
\lambda_{\mathrm{reg}} \sum_{k=0}^{K} \gamma^{\,k-K}\|W_k\|_F^2,
\end{equation}
where $k$ indexes the sequence of trainable weight matrices in order (GCN blocks first, classifier last).
We set $\lambda_{\mathrm{reg}}=\texttt{1e-4}$. When $\gamma>1$, later layers receive stronger penalty (back-heavy).
\end{itemize}

\paragraph{Selecting $\gamma$.}
For the front-heavy method, we optionally sweep $\gamma\in\{1.0,1.2,1.5,2.0,3.0,4.0\}$ and select the best $\gamma$ by source-domain validation accuracy.
For each $\gamma$, the model is trained on the source train split, evaluated on the source validation split, and the $\gamma$ with the highest validation accuracy is selected.
%

\subsubsection{Training protocol and implementation details}

\paragraph{Data splitting.}
For each source domain, we randomly split the source dataset into 80\% training and 20\% validation (seeded permutation).
We train on the training split and use the validation split for early stopping and (for \textbf{GCN+Front}) $\gamma$ selection.
We evaluate on the entire target domain.

\paragraph{Optimization and early stopping.}
We use Adam with learning rate \texttt{5e-4}.
For \textbf{GCN+L2}, we use Adam weight decay \texttt{1e-4}.
For \textbf{GCN}, \textbf{GCN+Front}, and \textbf{GCN+Back}, the optimizer weight decay is set to zero, and the front-/back-heavy penalty is added explicitly to the loss only for \textbf{GCN+Front} and \textbf{GCN+Back} respectively.
We train up to \texttt{800} epochs with patience \texttt{100} based on source validation accuracy, however due to early stopping the effective epochs commonly end with less than 200.
We use \texttt{ReduceLROnPlateau} with factor \texttt{0.5} and minimum LR \texttt{1e-5}.
Default batch size is \texttt{128}, hidden dimension \texttt{128}, dropout \texttt{0.2}.

\paragraph{Seeds and reporting.}
We run 5 random seeds (\texttt{0--4}) for each source$\rightarrow$target pair and each method.
%
%
We report mean$\pm$std of target accuracy across seeds.

%

\begin{table}[ht]
\centering
\resizebox{0.8\columnwidth}{!}{%
\begin{tabular}{lcccccc}
\toprule
Methods & M0$\to$M1 & M0$\to$M2 & M0$\to$M3 & M1$\to$M2 & M1$\to$M3 & M2$\to$M3 \\
\midrule
GCN & $\underline{74.50 \pm 0.37}$ & $\mathbf{69.30 \pm 0.88}$ & $53.28 \pm 1.34$ & $69.35 \pm 1.72$ & $49.91 \pm 2.85$ & $53.82 \pm 1.26$ \\
+ L2 & $74.24 \pm 0.40$ & $\underline{69.21 \pm 0.52}$ & $\underline{53.62 \pm 1.01}$ & $69.39 \pm 1.48$ & $49.83 \pm 2.70$ & $\underline{54.70 \pm 2.49}$ \\
+ Front & $74.48 \pm 0.30$ & $68.51 \pm 0.94$ & $\mathbf{53.67 \pm 0.94}$ & $\underline{69.80 \pm 1.88}$ & $\underline{51.25 \pm 3.67}$ & $\mathbf{54.96 \pm 2.50}$ \\
+ Back & $\mathbf{74.56 \pm 0.42}$ & $68.63 \pm 1.22$ & $53.14 \pm 1.59$ & $\mathbf{69.94 \pm 2.13}$ & $\mathbf{51.59 \pm 4.57}$ & $54.15 \pm 2.68$ \\
\midrule
Methods & M3$\to$M2 & M3$\to$M1 & M3$\to$M0 & M2$\to$M1 & M2$\to$M0 & M1$\to$M0 \\
\midrule
GCN & $56.22 \pm 1.78$ & $44.26 \pm 0.66$ & $\underline{57.18 \pm 0.36}$ & $75.63 \pm 0.26$ & $\underline{68.56 \pm 0.40}$ & $60.57 \pm 5.38$ \\
+ L2 & $56.59 \pm 2.05$ & $44.39 \pm 0.87$ & $\mathbf{57.23 \pm 0.46}$ & $\mathbf{75.77 \pm 0.52}$ & $\underline{68.56 \pm 0.18}$ & $61.55 \pm 4.46$ \\
+ Front & $\underline{56.92 \pm 1.68}$ & $\underline{44.54 \pm 0.85}$ & $\underline{57.18 \pm 0.27}$ & $\underline{75.66 \pm 0.52}$ & $\mathbf{68.60 \pm 0.57}$ & $\mathbf{62.66 \pm 4.47}$ \\
+ Back & $\mathbf{57.47 \pm 0.66}$ & $\mathbf{44.65 \pm 0.86}$ & $57.14 \pm 0.12$ & $75.44 \pm 0.50$ & $68.34 \pm 0.49$ & $\underline{62.60 \pm 5.12}$ \\
\bottomrule
\end{tabular}%
}
\caption{Accuracy (in \%) on Mutagenicity (source$\to$target).}
\label{tab:mutagenicity-stratified}
\end{table}

\paragraph{Why non-uniform layerwise (stratified) regularization matches Theorem~\ref{thm:opt_main}.}
Eq.~\eqref{thm:perturb_bound:goal} exhibits a layerwise magnification structure, where perturbations introduced at layer $l$ can be amplified by downstream factors via $\prod_{l'>l} C_2^{(l')}$.
This motivates allocating regularization \emph{non-uniformly} across layers, rather than enforcing a single uniform penalty, as a practical way to reduce the stability-driven term $\Delta_{\Gamma,\Theta}$.
Our front-heavy and back-heavy penalties are symmetric instantiations of this idea, differing only in where the larger weights are placed.
While $\|W\|_F$ is only a proxy for spectral/Lipschitz magnitude, it is a standard, stable surrogate and suffices to test the qualitative implication that \emph{layer-aware} control can improve transfer.

\paragraph{Results}
Table~\ref{tab:mutagenicity-stratified} reports transfer accuracy (mean$\pm$std, \%) on all 12 source$\rightarrow$target tasks.
Across these shifts, the main empirical message is that \emph{non-uniform regularization} tends to be more beneficial than a uniform $\ell_2$ penalty: both \textbf{GCN+Front} and \textbf{GCN+Back} improve the average transfer accuracy over \textbf{GCN+L2}.
Concretely, averaged over all transfers, \textbf{GCN+Front} improves by $\approx 0.47$ points over vanilla \textbf{GCN} and $\approx 0.26$ points over \textbf{GCN+L2}; \textbf{GCN+Back} improves by $\approx 0.42$ points over \textbf{GCN} and $\approx 0.21$ points over \textbf{GCN+L2}.
In terms of per-transfer comparisons, \textbf{GCN+Front} outperforms \textbf{GCN} on 9/12 transfers (ties 1/12), while \textbf{GCN+Back} outperforms \textbf{GCN} on 7/12 transfers.
Importantly, the difference between \textbf{Front} and \textbf{Back} is small overall: each wins on 6/12 tasks, and their mean performance differs by only $\approx 0.05$ points, suggesting that \emph{the presence of layerwise weighting matters more than its direction} in this setting.


In the following subsections, we complement the target-domain accuracy results with three mechanism-oriented diagnostics (setting A/B/C), aiming to test whether non-uniform regularization is consistent with a layerwise amplification view.


\subsection{Mechanism Evidence for Layerwise Product Structure}
\label{sec:exp3:mechanism}

\subsubsection{Setting A: Layerwise proxy and downstream product}
\label{sec:exp3_A}

\paragraph{Motivation and goal.}
This experiment is designed to check \emph{mechanism consistency}.
Our theory suggests that non-uniform regularization changes a layerwise multiplicative structure, which should be visible from trained weights.
So the key question in Setting A is:
\emph{Do front-/back-heavy regularizers systematically change layerwise downstream amplification proxies, compared with uniform L2 regularization?}

\begin{table}[ht]
\centering
\resizebox{0.8\columnwidth}{!}{%
\begin{tabular}{lcccccc}
\toprule
Methods & M0$\to$M1 & M0$\to$M2 & M0$\to$M3 & M1$\to$M2 & M1$\to$M3 & M2$\to$M3 \\
\midrule
+ L2 & $\underline{74.30 \pm 0.51}$ & $\mathbf{69.17 \pm 0.54}$ & $\mathbf{53.56 \pm 1.02}$ & $69.39 \pm 1.65$ & $49.83 \pm 3.02$ & $\mathbf{54.70 \pm 2.78}$ \\
+ Front & $\mathbf{74.54 \pm 0.29}$ & $\underline{68.15 \pm 0.93}$ & $52.68 \pm 1.09$ & $\mathbf{69.82 \pm 2.19}$ & $\underline{51.03 \pm 3.59}$ & $\underline{54.21 \pm 2.71}$ \\
+ Back & $74.08 \pm 0.44$ & $67.99 \pm 1.03$ & $\underline{52.95 \pm 1.04}$ & $\underline{69.69 \pm 2.12}$ & $\mathbf{51.68 \pm 4.27}$ & $54.00 \pm 2.77$ \\
\midrule
Methods & M3$\to$M2 & M3$\to$M1 & M3$\to$M0 & M2$\to$M1 & M2$\to$M0 & M1$\to$M0 \\
\midrule
+ L2 & $\underline{56.55 \pm 2.24}$ & $\mathbf{44.37 \pm 0.95}$ & $\underline{57.14 \pm 0.32}$ & $\mathbf{75.77 \pm 0.58}$ & $\mathbf{68.56 \pm 0.20}$ & $\mathbf{61.55 \pm 4.99}$ \\
+ Front & $56.33 \pm 2.68$ & $44.06 \pm 1.04$ & $\mathbf{57.23 \pm 0.25}$ & $\underline{75.50 \pm 0.50}$ & $\underline{68.49 \pm 0.46}$ & $61.40 \pm 5.45$ \\
+ Back & $\mathbf{56.88 \pm 1.93}$ & $\underline{44.35 \pm 0.80}$ & $57.12 \pm 0.32$ & $75.41 \pm 0.52$ & $68.34 \pm 0.57$ & $\underline{61.46 \pm 5.65}$ \\
\bottomrule
\end{tabular}%
}
\caption{Accuracy (in \%) on Mutagenicity (source$\to$target) under Regularization Strength $L=3$.}
\label{tab:mutagenicity-stratified-fixedbudget}
\end{table}

\paragraph{What we measure.}
For each trained model checkpoint, we compute a layerwise norm proxy from weight matrices.
Let $W_l$ be the weight matrix at layer $l$ (here: \texttt{conv0}, \texttt{conv1}, \texttt{classifier}), and define
\[
C_l := \|W_l\|_F.
\]
For each layer $l$, we define the downstream product proxy
\[
P_l := \prod_{j \ge l} C_j,
\]
which includes layer $l$ itself.
For example, with three layers: 
$P_1=C_1C_2C_3$, $P_2=C_2C_3$, $P_3=C_3$.
We also summarize an overall product proxy using $P_1$.

\paragraph{Experimental protocol.}
We compare \texttt{gcn\_front} and \texttt{gcn\_back} against the uniform baseline \texttt{gcn\_l2}.
To avoid confounding by data split randomness, we use paired comparisons under the \emph{same} transfer task and seed.
In total, we have 12 source$\rightarrow$target transfer directions and 5 random seeds, i.e., $12\times 5=60$ paired comparisons for each method-vs-baseline contrast.

For each pair, we compute:
\[
\Delta P_1 = P_1(\text{method}) - P_1(\texttt{gcn\_l2}),\quad
\Delta \mathrm{Acc} = \mathrm{Acc}(\text{method}) - \mathrm{Acc}(\texttt{gcn\_l2}).
\]
We then report (i) sign statistics and mean shifts, and (ii) Spearman correlation between $\Delta P_1$ and $\Delta \mathrm{Acc}$.

\begin{figure}[ht]
    \centering
    \includegraphics[width=0.60\linewidth]{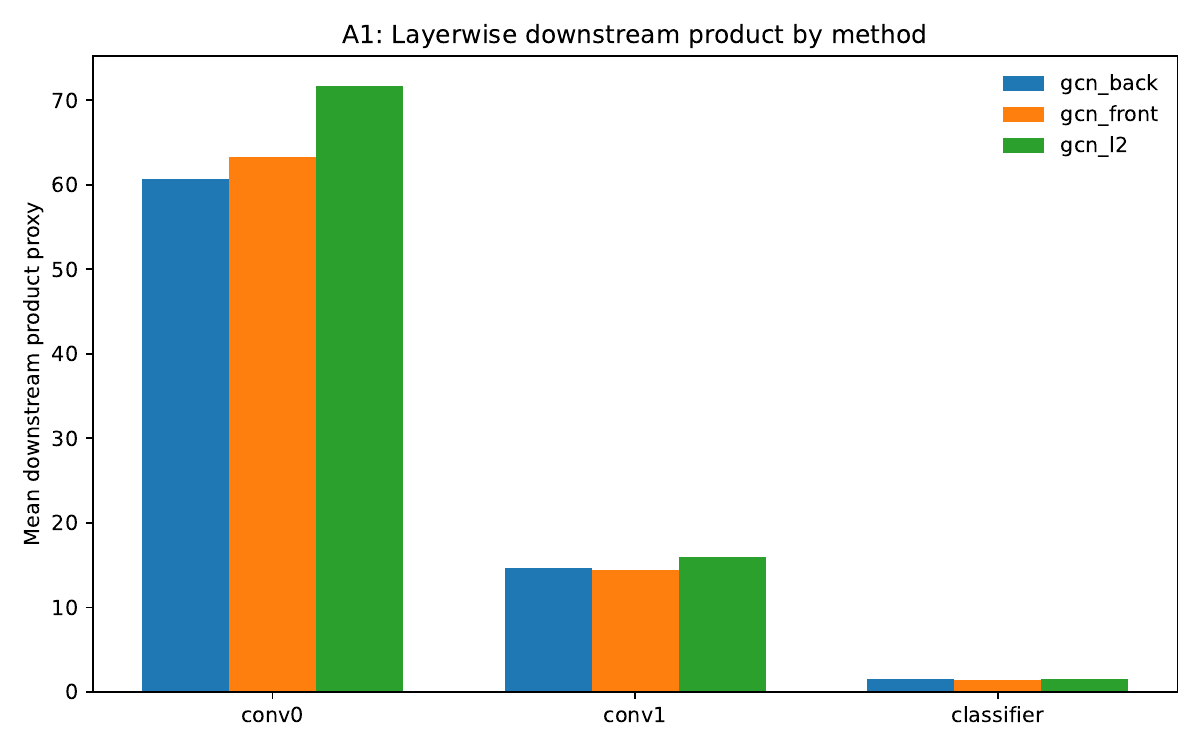}
    \caption{\textbf{Setting A: layerwise downstream-product proxy by method.}
    For each layer $l$, we plot the mean of $P_l=\prod_{j\ge l}\|W_j\|_F$ over all runs.
    Both non-uniform methods show smaller downstream products than \texttt{gcn\_l2}, especially at earlier layers.}
    \label{fig:exp3_A1}
\end{figure}

\begin{figure}[ht]
    \centering
    \includegraphics[width=0.48\linewidth]{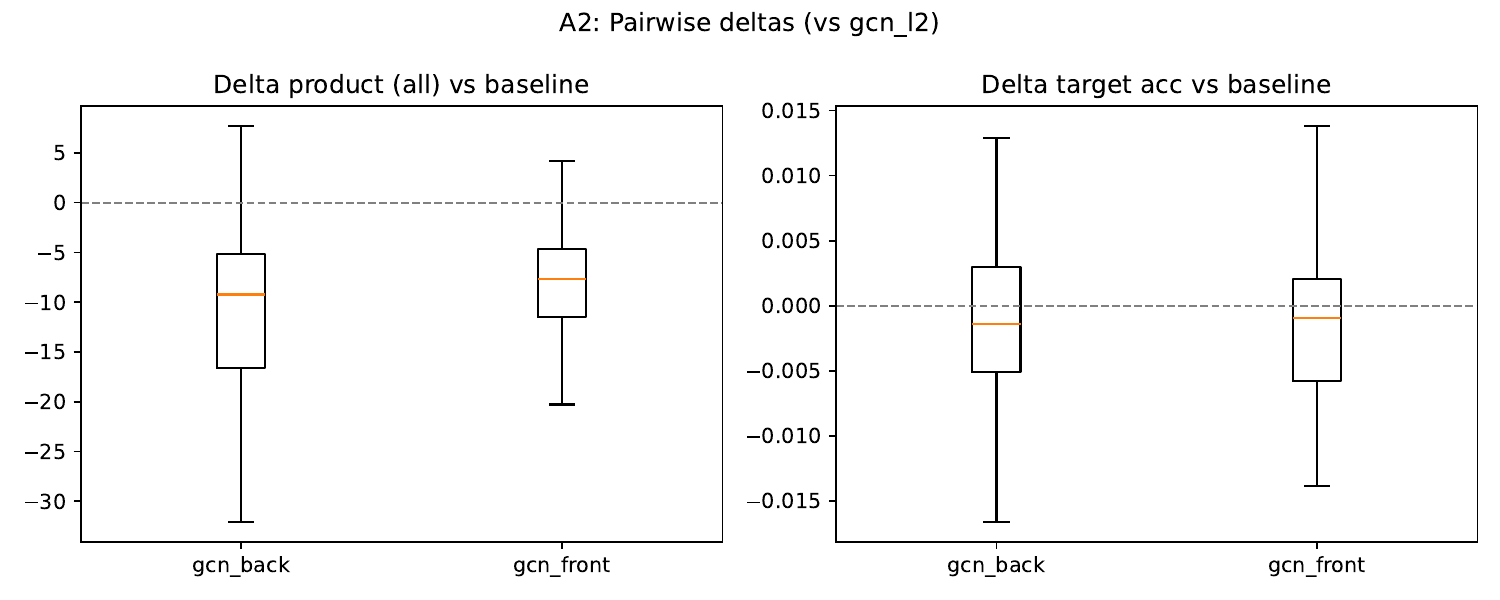}
    \includegraphics[width=0.48\linewidth]{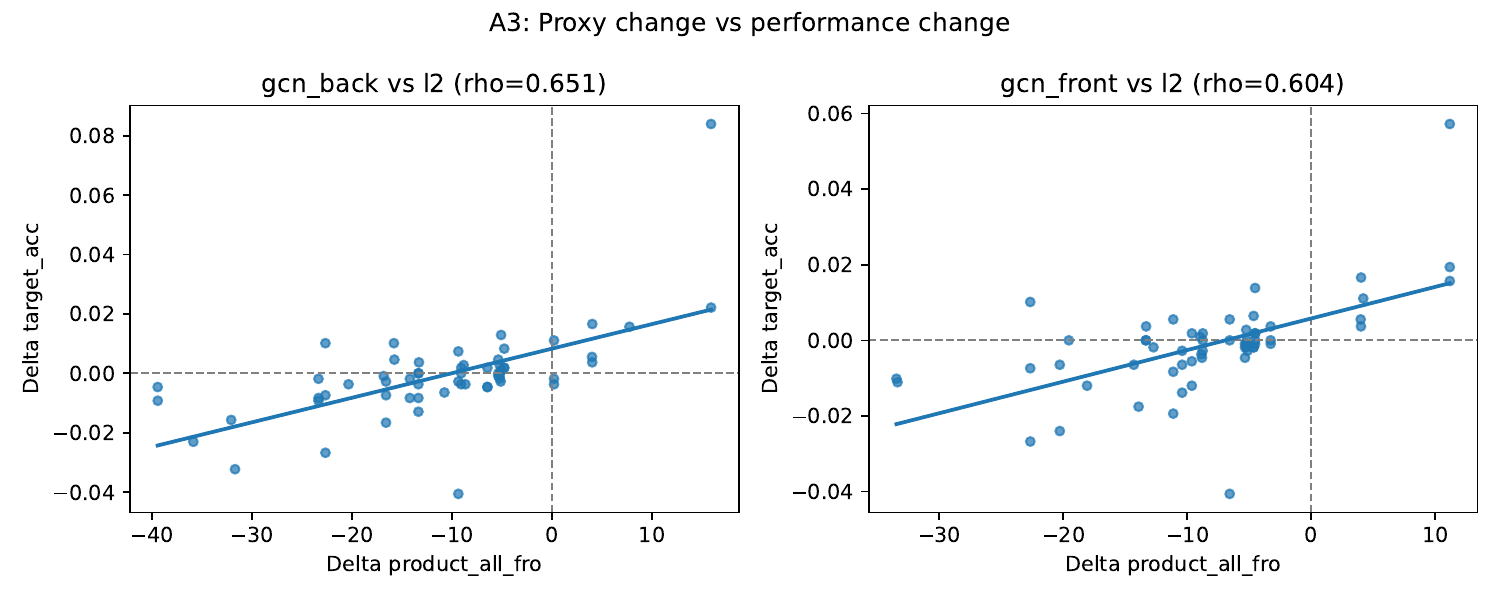}
    \caption{\textbf{Setting A: paired-change analysis against \texttt{gcn\_l2}.}
    Left: distributions of paired shifts in product proxy ($\Delta P_1$) and target accuracy ($\Delta\mathrm{Acc}$).
    Right: relation between $\Delta P_1$ and $\Delta\mathrm{Acc}$ (Spearman $\rho=0.604$ for \texttt{front} vs \texttt{l2}, $\rho=0.651$ for \texttt{back} vs \texttt{l2}; each computed from 60 paired comparisons = 12 transfer directions $\times$ 5 seeds).}
    \label{fig:exp3_A2}
\end{figure}

\paragraph{Results and interpretation.}
Figures~\ref{fig:exp3_A1}--\ref{fig:exp3_A2} show a clear structural effect:
relative to \texttt{gcn\_l2}, both \texttt{gcn\_front} and \texttt{gcn\_back} reduce downstream product proxies in most paired runs
(negative $\Delta P_1$ in 88.3\% and 85.0\% of comparisons, respectively).
The mean paired shifts are
$\Delta P_1=-8.40$ (\texttt{front} vs \texttt{l2})
and $\Delta P_1=-11.02$ (\texttt{back} vs \texttt{l2}).

At the same time, this proxy reduction does not automatically yield better target accuracy:
mean paired accuracy shifts are
$\Delta \mathrm{Acc}=-0.00123$ (\texttt{front} vs \texttt{l2})
and $-0.00080$ (\texttt{back} vs \texttt{l2}).
The Spearman correlations between $\Delta P_1$ and $\Delta \mathrm{Acc}$ are
$\rho=0.604$ (\texttt{front} vs \texttt{l2}, 60 paired comparisons) and
$\rho=0.651$ (\texttt{back} vs \texttt{l2}, 60 paired comparisons).
Hence, Setting A supports a \emph{mechanism-level} claim:
non-uniform regularization does reshape the layerwise product structure,
but in this configuration the proxy magnitude is not a monotonic predictor of transfer accuracy.

\subsubsection{Setting B. Layerwise perturbation sensitivity}
\label{sec:exp3_B}

\paragraph{Motivation and goal.}
Setting A shows that non-uniform regularization changes layerwise product proxies.
Setting B asks a more direct question:
\emph{which layer is actually more fragile to perturbation, and does non-uniform regularization change this fragility pattern?}
This is a mechanism-oriented test and does not require retraining.

\paragraph{Experimental setup.}
For each trained checkpoint, we perturb one layer at a time while keeping all other layers fixed.
Let $W_l$ be the weight matrix of layer $l$.
For each layer $l$, we add a random perturbation $\Delta W_l$ with controlled relative size
\[
\|\Delta W_l\|_F = \epsilon \|W_l\|_F,\quad
\epsilon \in \{0.002, 0.005, 0.01, 0.02, 0.05\}.
\]
For each $(\text{model}, l, \epsilon)$, we sample 8 random perturbation directions and measure target-loss change
\[
\Delta \ell = \ell_{\text{perturbed}} - \ell_{\text{original}}.
\]
To avoid sign cancellation across random directions, we use $\lvert \Delta \ell \rvert$ as the main sensitivity metric.

Data volume:
there are 12 transfer directions and 5 seeds, i.e., 60 trained models per method;
for a fixed method/layer/$\epsilon$, this gives $60\times 8=480$ perturbation evaluations in total.

\begin{figure}[ht]
    \centering
    \includegraphics[width=0.72\linewidth]{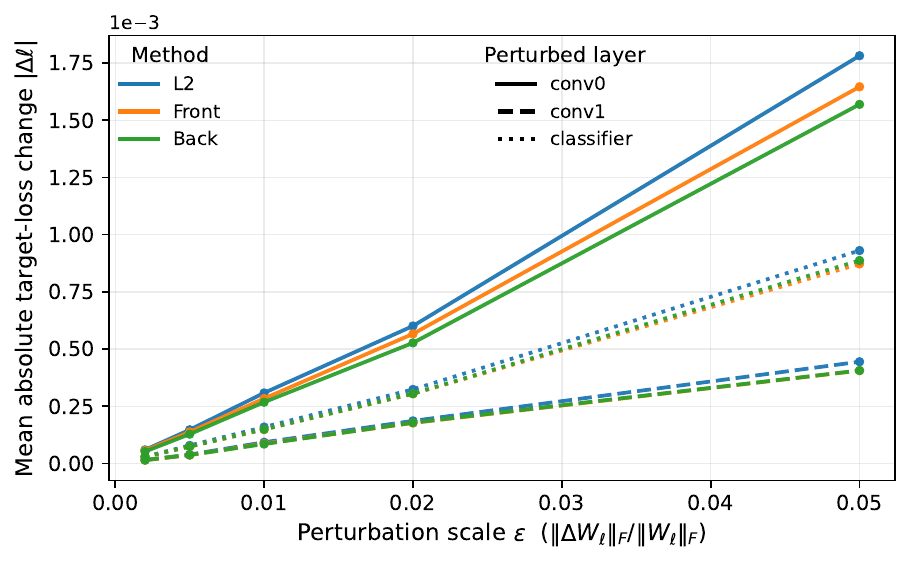}
    \caption{\textbf{Setting B: combined absolute sensitivity curves.}
    Mean $\lvert\Delta\ell\rvert$ versus perturbation magnitude $\epsilon$ for all methods in one panel.
    Early-layer sensitivity (\texttt{conv0}) is consistently dominant.
    Compared with \texttt{gcn\_l2}, both \texttt{gcn\_front} and \texttt{gcn\_back} reduce early-layer sensitivity, most clearly at larger $\epsilon$.}
    \label{fig:exp3_B_main}
\end{figure}

\begin{figure}[ht]
    \centering
    \includegraphics[width=0.49\linewidth]{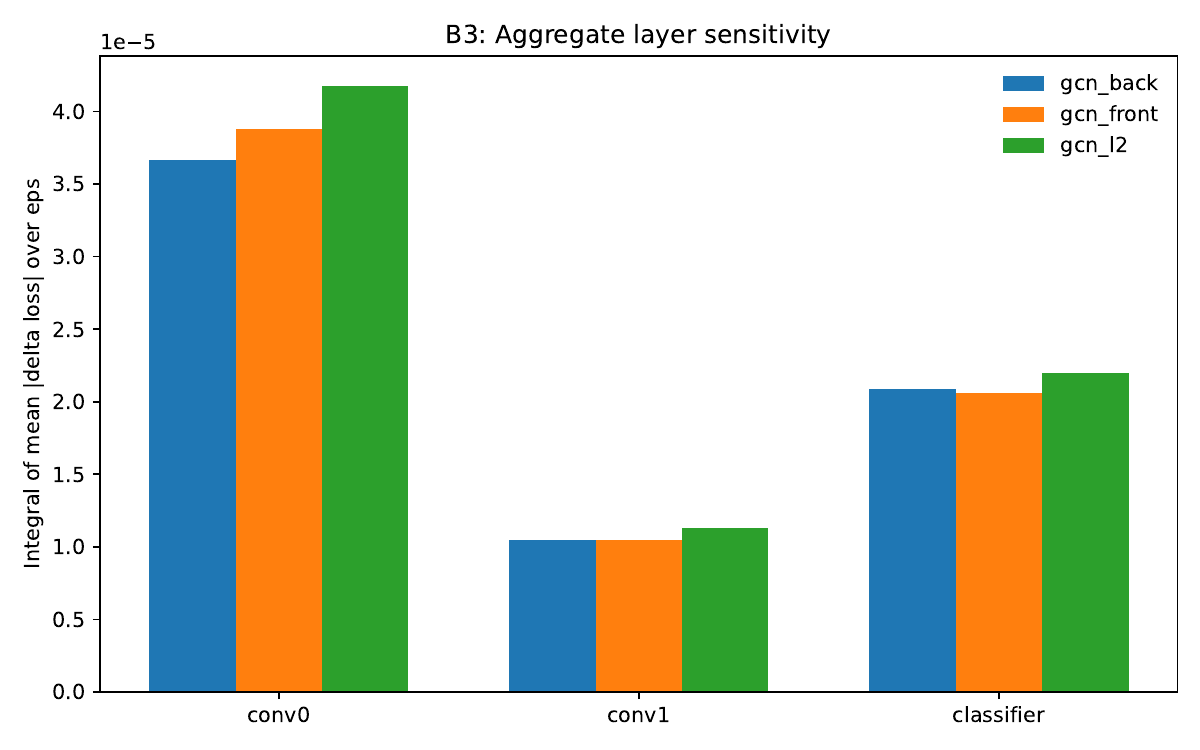}
    \includegraphics[width=0.49\linewidth]{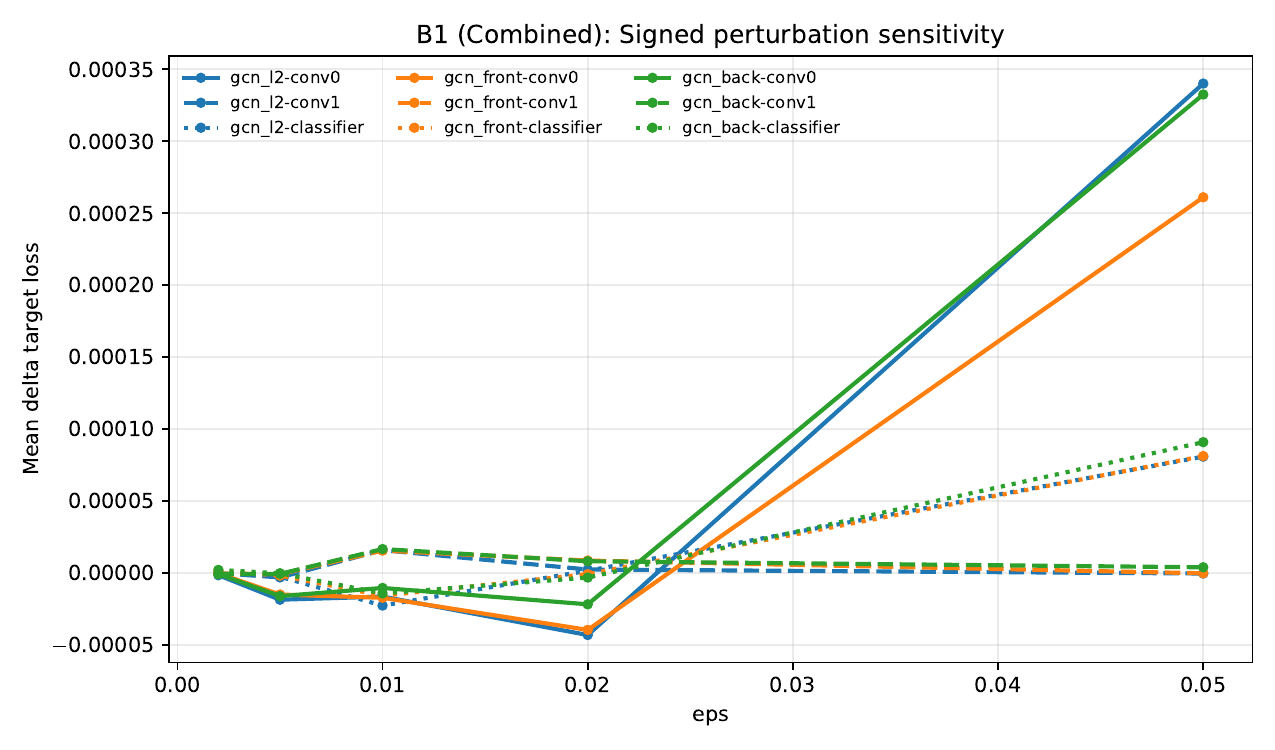}
    \caption{\textbf{Setting B: combined absolute sensitivity curves.}
    Left: integrated $\lvert\Delta\ell\rvert$ over $\epsilon$ (aggregate sensitivity), confirming strongest fragility at \texttt{conv0}.
    Right: combined signed $\Delta\ell$ curves; signed averages can be small due to cancellation across perturbation directions, so absolute response is used as the primary metric in Fig.~\ref{fig:exp3_B_main}.}
    \label{fig:exp3_B_support}
\end{figure}

\paragraph{Results and interpretation.}
Across all methods, early-layer perturbation is clearly most harmful:
\texttt{conv0} has much larger $\lvert\Delta\ell\rvert$ than \texttt{conv1}, consistent with a layerwise amplification pattern.

Using \texttt{gcn\_l2} as reference (mean over all transfer tasks, seeds, $\epsilon$, and perturbation directions):
\[
\lvert\Delta\ell\rvert_{\text{conv0}} \approx 5.80\times10^{-4},\quad
\lvert\Delta\ell\rvert_{\text{classifier}} \approx 3.04\times10^{-4},\quad
\lvert\Delta\ell\rvert_{\text{conv1}} \approx 1.56\times10^{-4}.
\]
At the largest perturbation level ($\epsilon=0.05$), the same ordering remains:
\[
1.78\times10^{-3}\ (\texttt{conv0})\ >\ 9.30\times10^{-4}\ (\texttt{classifier})\ >\ 4.45\times10^{-4}\ (\texttt{conv1}).
\]

Compared with \texttt{gcn\_l2}, both non-uniform variants reduce early-layer sensitivity:
for \texttt{conv0}, mean $\lvert\Delta\ell\rvert$ decreases to
$5.38\times10^{-4}$ (\texttt{gcn\_front}) and
$5.09\times10^{-4}$ (\texttt{gcn\_back}).
Therefore, Setting B provides direct evidence that non-uniform regularization reshapes \emph{where} the model is fragile, not only the final accuracy.

\subsubsection{Setting C. Fixed-budget $\alpha$-family}
\label{sec:exp3_C}

\paragraph{Motivation and goal.}
Settings A and B establish mechanism consistency and layerwise sensitivity.
Setting C asks a complementary question:
\emph{under a fixed regularization budget, does a non-uniform allocation of regularization strength ($\alpha\neq1$) provide practical performance headroom over the uniform choice ($\alpha=1$)?}
Rather than comparing only endpoint designs, we evaluate a continuous $\alpha$ family.

\paragraph{Experimental setup.}
To compare non-uniform layerwise regularizers fairly, we control the \emph{total regularization budget} across different shape parameters.
Let $W_l$ denote the trainable weight matrix of layer $l$ ($l=1,\dots,L$), and let $\alpha>0$ be the non-uniformity parameter.
For a given direction (front-heavy or back-heavy), we first define unnormalized layer coefficients $c_l(\alpha)$:
\[
\text{front-heavy:}\quad c_l(\alpha)=\alpha^{-(l-1)}, 
\qquad
\text{back-heavy:}\quad c_l(\alpha)=\alpha^{\,l-L}.
\]
These coefficients are then normalized to keep the same total budget:
\[
\tilde c_l(\alpha)=\frac{L\,c_l(\alpha)}{\sum_{j=1}^{L}c_j(\alpha)},
\qquad\Rightarrow\qquad
\sum_{l=1}^{L}\tilde c_l(\alpha)=L.
\]
The resulting regularizer is
\[
\mathcal R_{\alpha}(W)
=
\lambda_{\mathrm{reg}}
\sum_{l=1}^{L}\tilde c_l(\alpha)\,\|W_l\|_F^2.
\]
Hence, varying $\alpha$ changes only \emph{how} the regularization budget is distributed across layers, not the total amount.
This fixed-budget design isolates the effect of non-uniform allocation and avoids confounding from simply increasing/decreasing overall regularization strength.
In our experiments, we sweep $\alpha\in\{0.5,0.7,1.0,1.4,2.0\}$ and compare each run to the uniform reference $\alpha=1$.
%
For both \texttt{gcn\_front} and \texttt{gcn\_back}, we sweep
\[
\alpha \in \{0.5,\,0.7,\,1.0,\,1.4,\,2.0\},
\]
with budget fixing enabled so total regularization strength is comparable across $\alpha$.
For each transfer direction and seed, we record target accuracy for every $\alpha$ and compute
\[
\Delta_{\text{best-vs-uniform}}
= \max_{\alpha}\mathrm{Acc}(\alpha)-\mathrm{Acc}(\alpha=1).
\]
Data scale is 12 transfer directions $\times$ 5 seeds $=60$ runs per method.

\begin{figure}[ht]
    \centering
    \includegraphics[width=0.62\linewidth]{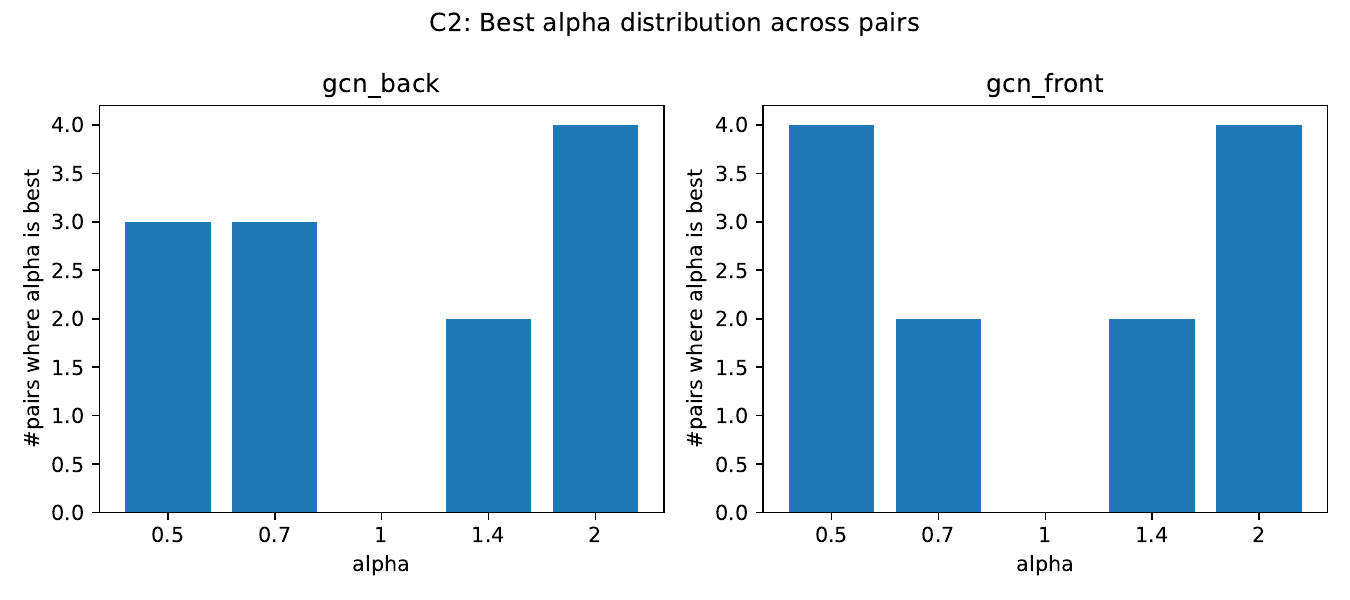}
    \caption{\textbf{Setting C: pair-level best-$\alpha$ distribution under budget fixing.}
    For each transfer pair, we mark the $\alpha$ that gives the best mean target accuracy.
    Best points are spread across multiple $\alpha\neq1$ values, indicating that non-uniform allocations are systematically useful.}
    \label{fig:exp3_C_main}
\end{figure}

\paragraph{Results and interpretation.}
The central finding is robust:
non-uniform allocations frequently outperform the uniform setting.
For both methods, the mean best-vs-uniform gain is about $+0.0042$
(\texttt{front}: $+0.00421$, \texttt{back}: $+0.00420$),
and gains are positive in 83.3\% of runs.
At pair level, all 12 transfer directions have at least one $\alpha\neq1$ that beats $\alpha=1$.

Figure~\ref{fig:exp3_C_main} summarizes \emph{where} the best $\alpha$ values lie across transfer pairs.
The best $\alpha$ is distributed across multiple non-uniform values (not concentrated at $\alpha=1$), which directly supports the claim that non-uniform regularization has usable headroom.
At the same time, directional preference (front vs back) is not dominant in this sweep and should be interpreted conservatively.


\section{Bound Illustration}
\label{app:bound}

We provide further experiments to demonstrate how key factors and assumptions affect our bound in Eq. (\ref{thm:DA:eq}).
We assume zero approximation and label error, i.e., $\varepsilon_3,\varepsilon_4=0$, focusing on identifiable hypothesis family and invariant ground truth in graph signal and latent spaces.
For the remaining experiment setting, we adopt and extend the ones as used in \citet{maskey2022generalization}.
Specifically, we consider 3-class graph classification,  i.e., $C=3$, with a node size of $N=150$ and graph instance number of $m_j=10^5$ per class for both source and target domains.  
We investigate the bound at a $90\%$ probability for it to hold, i.e.,  $\rho=0.9$, under a DA setting that assigns a distribution shift of $s=2$ between the source and target latent Gaussians. 
For the RGM setting, we study a latent space with $C_{\mathcal{X}}=1$ and $D_{\mathcal{X}}=4$, a  kernel function with $L_{W}^{\infty}=1$, $\left\| W \right\|_{\infty}=W_{\max}=0.5$, $d_{\min}=1$ and $C_{\nabla W}=1$, and a  feature-mapping with  $\alpha=1$, $L_f=0.5$, and $\left\| f \right\|_{\infty}=0.5$. 
For the vRKHS setting, we adopt $\lambda_{r}=0.1$, $K_{\max}=10 $ and $ L_2=1$.
The experimented MLP classifier and MPNN message and update functions are two-layer MLPs, i.e., $L=2$, where each layer is $0.5$-Lipschitz continuous, the layer-wise activation function satisfies $\alpha=1 $, $\beta=1.2$ and $L_\sigma=1$, the layer-wise weight matrix has a norm $ \left\| W_l \right\|_F, \|\Theta_{\Phi_t^{(l)}}  \|_F,  \|\Theta_{\Psi_t^{(l)}}  \|_F=0.5$  and  a condition number upper bound  $\kappa=1.1$.
Also the layer-wise message and update functions have zero formal bias $ \| \Phi^{(l)}(0,0)  \|_{\infty}, \| \Psi^{(l)}(0,0)  \|_{\infty}=0$, and the MPNN has a total of $T=2$ layers.
Our result is derived with the assist of top of chaining \citep{vershynin2018high} and triangle inequalities that are suitable for studying quantities within a local region. 
This serves the practical need of analyzing sufficiently good hypothesis, rather than wrong or poor hypothesis being weakly trained.
Therefore, we study construction change by applying $10\%$ norm changes to network weight perturbation, therefore  $\left\| \Delta W_l \right\|_F,  \|\Delta \Theta_{\Phi_t^{(l)}}  \|_F,  \|\Delta\Theta_{\Psi_t^{(l)}}  \|_F=0.05$; and applying an RGM deformation with $N_{P_\tau} = 1$ and $\left\| \nabla \tau \right\|_{\infty}=\frac{1}{2}$.
Under the above default setting, we vary a few key factors later to observe their effect.
To follow the tradition \citep{maskey2022generalization}, a default plot setting is in $\log_{10}$ scale.

\subsection{Key Factors of Interest}

Lem. \ref{proof:norm2diff} distills the effect of MLP classifier  in order to separately analyze the effect of MPNN feature extractor. 
It can be seen from Eq. (\ref{eq:smoothness}) that the MLP classifier impacts the loss term $\lVert \ell_{h,g_D} \rVert_{\mathcal{H}_{K_\ell}}^2 $ through two quantities $L_{NN}$ and $G_{NN}$.
The  quantity $L_{NN}$ depends on the activation choice and the weights of the MLP classifier.
The quantity  $G_{NN}$  reveals how the difference between a  trained MLP classifier and the ground-truth classifier affects the bound.
There exists an interesting relation between the classifier and feature extractor.
The intrinsic complexity of the MLP classifier impacts the  quality requirement of the feature extractor, through multiplying $\left\lVert \Bar{h}_G(Z) - \Bar{g}_{G_D}(Z) \right\rVert_{\infty}$ by $L_{NN}$.
This mixed term, together with the quality of the MLP classifier reflected by $G_{NN}$,  jointly determine the quality of the final hypothesis.

After the convergence and optimization analysis,  the feature extractor disagreement is bounded by  $\left\lVert \Bar{h}_G(Z) - \Bar{g}_{G_D}(Z) \right\rVert_{\infty} \leq \Delta_{N} + \Delta_{\Gamma,\Theta} +\varepsilon_3 + \varepsilon_4$. 
The convergence error bound $\Delta_{N}$ reduces with graph size, which we have studied in the main paper. 
Here we focus on the optimization error bound $\Delta_{\Gamma,\Theta}$, re-organized as below:
\begin{align}
    \label{app:perturb_bound:goal}
    \Delta_{\Gamma,\Theta}
    = \; & \underbrace{\sum_{l=1}^{T} \left (C_{T_2}^{(l)} + C_{T_4}^{(l)}\right) \prod_{l'=l+1}^{T} L_{\Psi^{(l')}} \left( 1 +  C_{T_3}^{(l')}\right)}_{T_1 \textmd{: network weight change}}  + \underbrace{\left(\tilde{C_1}^{(T)} + \tilde{C_2}^{(T)} \lVert f \rVert_{\infty}\right) N_{P_\tau}}_{T_3 \textmd{: RGM deformation}} + \nonumber \\
     & \underbrace{\tilde{C_3}^{(T)} C_{\nabla w} \lVert \nabla \tau \rVert_{\infty}}_{T_2 \textmd{: both}}. 
\end{align}
The first term $T_1$ is a layer-wise accumulation  of the quantities $\{C_{T_i}^{(l)}\}_{i=1}^4$   defined in Thm. \ref{proof:perturb_weights}.
The main key factors that affect these quantities include (1) weight changes $\lVert \Delta \Theta_{\Psi}^{(l)} \rVert_{\infty}$ for quantifying the difference between the learned hypothesis and  ground-truth feature extractors,  (2) RGM kernel property $\lVert W \rVert_{\infty}$, and (3)  Lipschitz constants relevant to the ground truth, e.g., $L_{\tilde{\Phi}}^{(t)}$ and $L_{\tilde{\Psi}}^{(t)}$, which reflect  the problem complexity.
The third term $T_3$ depends on $\{\tilde{C}_i\}_{i=1}^2$, which contain not only ground-truth relevant Lipschitz constants as in (3), but also the RGM deformation strength  $N_{P_\tau}$. 
It serves  as a dual form of $T_1$, mixing  RGM deformation and problem complexity.
While the term $T_2$ depends on both  ground-truth relevant Lipschitz constants as in (3) and RGM complexities, e.g., $C_{\nabla w}$ and  $\lVert \nabla \tau \rVert_{\infty}$.

Based on these, we identify  key factors to observe in later experiments. These include MLP classifier weight changes  $\frac{\left\| \Delta W_l \right\|_F}{\left\| W_l \right\|_F}$ that reflect the    difference between the hypothesis and ground-truth classifiers; MPNN weight changes $\frac{\left\| \Delta \Theta_{\Box^{(l)}} \right\|_F}{\left\| \Theta_{\Box^{(l)}} \right\|_F}$ for message and update functions that reflect the    difference between the hypothesis and ground-truth feature extractors;  a series of factors relevant to RGMs which we detail later in Section \ref{exp:RGM}; a series of factors relevant to MPNNs which we detail later in Section \ref{exp:mpnn}; and additional results on latent domain divergence, class number, and convergence sample complexity to complement the main results in Section \ref{app:exp3:sim}.

\subsection{Hypothesis Quality Implication}

We plot the bound  trend in Fig. \ref{fig:ratio}  by varying the  MLP weight change    $\frac{\left\| \Delta W_l \right\|_F}{\left\| W_l \right\|_F}$ and MPNN weight change  $\frac{\left\| \Delta \Theta_{\Box^{(l)}} \right\|_F}{\left\| \Theta_{\Box^{(l)}} \right\|_F}$ from 0 to 1.
A  larger change  indicates a larger relevant difference between the hypothesis and ground truth.
As the change increases, which means the learned classifier and feature extractor become  worse, the bound inflates.
It is interesting to see that, although both Fig. \ref{fig:ratio:a} and \ref{fig:ratio:b} demonstrate a quasi-exponential increase, the latter is steeper.
It indicates that the weights of the MLP classifier affect  less  the final bound, as compared to the weights of the MPNN feature extractor.  
An insight of this result is  to reduce accumulated co-influence between the classifier and feature extractor.
This supports the existing practice to decouple feature extractor and classifier  and aim at constructing latent representation  space  with high separability, e.g., a class embeddings space  that is as linearly separable as possible. 

\begin{figure}[t]
    \centering
    \subfigure[Classifier  $\frac{\left\| \Delta W_l \right\|_F}{\left\| W_l \right\|_F}$.]{\label{fig:ratio:a}\includegraphics[width=0.48\linewidth]{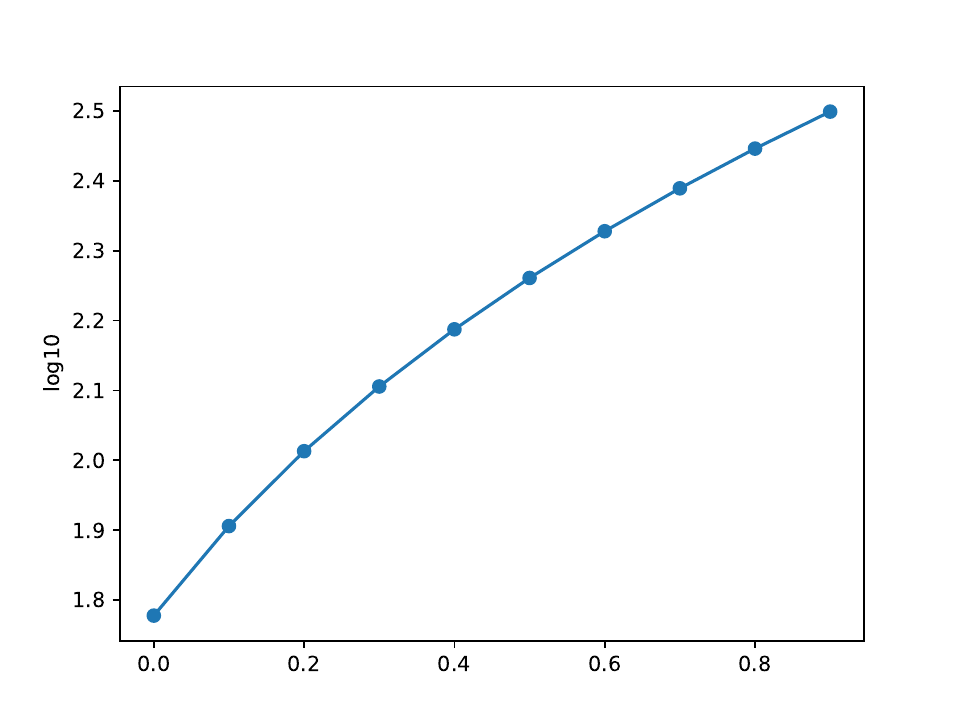}}
    \subfigure[Feature Extractor $\frac{\left\| \Delta \Theta_{\Box^{(l)}} \right\|_F}{\left\| \Theta_{\Box^{(l)}} \right\|_F}$.]{\label{fig:ratio:b}\includegraphics[width=0.48\linewidth]{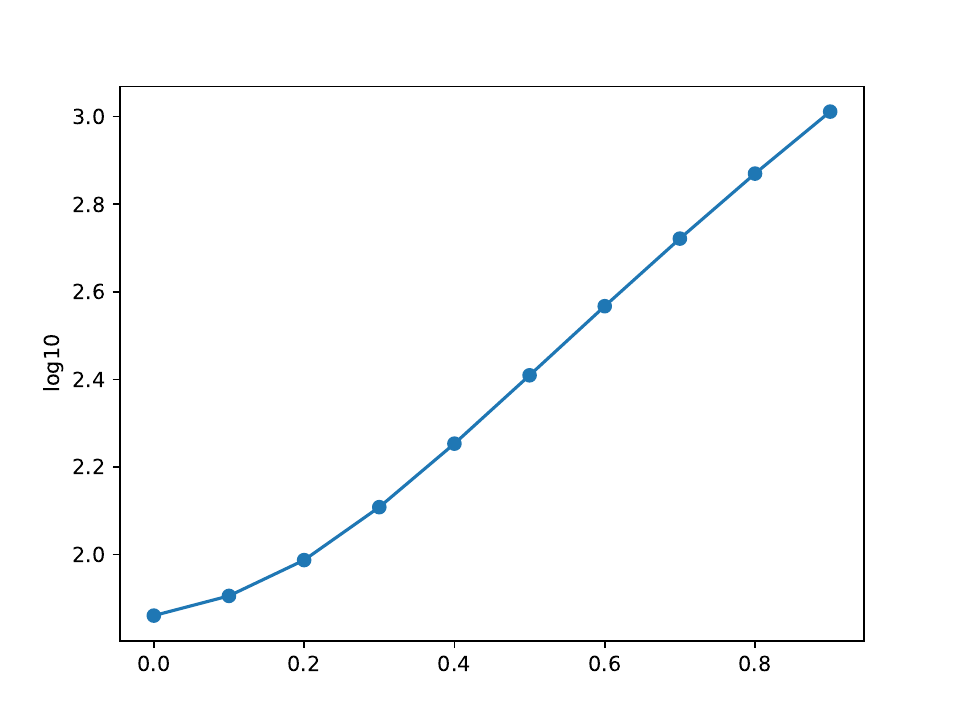}}
    \caption{Bound changes along with the varying weight changes of the MLP classifier and of the MPNN feature extractor.}
\label{fig:ratio}
\end{figure}

\subsection{ RGM Implications}
\label{exp:RGM}

We experiment with varying RGM quantities of interest and observe how the changes   affect the bound, and obtain insights on their   role  in DA.  

\begin{figure}[ht]
    \centering
    \subfigure[Lipschitz $L_W^{\infty}$.]{\label{fig:W:a}\includegraphics[width=0.48\linewidth]{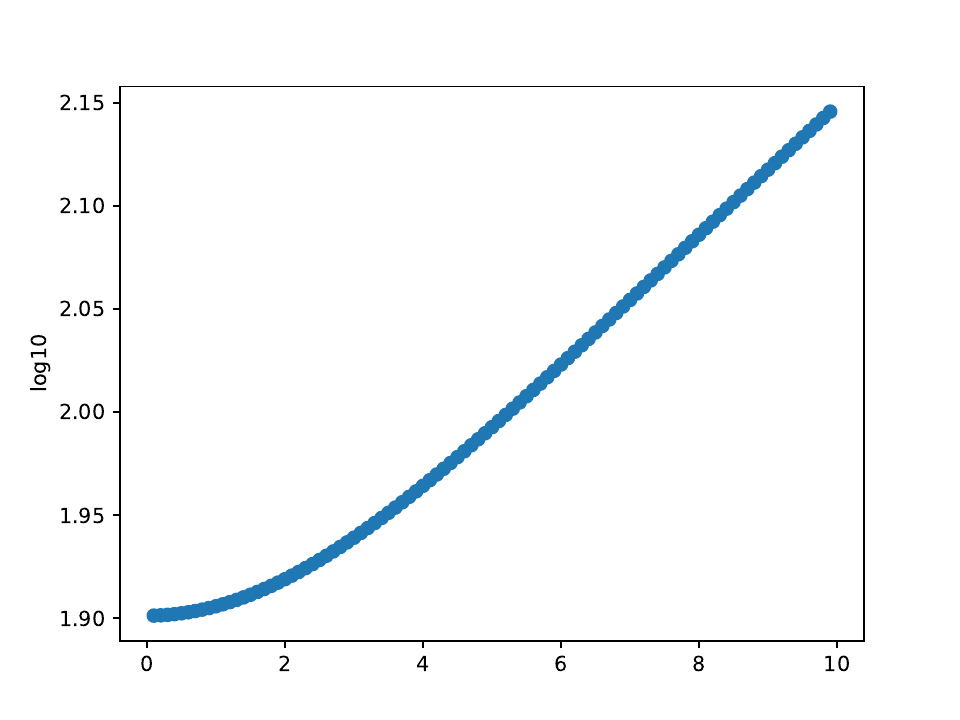}}
    \subfigure[Infinity norm $\left\| W \right\|_{\infty}$.]{\label{fig:W:b}\includegraphics[width=0.48\linewidth]{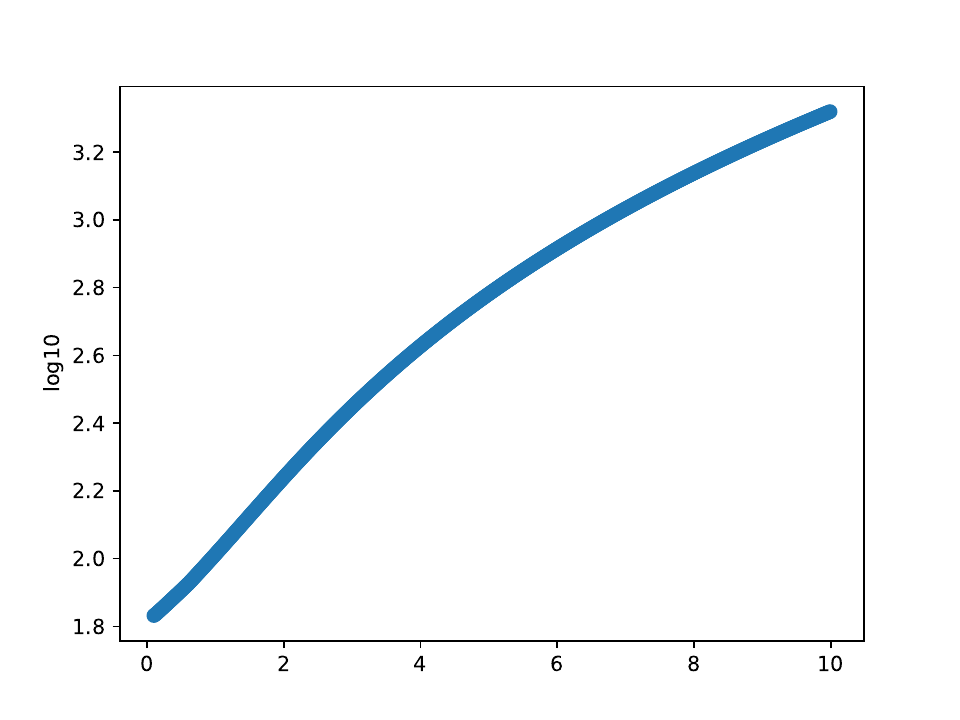}}

    \vspace{1em}
    
    \subfigure[Minimum degree $d_{\min}$.]{\label{fig:W:c}\includegraphics[width=0.48\linewidth]{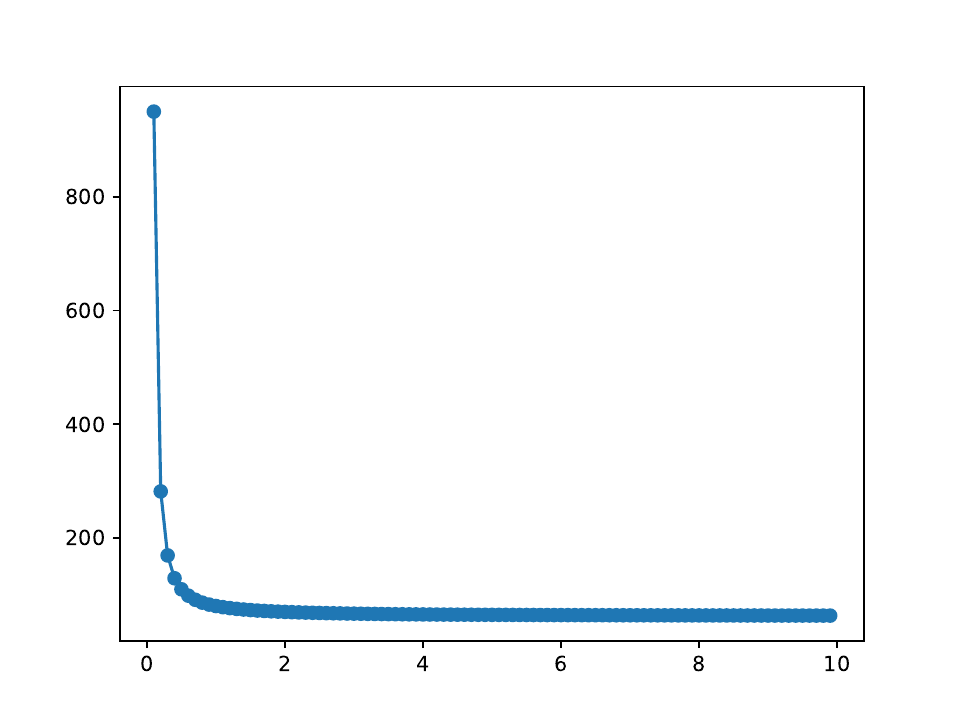}}
    \subfigure[Minimum degree $d_{\min}$ (Zoom In).]{\label{fig:W:d}\includegraphics[width=0.48\linewidth]{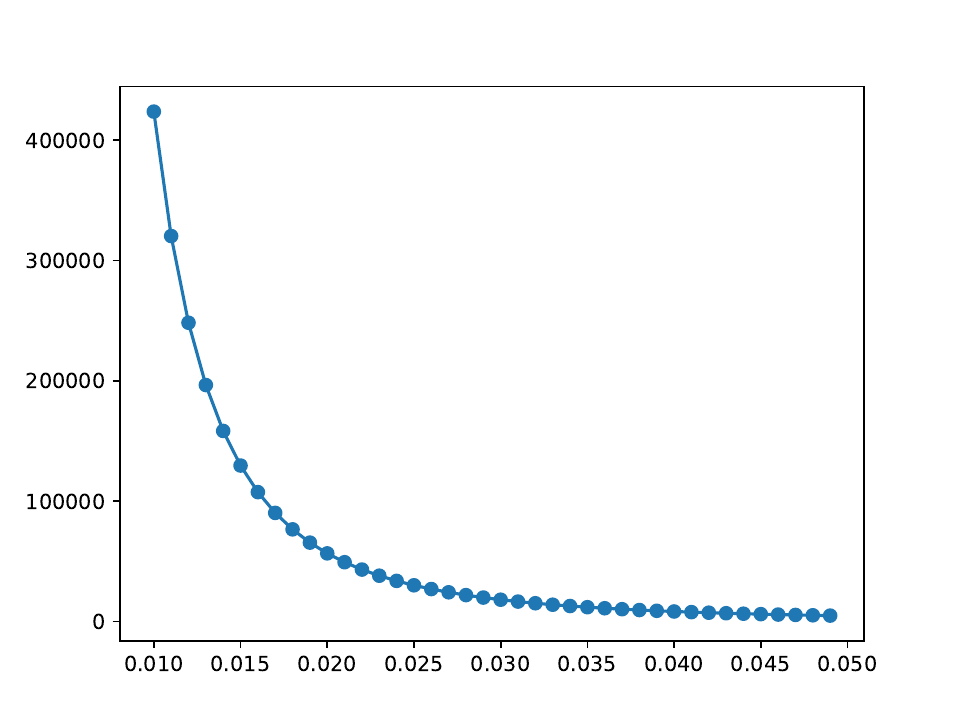}}
    \caption{Bound changes along with varying quantities regarding to  the RGM kernel function $W$.}
    \label{fig:W}
\end{figure}

\textbf{Change of kernel function $W$.} 
We vary the Lipschitz constant  and infinity norm of the kernel function, i.e., $L_W^{\infty}$ and $\left\| W \right\|_{\infty}$ from 0 to 10, and observe the bound change in   Fig. \ref{fig:W:a} and \ref{fig:W:b}.
These two factors are indicators of RGM complexity.  
A polynomial bound increase  can be observed given increasing values of $L_W^{\infty}$ and $\left\| W \right\|_{\infty}$.
This shows that more complicated RGM structure (determined by $W$) corresponds to worse hypothesis transferability, presumably due to higher difficulty of learning.
Our bound flags an important impact of the graph degree lower bound $d_{\min}$ over DA.

Next, we vary $d_{\min}$ from a small value close to 0 to 10 and present the bound change in Fig. \ref{fig:W:c}, 
then  zoom in the trend in Fig. \ref{fig:W:d}. 
To have a more detailed view of the trend, we do not use $\log_{10}$ scale.
A clear elbow change point can be observed around $0.015$.  
It is widely known that the sparsity factor $\alpha_N$ of an RGM is the key parameter that controls edge density \citep{keriven2020convergence}. 
The minimum degree $d_{\min}$ directly reflects how sparse the graph is.
In our bound calculation, we fix the graph size $N=150$.
A change point of $0.015 \approx \frac{\log (150)}{150} = \frac{\log (N)}{N}$ is observed   Fig. \ref{fig:W:d}.
This relatively sparse level with factor $\alpha_N \sim \frac{\log( (N)}{N}$ is exactly what has been broadly studied, e.g., in \citet{keriven2020convergence}. 
A classic result in random graph theory, shown by \citet{bollobas1998random}, is the phase transition property, i.e., in an Erd{\H{o}}s-R{\'{e}}nyi graph $G(N,p)$, when $p<\frac{log (N)}{N}$ the graph is almost surely disconnected, when $p>\frac{log (N)}{N}$ the graph is almost surely connected. 
Our bound faithfully reflects such random graph property and naturally links it to DA generalization error.

\begin{figure}[t]
    \centering
    \subfigure[Lipschitz $L_f$.]{\label{fig:f:a}\includegraphics[width=0.48\linewidth]{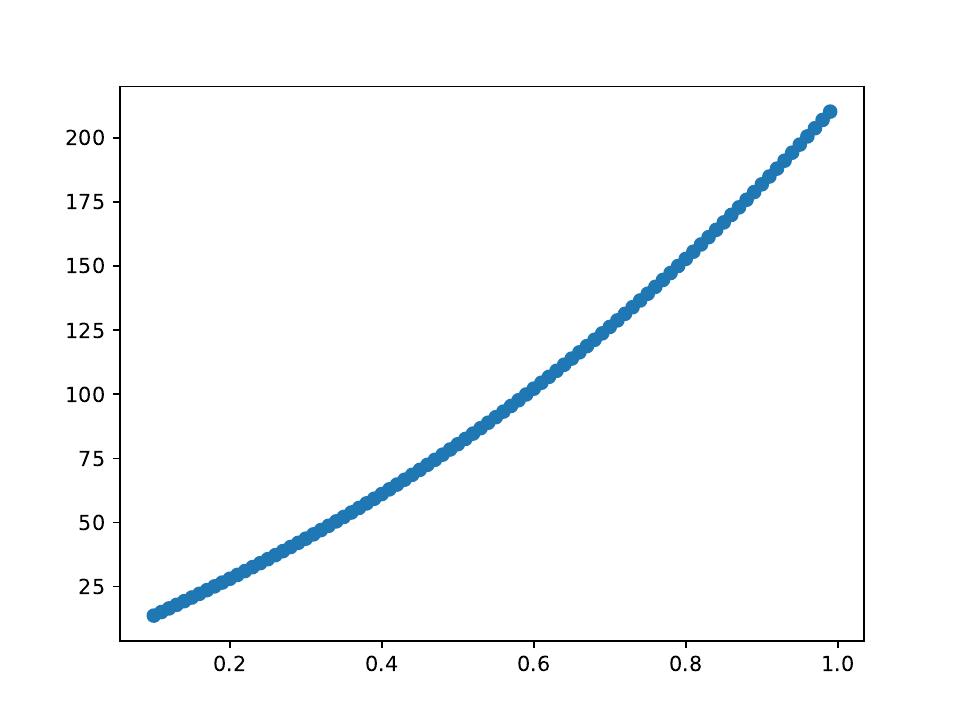}}
    \subfigure[Infinity norm $\left\| f \right\|_{\infty}$.]{\label{fig:f:b}\includegraphics[width=0.48\linewidth]{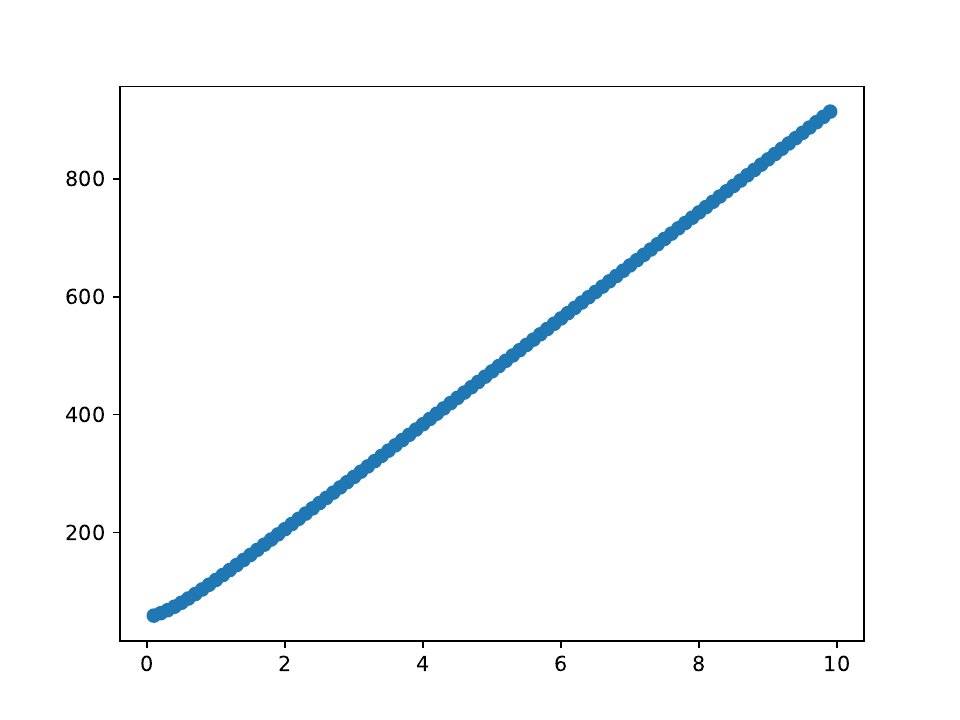}}

    \vspace{1em}
    
    \subfigure[Joint shift of $L_f$ and $\left\| f \right\|_{\infty}$]{\label{fig:f_together}\includegraphics[width=0.48\linewidth]{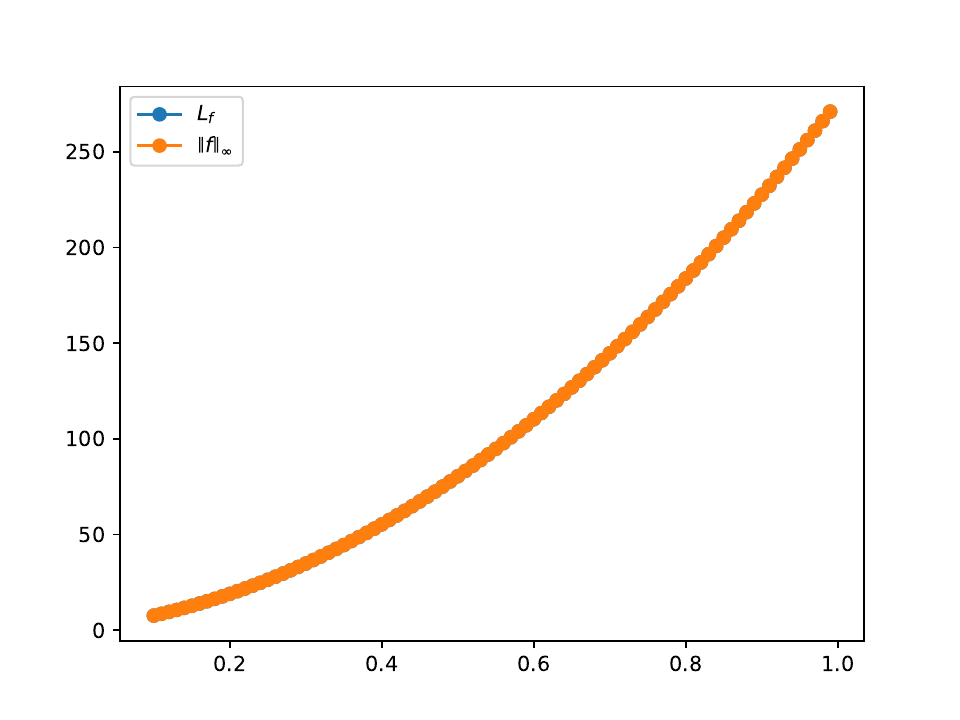}}
    \caption{Bound changes along with complexity of feature mapping $f$ (joint implication).}
    \label{fig:f}
\end{figure}

\textbf{Change of feature mapping $f$.} 
Fig. \ref{fig:f} illustrates the bound changes by varying  the Lipschitz constant $L_f$ and infinity norm $\left\| f \right\|_{\infty}$  of the feature mapping function, where we report the original bound values.
Close-to  linear trends are observed in both plots. 
Next, we jointly shift $L_f$ and $\left\| f \right\|_{\infty}$ by simultaneously shifting them from $0.1$ to $1$, and plot the change in  Fig. \ref{fig:f_together}. We observe that the concurrent change of both leads to higher generalization error than singly shifting $L_f$.
This accumulated effect from reduced function complexity of the feature mapping results in  bound changes in polynomial fashions.

\textbf{Change of space dimension $D_{\mathcal{X}}$.} 
In this experiment, we vary the quantity $D_{\mathcal{X}}$ of the latent space from 1 to 250. The Minkowski dimension is the lower bound of all such $D_{\mathcal{X}}$.
The original bound values are reported in Fig. \ref{fig:space} for different dimensions. 
The bound increases as $D_{\mathcal{X}}$ increase,  which correlates well with existing result, e.g.,  \citet{wang2025generalization} shows that generalization error is proportional to latent space (manifold) dimension.
It is also interesting to observe that, as the latent space dimension dimension grows the bound value eventually converges.

\subsection{MPNN Implications}
\label{exp:mpnn}

\begin{figure}[ht!]
    \centering
    \includegraphics[width=0.59\linewidth]{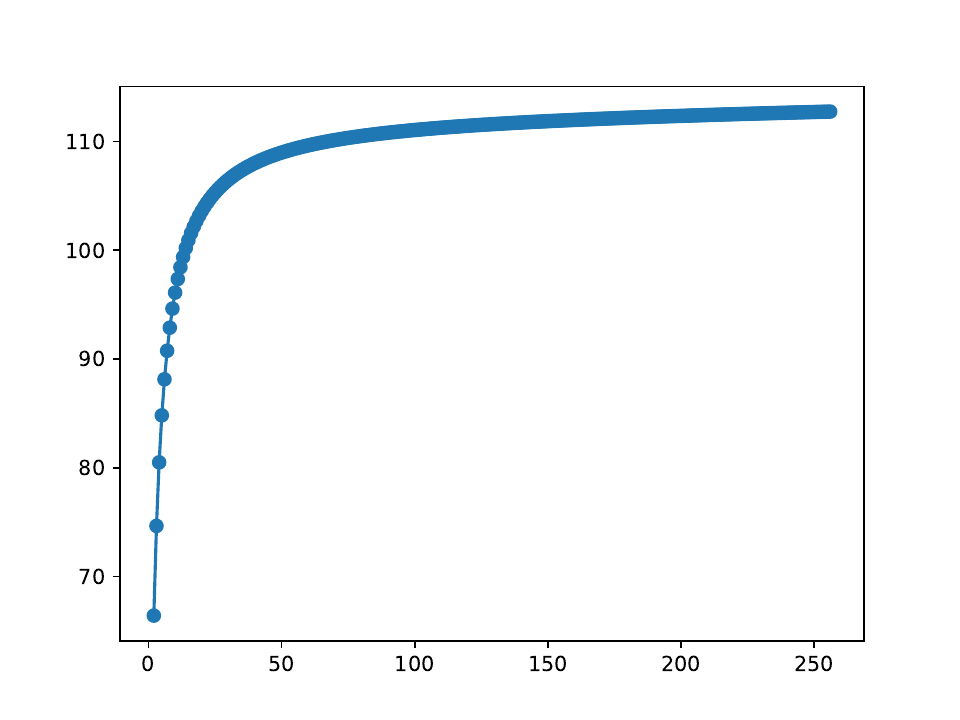}
    \caption{Bound change against $D_{\mathcal{X}}$.}
    \label{fig:space}
\end{figure}
We vary key factors of the MPNN feature extractor to  observe the bound changes.
Firstly, we jointly shift the MPNN Lipschitz constants $L_\Phi,L_\Psi$ from 0 to 10, and report the bound changes in log scale in Fig. \ref{fig:mpnn:a}.
When the Lipschitz constants are less than a value around $1$, the bound goes up linearly, while, after this a clear quasi-exponential increase is observed. 
This indicates that the Lipschitz constants of the MPNN layers should be maintained in a limited region.
This can justify existing regularization techniques that control complexity of graph neural networks.
\begin{figure}[ht!]
    \centering
    \subfigure[MPNN Lipschitz $L_\Box$.]{\label{fig:mpnn:a}\includegraphics[width=0.48\linewidth]{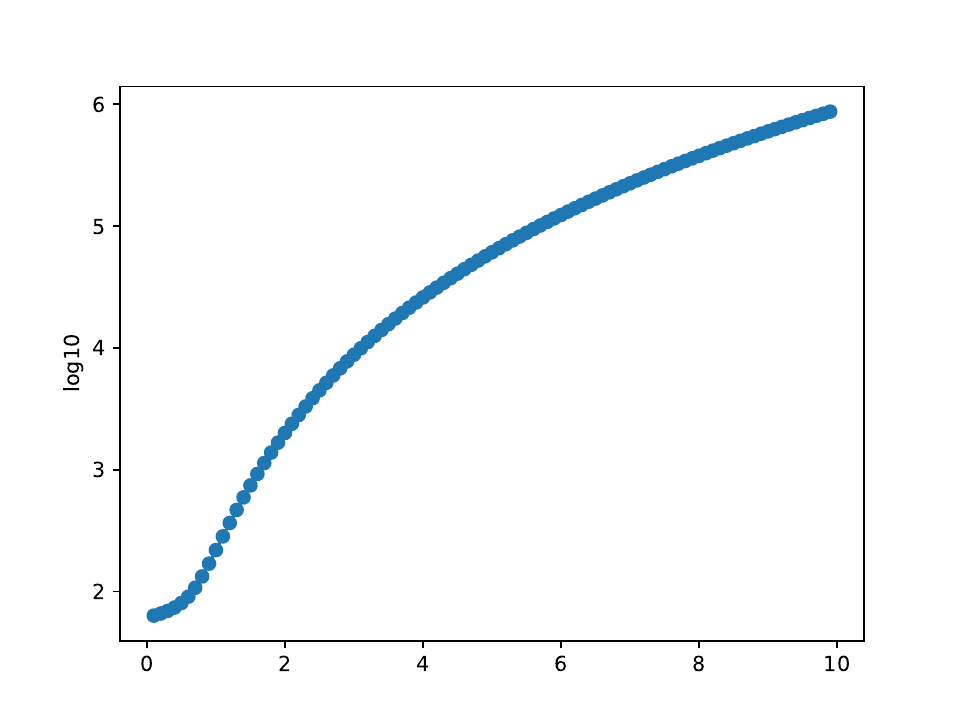}}
    \subfigure[MPNN layer $T$.]{\label{fig:mpnn:b}\includegraphics[width=0.48\linewidth]{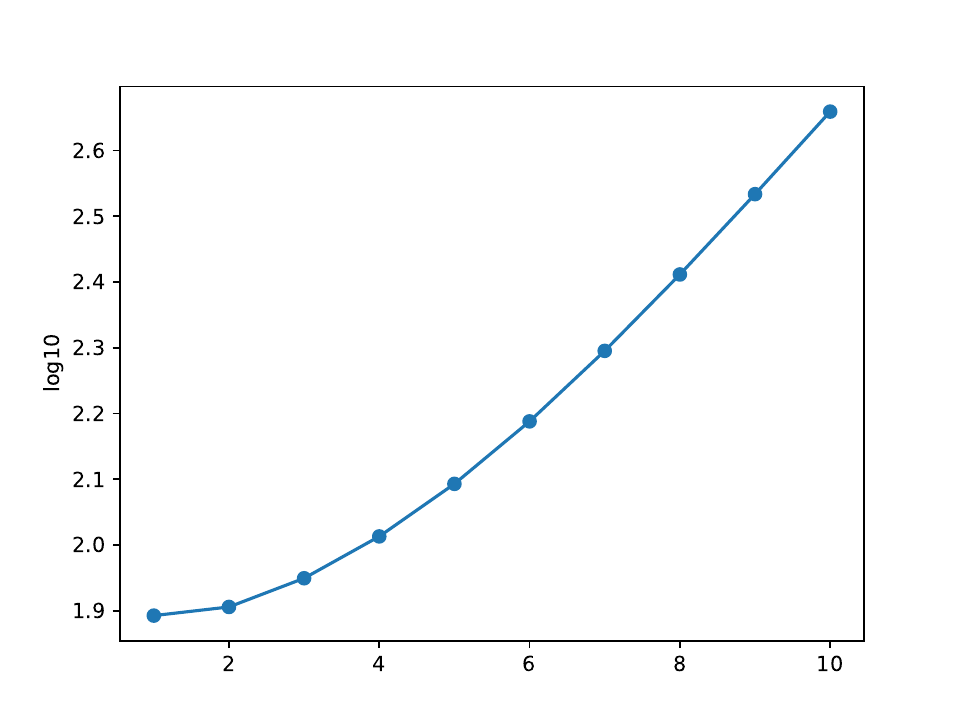}}
    \caption{Bound changes along with: (a) MPNN Lipschitz $L_\Phi,L_\Psi$, (b) MPNN number of layers $T$.}
    \label{fig:mpnn}
\end{figure}
\begin{figure}[ht!]
    \centering
    \includegraphics[width=\linewidth]{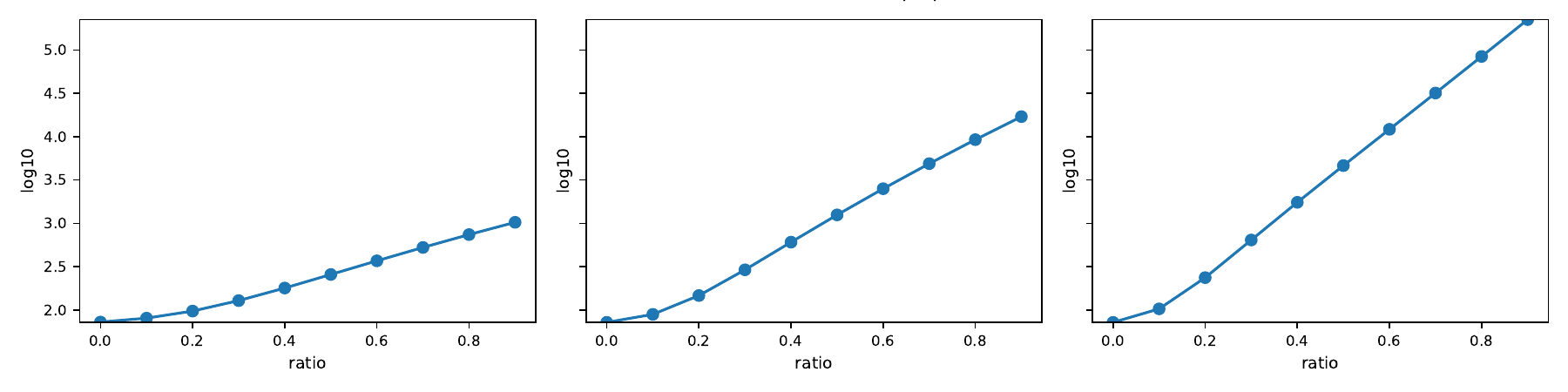}
    \caption{MPNN number of layers $T$ affect bound change rates over $\frac{\left\| \Delta \Theta_{\Box^{(l)}} \right\|_{F}}{\left\| \Theta_{\Box^{(l)}} \right\|_{F}}$.}
    \label{fig:mpnn_T}
\end{figure}
Next, we vary the number of MPNN layers $T$ from 1 to 10 and report the bound change in Fig. \ref{fig:mpnn:b}, where    an exponential impact of $T$ can be observed. 
This aligns well with practical observations where a larger number of MPNN layers could result in performance degradation.
We demonstrate bound change against varying MPNN weight changes $\frac{\left\| \Delta \Theta_{\Box^{(l)}} \right\|_{F}}{\left\| \Theta_{\Box^{(l)}} \right\|_{F}}$ from 0 to 1 to quantify the difference towards the ground truth, for different values of $T=2,3,4$ in Fig. \ref{fig:mpnn_T}.
An impact of $T$ over the rate of bound change can be observed. 
A quasi-exponential trend is observed for  all choices of $T$, but a larger $T$ results in a faster bound increase. 
Specifically, a larger $T$, representing a more complicated hypothesis space, results in a more inflated bound when ratio is high. A high $\frac{\left\| \Delta \Theta_{\Box^{(l)}} \right\|_{F}}{\left\| \Theta_{\Box^{(l)}} \right\|_{F}}$ indicates the larger difference between neural network weights of ground truth and hypothesis function, essentially representing either a more complex labeling function or a bad (poorly trained) hypothesis function.
What $T$ (larger hypothesis space) affects is how hypothesis function behaves in bad cases, indicating that a more complex hypothesis function performs worse in bad cases.
This aligns with well-known deep learning practice in challenges of training larger models.

\begin{figure}[ht!]
    \centering
    \subfigure[Shifts $s$.]{\label{fig:shift_class:a}\includegraphics[width=0.48\linewidth]{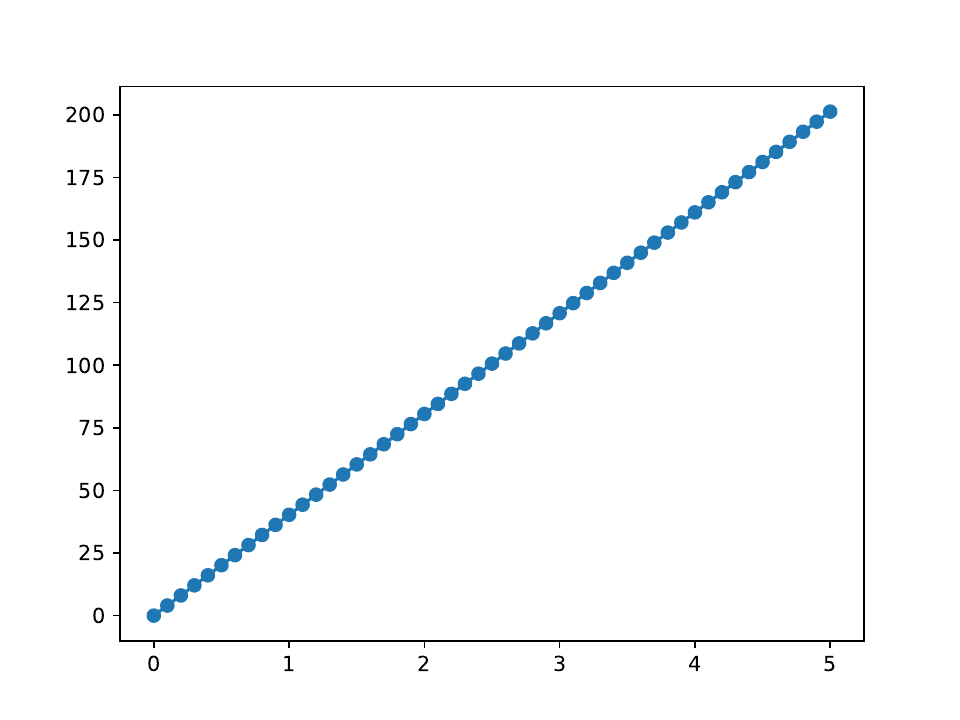}}
    \subfigure[Number of Classes $C$.]{\label{fig:shift_class:b}\includegraphics[width=0.48\linewidth]{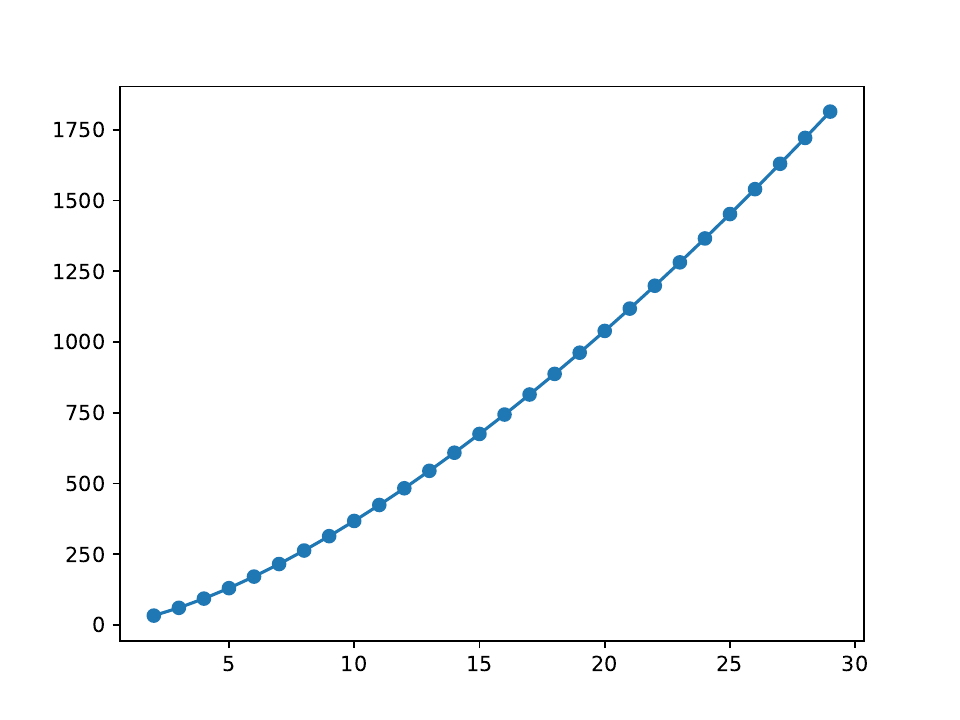}}
    \caption{Bound changes along with: (a) latent Gaussian shifts $s$, (b) number of classes $Cisn$.}
    \label{fig:shift_class}
\end{figure}

\subsection{On Latent Domain Divergence and Number of Classes}

In the end, we report bound changes by varying the latent Gaussian distribution shift  $s$ and class number  $C$  in Fig. \ref{fig:shift_class}, where the original bound values are reported. 
A linear increase against the distribution shift increase is observed, while the bound value increases in polynomial against the class number increase.
The linear relation between latent Gaussian shifts and generalization error inspires further work on imposing probabilistic or geometric constraints in latent space for future studies of domain shifts based on RGMs.
\begin{figure}[ht!]
    \centering
    \includegraphics[width=0.6\linewidth]{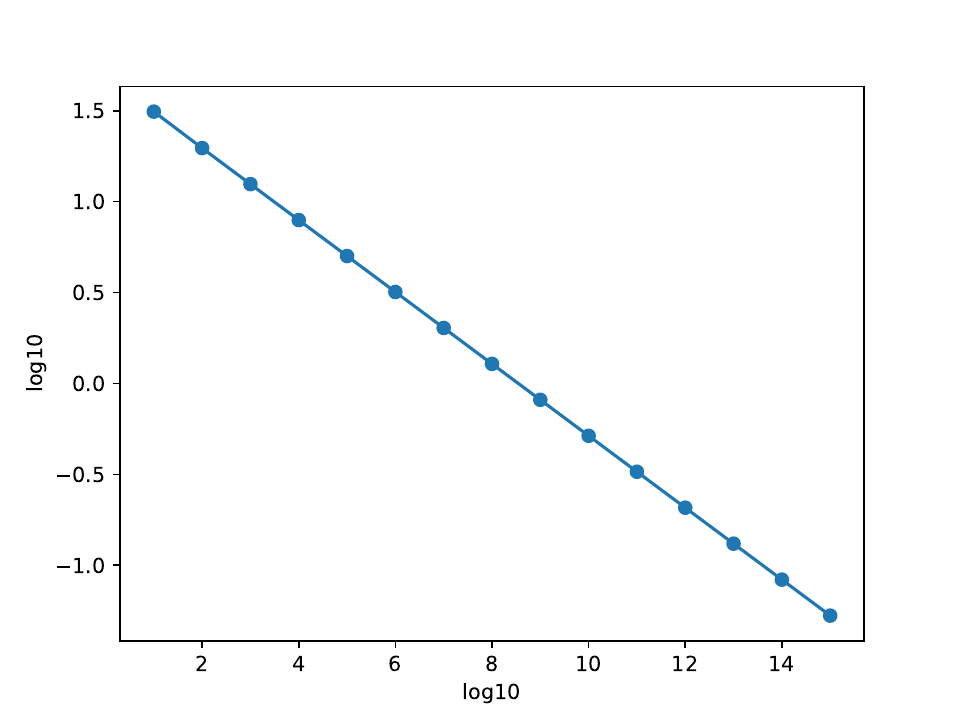}
    \caption{Number of nodes $N$ affects $\Delta_N$ change in polynomial.}
    \label{fig:convergence_sample_complexity}
\end{figure}
In Fig. (\ref{fig:convergence_sample_complexity}), we report $log_{10} N$ where $N$ is number of nodes in a graph, and how $log_{10} \Delta_N$ changes along with it. A clear polynomial trend is observed, indicating significance in sampling larger training graphs for domain adaptation.

\newpage
\section{Practical Implications}
\label{app:practical}

We present a random graph perspective for graph DA. By imposing an RGM generative process, graph distributions and their shifts between domains are formalized. 
Its practical impact depends on how well real-world graphs align with the RGMs.
Although focusing on one generic MPNN hypothesis class, our proposed error analysis framework has a potential applicability across popular GNNs, and the obtained theoretical results offer rich insights on algorithm development. 

\textbf{RGM Applicability.} There has been a long history of modeling real-world graphs by RGMs, e.g., (social) networks \citep{hoff2002latent,abbe2018community,fortunato202220}, and more generally, directed acyclic networks such as citation networks \citep{karrer2009random}.
The specific context necessitates specific RGM variants, however, the benefits are clear if we impose such a known structure to study distribution shifts.
For a finer  modeling, a promising approach is to integrate kernel functions and distribution families through copula functions \citep{sklar1959fonctions,nelsen2006introduction,idowu2025generating}.

%
%

\textbf{Algorithm Insights.}
Advancement on RGM inference has laid solid  foundations for learning RGM from observed graphs \citep{newman2006modularity,amini2013pseudo,bickel2011method,rohe2011spectral,wang2017likelihood,chen2018network,ma2021determining}. This, together with our results, enables design of RGM learning algorithms under various scenarios. 
For instance, we attribute graph DA error to shifts of  RGM latent   distributions and kernels, which results in  an indicator of hypothesis transferability based on Wasserstein distance between the latent distributions.
%
This indicator has  potential to help out-of-distribution detection and hardness analysis \citep{yang2024generalized,redko2019advances}.
Eq. (\ref{lem:wasserstein:goal}) indicates that  Wasserstein distances between  graph distributions can be  approximated through latent RGM distributions, e.g., to estimate $\mathcal{W}_2 (\hat{P}_S^j, \hat{P}_T^j)$ instead in practice.
This can help  develop graph distribution matching algorithms, e.g., by including the estimated distance  as a regularization term.
Eq. (\ref{thm:perturb_bound:goal}) analyzes hypothesis function changes caused by RGM perturbation.
This can be extended to model adversarial domains through perturbation-based augmentation for adversarial training, and serves as an alternative to \citet{luo2024gala} that also perturbs graphs.
Eq. (\ref{thm:DA:eq}) indicates quantities for RGMs to be more transferable.
This can be used to define specific RGMs to generate pseudo labels for target domain examples to improve domain alignment, e.g., under the contrastive learning framework for DA \citep{yin2023coco}.


\section{Some Related Works}
\label{app:relatedwork}

\textbf{Graph Domain Adaptation.} 
Regarding algorithm development for graph DA, existing works consider shifts of node attributes \citep{shen2020network,shen2020adversarial} and adjacency structures \citep{guo2022learning}. 
For instance, inspired by general DA theory and spectral graph theory, \citet{you2023graph} proposes a regularization-based algorithm that applies to both node and structure shifts for graph convolutional networks (GCNs).
\citet{bevilacqua2021size} study graph size extrapolation, i.e., a specific type of graph DA, through graphon theory \citep{lovasz2012large} and causality \citep{balke2022probabilistic}.
Regarding theory development, established key theories for graph DA restrict  to  GCNs, building on graph spectral theory \citep{keriven2020convergence,meng2023transfer}. 
To accommodate more generic classes of graph learning models, e.g., MPNNs, current achievements focus on standard generalization analysis \citep{garg2020generalization,maskey2022generalization},  without considering any distribution shift.

\textbf{Random Graph Model.}
\citet{erdds1959random} formulated the simplest RGM, where all pairs of nodes are linked by a constant probability.
Later on, RGMs started to consider links as independent random variables conditioned on nodes \citep{allman2011parameter}. 
\citet{gilbert1961random} developed the latent position model (LPM), which generates edges using node positions sampled in a latent Euclidean space \citep{kaur2023latent}. 
This idea has been adapted and applied to social network analysis \citep{hoff2002latent}.
Recently, there is a refreshing interest  in using LPMs to handle complex networks \citep{kaur2023latent}.
Parameters of RGM are known to be identifiable up to certain equivalent classes \citep{allman2009identifiability,allman2011parameter,athreya2018statistical}.
This is beneficial when being used to analyze large families of hypothesis functions, and we exploit this property in our analysis.

\textbf{Random Graphs in Machine Learning.} 
There have been usages of random graphs in the machine learning (ML) community.
For instance, a particular class of RGMs, known as stochastic block models \citep{holland1983stochastic}, has been prevalently applied to model social communities \citep{abbe2018community}. 
Graphon \citep{borgs2008convergent,lovasz2006limits,lovasz2012large}, defined as the limit object of a sequence of graphs, has been shown effective for analyzing  the stability and transferability of graph filters \citep{ruiz2021graph,gama2020stability,levie2021transferability}.
By exploiting a continuous counterpart of GCN defined based on RGMs, \citet{keriven2020convergence,keriven2021universality} proved a convergence result of GCN to their limit objects, and used it to develop further universality and stability results for GCNs.
\citet{maskey2022generalization} extended the convergence analysis to the more generic hypothesis class of MPNNs,  and, based on it, developed a generalization error bound for graph classification under the standard learning setting. 
%
However, research on analyzing generalization error for graph DA with respect to MPNNs is missing.
%

\end{document}